\documentclass{article}

\usepackage[letterpaper,top=2cm,bottom=2cm,left=3cm,right=3cm,marginparwidth=1.75cm]{geometry}
\usepackage[square,numbers]{natbib}

\usepackage[utf8]{inputenc} % allow utf-8 input
\usepackage[T1]{fontenc}    % use 8-bit T1 fonts
\usepackage{hyperref}       % hyperlinks
\usepackage{url}            % simple URL typesetting
\usepackage{booktabs}       % professional-quality tables
\usepackage{amsfonts}       % blackboard math symbols
\usepackage{nicefrac}       % compact symbols for 1/2, etc.
\usepackage{microtype}      % microtypography
\usepackage{xcolor}         % colors
\usepackage{color}
\usepackage{bbding}
\usepackage{pifont}
\usepackage{bbm}
\usepackage{soul}
\usepackage{multirow}
\usepackage{framed}
\usepackage{graphicx}
\usepackage{amsmath}
\usepackage{amssymb}
\usepackage{booktabs}
\usepackage{multicol}
\usepackage{caption}
\usepackage{subfig}
\usepackage{tabularx}
\usepackage{makecell}
\usepackage{cases}

\usepackage{amsmath, amsthm, amssymb}
\usepackage{xcolor}
\usepackage{tabularx}
\usepackage{breqn}
\allowdisplaybreaks[4]
\usepackage{soul}
\usepackage{footmisc}

\newtheorem{definition}{{Definition}}

\newtheorem{theorem}{Theorem}
\newtheorem{proposition}{Proposition}
\newtheorem{propositionappendix}{Proposition}
\newtheorem{assumption}{Assumption}
\newtheorem{lemma}{Lemma}
\newtheorem{corollary}{Corollary}
\newtheorem{hypothesis}{Hypothesis}
\newtheorem{lemmaappendix}{Lemma}
\newtheorem{corollaryappendix}{Corollary}
\newtheorem{theoremappendix}{Theorem}

\newcommand{\mybinom}[2]{\Bigl(\begin{array}{@{}c@{}}#1\\#2\end{array}\Bigr)}

\title{Proving Common Mechanisms Shared by Twelve Methods of Boosting Adversarial Transferability}

\author{%
Quanshi Zhang\thanks{These authors contribute equally to this paper.} \thanks{This work is supervised by Dr. Quanshi Zhang (the corresponding author), zqs1022@sjtu.edu.cn. He is with the Department of Computer Science and Engineering,
the John Hopcroft Center and the MoE Key Lab of Artificial Intelligence, AI Institute, at the Shanghai Jiao Tong University, China.} \textsuperscript{1},  Xin Wang\textsuperscript{*1}, Jie Ren\textsuperscript{*1},
\\
Xu Cheng\textsuperscript{1}, Shuyun Lin\thanks{This work was done when Shuyun Lin was an intern at Shanghai Jiao Tong University.} \textsuperscript{1,2}, Yisen Wang\textsuperscript{3}, and Xiangming Zhu\textsuperscript{1}
\\
\\
\textsuperscript{1} Shanghai Jiao Tong University
\\\textsuperscript{2} Tsinghua University
\\
\textsuperscript{3} Peking University
}

\begin{document}

\date{}
\maketitle

\begin{abstract}
Although many methods have been proposed to enhance the transferability of adversarial perturbations, these methods are designed in a heuristic manner, and the essential mechanism for improving adversarial transferability is still unclear.
This paper summarizes the common mechanism shared by twelve previous transferability-boosting methods in a unified view, \emph{i.e.} these methods all reduce game-theoretic interactions between regional adversarial perturbations.
To this end, we focus on the attacking utility of all interactions between regional adversarial perturbations, and we first discover and prove the negative correlation between the adversarial transferability and the attacking utility of interactions.
Based on this discovery, we theoretically prove and empirically verify that twelve previous transferability-boosting methods all reduce interactions between regional adversarial perturbations.
More crucially, we consider the reduction of interactions as the essential reason for the enhancement of adversarial transferability.
Furthermore, we design the interaction loss to directly penalize interactions between regional adversarial perturbations during attacking.
Experimental results show that the interaction loss significantly improves the transferability of adversarial perturbations.
\end{abstract}

\section{Introduction} \label{sec:intro}
Adversarial attacks of deep neural networks (DNNs) has received considerable attention in recent years~\cite{ma2018characterizing,pgd2018,wang2019dynamic,ilyas2019adversarial,duan2020adversarial,wu2020adversarial}.
Adversarial transferability~\cite{szegedy2013intriguing,goodfellow2014explaining} is one of the intriguing properties of adversarial examples, which refers that adversarial examples generated on a source DNN can also fool other DNNs.
A high adversarial transferability indicates that adversarial examples generated on a DNN can easily attack a different DNN, which is of significant value in black-box attacks.
Therefore, many methods have been proposed to enhance adversarial transferability~\cite{mim,wu2018understanding,Wu2020Skip,xie2019improving,dong2019evading,guo2020backpropagating,gao2020patch, huang2019enhancing}.

As Table~\ref{tab:previous methods} shows, many previous methods of enhancing adversarial transferability are {predominantly} designed based on various heuristics, although they all could effectively boost adversarial transferability.
However, the core mechanism for the enhancement of adversarial transferability {remains} unclear.
For example, some methods~\cite{Wu2020Skip, guo2020backpropagating} modified the gradients in back-propagation to boost adversarial transferability. {In some studies~\cite{wu2018understanding,xie2019improving}, adversarial perturbations were extracted} under various image transformations to boost adversarial transferability,
while \cite{huang2019enhancing} generated adversarial perturbations by increasing perturbation on features in a pre-specified layer to improve adversarial transferability.

Compared to heuristic methods, in this paper, we are more interested in exploring the {core} reason for the effectiveness of the above methods. To this end, we believe a solid explanation for  adversarial transferability should satisfy the following three requirements.

\textit{Requirement 1:} The explanation is supposed to reveal the common mechanism shared by lots of previous methods of boosting adversarial transferability.

\textit{Requirement 2:} The explanation is supposed to enable us to identify the algorithmic components in some previous methods that were designed in an experimental manner but {in conflict} with the purpose of boosting adversarial transferability. More crucially, removing such algorithmic conflicts may further boost adversarial transferability.

\textit{Requirement 3:} It is supposed to use the explanation to design new loss functions {and} further boost adversarial transferability.

In this way, if an explanation satisfies the above three requirements, then we believe this explanation {most probably will} reflect the essence of adversarial transferability.

\textbf{Using interactions to analyze the attacking utility.}
Before we investigate the common mechanism of boosting adversarial transferability, we need to first represent the attacking utility from a new perspective, \emph{i.e.}, using game-theoretic interactions to explain the attacking utility.
Specifically, given an input sample $x$, the adversarial attack adds human-imperceptible perturbations $\delta$ on the input to make the DNN's output $f(x')$ incorrect, \emph{e.g.}, being implemented as $\max_\delta \textit{Loss}(x'\!=\!x\!+\!\delta)$.
Here, the attacking utility of the perturbation $\delta$ is defined as the increase of the classification loss $\textit{Loss}(x')-\textit{Loss}(x)$.
As shown in Figure~\ref{fig:intro_interaction}, we can consider each dimension, or divide the perturbation map into multiple regions and consider each region as a \textit{perturbation unit}.
Note that the attacking utility is usually not caused by each single perturbation unit in $\delta$ independently, but is based on interactions between different perturbation units.

\begin{figure}
    \centering
    \includegraphics[width=0.8\linewidth]{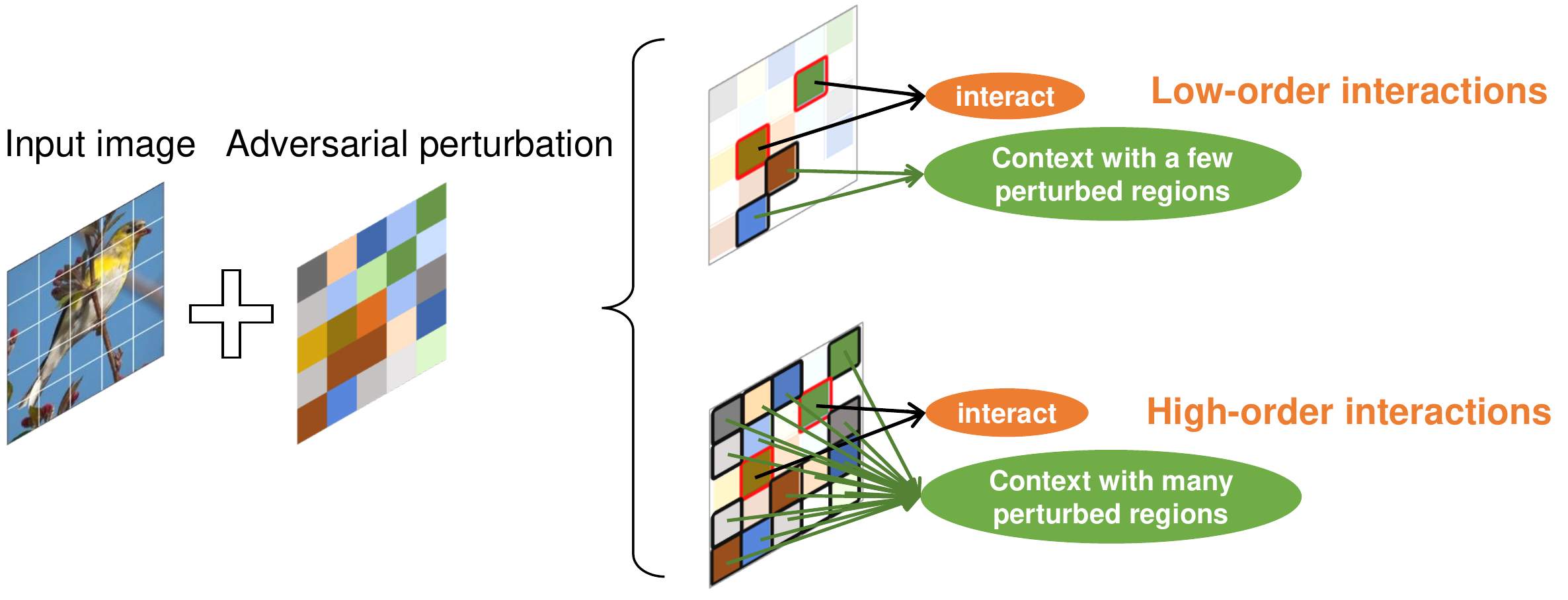}
    \vspace{-10pt}
    \caption{Schematic diagram for the multi-order interaction between perturbation units. In this toy example, we split the perturbation into $5\times 5=25$ regions  (perturbation units) for clarity.}
    \label{fig:intro_interaction}
\end{figure}

Accordingly, we propose to use the interaction between perturbation units to explain the attacking utility.
We divide the entire perturbation map into $n$ regions, and take each regional perturbation as a unit, as shown in Figure~\ref{fig:intro_interaction}.
Let $\Omega=\{1,2,\ldots,n\}$ denote the set of all perturbation units.
In the attack, the attacking utility when we perturb two units $(a,b)$ in the input together may not be equal to the sum of utilities when we individually perturb only one unit $a$ or $b$.
Therefore, the interaction between $(a,b)$ is defined as the difference in the attacking utility when we consider $(a,b)$ as a coalition and perturb them together \emph{w.r.t.} the case when we only perturb $a$ or only perturb $b$.
Furthermore, Zhang~\emph{et al.}~\cite{dropout2020zhang} proved that the overall interaction between two perturbation units can be decomposed into different complexities.
Some interactions are computed based on many contextual variables, which are complex. Such interactions are defined as high-order interactions, where the order refers to the number of contextual variables.
In comparison, some interactions are based on a small number of contextual variables. Such interactions represent simple collaborations between perturbation units and are defined as low-order interactions.

Based on the above interaction, we can explain the overall attacking utility and the adversarial transferability.
Specifically, we mathematically decompose the overall attacking utility into the weighted sum of interactions between different units, \emph{i.e.,} $\textit{Loss}(x')-\textit{Loss}(x)=\sum_{s=0}^{n-2}\sum_{a\ne b\in \Omega}\lambda^{(s)} I^{(s)}_{ab}+\text{constant}$, where $I^{(s)}_{ab}$ denotes the interaction between $(a,b)$ with the order $s$.
In this way, interactions {are used} to progressively explain the adversarial transferability from the following four perspectives.

\textbf{1. Proving the negative correlation between transferability and interactions.}
We prove that adversarial perturbations with higher transferability tend to exhibit smaller interactions between perturbation units.
Such a correlation is both theoretically proven in Section~\ref{subsec:proof for negative correlation} and experimentally verified in Figure~\ref{fig:relation}.

The negative correlation between adversarial transferability and interactions is proven from the following four perspectives.
(1) We derive a closed-form solution to the infinite-step adversarial attack.
Based on this, we prove that compared to the single-step attack, the multi-step attack {generally} generates perturbations {that have} larger interactions.
(2) We prove that adversarial perturbations generated by the multi-step attack is less transferable {compared to} those generated by the single-step attack.
The above conclusions could imply that a larger interaction usually {is indicative of} lower transferability.
(3) {Taking inspiration from the} above two conclusions, we prove that such a negative correlation is {due to} the phenomenon that different DNNs may classify different categories as the second-largest class.
This phenomenon explains why a large interaction causes low transferability.
(4) In addition, we prove that high-order interactions are much larger than low-order interactions, and high-order interactions {occupy} major numerical components of the overall attacking utility.
Meanwhile, we find that high-order interactions are less transferable than low-order interactions.

\textbf{2. Unifying the twelve previous transferability-boosting methods.}
Based on the above theoretical theorems, we prove that \textbf{though many previous methods are based on different intuitions and observations, they all share a common mechanism that \textit{they all reduce interactions between perturbation units.}}
More specifically, we prove that twelve methods in Table~\ref{tab:previous methods}, including the MI Attack, VR Attack, GhostNet Attack, LinBP Attack, Adaptive Attack, FD Attack, IL Attack, PI Attack, adversarially-trained model attack, RAP Attack, IA Attack, and FI Attack, all {lead to a reduction in} interactions between perturbation units.
Therefore, the decrease of interactions between perturbation units may {to some extent show} a common effective factor for the improvement of adversarial transferability.

\textbf{3. {Removing algorithmic conflicts to theoretically improve some previous methods by}}.
The above explanation {allows} us to identify some algorithm components/tricks in some methods that may occasionally {be in} conflict with the decrease of interactions.
This is because previous methods are {mostly} designed based on intuitions {that lack} clear theoretical foundation. Thus, some methods may not always exhibit the common mechanism in all cases \emph{w.r.t.} different inputs.
Therefore, we revise previous methods to eliminate these conflicts {so as to} purely reduce interactions.
Surprisingly, we find that in these cases, the adversarial transferability is usually further boosted, which also proves the {solid basis} of our theoretical explanation.

\textbf{4. {Boosting adversarial transferability by proposing a new loss function.}}
Based on above theoretical proof, we propose a new loss to directly penalize interactions between perturbation units during attacking, which further boosts the adversarial transferability.

\begin{table}[t!]
    \centering
    \caption{Previous methods of boosting adversarial transferability.
    \vspace{-5pt}}
    \resizebox{\linewidth}{!}{
    \begin{tabular}{p{0.25\linewidth} | p{0.75\linewidth}}
        \hline
        Paper & Method\\
        \hline
        \makecell[l]{Momentum iterative (MI)
        \\~ Attack~\cite{mim}} & \makecell[l]{Utilized the momentum of gradients during the generation of adversarial\\~ perturbations.}\\
        \makecell[l]{Variance-reduced (VR)\\~ Attack~\cite{wu2018understanding}} & \makecell[l]{Added different Gaussian noises to an input image, and averaged the\\~ gradients for attacking.}\\
        \makecell[l]{Translation-invariant (TI)\\~ Attack~\cite{dong2019evading}} & \makecell[l]{Generated adversarial perturbations by computing the gradients \emph{w.r.t.} to a\\~ set of translated versions of the original input.}\\
        \makecell[l]{Skip gradient method \\~(SGM)~\cite{Wu2020Skip}} & \makecell[l]{Used more gradient information of skip connections to improve the\\~ transferability.}\\
        \makecell[l]{GhostNet Attack~\cite{li2020learning}} & \makecell[l]{Added dropout operations at different layers and randomized the influence\\~ of gradients from skip connections in each iteration of generating\\~ adversarial perturbations.}\\
        \makecell[l]{Linear back-propagation \\~(LinBP) Attack~\cite{guo2020backpropagating}} & \makecell[l]{Computed gradients by taking non-linear activation functions as identity\\~ transformations in  the back-propagation.}\\
        \makecell[l]{Activation Attack (AA)\\~~\cite{inkawhich2019feature}} & \makecell[l]{Generated adversarial examples by minimizing the difference of intermediate\\~-layer features between the original input and an input from the target \\~ class.}\\
        \makecell[l]{Feature distribution (FD)\\~ Attack~\cite{Inkawhich2020Transferable}} & \makecell[l]{Trained auxiliary models to explicitly modeled the feature distribution of\\~ each class, and  generated perturbations by maximizing the probability of\\~ the target class.}\\
        \makecell[l]{Intermediate level (IL)\\~ Attack~\cite{huang2019enhancing}} & \makecell[l]{Fine-tuned an existing adversarial perturbation by increasing its perturbation\\~ on features in  a pre-specified  layer of the source model.}\\
        \makecell[l]{Diversity input (DI)\\~ Attack~\cite{xie2019improving}} & \makecell[l]{Augmented input images by random transformations including resizing\\~ and padding, and fed  the transformed images into the classifier to compute\\~ gradients for attacking.}\\
        \makecell[l]{Patch-wise iterative (PI)\\~ Attack~\cite{gao2020patch}} & \makecell[l]{Boosted the transferability by distributing the perturbations out of the norm\\~ constraint to neighboring units.}\\
        \makecell[l]{Spriger \emph{et al.}~\cite{springer2021adversarial}} & \makecell[l]{Found that adversarial examples generated on an adversarially trained DNN\\~ were usually more transferable  than those generated on a normally trained\\~ DNN.}\\
        \makecell[l]{Reverse adversarial\\~ perturbation (RAP)\\~ Attack~\cite{qin2022boosting}}  & \makecell[l]{Generated an adversarial example in flat regions of the loss landscape by\\~ solving the min-max optimization.}\\
        \makecell[l]{Intrinsic adversarial (IA)\\~ Attack ~\cite{zhu2022rethinking}} & \makecell[l]{Increased the alignment between gradients of the adversarial loss and the\\~ gradient of joint data distribution by smoothing the activation function and\\~ weakening the impact of residual modules.}\\
        \makecell[l]{Feature importance-aware\\~(FI) Attack \cite{wang2021feature}} & \makecell[l]{Generated adversarial perturbations by  disrupting important features and\\~ improving trivial features in an intermediate layer.}\\
        \hline
    \end{tabular}}
    \label{tab:previous methods}
\end{table}

\textbf{Contributions.}
Essentially, the main contributions of this paper {appear} in the following two {aspects}.

$\bullet$ First, we use the game-theoretic interaction to prove a unified explanation of the common mechanism shared by transferability-boosting methods. This indicates a {core} reason for the improvement of adversarial transferability.

$\bullet$ Second, based on our explanations, we discover algorithmic components/tricks in previous methods that may be occasionally in conflict with the decrease of interactions. Then, we revise these components to remove conflicts and thus further improving previous methods.
Besides, the above theoretical findings also enable us to design new methods to boost adversarial transferability.
In other words, our research may summarize a new theoretical perspective to revisit the common mechanism that determines adversarial transferability, and {may even assist in} designing new algorithms in a theory-driven manner.

A preliminary version of this paper appeared in \cite{wang2021a}.
However, the previous version~\cite{wang2021a} did not theoretically prove the negative correlation between adversarial transferability and interactions, and it just unified five previous transferability-boosting methods.
Beyond the previous version, we theoretically prove the negative correlation between adversarial transferability and interactions. Then, we prove that a total of twelve previous transferability-boosting methods can reduce interactions between perturbation units.
In addition, we identify and revise some algorithmic components that may be occasionally in conflict with the reduction of interactions, thereby boosting the performance of previous methods.

\section{Related work}

\textbf{Adversarial transferability.}
Adversarial attacks can be roughly divided into two categories, \emph{i.e.} white-box attacks~\cite{szegedy2013intriguing, goodfellow2014explaining,jsma,carlini2017towards,bim,Su2019Onepixel,pgd2018} and black-box attacks~\cite{liu2016delving,substitute,zoo, bhagoji2018practical,ilyas2018black,bai2020improving}.
A specific type of black-box attack~\cite{mim,wu2018understanding,xie2019improving,Wu2020Skip} is to generate adversarial examples on a surrogate/source DNN and transfer these adversarial examples to attack a different target DNN.
In this way, the {strength of adversarial transferability brings about} the development of black-box attacks.
Thus, as Table~\ref{tab:previous methods} shows, many methods have been proposed to enhance the transferability of adversarial examples.

However, it is still unclear why the above methods could effectively boost the adversarial transferability.
To this end, we aim to explain and improve the transferability based on the game theory.
We reveal that though the above transferability-boosting methods are proposed based on different intuitions and observations, almost {all of them} all reduce interactions between adversarial perturbation units.

\textbf{Interaction.}
Using interactions between input variables  to analyzing DNNs is a new direction in recent years.
Based on the Shapley value~\cite{shapley1953value}, Grabisch and Roubens~\cite{grab1999An} proposed the Shapley interaction index to measure interactions between players in game theory.
Based on the Shapley interaction index, Lundberg~\emph{et al.}~\cite{Lundberg2017tree} quantified pairwise interactions in tree-ensemble models.
Chen~\emph{et al.}~\cite{chen2020generating} proposed to use the Shapley interaction index to generate hierarchical explanations for DNNs that were learned towards NLP tasks.
Tsang~\emph{et al.}~\cite{tsang2018detecting} measured interactions between input variables using weights of intermediate layers in the DNN.
Murdoch~\emph{et al.}~\cite{Murdoch2018beyond} detected interactions between input variables in LSTMs by using contextual decomposition (CD), and Sigh~\emph{et al.}~\cite{Sigh2019hirachical} extended this method to CNNs.
Jin~\emph{et al.}~\cite{jin2020hierachical} further improved the CD algorithm to measure the contextual independence of words and hierarchically explain LSTMs.
Based on the Integrated Gradients~\cite{sundararajan2017axiomatic}, which was an attribution method, Janizek~\emph{et al.}~\cite{janizek2020explaining} proposed the Integrated Hessian to quantify interactions between a pair of input variables.
This method required to replace ReLU functions in the DNN with  SoftPlus functions.
Unlike previous studies, this paper focuses on using the game-theoretic interaction to both explain and improve the adversarial transferability.

\section{The relationship between transferability and interactions}
\label{sec:negative}

\textbf{Adversarial attack and transferability.}
Given an input sample $x\in\mathbb{R}^n$ with the ground-truth label $y\in\{1, 2, \dots, C\}$, let $h(x)\in\mathbb{R}^C$ denote the output of the DNN before the softmax layer.
We use $h_{y'}(x)$ to denote the output value of the $y'$-th class.
To simplify the story, in this study, we mainly focus on the untargeted attack, which adds an imperceptible perturbation $\delta\in\mathbb{R}^n$ on the sample $x$ to make the DNN {incorrectly} classify the perturbed sample $x'=x+\delta$, \emph{i.e.}, $\arg\max_{y'} h_{y'}(x')\ne y$.
The objective of the untargeted attack is usually formulated as follows.
\begin{equation}
\begin{split} \label{eq:attack}
    \!\!\max_{\delta}\; \textit{Loss}(h(x+\delta), y) \; \text{s.t.}\; \|\delta\|_p\le \epsilon,\; x+\delta\in[0, 1]^n,
\end{split}
\end{equation}
where $\textit{Loss}(h(x+\delta), y)$ is referred to {as} the classification loss, and $\epsilon$ is a constant of the norm constraint.
For simplification, we use $\textit{Loss}(x')=\textit{Loss}(h(x'), y)$ to represent the classification loss on the input $x'$.

Adversarial transferability refers to the transferability of adversarial examples on two different DNNs.
Given a trained DNN, namely a \textit{source DNN}, we generate adversarial examples on the source DNN, and expect to attack another DNN (namely the \textit{target DNN}).
If the adversarial example generated on the source DNN leads the target DNN to make incorrect predictions, then we consider this adversarial example exhibit a good transferability.

\textbf{Shapley values for the attacking utility.}
We consider the adversarial attack as a game, and  we use the Shapley value in game theory~\cite{shapley1953value} to analyze the attacking utility of adversarial perturbations.
We can consider the perturbation in each pixel (\emph{i.e.} a perturbation unit) as a player in the game, and the attacking utility $\textit{Loss}(x')-\textit{Loss}(x)$ can be considered as the overall reward won by the collaboration of different perturbation units. Note that sometimes, to reduce the computational cost of the analysis, we do not directly explain the attacking utility at the pixel level. Instead, we divide the perturbation $\delta\in \mathbb{R}^n$ into $K$ regions, and the perturbation vector within each region is termed a single perturbation unit, as shown in Figure~\ref{fig:intro_interaction}.\footnote{In experiments in Section~\ref{sec:IR loss}, we divided perturbations into $16\times 16$ regions and took each region as a perturbation unit.}
Thus, the Shapley value of each player measures the numerical utility of perturbations within a region in adversarial attacking.
Without loss of generality, let $\Omega\!=\!\{1, 2, \dots, n\}$ denote the set of all players (perturbation units).
For a subset of perturbation units $S\!\subseteq\! \Omega$, let $\delta^{(S)}\!\in\!\mathbb{R}^n$ denote the perturbation map that only contains perturbation units in $S$, as follows.
\begin{equation}
    \forall a\in S,~\delta^{(S)}_a=\delta_a;\;  \forall i\notin S,~\delta^{(S)}_a=0
\end{equation}
Then, the attacking utility $v(S)$ of perturbation units in $S$ is defined as the loss increase when only perturbations on units in $S$ are added to the input units in $S$ are added to the input, as follows.
\begin{equation}
    v(S)\triangleq \textit{Loss}(x+\delta^{(S)})-\textit{Loss}(x)
    \label{eq: v(S)}
\end{equation}
The overall attacking utility of all perturbation units corresponds to $v(\Omega)=\textit{Loss}(x'=x+\delta)-\textit{Loss}(x)$.
In this way, Ancona~\emph{et al.}~\cite{ancona2019explaining} and Weber~\emph{et al.}~\cite{weber1988probabilistic} have proven that the Shapley value $\phi(a|\Omega)$ is a metric that fairly allocates the total game reward (attacking utility) gained by all players (all perturbation units) $v(\Omega)-v(\emptyset)$ to each player (each perturbation unit), as follows.
\begin{equation}
    \begin{split}
    &\phi(a|\Omega)\triangleq \!\!\!\! \sum_{{S \subseteq \Omega \backslash\{a\}}}\!\!\!\!\! \frac{|S| !(n-|S|-1) !}{n !}[v(S \cup\{a\})-v(S)],\\
    &\text{which satisfies }v(\Omega)-v(\emptyset) = {\sum}_{a}\phi(a|\Omega).
\end{split}
\label{eq:shapley}
\end{equation}
The Shapley value satisfies four desirable axioms, \emph{i.e.} the \textit{linearity}, \textit{dummy}, \textit{symmetry}, and \textit{efficiency} axioms (please see the Appendix~\ref{append:shapley-axiom} for details).

\subsection{Interactions}
\label{subsec:interactions}

In order to explain adversarial transferability, let us first define the interaction.
In the attack, perturbation units in $\delta$ do not work independently to contribute to the attacking utility. Instead, different perturbation units may interact with each other to form certain patterns.
The interaction between two perturbation units $(a, b)$ measures the compositional contribution of each specific pairwise interaction between $(a,b)$ to the overall attacking utility.
Specifically, the interaction measures the difference between the joint contribution of two units when they are perturbed together and the sum of their individual contributions~\cite{dropout2020zhang}, as follows.
\begin{equation}
    I_{a b}(\delta) \triangleq \phi(S_{a b}|\Omega') - \left[\phi(a|\Omega\setminus\{b\}) + \phi(b|\Omega\setminus\{a\})\right] \label{eq:interaction}
\end{equation}
where $S_{ab}=\{a,b\}$ denotes a singleton unit when $(a, b)$ are perturbed simultaneously or not perturbed simultaneously.
In this case, the set of all perturbation units can be considered as $\Omega'=\Omega\cup \{S_{ab}\}\setminus\{a,b\}$.
Then, when $a$ and $b$ are considered as a single player in the computation of the Shapley value, the contribution of the singleton unit $(a, b)$ to the attacking utility is quantified as $\phi(S_{ab}|\Omega')$.
$\phi(a|{\Omega\setminus\{b\}})$ denotes the contribution of the perturbation unit $a$, when the unit $b$ is always absent.
Similarly, $\phi(b|{\Omega\setminus\{a\}})$ denotes the individual contribution of the perturbation unit $b$ when the unit $a$ is always absent.

Essentially, the above definition of the interaction reflects the relationship like ``$1\!+\! 1\!>\! 2$'' between two perturbation units.
Specifically, the positive interaction $I_{a b} > 0$ means that perturbation units $a$ and $b$ cooperate with each other to boost the attacking utility, \emph{i.e.}, ``$1\!+\! 1\!>\! 2$.'' In contrast, the negative interaction $I_{a b}<0$ means that perturbation units $a$ and $b$ conflict with each other, thereby weakening the attacking utility, \emph{i.e.}, ``$1\!+\! 1\!<\! 2$.''
When $I_{a b}\approx 0$, it means that there is no interaction between perturbation units $a$ and $b$, \emph{i.e.}, ``$1\!+\! 1\!\approx\! 2$.''

The definition of the interaction in Eq.~\eqref{eq:interaction} is equivalent to the change of the unit $a$'s contribution to the attacking utility, when we compare the unit $a$'s contributions computed in the following two cases, \emph{i.e.}, (1) the case when the unit $b$ is present and (2) the case when the unit $b$ is absent.
\begin{equation}
    I_{a b} =  \mathbb{E}_{S\subseteq \Omega\setminus\{a,b\}} \underbrace{\{[v(S\cup\{a,b\})-v(S\cup\{a\})]-[v(S\cup\{b\})-v(S)]\}}_{\Delta v(a,b|S)} \label{eq:interaction-2}
\end{equation}
Please see the supplementary material for the proof.
This definition is just the same as the Shapley interaction index~\cite{grab1999An} between two players.

\textbf{Multi-order interactions as numerical components of the overall attacking utility.}
Furthermore, the pairwise interaction $I_{a b}$ can be decomposed into different orders as follows~\cite{dropout2020zhang}.
\begin{equation}
    I_{a b}=\frac{1}{n-1}\sum_{s=0}^{n-2}I_{a b}^{(s)},\;\;
    I^{(s)}_{ab}\triangleq \sum_{\substack{S\subseteq \Omega\setminus\{a,b\},|S|=s}}  \Delta v(a,b,S)
     \label{eq:multi-order interaction}
\end{equation}
The order $s$ reflects the contextual complexity of the interaction, which means that $(a, b)$ collaborate with other $s$ contextual variables.
As Figure~\ref{fig:intro_interaction} shows,  low-order interactions (\emph{i.e.}, $s$ is small) represent simple collaborations between $(a, b)$ and a few contextual variables.
In contrast, high-order interactions (\emph{i.e.}, $s$ is large) represent complex collaborations between $(a, b)$ and a massive number of contextual variables.

\begin{theorem}
(\textbf{Decomposing the attacking utility into multi-order interactions,} proof in \cite{robustness2021ren} and Appendix~\ref{append:multi-order interaction}.)
The overall attacking utility can be decomposed into compositional utilities of multi-order interactions, \emph{i.e.}, $v(\Omega)-v(\emptyset)=\sum_{a\in\Omega}[v(\{a\})-v(\emptyset)] + \sum_{a\ne b\in \Omega}\sum_{s=0}^{n-1} \frac{n-1-s}{n(n-1)}I^{(s)}_{ab}$.
\end{theorem}

Therefore, such a decomposition enables us to analyze the adversarial utility in a more fine-grained manner.

\begin{table}[t!]
\centering
\caption{Notations used in the paper.
\vspace{-5pt}}
\label{tab:notation}
    \begin{tabular}{cp{10.5cm}}
    \hline
    $x$ & An input sample\\
    $y\in\{1, 2, \dots, c\}$ & The ground truth label of $x$\\
    $\delta$ & The adversarial perturbation added on $x$\\
    $\Omega=\{1, 2, \dots, n\}$ & The set of all perturbation units\\
    $I_{a b}(\delta)$ & The interaction between perturbation units $(a, b)$  in $\delta$\\
    $I^{(s)}_{ab}(\delta)$ & The $s$-th order interaction between perturbation units $(a, b)$ in $\delta$\\
    $\alpha$ & The step size for generating adversarial perturbation $\delta$\\
    $m$ & The number of steps for generating adversarial perturbation $\delta$\\
    $\beta$ & The adversarial attacking strength $\beta=\alpha m$\\
    $h(\cdot)$ & A DNN, where $h(x)\in\mathbb{R}^c$\\
    $\textit{Loss}(\cdot)$ & The classification loss function\\
    $g$ & The gradient of the loss function \emph{w.r.t.} the input $x$, \emph{i.e.} $g=\nabla_x \textit{Loss}(x)$\\
    $H$ & The Hessian of the loss funtion \emph{w.r.t.} the input $x$, \emph{i.e.} $H=\nabla_x^2 \textit{Loss}(x)$\\
    $\lambda_i$ & The $i$-th eigenvalue of the Hessian $H$\\
    $v_i$ & The $i$-th eigenvector of the Hessian $H$\\
    $\gamma_i$ & The coefficient of the gradient $g$ along the $i$-th eigenvector of $H$, and $g=\sum_{i=1}^n \gamma_i v_i$\\
    \hline
    \end{tabular}
\end{table}

\subsection{Multi-step attack vs. single-step attack}\label{subsec:multi-step vs single-step}

Before the proof of the negative correlation between adversarial transferability and interactions, let us first focus on a special case.
\emph{I.e.,} let us investigate the adversarial transferability and interactions on the multi-step attack and the single-step attack.
In fact, in this subsection, we prove the following two conclusions.

First, we prove that the multi-step attack generates perturbations with larger interactions than the single-step attack.
Then, this conclusion can be combined with the previous discovery in \cite{xie2019improving}, \emph{i.e.,} the multi-step attack usually generates perturbations with lower transferability than the single-step attack.
This discovery can be explained by considering the multi-step attack more overfitted to the source DNN than the single-step attack. Therefore, in this paper, we propose the hypothesis that interactions and adversarial transferability are negatively correlated.

Second, we prove the closed-form solution to adversarial perturbations and interactions when we conduct the following simplification on adversarial attacking.
We use these conclusions to prove the common mechanism shared by twelve methods of boosting adversarial transferabiltiy.

\textbf{Closed-form solution to adversarial perturbations.}
Let us focus on adversarial perturbations generated via the multi-step attack.
Let $\alpha$ denote the step size, and let $m$ denote the number of steps in attack.
Then, the generated adversarial perturbation after $m$ steps can be represented as follows.
\begin{equation}
    \delta^{(m)} = \alpha\sum_{t=0}^{m-1}g_{x+\delta^{(t)}}
    \label{eq: attack}
\end{equation}
In general, there are two widely-used variants of the adversarial attack, \emph{i.e.,} the $\ell_\infty$ attack and the $\ell_2$ attack~\cite{pgd2018}.
The $\ell_\infty$ attack regularizes the gradient with the \text{sign} function, and the $\ell_2$ attack normalizes the gradient.
\begin{equation}
    \delta^{(m)}=\alpha_t\sum_{t=0}^{m-1}g_{x+\delta^{(t)}}=\begin{cases}
    \alpha_t\sum_{t=0}^{m-1} \text{sign}(\nabla_x Loss(x+\delta^{(t)})), & \ell_\infty~\text{attack}\\
    \alpha_t\sum_{t=0}^{m-1} (\nabla_x Loss(x+\delta^{(t)})/\Vert \nabla_x Loss(x+\delta^{(t)})\Vert), & \ell_2~\text{attack}\\
    \end{cases}
\label{eq: delta with norm}
\end{equation}
Without loss of generality, we first consider the attack without regularization or normalization on gradients, as follows.
\begin{equation}
    g_{x+\delta^{(t)}}=\nabla_x Loss(x+\delta^{(t)})
    \label{eq: attack without sign}
\end{equation}
To simplify our analysis, we use the second-order Taylor expansion to approximate the loss function $\textit{Loss}(\cdot)$, where we ignore terms higher than the second order.

\begin{lemma}
(\textbf{The closed-form solution to adversarial perturbations,} proof in Appendix~\ref{lemmaappendix:delta}.)
\label{lemma:delta}
Let us first consider the adversarial perturbation after $m$ steps in Eq.~\eqref{eq: attack without sign}.
Let us decompose the Hessian matrix ($H=\nabla^2_x \textit{Loss}(x)\in \mathbb{R}^{n\times n}$) via eigen-decomposition $H\!=\!V \Lambda V^{\text{T}}$ subject to $V=[v_1,v_2,\ldots,v_n], \Lambda=\text{diag}(\lambda_1,\lambda_2,\ldots,\lambda_n)$, where $\lambda_i \in \mathbb{R}$ and $v_i\in \mathbb{R}^{n}$ denote the $i$-th eigenvalue and eigenvector of $H$, respectively.
Furthermore, $g\!=\!\sum_{i=1}^n \gamma_i v_i$  decomposes the gradient $g=\nabla_x \textit{Loss}(x)$ along eigenvectors of $H$, {where $\gamma_i = g^T v_i$.}
Then, when we attack a DNN via gradient ascent for $m$ steps with the step size $\alpha$, based on Assumption~\ref{assumption:taylor}, the generated adversarial perturbation is { $\delta^{(m)} =V D^{(m)} \boldsymbol{\gamma}$}, where $\boldsymbol{\gamma}=[\gamma_1,\gamma_2,\ldots,\gamma_n]^T$ and { $D^{(m)}\in \mathbb{R}^{n\times n}$ is a diagonal matrix. If $\lambda_i\ne0$, $D^{(m)}_{ii}=\frac{(\alpha\lambda_i+1)^m-1}{\lambda_i}$; otherwise, $D^{(m)}_{ii}=\alpha m$}.
\end{lemma}

In fact, there are many factors that influence the adversarial perturbation in the multi-step attack, \emph{e.g.}, the step number $m$ and the step size $\alpha$. Therefore, to simplify the story, we investigate the more generic form of multi-step attacking, \emph{i.e.}, attacking for infinite steps with the infinitesimal step size. We consider all other forms of multi-step attacks as the approximation of such an infinite-step attack with the infinitesimal step size.

\begin{corollary}
(Proof in Appendix~\ref{corollaryappendix:delta}.)
\label{corollary:delta}
Let $\beta=\alpha m$ denote the total adversarial strength.
If we use the infinite-step adversarial attack with an infinitesimal step size to simplify the multi-step adversarial attack, \emph{i.e.}, $m\!\to\! \infty$, and $\alpha=\beta / m$ is infinitesimal, then the generated adversarial perturbation is given as $\delta^{(\infty)}=VD\boldsymbol{\gamma}=\sum_{i=1}^n D_{ii} \gamma_i v_i$, where if $\lambda_i\ne0$, $D_{ii}=\frac{\exp(\beta\lambda_i)-1}{\lambda_i}$; otherwise, $D_{ii}=\beta$.
\end{corollary}

\textbf{Compared with the single-step attack, the multi-step attack tends to emphasize gradient components in $g$ towards top-ranked eigenvectors of the Hessian matrix.}
In order to obtain such a conclusion, we first need to ensure the single-step attack and the multi-step attack to have similar attacking strength, which ensures the fairness of comparing the single-step attack and the multi-step attack. Therefore, we set the step size $\eta$ of the single-step attack as the accumulated step size of all the $m$ steps in the multi-step attack, $\eta=\beta=\alpha m$. Thus, the adversarial perturbation generated by the single-step attack is given as follows.
\begin{equation}
    \delta^{(\text{single})}=\eta g=\beta\nabla_x Loss(x)
    \label{eq: delta_single}
\end{equation}

In this way, Lemma~\ref{lemma:delta} and Corollary~\ref{corollary:delta} show that unlike the single-step attack with the adversarial perturbation $\beta g$, the multi-step attack usually emphasizes gradient components in $g$ towards top-ranked eigenvectors of the Hessian matrix and restrains gradient components in $g$ along  low-ranked eigenvectors of the Hessian matrix, when it generates adversarial perturbations.

\textbf{The $\ell_2$ attack and the $\ell_\infty$ attack with normalized gradients.}
In fact, the closed-form solution in Corollary~\ref{corollary:delta} can also be used to analyze the $\ell_2$ PGD attack and the $\ell_\infty$ PGD attack~\cite{pgd2018}.
The PGD Attack  normalizes gradients by using $\ell_2$-norm or the sign function, as shown in Eq.~\eqref{eq: delta with norm}.
We can still roughly approximate such a PGD Attack using infinite-step attack with the infinitesimal step size, as long as the step size for the $t$-th step $\alpha_t$ in the PGD Attack is small enough. In this case, when the step size is infinitesimal, we can prove that the $\ell_2$ attack usually converges to the infinite-step attack.
For the $\ell_\infty$ attack, we can approximately consider that for a complex nonlinear DNN in real applications, the sign function multiplied with a tiny step size is an approximation for its gradient, because the small gradients are usually unstable.

\textbf{Closed-form solution to interactions.}
In fact, we study the sum of interactions between all pairs of perturbation units. Theoretically, the sum of interactions $\sum_{a,b\in\Omega} I_{ab}(\delta)$ can be computed as follows.
\begin{equation}
    {\sum}_{a, b\in\Omega} I_{a b}(\delta) = {\sum}_{a\in\Omega} \left[v(\Omega)-v(\Omega\setminus \{a\}) - v(\{a\}) + v(\emptyset)\right]
    \label{eq:expect}
\end{equation} Please see Appendix~\ref{append:sum of interaction} for the proof.

\begin{lemma}
(\textbf{The closed-form solution to interactions,} proof in Appendix~\ref{append:interaction} .)
\label{lemma:interaction}
For the adversarial perturbation $\delta$, the sum of pairwise interactions between perturbation units is\footnote{We prove that the interaction between two input units $(a, b)$ can be approximated as $I_{a b}=\delta_a H_{ab} \delta_b$. Accordingly, we can extend the definition of the interaction to measure the interactive influence a unit $a$ on itself as $I_{aa}= \delta_a H_{aa} \delta_a$. In other words, if we can double the perturbation on $a$, $x'_a=x_a+2\delta_a$, then the absence of the original perturbation $\delta_a$ can further increase the importance of the additional perturbation on the unit by $I_{aa}$. Please see the supplementary material for details.} $\sum_{a,b\in \Omega} I_{a b}(\delta)=\delta^T H \delta$.
Moreover, according to Corollary~\ref{corollary:delta}, if the adversarial perturbation  is given as the result of the infinite-step adversarial
attack with an infinitesimal step size $\delta^{(\infty)}$, then the sum of pairwise interactions inside is given as $\sum_{a,b\in \Omega} I_{ab}(\delta^{(\infty)})=\sum_{i=1}^n \lambda_{i}(D_{ii}\gamma_{i})^2$.
\end{lemma}

\begin{corollary}
(Proof in Appendix~\ref{appendix:positive-semi-definite}.)
\label{corollary:semi-definite}
According to \cite{yao2018hessian}, if the loss function $Loss(x)$ for attacking is formulated as a cross entropy upon the softmax function or a sigmoid function, then the Hessian matrix $H$ is semi-definite. Therefore, $\sum_{a,b\in\Omega} I_{a b} \ge 0$.
\end{corollary}

\begin{theorem}
(\textbf{Multi-step attack generates perturbations with larger interactions than the single-step attack,} proof in Appendix~\ref{append:single vs multi}.)
\label{theorem: multi-step have large interactions}
We consider the multi-step attack with infinite steps, \emph{i.e.,} $\delta^{(\text{multi})}=\delta^{(\infty)}$ defined in Corollary~\ref{corollary:delta}, and the single-step attack defined in Eq.~\eqref{eq: delta_single} and Eq.~\eqref{eq: same norm}.
Then, we have $\sum_{a, b\in\Omega}[{I_{a b}(\delta^{(\text{multi})})}]\ge\sum_{a, b\in\Omega}[{I_{a b}(\delta^{(\text{single})})}]$.
\end{theorem}

\textbf{Compared with the single-step attack, the multi-step attack generates perturbations with larger interactions.} Theorem~\ref{theorem: multi-step have large interactions} enables us to compare adversarial perturbations between the single-step and the multi-step attack.
Note that we need to ensure the fairness of the comparison between the single-step attack and the multi-step attack.
Here, there are two ways to ensure the fairness.

$\bullet$
We set the accumulated attacking strength of the single-step attacks to be the same as that of the multi-step attack. That is $\delta^{(\text{single})}=\beta g$, just like in Eq.~\eqref{eq: delta_single}.

$\bullet$
Alternatively, we can also assume that the norm of the perturbation generated by the single-step attack is the same as that in the multi-step attack, as follows.
\begin{equation}
    \Vert\delta^{(\text{single})}\Vert_2=\Vert\delta^{(\text{multi})}\Vert_2
    \label{eq: same norm}
\end{equation}
Then, Theorem~\ref{theorem: multi-step have large interactions} theoretically proves that the multi-step attack usually generates perturbations with larger interactions than the single-step attack.

\textbf{Discussion on the transferability:} The above theorem shows that adversarial perturbations generated by the multi-step attack exhibit larger interactions than those generated by the single-step attack.
In addition, Xie~\emph{et al.}~\cite{xie2019improving} found that the perturbations generated by the  multi-step attack were usually less transferable than those generated by the signle-step attack.
An intuitive explanation for this phenomenon is that the multi-step attack is more likely to over-fit the source DNN, thereby yielding a low transferability.
Therefore, we propose the following hypothesis.
\begin{hypothesis}
The adversarial transferability and interactions of adversarial perturbations are negatively correlated.
\label{hypo: correlation}
\end{hypothesis}

\subsection{Understanding negative correlation between adversarial transferability and interactions}
\label{subsec:proof for negative correlation}

In Section~\ref{subsec:multi-step vs single-step}, we prove that the multi-step attack generates perturbations with larger interactions than the signle-step attack. Besides, \citet{xie2019improving} found that the multi-step attack generates perturbations with higher transferability than the single-step attack.
Therefore, we propose Hypothesis~\ref{hypo: correlation} that the adversarial transferability and interactions of adversarial perturbations are negatively correlated.
In this subsection, we aim to theoretically prove this correlation between adversarial transferability and interactions.
To this end, the basic idea consists of two parts. First, we prove that a low interaction between perturbation units is usually owing to the phenomenon that the DNN allocates relatively balanced attention to all categories (except for the ground-truth category). Second, we can consider that the adversarial transferability can be naturally explained by such balanced attention, to some extent.

\begin{theorem} \label{theorem:balance}
(Proof in Appendix~\ref{append:balance}.)
Under Assumption~\ref{assum:ortho-weights},  if the prediction scores on other categories are more balanced (following Definition~\ref{def:balance}), then the sum of pairwise interactions $\sum_{a, b\in\Omega}[{I_{a b}(\delta^{(\text{single})})}]$ will be smaller.
\end{theorem}

\textbf{Conclusion 1: negative relationship between the balanced prediction and the interaction.} Theorem~\ref{theorem:balance} shows the negative relationship between the balance of predictions and the interaction between perturbation units.

\textbf{Two assumptions for adversarial transferability.}
In fact, the mathematical analysis of adversarial transferability is an ill-defined problem. It is because essentially, the adversarial transferability depends on the difference between the source DNN and the target DNN, and we do not have any prior knowledge about such a difference.
However, according to the long-term experimental observation, we can still study the transferability based on the following two assumptions.

$\bullet$
For a specific category, the attention or visual concepts encoded in different DNNs have some similarities. Otherwise, we have no way to discuss the adversarial transferability.

$\bullet$
Although different DNNs are all trained and all assign the highest inference score to the ground-truth category, the second-largest and third-best categories in different DNNs are usually different.
Such a difference in the second-largest and third-best categories significantly hurts the adversarial transferability.
For example, if the ground-truth category and the second-largest category in a DNN are the cat and the dog, respectively. Then, the attacking method will generate adversarial perturbation with a clear direction, \emph{i.e.}, weakening the cat category and strengthening the prediction on the dog category.
If the second-largest category of the target DNN is not the dog, then the transferability of the above adversarial perturbation on this DNN will be significantly hurt.

\textbf{Conclusion 2: positive relationship between the balanced prediction and the adversarial transferability.}
The balanced predictions for the second-largest categories will boost the adversarial transferability.
In the above example, if the target DNN considers the horse as the second-largest category, then the transferability of the perturbation is low on this target DNN.
In comparison, given a DNN that generates balanced probabilities for different categories (except for the ground-truth category), the attacking does not have a clear direction towards any specific category. In other words, the attacking mainly weakens the prediction on the ground-truth category, and the attacking method strengthens other categories in a more balanced manner.
Thus, such perturbations exhibit a high transferability.

Therefore, the above two conclusions indicate \textbf{the negative correlation between interactions and adversarial transferability} under the two assumptions about the adversarial transferability.
Note that the balanced predictions here do not means that all other categories have the same prediction score.
Instead, we only focus on all categories that are similar to the ground-truth category, and consider whether the prediction on such categories are balanced or not. Otherwise, the exact uniform classification probability over all categories will lead to an over-complicated feature representation with a low transferability.
In sum, the balanced prediction is also an intuitive explanation for the adversarial transferability.

\begin{figure*}[t]
    \centering
    \includegraphics[width=1.0\linewidth]{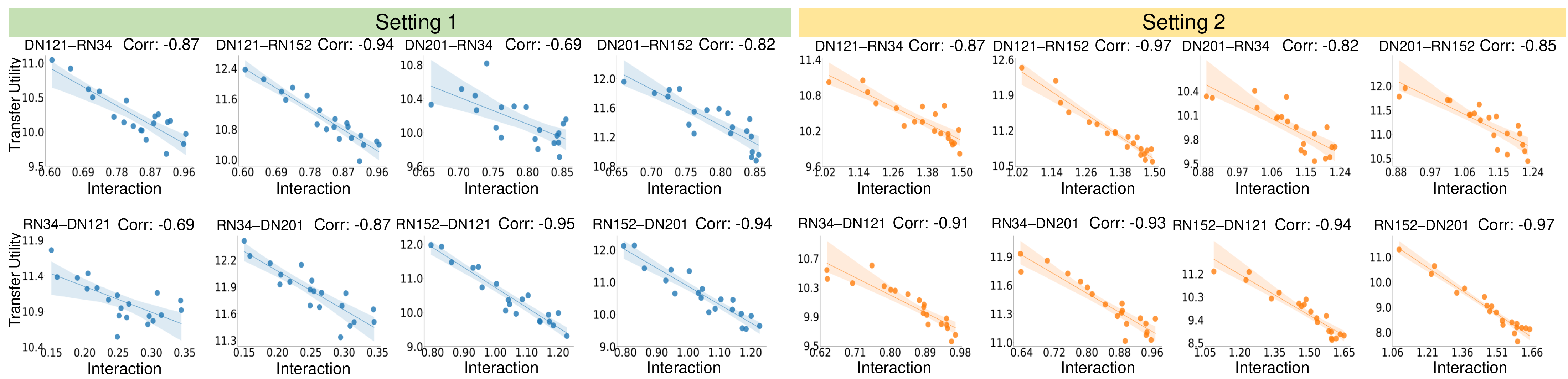}
    \vspace{-15pt}
    \caption{The negative correlation between the transfer utility and the interaction{, which is measured by} the Pearson correlation. The blue shade in each subfigure represents the 95\% confidence interval of the linear regression.}
    \label{fig:relation}
\end{figure*}

\textbf{Empirical verification of the negative correlation between adversarial transferability.}
To this end, we first generated the adversarial example $x'=x+\delta$ on the source DNN.
Then, given a target DNN $h^{(t)}$, the transfer utility of $\delta$ was measured by $\textrm{\textit{Transfer Utility}}=[\max_{y^\prime\ne y}h^{(t)}_{y^\prime}(x') - h^{(t)}_y(x')] - [\max_{y^\prime\ne y}h^{(t)}_{y^\prime}(x) - h^{(t)}_y(x)]$.
The transfer utility measures the change of the target DNN's prediction gap between the ground-truth category and the second-largest category caused by the adversarial perturbation.
The average interaction was measured by $\textit{Interaction}=\frac{1}{n(n-1)} \sum_{a,b\in\Omega}I_{ab}(\delta)$ following Eq.~\eqref{eq:expect}, which was computed on the source DNN.

Specifically, we generated adversarial perturbations on four DNNs, including ResNet-34/152(RN-34/152)~\cite{he2016deep} and DenseNet-121/201(DN-121/201)~\cite{huang2017densely}.
We randomly sampled 50 images from the validation set of the ImageNet dataset~\cite{imagenet2015}.
Given an input image $x$, we generated adversarial perturbations by solving the relaxed form of Eq.~\eqref{eq:attack} via gradient descent, \emph{i.e.} $\min_\delta -Loss(x+\delta) + c\cdot \|\delta\|^p_p \; \text{s.t.} \; x+\delta\in[0, 1]^n$, where $c\in\mathbb{R}$ is a scalar.
In this way, we  gradually changed the value of $c$ and set different values of $p$\footnote{We set $p=2$ as the setting 1, and $p=5$ as the setting 2. To this end, the performance of adversarial perturbations is not the key issue in the experiment. Instead, we just randomly set the $p$ value to examine the trustworthiness of the negative correlation under various attacking conditions (even in extreme attacking conditions).} as different hyper-parameters to generate different adversarial perturbations on each input image.

In this way, we could compare different adversarial perturbations generated on the same image, and explored the relationship between their adversarial transferability and interactions.
For fair comparisons between different adversarial perturbations, we selected a constant $\tau$ and took $\|\delta\|_2=\tau$ as the stopping criteria of all adversarial attacks.
Given each input image, we transferred its adversarial examples generated on each ResNet to DenseNets, and transferred its adversarial examples generated on each DenseNet to ResNets.
As shown in Figure~\ref{fig:relation}, there was a strong negative correlation between the transfer utility and the interaction.
Each subfigure was obtained from a specific pair of source DNN and target DNN.
Each point in the subfigures shows the average transfer utility and the average interaction of adversarial perturbations, which were computed on adversarial perturbations on all testing images \emph{w.r.t.} specific hyper-parameters $c$ and $p$.

\section{Unifying twelve previous transferability-boosting attacks}\label{sec:explaining}

In this section, we prove that the decrease of interactions between perturbation units is a common mechanism shared by many classic transferability-boosting methods, although they were not originally designed to do so.
Before proving the common mechanism of previous methods, we first clarify some assumptions used in the proof.

(1) We assume that the source DNN is learned for the classification task, and the classification loss is formulated as the cross-entropy loss. This assumption is used in proofs for all transferability-boosting methods below.

(2) In order to simplify the analysis, we use the second-order Taylor expansion to approximate the loss function
$Loss(\cdot)$, where we ignore terms of higher than the second order.
This assumption is mathematically formulated in Assumption~\ref{assumption:taylor}, and used in proofs for all the following transferability-boosting methods.

(3) Since the DNN is highly non-linear, the change of the activation states is unpredictable for analysis.
To simplify our analysis, we assume that the adversarial perturbation does not significantly change gating states in each activation layer.
This assumption is mathematically formulated in Assumption~\ref{assumption:activation}, and used in proofs for all the following transferability-boosting methods.

(4) We notice that almost all attacking methods use the clip or sign operations to constrain the magnitude of perturbations, so such operations are not the essential difference between different methods.
Therefore, we usually ignore the clip and sign operations in all transferability-boosting methods, as well as the classic multi-step attacking method. In this way, we compare such methods all without the clip operation and the sign operation, in order to prove that all transferability-boosting methods reduce the interaction between perturbation units.

\textbullet{
\textbf{M1: MI Attack}~\cite{mim} incorporats the momentum of gradients in attacking to boost the adversarial transferability.
The gradient used in step $t$ of the MI Attack is computed as $g_{mi}^{(t)}= g_{mi}^{(t-1)}+\nabla_x Loss\left(x+\delta^{(t-1)}\right)/\left\|\nabla_x Loss\left(x+\delta^{(t-1)}\right)\right\|_{1}$, where $g_{mi}^{(t)}$ is the accumulated gradient.

To simplify the analysis, we ignore the normalization for $\nabla_x Loss(x+\delta^{(t-1)})$.
Note that $g_{mi}^{(t)}$ is the accumulated gradient.
The MI Attack is very complicated. Therefore, we first prove that Eq.~\eqref{eq:approximate MI} can approximate the MI Attack. Based on the formulation in Eq.~\eqref{eq:approximate MI}, we prove that the MI Attack generates perturbations with smaller interactions than the multi-step attack.
\begin{equation}
   g_{mi}^{(t)}=\mu g_{mi}^{(t-1)}+(1-\mu) \nabla_x Loss(x+\delta^{(t-1)}),\; \text{where}\; \mu=(t-1)/t
   \label{eq:approximate MI}
\end{equation}

Besides, a another important issue is that we need to make sure the MI Attack and the multi-step attack to have the same attacking strength. To this end, according to Eq.~\eqref{eq: attack}, $\beta=\alpha m$ measures the accumulated step size, which can reflect the attacking strength of the multi-step attack.
Therefore, for the MI Attack, we also require its accumulated attacking step size to be the same $\beta$ value.
Then, we prove the following proposition, which shows that perturbations generated by the MI Attack exhibit smaller interactions than perturbations generated by the multi-step attack.
}

\begin{proposition} \label{pro:mi}
(Perturbations of the MI Attack exhibit smaller interactions than perturbations of the multi-step attack, proof in {Appendix}~\ref{append:mi}.)
According to Eq.~\eqref{eq: attack}, the accumulated step size $\beta=\alpha m$ reflects the attacking strength of the multi-step attack. Similarly, we set the MI Attack to have the same accumulated step size for a fair comparison.
Then, if the classification loss for attacking is formulated as the cross-entropy loss,  we have $\sum_{a, b\in\Omega}[{I_{a b}(\delta^{(\text{mi})})}]\le\sum_{a, b\in\Omega}[{I_{a b}(\delta^{(\text{multi})})}]$.
\label{proposition2}
\end{proposition}

\textbullet{
\textbf{M2: LinBP Attack}~\cite{guo2020backpropagating} boosts the adversarial transferability via  \textit{linear back-propagation}.
In contrast to the standard back-propagation, the linear back-propagation is referred to as the process that the gradient back-propagates linearly as if there is no activation function encountered.
We use $g^{(t)}_{lbp}$ to denote the gradient computed in step $t$ of the LinBP Attack.

For a fair comparison, we need to control the magnitude of the LinBP Attack and the multi-step attack to the same level, \emph{i.e.,} we set perturbations generated by the LinBP and perturbations generated by the multi-step attack to have the same $\ell_2$ norm.
In this way, we compare perturbations generated by the two attacks in the scenario of binary classification.
Then, we prove the following propostion, which shows the LinBP Attack generates perturbations with smaller interactions than the multi-step attack.
}

\begin{proposition}
\label{pro:linbp}
(Perturbations of the LinBP Attack exhibit smaller interactions than perturbations of the multi-step attack, proof in {Appendix}~\ref{append:lbp}.)
For the multi-step attack, we consider the multi-step attack with infinite steps as a typical case for analysis, \emph{i.e.,} $\delta^{(\text{multi})}=\delta^{(\infty)}$.
For the LinBP Attack, we also analyze perturbations generated by the LinBP Attack for infinite steps with the infinitesimal step size $\alpha^{(lbp)}\to 0$, which is computed as $\delta^{(\textrm{lbp})} = \lim_{m\rightarrow +\infty}\alpha^{(lbp)}\sum_{t=1}^{m} g^{(t)}_{\textrm{lbp}}$.
If the classification loss for attacking  is formulated as the cross-entropy loss on the sigmoid function for binary classification, then under the assumption that $\|\delta^{(\textrm{multi})}\|_2 = \| \delta^{(\text{lbp})}\|_2$,  we have $\sum_{a, b\in\Omega}[{I_{a b}(\delta^{(\text{lbp})})}]\le\sum_{a, b\in\Omega}[{I_{a b}(\delta^{(\text{multi})})}]$.
\end{proposition}

\textbullet{
\textbf{M3: Adversarially-Trained Model Attack}~\cite{springer2021adversarial} finds that adversarial perturbations generated on adversarially-trained DNNs~\cite{pgd2018} are more transferable than those generated on normally-trained DNNs. Adversarial training~\cite{pgd2018} is an effective and the most widely-used method for adversarial defense, which is formulated as $\min_\theta \mathbb{E}_{x}\left[\max _{\|\delta\|_{p} \leq \varepsilon} \textit{Loss}_\theta(x+\delta)\right]$. $\theta$ denotes the parameters of the DNN.
We use $\theta^{\textrm{(adv)}}$ to denote parameters of the DNN after adversarial training, and use $\theta^{\textrm{(nor)}}$ to denote parameters of the normally-trained DNN.

We consider that the cross-entropy loss is used for both adversarial training and adversarial attacking.
Moreover, in order to fairly compare adversarial perturbations generated on the adversarially-trained DNN and those generated on the normally-trained DNN, we assume that magnitudes of the perturbations generated on the two DNNs are the same.
Then, we prove the following proposition, which shows that adversarial perturbations generated on an adversarially-trained DNN exhibit smaller interactions than adversarial perturbations generated on the normally-trained DNN.
}

\begin{proposition}
\label{pro:at}
(Perturbations generated with the multi-step attack on the adversarially-trained DNN have smaller interactions  than perturbations generated on the normally-trained DNN, proof in {Appendix}~\ref{append:at}.)
Let us assume that the classification loss for both attacking and adversarial training is the cross-entropy loss.
Let $\delta_{\theta^{(\text{adv})}}=\arg\max_{\|\delta\|_p \le \epsilon} \textit{Loss}_{\theta^{\text{(adv)}}}(x + \delta)$ and $\delta_{\theta^{(\text{nor})}}=\arg\max_{\|\delta\|_p \le \epsilon} \textit{Loss}_{\theta^{\text{(nor)}}}(x + \delta)$ be perturbations generated on the adversarially-trained DNN and perturbations generated on the normally-trained DNN, respectively.
Then, we have $\mathbb{E}_x \left[\sum_{a, b\in\Omega}I_{ab}\left(\delta_{\theta^{(\text{adv})}}\right) \right] \le  \mathbb{E}_x \left[\sum_{a, b\in\Omega}I_{ab}\left(\delta_{\theta^{(\text{nor})}}\right) \right]$.
\end{proposition}

\textbullet{
\textbf{M4: PI Attack}~\cite{gao2020patch} sets a norm constraint for adversarial perturbations on all units. Then, the PI Attack uniformly distributes the perturbations exceeding the norm constraint to its neighboring units.
We prove that the perturbation after $t$ steps in the PI Attack can be approximated as $\delta^{(t)}_{\textrm{pi}}= \delta^{(t-1)}_{\textrm{pi}} + \alpha\cdot A \nabla_{x} \textit{Loss}(x + \delta^{(t-1)}_{\textrm{pi}})$, where $A\in\mathbb{R}^{n\times n}$ is a matrix that uniformly distributes the perturbation on each unit that exceeds constraint to its $K$ surrounding units.
\begin{equation}
    A=\left[\begin{array}{ccccc}
    1-\tau_{1} & \frac{\tau_{2}}{K} & 0 & \ldots & \frac{\tau_{n}}{K} \\
    \frac{\tau_{1}}{K} & 1-\tau_{2} & 0 & \ldots & 0 \\
    0 & 0 & 1-\tau_{3} & \ldots & \frac{\tau_{n}}{K} \\
    \vdots & \vdots & \ddots & \ddots & \vdots \\
    \frac{\tau_{1}}{K} & 0 & \frac{\tau_{3}}{K} & \ldots & 1-\tau_{n}
    \end{array}\right]
\end{equation}
If perturbation on unit $a$ does not exceeds the norm, then $\tau_a = 0$; otherwise, $\tau_a$ is computed as the ratio of the perturbation exceeding the norm.
In this way, the remained perturbation on the unit $a$ is $(1-\tau_a)\nabla_x Loss(x+\delta_{pi}^{(t-1)})$.
Then, in the $a$-th column of the matrix $A$, the PI Attack identifies $K$ neighboring units of the unit $a$, and assigns $\frac{\tau_a}{K}$ of the exceeding perturbation on the unit $a$ to each neighboring unit.
The entries corresponding to other $n-1-K$ units are set to zero.

According to Eq.~\eqref{eq: attack}, $\beta=\alpha m$ measures the accumulated step size, which can reflect the attacking strength of the multi-step attack.
Similarly, the PI Attack is also conducted with multiple steps.
Therefore, for the PI Attack, we also set the accumulated attacking step size to be the same $\beta$ value.
In this way, we can ensure that perturbations generated by the multi-step attack and those generated by the PI Attack to have the same attacking strength for a fair comparison.
In the scenario of the binary classification, we prove the following proposition, which shows that the PI Attack generates perturbations with smaller interactions than the multi-step attack.
}

\begin{proposition}
\label{pro:pi}
(Perturbations of the PI Attack exhibit smaller interactions than perturbations of the multi-step attack, proof in {Appendix}~\ref{append:pi}.)
In the binary classification, let us consider that the classification loss for attacking is formulated as the cross-entropy loss on the sigmoid function.
We notice that both the the PI Attack and the multi-step attack are conducted with multiple steps. In this way, we conduct both the PI Attack and the multi-step attack for infinite steps with an infinitesimal step size.
Let $\delta^{(pi)}$ denote the adversarial perturbation generated by the PI Attack.
For the multi-step attack, the accumulated step size $\beta=\alpha m$ reflects the attacking strength. For the PI Attack, we also set the accumulated step size to be the same as $\beta$, in order to fairly compare the PI Attack and the multi-step attack.
Then, under Assumption~\ref{assumption:tau-mono} and Assumption~\ref{assumption:mean-v}, we have  $\sum_{a, b\in\Omega}[I_{a b}(\delta^{(\textrm{pi})})] \le \sum_{a, b\in\Omega}[I_{ab}(\delta^{(\textrm{multi})})]$.
\end{proposition}

\textbullet{
\textbf{M5: IA Attack}~\cite{zhu2022rethinking}.
Zhu \emph{et al.}~\cite{zhu2022rethinking} find that out-of-distribution (OOD) adversarial examples exhibit high adversarial transferability.
Based on this finding, the IA Attack~\cite{zhu2022rethinking} is proposed to finetune the pre-trained DNN, in order to make it able to generate more OOD adversarial examples.
Given the input $x$ and a DNN parameterized by $\theta$, $p_{\theta}(y \mid {x})$ denotes the probability of $x$ being classified to the ground-truth category $y$.
The objective of the IA Attack is given as $\max_{\theta}   \mathbb{E}_{(x, y)}\left[ -\nabla_{x}\cdot\frac{\nabla_{x} \log ( p_{\theta}(y \mid x))}{\left\|\nabla_{x} \log ( p_{\theta}(y \mid x))\right\|_{2}} \right]$.
Let $\theta^{\text{(nor)}}$ and $\theta^{\text{(ia)}}$ denote the normally-trained DNN and the DNN finetune by the IA Attack, respectively.

In order to ensure the fair comparison between perturbations generated on the normally-trained DNN and those generated on the DNN finetuned by the IA Attack, for a specific sample $x$, we assume that the classification power of the normally-trained DNN and the DNN finetuned by the IA Attack is the same.
Then, given an input $x$, the normally-trained DNN $\theta^{(\text{nor})}$, and the DNN finetuned by the IA Attack $\theta^{(\text{ia})}$, we generate adversarial perturbations on each DNN by using the single-step attack with the same step size for a fair comparison.
We prove that the perturbation generated on the DNN finetuned by the IA Attack exhibits smaller interactions than that generated on the normally-trained DNN.
}

\begin{proposition}
\label{pro:ia}
(Perturbations generated with the single-step attack on the DNN finetuned by the IA Attack has smaller interactions  than perturbations generated with the single-step attack on the normally-trained DNN, proof in {Appendix}~\ref{append:ia}.)
Let $\delta^{(\text{single})}_{\theta^{(\text{nor})}}$ and $\delta^{(\text{single})}_{\theta^{(\text{ia})}}$ denote adversarial perturbations generated on the pre-trained DNN and the DNN finetune by the IA Attack, respectively.
Let us consider the classification task on $c$ categories. Besides the ground-truth category, there are totally $c-1$ incorrect categories.
For the $c-1$ incorrect categories, let $p_k=p_{\theta^{(nor)}}(k|x)$ and $q_k=p_{\theta^{(ia)}}(k|x)$ denote the probability of the $k$-th category predicted by the normally-trained DNN and the probability of the $k$-th category predicted by the DNN finetuned by the IA Attack, respectively.
Without loss of generality, let $(i_1,i_2,\dots,i_{c-1})$ denote a permutation of the $c-1$ incorrect categories in the normally-trained DNN, which ensures that $p_{i_1}\ge p_{i_2} \ge \dots \ge p_{i_{c-1}}$.
Similarly, let $(j_1,j_2,\dots,j_{c-1})$ denote a permutation of the $c-1$ incorrect categories in the DNN finetuned by the IA Attack, which ensures that $q_{j_1} \ge q_{j_2} \ge \dots \ge q_{j_{c-1}}$.
If the classification loss of the DNN is formulated as the cross-entropy loss, then under Assumption~\ref{assum:ortho-weights} and Definition~\ref{def:balance} that $p_{\theta^{(nor)}}(y|x)=p_{\theta^{(ia)}}(y|x)$, predictions made by the DNN finedtuned by the IA Attack are more balanced than predictions made by the normally-trained DNN, \emph{i.e.,} $\forall 1\le k \le c-2$, $q_{j_{k}} - q_{j_{k+1}} \le p_{i_{k} } - p_{i_{k+1}}$.
\end{proposition}

The above proposition proves that the DNN finetuned by the IA Attack usually generates more balanced predictions on the $c-1$ incorrect categories than the normally-trained DNN.
On the other hand, given two adversarial perturbations generated on two different DNNs, we assume that the predictions made by a DNN is more balanced than the other DNN.
Then, when the two perturbations exhibit the same attacking strength, Theorem~\ref{theorem:balance} has proven that under Assumption ~\ref{assum:ortho-weights} and Definition~\ref{def:balance}, perturbations generated on the DNN that makes more balanced predictions exhibit smaller interactions than perturbations generated on the other DNN.
Therefore, let us consider the single-step attack on the normally-trained DNN and the single-step on the DNN finetuned by the PI Attack.
In order to fairly compare perturbations generated on the two DNNs, we set the step size of the single-step attack on the two DNNs to be the same.
Then, we can prove that perturbations generated on the DNN finetuned by the IA attack exhibit smaller interactions than perturbations generated on the normally-trained DNN, \emph{i.e.},
$\sum_{a, b\in\Omega}I_{ab}\left(\delta^{(\text{single})}_{\theta^{(\text{ia})}}\right) \le \sum_{a, b\in\Omega}I_{ab}\left(\delta^{(\text{single})}_{\theta^{(\text{nor})}}\right)$.

\textbullet{
\textbf{M6: RAP Attack}~\cite{qin2022boosting} boosts the adversarial transferability by generating adversarial perturbations locating at the flat loss landscape.
The objective of the untargeted RAP Attack~\cite{qin2022boosting} can be formulated as $\underset{\|\delta_{\text{rap}}\|_p \le \epsilon}{\text{max}}\,\, \underset{\|r\|_p \le \epsilon_r}{\text{min}} \,\, \textit{Loss} (x  + r + \delta_{\text{rap}}) $.
In each iteration, the RAP Attack first updates $r$ to minimize the classification loss, in order to seek a flat loss landscape.
Then, the RAP Attack updates the adversarial perturbation $\delta$ to maximize the classification loss.

For the multi-step attack, the accumulate step size $\beta=\alpha m$ represent the attacking strength of the generated perturbations. Similarly, the RAP Attack are also conducted for multiple steps. Therefore, we set the accumulated step size of the RAP attack to be the same as $\beta$, thereby fairly comparing perturbations generated by the multi-step attack and those generated by the RAP Attack.
We prove the following proposition, which shows that the RAP Attack generates perturbations with smaller interactions than the multi-step attack.
}

\begin{proposition}
\label{pro:rap}
(Perturbations of the RAP Attack exhibit smaller interactions than perturbations of the multi-step attack, proof in {Appendix~\ref{append:rap}}.)
Let $\delta^{(\text{rap})}$ denote the adversarial perturbation generated by the RAP attack.
The classification loss for attacking is formulated as the cross-entropy loss, and we use the infinite-step attack with the infinitesimal step size  to simplify both the RAP Attack and the multi-step attack.
For the fair comparison, we set the accumulate step size (attacking strength) of the RAP Attack to be the same as the accumulated step size of the multi-step attack, \emph{i.e.,} $\beta=\alpha m$.
Then, we have $\sum_{a, b\in\Omega}[{I_{a b}(\delta^{(\text{rap})})}]\le\sum_{a, b\in\Omega}[{I_{a b}(\delta^{(\text{multi})})}]$.
\end{proposition}

\textbullet{
\textbf{M7: IL Attack}~\cite{huang2019enhancing} generates more transferable adversarial perturbations by using intermediate features.
The IL Attack has two stages. In the first stage, the IL Attack generates a relatively weak adversarial perturbation $\delta_{\textrm{base}}$. Then, in the second stage, the IL Attack generates another adversarial perturbation $\delta^{(\text{il})}$ towards the direction of feature changes in the target layer, which are caused by the previous perturbation $f_l(x+\delta_{\textrm{base}}) - f_l(x)$, \emph{i.e.,}  $\max_{\|\delta^{(\text{il})} \|_p \le \epsilon } [f_l(x+\delta^{\text{(il)}}) - f_l(x) ]^T [f_l(x+\delta_{\textrm{base}}) - f_l(x) ]$, where $f_l(\cdot)$ denotes the feature of the $l$-th layer.

For a fair comparison, we need to control the magnitude of the perturbation generated by the IL Attack and the magnitude of the perturbation generated by the multi-step attack to be the same. In order words,  we need to ensure that the IL Attack and the multi-step attack generate perturbations with the same $\ell_2$ norm.
In this way, we compare interactions of perturbations generated by the IL Attack and interactions of perturbations generated by the multi-step attack.
Then, we prove the following proposition, which shows that the IL Attack generates perturbations with smaller interactions than the multi-step attack.
}

\begin{proposition}
\label{pro:il}
(Perturbations of the IL Attack exhibit smaller interactions than perturbations of the multi-step attack, proof in {Appendix~\ref{append:il}}.)
Let $\delta^{(\text{il})}$ denote the adversarial perturbation generated by the IL Attack.
We notice that both the multi-step attack and the IL Attack are conducted with multiple steps. Without loss of generality, we conduct both the multi-step attack and the IL Attack with infinite steps with the infinitesimal step size.
We also control the magnitude of perturbations generated by the IL Attack to be the same as the magnitude of perturbations generated by the multi-step attack, \emph{i.e.,} $\Vert\delta^{(\text{multi})}\Vert_2= \Vert \delta^{(\text{il})}\Vert_2$.
In this way, we can fairly compare perturbations generated by the IL Attack and perturbations generated by the multi-step attack.
If the classification loss for attacking  is formulated as the cross-entropy loss, then we have $\sum_{a, b\in\Omega}[I_{ab}(\delta^{(\textrm{il})})] \le \sum_{a, b\in\Omega}[I_{ab}(\delta^{(\textrm{multi})})]$.
\end{proposition}

\textbullet{
\textbf{M8: VR Attack}~\cite{wu2018understanding} smooths the gradient in attacking by adding the Gaussian noise on the input, which is computed as follows.
$L^{(\text{vr})}(x)  =  \mathbb{E}_{\xi\sim\mathcal{N}(0, \sigma^2 I)}\left[\textit{Loss}(\hat{x} = x + \xi)\right]$.
Besides, we assume that the variance of noise $\sigma^2$ in the VR Attack is small, so we approximate the classification loss using the Taylor expansion with terms of the first and second orders.
Let $\hat{p}$ denote the output probability of the DNN on the input with a Gaussian noise.
If we approximate the expected probability $\bar{p} = \mathbb{E}_{\xi\sim\mathcal{N}(0, \sigma^2I)}[\hat{p}]$ using $\bar{p}(\tau)$ in Eq.~\eqref{eq:pi-tau}, we have the following proposition.

\begin{proposition}\label{pro:p-bar-approx}
(Proof in {Appendix}~\ref{append:vr}.)
    For the $i$-th category, if the output probability of the original sample $p_i \le 0.5$, then the probability under the noise is boosted, \emph{i.e.,} $\bar{p}_i (\tau) - p_i \ge 0$; if $p_i \ge 0.5$, then $\bar{p}_i (\tau) - p_i \le 0$.
\end{proposition}

Proposition~\ref{pro:p-bar-approx} shows that Gaussian noises in the VR Attack make the output probability $\bar{p}(\tau)$ more balanced than $p$, thus we make Assumption~\ref{assum:vr-balance} based on this conclusion.
In order to fairly compare the single-step VR Attack and the normal single-step attack, we ensure perturbations generated by the single-step VR Attack to have the same attacking strength as perturbations generated by the normal single-step attack by using the normalization in Eq.~\eqref{eq:norm-ce-vr}.
Then, under Assumption~\ref{assum:ortho-weights} and Assumption~\ref{assum:vr-balance}, we prove the following proposition, which shows that the single-step VR Attack generates perturbations with smaller interactions than the normal single-step attack.}

\begin{proposition} \label{pro:vr}
(The single-step VR Attack generates perturbations with smaller interactions than the normal single-step attack, proof in {Appendix}~\ref{append:vr}.)
    {Let the classification loss $\textit{Loss}(x)$  be formulated as the cross-entropy loss in Eq.~\eqref{eq:softmax-ce}}.
    In order to fairly compare the single-step VR Attack and the normal single-step attack, we ensure perturbations generated by the single-step VR Attack to have the same attacking strength as perturbations generated by the normal single-step attack by using the normalization in Eq.~\eqref{eq:norm-ce-vr}.
    Then, under Assumption~\ref{assum:ortho-weights} and Assumption~\ref{assum:vr-balance},  we have $\sum_{a, b\in\Omega} \left[I_{ab}\left(\delta^{(\text{single})}_{{\text{vr}}}\right)\right] \le  \sum_{a, b\in\Omega} \left[I_{ab}\left(\delta^{(\text{single})}_{{\text{ce}}}\right) \right]$.
\end{proposition}

\textbullet{
\textbf{M9: SGM Attack}~\cite{Wu2020Skip} uses the gradient information of the skip connection in ResNets to improve the transferability of adversarial perturbations.
As discussed in Appendix~\ref{append:sgm}, the SGM Attack modifies the gradient in back-propagation of of ResNets to reduce gradient information from the residual modules, which can be considered as adding a specific dropout operation on the gradient.}

In our previous work~\cite{dropout2020zhang}, we have proven that the dropout operation can reduce the significance of interactions.
Thus, this also explains the reason why the SGM Attack generates perturbations with smaller interactions than the multi-step attack.

\textbullet{
\textbf{M10: GhostNet Attack}~\cite{li2020learning} proposes to generate more transferable adversarial perturbations on the GhostNet. The GhostNet uses two kinds of erosions to revise a pre-trained DNN, including the dropout erosion and the skip connection erosion. The dropout erosion means that \citet{li2020learning} densely applied Dropout operations after each layer of the DNN.
The skip connection erosion means that the GhostNet modifies the weight of features from the skip connection in each layer of ResNets. In the forward propagation given each input sample, the GhostNet randomly samples a weight from a uniform distribution to represent the importance of features from the skip connection.

In Appendix~\ref{append:ghost}, we show that we can roughly consider that the dropout erosion is similar to the SGM Attack, which has been proven to reduce interactions between perturbation units.
In Appendix~\ref{append:ghost}, we show that we can roughly consider the skip connection erosion as a special technique to add noises to the output of the DNN.
This is similar to the VR Attack that adds noises to the input, and the VR Attack has been proven to be equivalent to adding noises to the output of the DNN.
In this way, as proven in Appendix~\ref{append:vr}, the VR Attack can reduce interactions, so we can consider that the skip connection erosion also reduces interactions.
}

\textbullet{
\textbf{M11: FD Attack}~\cite{Inkawhich2020Transferable} puts a two-layer network for each category upon the feature of an intermediate layer in the DNN. Thus, the FD Attack retrains a two-layer network for the classification of each specific category. We can consider the FD Attack actually uses the two-layer network to replace to deep layers upon the feature in the original DNN.
Then, it generates adversarial perturbations by maximizing the output probability of the two-layer network corresponding to the target category.
Note that the FD Attack only focuses on the targeted attack.

Obviously, the newly trained two-layer (shallow) networks usually encode more linear feature representations than the original deep  network.
In this way, we can consider eigenvalues of the Hessian matrix in the two-layer network are smaller than those in the original DNN.
If the newly-trained two-layer network exhibits the similar classification power to the original deep network, then, according to Lemma~\ref{lemma:interaction}, the two-layer network learned by the FD Attack is very likely to generate perturbations with smaller interactions than the original DNN.
}

\textbullet{
\textbf{M12: FI Attack}~\cite{wang2021feature} generates adversarial perturbations as follows. For the input $x$, the FI Attack first randomly masks the input by the mask matrix $M\in\{0,1\}^n$.
Then, the FI Attack feeds the input under different masks into the DNN.
For each specific mask $M$, let $Loss(M\odot x)$ denote the loss function on the masked input, and let $f_k(M\odot x)$ denote the feature of the $k$-th layer, where $\odot$ denotes the element-wise multiplication.
Then, $g^{(M)}=\frac{\partial Loss(M\odot x)}{\partial f_k(M\odot x)}$ denote the gradient of the loss \emph{w.r.t.} the feature of the $k$-th layer.
The FI Attack computes the average gradient over different masks $\bar g = \mathbb{E}_{M\sim \text{Bernoulli}(q)}\frac{\partial Loss(M\odot x)}{\partial f_k(M\odot x)}$.
The objective of the FI Attack is to estimate the adversarial perturbation to further push intermediate-layer features away from the direction of the average gradient $\bar g$, \emph{i.e.,} $\min_{\delta} \sum {\bar g \odot f_k(x+\delta)}$.

In order to prove that the FI Attack generates perturbations with smaller interactions than the multi-step attack, we will prove the following three propositions.
(1) In the pairwise interaction between perturbation units, high-order interactions are major components, while low-order interactions constitute only a small proportion. (2) In the FI Attack, when the input is masked, high-order interactions will be more significantly destroyed than low-order interactions. In other words, most high-order interactions are destroyed and low-order interactions are remained. Therefore, the generated perturbations mainly encode low-order interactions. (3) Since low-order interactions are usually small, the perturbation mainly containing low-order interactions is more likely to exhibit smaller interactions than the perturbation that contain both high-order and low-order interactions.
Therefore, interactions between perturbation units generated by the FI Attack are usually smaller than those generated by the multi-step attack.
}

\section{Interaction loss for boosting adversarial transferability}\label{sec:IR loss}

\textbf{Interaction loss.}
Based on proofs in earlier sections,
we propose to improve the transferability of adversarial perturbations by directly penalizing interactions between perturbation units during attacking.
Therefore, we design the interaction loss $\ell_{\text{interaction}}$ in  Equation \eqref{eq:attack}, and jointly optimize the classification loss and the interaction loss to generate adversarial perturbations.
This method is termed the interaction-reduced attack (IR Attack).
\begin{equation}
    \begin{split}
        \underset{\delta}{\text{max}}\,  \big[Loss(x+\delta) - \lambda \ell_{\textrm{interaction}}\big], \;
        \ell_{\textrm{interaction}}\!=\! \mathbb{E}_{a, b\in\Omega}\left[ I_{ab}(\delta)\right], \;
        \text{s.t.}~ \|\delta\|_p\le \epsilon,\; x\!+\!\delta\in[0, 1]^n,
    \end{split}
     \label{eq:inte loss}
\end{equation}
where $\lambda$ is a constant weight for the interaction loss $\ell_{\textrm{interaction}}$.
We note that when the dimension of images is high, the computational cost of the interactions loss is intolerable {even when} simplified according to Eq.~{\eqref{eq:expect}}.
Therefore, in order to further reduce the computational cost, we divide the input image into $16\times 16$ grids, and compute and penalize the interactions at the grid level rather than the pixel level.
In addition, we apply an efficient sampling method~\cite{dropout2020zhang} to approximate the expectation in Eq.~\eqref{eq:expect}.

Figure~\ref{fig:interaction} visualizes interactions between each perturbation unit and its adjacent perturbation units that were computed at the grid level.
We used adversarial perturbations generated with and without the interaction loss, respectively.
The result shows that the interactions loss successfully reduced interactions between perturbation units.

\begin{figure}[t]
    \centering
    \includegraphics[width=0.85\linewidth]{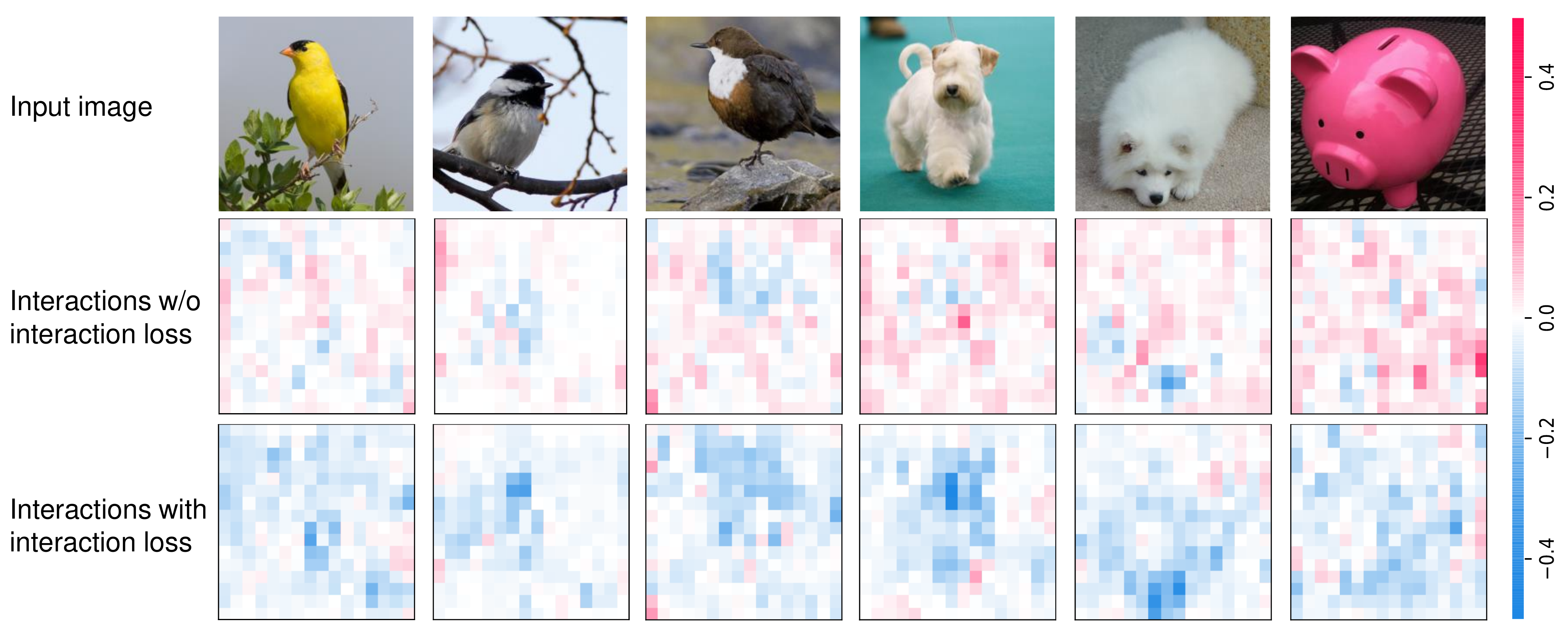}
    \vspace{-5pt}
    \caption{Visualization of interactions between neighboring perturbation units, which is obtained from perturbations generated with and without the interaction loss, respectively.
    Here, we ignore interactions between non-adjacent units to simplify the visualization, because adjacent units usually encode much more significant interactions than other units.
    The color of each pixel in heatmaps is computed as $color[i] \propto E_{j\in N_i}\left[ I_{i j}(\delta)\right]$, where $N_i$ denotes the set of adjacent perturbation units of the perturbation unit $i$.
    Note that in this experiment, we set $v(S)=\max_{y'\ne y}h_{y'}(x+\delta^{(S)})-h_y(x+\delta^{(S)})$, rather than $Loss(x+\delta^{(S)})-Loss(x)$ used in Corollary~\ref{corollary:semi-definite}.
    }
    \label{fig:interaction}
\end{figure}

\begin{table}[!t]
\renewcommand{\arraystretch}{1.1}
\renewcommand\tabcolsep{3.0pt}
\centering
\caption{
Transferability with and without the interaction loss: the success rates of $\ell_\infty$ and $\ell_2$ black-box attacks crafted on six source models, including AlexNet, VGG16, RN-34/152, DN-121/201, against seven target models.
Penalizing interactions between perturbation units boosted the transferability of adversarial perturbations.
\vspace{-5pt}
}
\resizebox{\linewidth}{!}{
\begin{tabular}{c|c|ccccccc}
\hline
  Source& Method & VGG-16 & RN152 & DN-201 & SE-154 & IncV3 & IncV4 & IncResV2  \\ \hline
  \multirow{2}{*}{{AlexNet}} & PGD $\ell_\infty$ & 67.0$\pm$1.6 & 27.8$\pm$1.1 & 32.3$\pm$0.4 & 28.2$\pm$0.7 & 29.1$\pm$1.5 & 23.0$\pm$0.4 & 18.6$\pm$1.5 \\
  & PGD $\ell_\infty$+IR & \textbf{78.7$\pm$1.0} & \textbf{42.0$\pm$1.5} & \textbf{50.3$\pm$0.4} & \textbf{41.2$\pm$0.6} & \textbf{43.7$\pm$0.5} & \textbf{36.4$\pm$1.5} & \textbf{29.0$\pm$1.0} \\\hline
  \multirow{2}{*}{{VGG-16}} & PGD $\ell_\infty$ & -- & 43.0$\pm$1.8 & 48.3$\pm$2.0 & 52.9$\pm$2.7 & 39.3$\pm$0.7 & 49.3$\pm$1.1 & 29.7$\pm$2.0 \\
  & PGD $\ell_\infty$+IR & -- & \textbf{63.1$\pm$1.6} & \textbf{70.0$\pm$1.1} & \textbf{71.2$\pm$1.5} & \textbf{57.6$\pm$1.0} & \textbf{68.6$\pm$3.2} & \textbf{49.2$\pm$1.2} \\\hline
  \multirow{2}{*}{{RN-34}} & PGD $\ell_\infty$ & 65.4$\pm$2.9 & 59.2$\pm$2.7 & 63.5$\pm$3.3 & 33.1$\pm$2.9 & 27.4$\pm$3.6 & 23.9$\pm$1.7 & 21.1$\pm$1.1 \\
  & PGD $\ell_\infty$+IR & \textbf{84.0$\pm$0.5} & \textbf{84.7$\pm$2.3} & \textbf{88.5$\pm$0.9} & \textbf{64.4$\pm$1.6} & \textbf{56.9$\pm$3.1} & \textbf{59.3$\pm$4.3} & \textbf{49.2$\pm$1.1} \\\hline
  \multirow{2}{*}{{RN-152}} & PGD $\ell_\infty$ &  51.6$\pm$3.2 & -- & 61.5$\pm$2.4 & 33.9$\pm$1.5 & 28.1$\pm$0.9 & 25.0$\pm$1.2 & 22.4$\pm$1.0 \\
  & PGD $\ell_\infty$+IR & \textbf{72.3$\pm$1.2} & -- & \textbf{82.1$\pm$1.3} & \textbf{61.1$\pm$0.9} & \textbf{53.6$\pm$0.8} & \textbf{50.6$\pm$3.5} & \textbf{46.0$\pm$2.3} \\ \hline
  \multirow{2}{*}{{DN-121}} & PGD $\ell_\infty$ & 68.6$\pm$1.1 & 63.6$\pm$3.2 & 86.9$\pm$1.5 & 46.1$\pm$1.5 & 37.3$\pm$1.6 & 37.1$\pm$2.1 & 28.9$\pm$2.8 \\
  & PGD $\ell_\infty$+IR & \textbf{85.0$\pm$0.3} & \textbf{84.8$\pm$0.4} & \textbf{95.1$\pm$0.2} & \textbf{70.3$\pm$1.7} & \textbf{61.1$\pm$2.5} & \textbf{62.1$\pm$2.0} & \textbf{53.5$\pm$0.3} \\ \hline
  \multirow{2}{*}{{DN-201}} & PGD $\ell_\infty$ & 64.4$\pm$1.4 & 67.8$\pm$0.2 & -- & 50.9$\pm$0.8 & 39.5$\pm$3.3 & 36.5$\pm$0.9 & 34.2$\pm$0.4 \\
  & PGD $\ell_\infty$+IR & \textbf{78.6$\pm$2.5} & \textbf{85.0$\pm$1.1} & -- & \textbf{73.9$\pm$0.5} & \textbf{61.6$\pm$1.8} & \textbf{63.7$\pm$0.6} & \textbf{56.4$\pm$2.1} \\ \hline \hline
  \multirow{2}{*}{AlexNet} & PGD $\ell_2$ & 85.1$\pm$1.5 & 58.9$\pm$1.0 & 60.2$\pm$2.1 & 55.1$\pm$1.5 & 56.0$\pm$3.7 & 49.6$\pm$3.4 & 44.6$\pm$3.3 \\
  & PGD $\ell_2$+IR & \textbf{91.6$\pm$1.1} & \textbf{72.0$\pm$1.6} & \textbf{76.8$\pm$1.0} & \textbf{69.0$\pm$1.0} & \textbf{73.0$\pm$0.8} & \textbf{63.1$\pm$2.1} & \textbf{59.4$\pm$1.9} \\ \hline
  \multirow{2}{*}{VGG-16} & PGD $\ell_2$ &  -- & 76.7$\pm$0.9 & 82.3$\pm$2.9 & 83.5$\pm$1.9 & 77.5$\pm$3.6 & 82.1$\pm$2.2 & 69.4$\pm$2.1 \\
  & PGD $\ell_2$+IR & -- & \textbf{86.5$\pm$0.9} & \textbf{88.9$\pm$1.5} & \textbf{89.6$\pm$1.2} & \textbf{85.2$\pm$1.1} & \textbf{88.3$\pm$1.4} & \textbf{80.4$\pm$0.4} \\ \hline
  \multirow{2}{*}{RN-34} & PGD $\ell_2$ & 88.2$\pm$1.4 & 86.2$\pm$0.4 & 89.6$\pm$1.3 & 66.9$\pm$1.1 & 64.2$\pm$2.9 & 60.0$\pm$1.9 & 55.2$\pm$1.8 \\
  & PGD $\ell_2$+IR & \textbf{95.2$\pm$0.2} & \textbf{95.4$\pm$0.1} & \textbf{96.7$\pm$0.6} & \textbf{86.7$\pm$1.2} & \textbf{84.3$\pm$0.6} & \textbf{81.8$\pm$1.9} & \textbf{80.4$\pm$1.9} \\ \hline
  \multirow{2}{*}{DN-121} & PGD $\ell_2$ & 89.4$\pm$1.1 & 86.8$\pm$1.0 & 97.6$\pm$1.0 & 75.6$\pm$1.7 & 70.1$\pm$2.9 & 70.4$\pm$4.4 & 66.5$\pm$4.7 \\
  & PGD $\ell_2$+IR & \textbf{94.2$\pm$0.1} & \textbf{93.3$\pm$0.8} & \textbf{97.7$\pm$0.3} & \textbf{87.8$\pm$0.7} & \textbf{84.5$\pm$0.7} & \textbf{84.2$\pm$0.1} & \textbf{82.4$\pm$0.1} \\  \hline
\end{tabular}}
\label{table:ir}
\end{table}

\begin{table*}[!t]
\renewcommand{\arraystretch}{1.1}
\renewcommand\tabcolsep{3.0pt}
\centering
\caption{
Transferability of perturbations generated on the ensemble model: the success rates of $\ell_\infty$ black-box attacks crafted on the ensemble model (RN-34+RN-152+DN-121) against nine target models.
\vspace{-10pt}
}
\resizebox{\linewidth}{!}{
\begin{tabular}{c|c|ccccccccccc}
\hline
  Source& Method & VGG-16 & RN-152  & DN-201 & SE-154 & IncV3 & IncV4 & IncResV2 & DPN-68 & NASN-L & PNASN  \\ \hline
  \multirow{2}{*}{Ensemble} & PGD $\ell_\infty$ & 86.6$\pm$1.2 & \textbf{99.9$\pm$0.1}&  \textbf{97.0$\pm$0.7} & 70.7$\pm$1.6 & 64.2$\pm$0.3 & 57.7$\pm$2.4 & 53.1$\pm$0.7 & 61.6$\pm$0.5& 59.6$\pm$0.4& 72.3$\pm$0.3 \\
  & PGD $\ell_\infty$+IR & \textbf{91.5$\pm$0.1} & 92.4$\pm$1.6 &  92.1$\pm$1.7 & \textbf{86.1$\pm$0.3} & \textbf{81.6$\pm$0.9} & \textbf{79.9$\pm$1.7} & \textbf{78.4$\pm$1.3} & \textbf{82.5$\pm$1.0} & \textbf{82.3$\pm$1.6} & \textbf{85.6$\pm$0.5} \\ \hline
\end{tabular}}
\label{table:ensemble}
\end{table*}

\begin{table*}[t]
\renewcommand{\arraystretch}{1.1}
\renewcommand\tabcolsep{3.0pt}
\centering
\caption{Transferability against the secured models: the success rates of $\ell_\infty$ black-box attacks crafted on RN-34 and DN-121 source models against three secured models.
\vspace{-5pt}
}
\resizebox{\linewidth}{!}{
\begin{tabular}{c|c|rrr|c|c|rrr}
\hline
  Source& \!Method \!&\!\! { IncV3$_{ens3}$} \!\!&\!\! { IncV3$_{ens4}$} \!\!&\!\!{ IncRes$_{ens3}$} \!&\!\!Source \!&\!\! Method \!\!&\!\! {  IncV3$_{ens3}$} \!\!&\!\!{ IncV3$_{ens4}$} \!\!&\!\!{ IncRes$_{ens3}$} \\ \hline
\multirow{4}{*}{{\makecell{ RN-34} }} & PGD $\ell_\infty$ & 9.8$\pm$0.1 & 10.0$\pm$0.5 & 5.7$\pm$0.3  & \multirow{4}{*}{\makecell{DN-121}} & PGD $\ell_\infty$ & 12.8$\pm$0.1 & 11.2$\pm$1.7 & 6.9$\pm$1.0  \\
 & PGD $\ell_\infty$+IR & \textbf{26.5$\pm$2.9} & \textbf{22.1$\pm$1.3} & \textbf{14.3$\pm$0.4} & & PGD $\ell_\infty$+IR  & \textbf{28.0$\pm$1.8} & \textbf{26.5$\pm$2.1} & \textbf{17.4$\pm$1.3} \\ \cline{2-5} \cline{7-10}
 & TI{\footnotemark \label{fn:TI}}   & 21.4$\pm$0.8 & 20.9$\pm$0.9 & 14.9$\pm$1.4 & & TI\footnotemark[2]  & 26.8$\pm$1.3 & 26.1$\pm$1.5 & 19.4$\pm$1.6 \\
 & TI\footnotemark[2] + IR & \textbf{33.6$\pm$0.4} &  \textbf{33.2$\pm$0.3} & \textbf{24.0$\pm$0.5} & & TI\footnotemark[2] + IR  & \textbf{38.0$\pm$2.5} & \textbf{42.2$\pm$7.7} & \textbf{29.0$\pm$1.4} \\ \hline
\end{tabular}
}
\label{table:secured}
\end{table*}
\footnotetext[2]{The TI Attack was designed oriented to the secured DNNs which were robustly trained via adversarial training.
 Thus, we applied the TI Attack to the secured models in Table~\ref{table:secured}.}

\begin{table}[t!]
\renewcommand{\arraystretch}{1.1}
\renewcommand\tabcolsep{3.0pt}
\centering
\caption{
Transferability with and without the interaction loss.
The success rates of $\ell_\infty$ black-box attacks crafted by different methods on four source models (RN-34/152, DN-121/201) against seven target models.
Penalizing the interaction between perturbation units enhanced the transferability of adversarial perturbations.
\vspace{-5pt}}
\resizebox{\linewidth}{!}{
\begin{tabular}{c|c|ccccccc}
\hline
  Source& Method & VGG-16 & RN152 & DN-201 & SE-154 & IncV3 & IncV4 & IncResV2  \\ \hline
  \multirow{5}{*}{{RN-34}}
  & MI & 80.1$\pm$0.5 & 73.0$\pm$2.3 & 77.7$\pm$0.5 & 48.9$\pm$0.8 & 46.2$\pm$1.2 & 39.9$\pm$0.5 & 34.8$\pm$2.5 \\
  & VR & 88.8$\pm$0.2 & 86.4$\pm$1.6 & 87.9$\pm$2.4 & 62.1$\pm$1.5 & 58.4$\pm$3.0 & 56.3$\pm$2.3 & 49.7$\pm$0.9 \\
  & SGM & 91.8$\pm$0.6 & 89.0$\pm$0.9 & 90.0$\pm$0.4 & 68.0$\pm$1.4 & 63.9$\pm$0.3 & 58.2$\pm$1.1 & 54.6$\pm$1.2 \\ \cline{2-2}
  & SGM+IR & 94.7$\pm$0.6 & 91.7$\pm$0.6 & 93.4$\pm$0.8 & 72.7$\pm$0.4 & 68.9$\pm$0.9 & 64.1$\pm$1.3 & 61.3$\pm$1.0 \\
  & HybridIR & \textbf{96.5$\pm$0.1} & \textbf{94.9$\pm$0.3} & \textbf{95.6$\pm$0.6} & \textbf{79.7$\pm$1.0} & \textbf{77.1$\pm$0.8} & \textbf{73.8$\pm$0.1} & \textbf{70.2$\pm$0.5} \\ \hline
  \multirow{5}{*}{{RN-152}}
  & MI & 70.3$\pm$0.6 & -- & 74.8$\pm$1.4 & 51.7$\pm$0.8 & 47.1$\pm$0.9 & 40.5$\pm$1.6 & 36.8$\pm$2.7 \\
  & VR & 83.9$\pm$3.4 & -- & 91.1$\pm$0.9 & 70.0$\pm$3.7 & 63.1$\pm$0.9 & 58.8$\pm$0.1 & 56.2$\pm$1.3 \\
  & SGM & 88.2$\pm$0.5 & -- & 90.2$\pm$0.3 & 72.7$\pm$1.4 & 63.2$\pm$0.7 & 59.1$\pm$1.5 & 58.1$\pm$1.2 \\ \cline{2-2}
  & SGM+IR & 92.0$\pm$1.0 & -- & 92.5$\pm$0.4 & 79.3$\pm$0.1 & 69.6$\pm$0.8 & 66.2$\pm$1.0 & 63.6$\pm$0.9 \\
  & HybridIR & \textbf{95.3$\pm$0.4} & -- & \textbf{96.9$\pm$0.2} & \textbf{84.7$\pm$0.7} & \textbf{80.0$\pm$1.2} & \textbf{77.5$\pm$0.8} & \textbf{75.6$\pm$0.6} \\ \hline
  \multirow{5}{*}{{DN-121}}
  & MI & 83.0$\pm$4.9 & 72.0$\pm$0.7 & 91.5$\pm$0.2 & 58.4$\pm$2.6 & 54.6$\pm$1.6 & 49.2$\pm$2.4 & 43.9$\pm$1.5 \\
  & VR & 91.5$\pm$0.5 & 88.7$\pm$0.5 & 98.8$\pm$0.2 & 75.1$\pm$1.3 & 74.3$\pm$1.7 & 75.6$\pm$3.0 & 69.8$\pm$1.3 \\
  & SGM & 88.7$\pm$0.9 & 88.1$\pm$1.0 & 98.0$\pm$0.4 & 78.0$\pm$0.9 & 64.7$\pm$2.5 & 65.4$\pm$2.3 & 59.7$\pm$1.7 \\ \cline{2-2}
  & SGM+IR & 91.7$\pm$0.2 & 90.4$\pm$0.4 & 94.3$\pm$0.1 & 87.0$\pm$0.4 & 78.8$\pm$1.3 & 79.5$\pm$0.2 & 75.8$\pm$2.7 \\
  & HybridIR & \textbf{96.9$\pm$0.4} & \textbf{96.8$\pm$0.4} & \textbf{99.1$\pm$0.4} & \textbf{90.9$\pm$0.5} & \textbf{88.4$\pm$0.8} & \textbf{87.8$\pm$0.8} & \textbf{87.1$\pm$0.4} \\ \hline
  \multirow{5}{*}{{DN-201}}
  & MI & 77.3$\pm$0.8 & 74.8$\pm$1.4 & -- & 64.6$\pm$1.0 & 56.5$\pm$2.5 & 51.1$\pm$2.1 & 47.8$\pm$1.9 \\
  & VR & 87.3$\pm$1.1 & 90.4$\pm$1.2 & -- & 78.0$\pm$1.5 & 75.8$\pm$2.1 & 75.8$\pm$1.3 & 71.3$\pm$1.2 \\
  & SGM & 87.3$\pm$0.3 & 92.4$\pm$1.0 & -- & 82.9$\pm$0.2 & 72.3$\pm$0.3 & 71.3$\pm$0.6 & 68.8$\pm$0.5 \\ \cline{2-2}
  & SGM+IR & 89.5$\pm$0.9 & 91.8$\pm$0.7 & -- & 87.3$\pm$1.2 & 82.5$\pm$0.8 & 80.3$\pm$0.3 & 81.5$\pm$0.5 \\
  & HybridIR & \textbf{94.4$\pm$0.1} & \textbf{96.9$\pm$0.5} & -- & \textbf{91.7$\pm$0.2} & \textbf{89.6$\pm$0.6} & \textbf{88.3$\pm$0.3} & \textbf{87.3$\pm$0.7} \\ \hline
\end{tabular}}
\label{table:ir+}
\end{table}

\textbf{Experimental settings.}
We generated adversarial perturbations on six source DNNs, including Alexnet~\cite{krizhevsky2012imagenet}, VGG-16~\cite{simonyan2015very}, ResNet-34/152  (RN-34/152)~\cite{he2016deep} and DenseNet-121/201  (DN-121/201)~\cite{huang2017densely}.
For adversarial perturbations generated on each source DNN, we evaluated their transferability on seven target DNNs, including VGG-16, ResNet-152 (RN-152), DenseNet-201 (DN-201), SENet-154 (SE-154)~\cite{hu2018squeeze}, InceptionV3 (IncV3)~\cite{szegedy2016rethinking}, InceptionV4 (IncV4)~\cite{szegedy2017inception}, and Inception-ResNetV2 (IncResV2)~\cite{szegedy2017inception}.
We also evaluated the transferability on secured target models from Tramer~\emph{et al.}~\cite{tramer2017ensemble}, including IncV3$_{ens3}$ (ensemble of three IncV3 networks), IncV3$_{ens4}$ (ensemble of four IncV3 networks), and IncResV2$_{ens3}$ (ensemble of three IncResV2 networks), which were learned via ensemble adversarial training.
Besides, in order to evaluate the transferability of perturbations generated on ensemble source models, we generated adversarial perturbations on the ensemble of RN-34, RN-152, and DN-121.
The target DNNs included the Dual-Path-Network (DPN-68)~\cite{chen2017dual}, the NASNet-LARGE (NASN-L)~\cite{zoph2018learning}, and the Progressive NASNet (PNASN)~\cite{liu2018progressive}, which were three state-of-the-art DNNs.
The proposed IR Attack was conducted in the scenario of the ensemble-based attacking~\cite{liu2016delving}.

\textit{Baseline methods.} We used five baseline attacking methods for comparison, including the PGD Attack~\cite{pgd2018}, MI Attack~\cite{mim}, VR Attack~\cite{wu2018understanding}, SGM Attack~\cite{Wu2020Skip}, and TI Attack~\cite{dong2019evading}.
Our IR Attack was implemented according to Eq.~\eqref{eq:inte loss}.
Besides, we also added the interaction loss $\ell_{\textrm{interaction}}$ to the MI, VR, and SGM Attacks as new implementations of our method, namely the MI+IR, VR+IR, and SGM+IR Attacks, respectively.
Please refer to Appendix~\ref{append:ir+} for details.
Moreover, since Section~\ref{sec:explaining} has proven that the MI, VR, and SGM Attacks can all reduce interactions during attacking, then we combined the IR Attack with all these techniques together as a new implementation of our method, namely the HybridIR Attack.

All the above attacks were conducted for 100 steps\footnote{Previous studies usually conducted attacks for 10 or 20 steps. For a fair comparison between different attacks, the number of steps was set to 100 together with the leave-one-out (LOO) validation. In other words, the transferability of each baseline attacking method was computed on the best adversarial perturbation during the 100 steps.
Please see the appendix~\ref{append:loo} for the motivation and the evidence of the LOO evaluation of transferability.\label{fn:1}} on 1000 images, which were randomly selected from the validation set in the ImageNet dataset.
We set $\epsilon=16/255$ for the $\ell_\infty$ attack, and $\epsilon=16/255\sqrt{n}$ following  \cite{mim} for the $\ell_2$ attack.
For all attacks, the step size was $2/255$.
Considering that DNNs with different depths usually exhibited different signal processing efficiency, in the IR Attack, we set $\lambda=1$ when the source DNN was a ResNet and $\lambda=2$ for other source DNNs.
All attacks were conducted using three different initial perturbations or different randomly sampled grids.

Table~\ref{table:ir} reports the success rates of the perturbations generated by the baseline PGD Attack~\cite{pgd2018} and the IR Attack when attacking target models.
For the IR Attack, we applied $\ell_\infty$ and $\ell_2$ constraints for perturbations, namely PGD $\ell_{\infty}$+IR and PGD $\ell_{2}$+IR, respectively.
Compared with the baseline PGD Attack, the interaction loss significantly improved the transferability on various source models against different target models.
For most source models and target models, interaction loss boosted the transferability by more than 10\%.
In particular, when the source DNN were DN-201 and the target DNN were IncV4, respectively, the interaction loss improved the transferability from 36.5\% to 63.7\% ($>$~27\% gain).

For ensemble models, Table~\ref{table:ensemble} shows that in most cases, the IR Attack improved the transferability of  perturbations by more than $10\%$.
Besides, as Table~\ref{table:secured} shows, the interaction loss also significantly improved the transferability of perturbations against the secured target DNNs.
Note that in order to make settings in Table~\ref{table:secured} to be consistent with \cite{tramer2017ensemble}, we did not use the LOO validation.

Table~\ref{table:ir+} shows that the interaction loss also further boosted the transferability of other attacking methods.
For example, the highest transferability achieved by the SGM Attack against IncResV2 was 68.8\% (when the source is DN-201).
The interaction loss further improved the transferability to 81.5\%  with more than 12\% gain.
Moreover, when we combined the three methods of reducing interactions together, namely the HybridIR Attack, the transferability was improved from a range of 54.6$\sim$98.8\% to 70.2$\sim$99.1\%.
Such improvement {due to} the interaction loss can be understood as follows.
Adversarial perturbations generated by different attacking methods were usually in different manifolds, thereby exhibiting different transferability. Based on these perturbations, the interaction loss can point out the optimization direction of further reducing interactions in a local manner. Thus, the interaction loss further boosts the transferability.

\textit{\textbf{Effects of the interaction loss.}}
We evaluated the transferability of perturbations generated by the IR Attack when setting different weights $\lambda$ for the interaction loss.
In particular, if $\lambda=0$, the IR Attack degraded to the PGD Attack.
We generated adversarial perturbations on two source DNNs (RN-34 and DN-121), and transferred perturbations generated on each source DNN to seven target DNNs (VGG16, RN-152, DN-201, SE-154, IncV3, IncV4, and IncResV2).
Figure~\ref{fig:lambda}~(a) shows the successful transferring rates with different values of $\lambda$.
We found that the transferability of the generated perturbations increased as the weight $\lambda$ increased.

\begin{figure}[t]
    \centering
    \begin{minipage}{0.46\linewidth}
    \caption{(a) The transferability of adversarial perturbations generated by the IR Attack using different values of $\lambda$.
    The success rates increased along with the increase of the value of $\lambda$.
    (b) The transferability of adversarial perturbations generated by only using the interaction loss (without the classification loss).
    Such adversarial perturbations still exhibited moderate adversarial transferability.
    Points localized at the last epoch represent the baseline transferability of noise perturbations.}
    \label{fig:lambda}
    \end{minipage}
    \hfill
    \begin{minipage}{0.5\linewidth}
     \includegraphics[width=\linewidth]{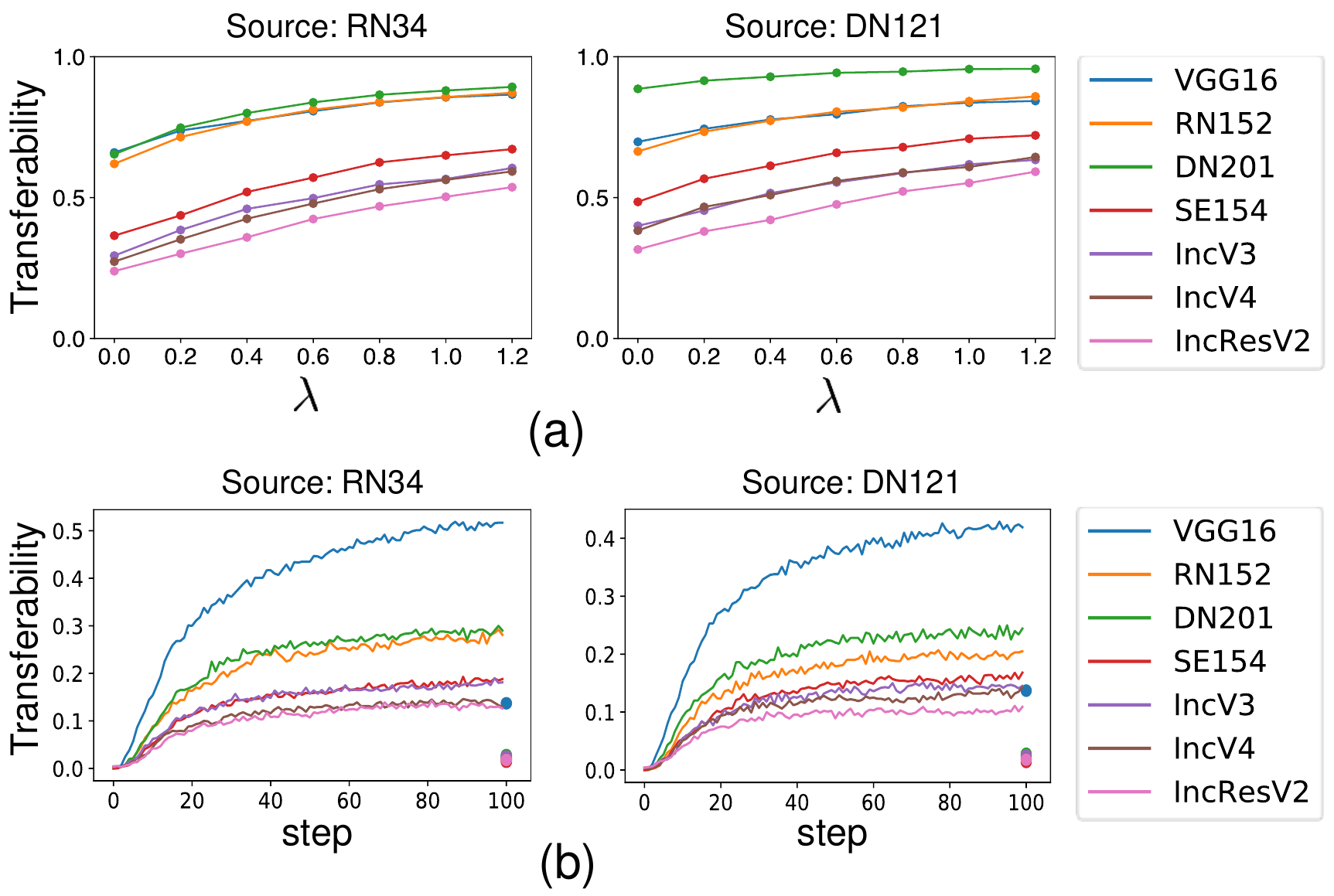}
    \end{minipage}
\vspace{-5pt}
\end{figure}

\textit{\textbf{Attack {solely} with the interaction loss.}}
In order to further evaluate the effects of the interaction loss, we exclusively used the interactions without the classification loss to generate adversarial perturbations on images from the  ImageNet validation set.
We attacked the RN-34 and DN-121 as source DNNs to generate adversarial perturbations for 100 steps\footref{fn:1}, and then evaluated the transferability of perturbations on seven target DNNs.
Besides, we compared such adversarial perturbations with noise perturbations, \emph{i.e.,} $\epsilon\cdot\text{sign}(noise)$, where $\epsilon=16/255$ and $noise \sim \mathcal{N}(0, \sigma^2 I)$.
Figure~\ref{fig:lambda}~(b) shows that the transferability of perturbations generated by only using the interaction loss were still considerable.
This phenomenon may be due to the fact that perturbations generated by only using the interaction loss led to a decrease in most interactions in the DNN, thereby damaging the inference patterns learned on the input.

\section{Conclusion}
In this paper, we have theoretically proven the negative correlation between the transferability and interactions inside adversarial perturbations.
In other words, the adversarial perturbations with higher transferability usually exhibit smaller interactions.
Then, we have proven that the decrease of interactions between perturbation units is a common mechanism shared by twelve classical methods of boosting adversarial transferability.
Furthermore, we have found that the above explanations help us in identifying and refining some algorithmic components that may not boost the adversarial transferability.
Moreover, we have proposed a new loss function {that directly penalizes interactions between perturbation units during attacking, thus significantly improving the adversarial transferability.}

\bibliographystyle{plainnat}
\bibliography{main}

\newpage
\onecolumn
\appendix
\setcounter{equation}{0}
\section{{Interaction based on the Shapley value}}

In this section, we discuss the properties of the Shapley value and those of interactions based on the Shapley value.

\subsection{Four axioms of the Shapley value} 
\label{append:shapley-axiom}
The Shapley value~\cite{shapley1953value} measures the importance of each player in a game.
Let $\Omega=\{1, 2, \dots, n\}$ be the set of all players.
Given $S\subseteq \Omega$, the reward function is denoted by  $v(S)$, which represents the reward obtained by players in $S$.
We use $\phi(a|\Omega)$ to denote the Shapley value of player $a$ in the game with all players $\Omega$ and reward function $v$, which is given as follows.
\begin{equation}
\phi(a|\Omega)=\sum_{S \subseteq \Omega \backslash\{a\}} \frac{|S| !(n-|S|-1) !}{n !}(v(S \cup\{a\})-v(S)). \label{eq:shapley_appendix}
\end{equation}

The Shapley value satisfies four axioms~\cite{weber1988probabilistic} as follows.

\textbullet{ \textit{Linearity axiom:} If there are two games  with reward functions $v$ and $w$, then { $v(S)$} and { $w(S)$} measure the reward obtained by players in {$S$} in games $v$ and $w$,   respectively.
Let $\phi_{v}(a|\Omega)$ and $\phi_w(a|\Omega)$ denote Shapley values of player $a$ in games $v$ and $w$, respectively.
If these two games are combined into a new game, and the reward function becomes $\textit{reward}(S)=v(S)+w(S)$, then the Shapley value is $\phi_{v+w}(a|\Omega)=\phi_{v}(a|\Omega) + \phi_{w}(a|\Omega)$} for each player $a$ in $\Omega$.

\textbullet{ \textit{Dummy axiom:} A player $a \in \Omega$ is said to be a dummy player, if {{$\forall S\subseteq \Omega\backslash \{a\}$}, $v(S\cup \{a\}) = v(S) + v(\{a\})$}.
Thus, { $\phi(a|\Omega) = v(\{a\})-v(\emptyset)$}, which means that player $a$ plays the game individually. }

\textbullet{ \textit{Symmetry axiom:} If { {$\forall S\subseteq\Omega\setminus\{a,b\}$}, $v(S \cup \{a\}) = v(S \cup \{b\})$}, then Shapley values of player $a$ and $b$ are equal, \textit{i.e.} { $\phi(a|\Omega)=\phi(b|\Omega)$} .}

\textbullet{ \textit{Efficiency axiom:} The sum of each player's Shapley value is equal to the overall reward won by the coalition $N$, \textit{i.e.} { $\sum_a \phi(a|\Omega) = v(\Omega)-v(\emptyset)$}.
This axiom ensures that the overall reward can be allocated to each player in the game.}

\subsection{Pairwise interaction} \label{append:consistent-interaction}
In this section, we prove that the two forms of pairwise interaction defined in Section~\ref{subsec:interactions} are equivalent.
Moreover, they are consistent  with the Shapley interaction index~\cite{grab1999An}.

\textbf{First form of  the pairwise  interaction.}
In Section~\ref{subsec:interactions}, the interaction between perturbation units $(\delta_a,\delta_b)$ is defined as the additional contribution as follows.
\begin{equation} \label{eq:first-form-interaction}
    I_{ab}(\delta) = \phi(S_{a b}|\Omega') - \left[\phi(a|\Omega\setminus\{b\}) + \phi(b|\Omega\setminus\{a\})\right],
\end{equation}
where $\phi(S_{a b}|\Omega')$ denotes the joint contribution of $a, b$, when perturbation units $a, b$ are regarded as a singleton unit $S_{a b}=\{a, b\}$.
$\phi(a|{\Omega\setminus\{b\}})$ and $\phi(b|{\Omega\setminus\{a\}})$ represent the individual contributions of units $a$ and $b$, respectively, when perturbation units $a, b$ work individually.

\textbf{Second form  of the pairwise interaction.}
In Section~\ref{subsec:interactions}, the interaction between perturbation units $(\delta_a,\delta_b)$ can also be defined as
the change of the  importance $\phi_a$ of the $a$-th perturbation unit when the $b$-th perturbation unit $\delta_b$ is perturbed \emph{w.r.t} the case when the $b$-th perturbation unit $\delta_b$ is not perturbed. 
If perturbation $\delta_b$ on the $b$-th perturbation unit increases the importance $\phi_a$ of the $a$-th perturbation unit, then a positive interaction exists between $\delta_a$ and $\delta_b$.
If perturbation $\delta_b$ decreases the importance $\phi_a$, it indicates a negative interaction.
Mathematically, this definition of the pairwise interaction can be expressed as follows.
\begin{equation} \label{eq:second-form-interaction}
    I'_{a b}(\delta) = \phi_{a, \text{w/}\, b} - \phi_{a, \text{w/o}\, b},
\end{equation}
where $\phi_{a, \text{w/}\, b}$ represents the importance of $\delta_a$, when $\delta_b$ is always present; and $\phi_{a, \text{w/o}\, b}$ represents the importance of $\delta_a$, when $\delta_b$ is always absent.

\textbf{Form  of the Shapley interaction index.}
The Shapley interaction index~\cite{grab1999An} was proposed as an extension  of the Shapley value~\cite{shapley1953value}, which measures the pairwise interaction as follows.
\begin{equation} \label{eq:sii}
    I^{\text{sii}}_{a b}(\delta) = \sum_{S \subseteq \Omega \backslash\{a, b\}} \frac{|S| !(n-|S|-2) !}{(n-1) !}\big[v(S \cup\{a, b\})- v(S \cup\{a\})-v(S \cup\{b\})+v(S)\big].
\end{equation}

We then prove that the first and second forms of pairwise interactions are equivalent.
Moreover, both forms of the pairwise interaction are consistent  with the Shapley interaction index $I^{\text{sii}}_{a b}(\delta)$.
\newline
\begin{proof}
    Let us analyze the first form of the pairwise interaction.
    When perturbation units $a$ and $b$ are regarded as a singleton unit $S_{a b}=\{a, b\}$, the importance  of $S_{a b}$ based on the Shapley value is given as follows.
    \begin{equation}
        \phi(S_{a b}|\Omega') = \sum_{S \subseteq \Omega \backslash\{a, b\}} \frac{|S| !(n-|S|-2) !}{(n-1) !}(v(S \cup\{a,b\})-v(S)),
    \end{equation}
    where $S_{a b} = \{a, b\}$ represents the coalition of perturbation units $a, b$.
    In this game, because perturbation units $a$ and $b$ are regarded as a singleton player, we can consider that only $n-1$ players are in the game. Consequently the set of total players changes to $\Omega' = \Omega\setminus\{a, b\}\cup S_{a b}$.

    In Eq.~(\ref{eq:first-form-interaction}), $\phi(a|{\Omega\setminus\{b\}})$ and $\phi(b|{\Omega\setminus\{a\}})$ represent the individual contributions of units $a$ and $b$, respectively, when perturbation units $a$ and $b$ work individually.
    The individual contribution of perturbation unit $a$, when perturbation unit $b$ is absent, is expressed as follows.
    \begin{equation}
        \phi(a|\Omega\setminus\{b\}) = \sum_{S \subseteq \Omega \backslash\{a, b\}} \frac{|S| !(n-|S|-2) !}{(n-1) !}(v(S \cup\{a\})-v(S)).
    \end{equation}

    In this game, because the perturbation unit $b$ is always absent, we can consider there are only $n-1$ players in the game.
    Consequently the set of players changes to $\Omega\setminus\{b\}$.

    Similarly, when perturbation unit $a$ is absent, the individual contribution of perturbation unit $b$ is given as follows.
    \begin{equation}
        \phi(b|{\Omega\setminus\{a\}}) = \sum_{S \subseteq \Omega \backslash\{a, b\}} \frac{|S| !(n-|S|-2) !}{(n-1) !}(v(S \cup\{b\})-v(S)).
    \end{equation}

    In this way, the first form of the pairwise interaction can be written as follows.
    \begin{equation} \label{eq:first-form-eqs-shapley-interaction}
        \begin{split}
            I_{a b}(\delta) &= \phi(S_{a b}|\Omega') - \left[\phi(a|\Omega\setminus\{b\}) + \phi(b|{\Omega\setminus\{a\}})\right] \\
            &=\sum_{S \subseteq \Omega \backslash\{a, b\}} \frac{|S| !(n-|S|-2) !}{(n-1) !}(v(S \cup\{a,b\})-v(S)) \\
            &-\left[\sum_{S \subseteq \Omega \backslash\{a, b\}} \frac{|S| !(n-|S|-2) !}{(n-1) !}(v(S \cup\{a\})-v(S)) + \sum_{S \subseteq \Omega \backslash\{a, b\}} \frac{|S| !(n-|S|-2) !}{(n-1) !}(v(S \cup\{b\})-v(S)) \right] \\
            &= \sum_{S \subseteq \Omega \backslash\{a, b\}} \frac{|S| !(n-|S|-2) !}{(n-1) !}\big[v(S \cup\{a, b\})- v(S \cup\{a\})-v(S \cup\{b\})+v(S)\big] \\
            &= I^{\text{sii}}_{a b}(\delta).
        \end{split}
    \end{equation}

    As shown in Eq.~(\ref{eq:first-form-eqs-shapley-interaction}), the first form of the pairwise interaction is equivalent to the Shapley interaction index.

    Then, we analyze the second form of the pairwise interaction.
    In Eq.~(\ref{eq:second-form-interaction}), when perturbation unit $b$ is always present,  the importance of perturbation unit $a$ is given as follows.
    In this game, because the perturbation unit $b$ is always present, we can consider that there are only $n-1$ players.
    \begin{equation}
        \phi_{a, \text{w/}\, b} = \sum_{S \subseteq \Omega \backslash\{a, b\}} \frac{|S| !(n-|S|-2) !}{(n-1) !}(v(S \cup\{a, b\})-v(S\cup \{b\})).
    \end{equation}

    When perturbation unit $b$ is always absent, the importance of perturbation unit $a$ is given as follows.
    In this game, because the perturbation unit $b$ is always absent, we can consider that there are only $n-1$ players.
    \begin{equation}
        \phi_{a, \text{w/o}\, b} = \sum_{S \subseteq \Omega \backslash\{a, b\}} \frac{|S| !(n-|S|-2) !}{(n-1) !}(v(S \cup\{a\})-v(S)).
    \end{equation}

    Thus, the second form of the pairwise interaction can be written as follows.
    \begin{equation} \label{eq:second-form-eqs-shapley-interaction}
        \begin{split}
            I'_{a b}(\delta) &= \phi_{a, \text{w/}\, b} - \phi_{a, \text{w/o}\, b} \\
            &=\sum_{S \subseteq \Omega \backslash\{a, b\}} \frac{|S| !(n-|S|-2) !}{(n-1) !}(v(S \cup\{a, b\})-v(S\cup \{b\})) \\
            &\qquad- \sum_{S \subseteq \Omega \backslash\{a, b\}} \frac{|S| !(n-|S|-2) !}{(n-1) !}(v(S \cup\{a\})-v(S)) \\
            &=\sum_{S \subseteq \Omega \backslash\{a, b\}} \frac{|S| !(n-|S|-2) !}{(n-1) !}\big[v(S \cup\{a, b\})- v(S \cup\{a\})-v(S \cup\{b\})+v(S)\big] \\
            &=I^{\text{sii}}_{a b}(\delta).
        \end{split}
    \end{equation}

    As shown in Eq.~(\ref{eq:second-form-eqs-shapley-interaction}), the first form of the pairwise interaction is equivalent to the Shapley interaction index.

    In this manner, we proved that the first and second forms of the pairwise interactions are equivalent, and that they are consistent  with the Shapley interaction index.
\end{proof}

\subsection{Multi-order interaction}
\label{append:multi-order interaction}
The pairwise interaction $I_{ab}$ can be decomposed into multi-order interactions as follows~\cite{zhang2020game}.
\begin{equation}
    \begin{split}
        &I_{ab}\!=\!\frac{1}{n-1}\sum_{s=0}^{n-2}I_{ab}^{(s)}\\
        &I^{(s)}_{a b}\triangleq \mathbb{E}_{\substack{S\subseteq \Omega\setminus\{a,b\},|S|=s}}\left[  v(S \cup\{a, b\})- v(S \cup\{a\})-v(S \cup\{b\})+v(S) \right].
    \end{split}
\end{equation}

The order $s$ reflects the contextual complexity of the interaction, which means that $(i,j)$ collaborates with $s$ other  contextual variables.

In game theory, we consider the adversarial attack as a game.
We consider each perturbation  unit in the adversarial perturbation $\delta\in\mathbb{R}^n$ as a player.
Without loss of generality, let $\Omega=\{1, 2, \dots, n\}$ denote the set of all players (perturbation units).
Each player (perturbation unit) aims to increase the classification loss, in order to fool the DNN.
Thus, the attacking utility can be defined as $v(S)\triangleq \textit{Loss}(x + \delta^{(S)}) - \textit{Loss}(x)$, where $\forall a\in S,~\delta^{(S)}_a=\delta_a;\;  \forall a\notin S,~\delta^{(S)}_a=0$.
The total attacking utility is given as $v(\Omega)-v(\emptyset) = \textit{Loss}(x'=x + \delta) - \textit{Loss}(x)$, which indicates the increase of the loss caused by adversarial perturbation $\delta$.

The following theorem demonstrates that the total attacking utility can be decomposed into compositional utilities of multi-order interactions.
\begin{theoremappendix} \label{theoremappendix:decomposition}
    \textbf{Decomposing the attacking utility into multi-order interactions.}
    The overall attacking utility can be decomposed into compositional utilities of multi-order interactions, \emph{i.e.}, $v(\Omega)-v(\emptyset)=\sum_{a\in\Omega}[v(\{a\})-v(\emptyset)] + \sum_{a\ne b\in \Omega}\sum_{s=0}^{n-1} \frac{n-1-s}{n(n-1)}I^{(s)}_{ab}$.
\end{theoremappendix}
\begin{proof}
    According to the efficiency property of the Shapley value~\cite{weber1988probabilistic}, we have
    \begin{equation} \label{eq:shapley-efficiency}
        v(\Omega)-v(\emptyset) = \sum_{a\in \Omega} \phi(a\mid \Omega).
    \end{equation}

    Similar to the multi-order interaction, we can decompose the Shapley value as follows.
    \begin{equation} \label{eq:multi-order-shapley-value}
        \begin{split}
            \phi(a\mid \Omega) &= \frac{1}{n} \sum_{s=0}^{n-1} \phi^{(s)}(a\mid \Omega) \\
            \phi^{(s)}(a\mid \Omega) &\triangleq \mathbb{E}_{\substack{S\subseteq \Omega\setminus\{a\},|S|=s}}\left[ v(S \cup\{a\})-v(S) \right].
        \end{split}
    \end{equation}

    Thus, according to Eq.~(\ref{eq:shapley-efficiency}) and Eq.~(\ref{eq:multi-order-shapley-value}),  we have
    \begin{equation} \label{eq:decompose-multi-order-shapley-value}
        \begin{split}
            v(\Omega)-v(\emptyset) &= \frac{1}{n} \sum_{a\in \Omega} \sum_{s=0}^{n-1} \phi^{(s)}(a\mid \Omega) \\
            &= \frac{1}{n} \sum_{a\in \Omega} \sum_{s=0}^{n-1} \phi^{(s)}(a\mid \Omega) - \phi^{(s-1)}(a\mid \Omega) + \phi^{(s-1)}(a\mid \Omega) - \phi^{(s-2)}(a\mid \Omega) + \dots - \phi^{(0)}(a\mid \Omega) + \phi^{(0)}(a\mid \Omega).
        \end{split}
    \end{equation}

    Note that according to Eq.~(\ref{eq:multi-order-shapley-value}), we have
    \begin{equation} \label{eq:diff-multi-order-shapley-value}
        \begin{split}
            &\phi^{(s)}(a\mid \Omega) - \phi^{(s-1)}(a\mid \Omega) \\
            &= \mathbb{E}_{\substack{S'\subseteq \Omega\setminus\{a\}, \atop |S'|=s}}\left[ v(S' \cup\{a\})-v(S') \right] - \mathbb{E}_{\substack{S\subseteq \Omega\setminus\{a\} \atop|S|=s-1}}\left[ v(S \cup\{a\})-v(S) \right] \\
            &= \mathbb{E}_{\substack{S\subseteq \Omega\setminus\{a\} \atop|S|=s-1}}\left[\mathbb{E}_{b\in \Omega\setminus \{S\cup \{a\}\}} \left[  v(S\cup\{b\} \cup\{a\})-v(S\cup\{b\})\right] \right] - \mathbb{E}_{\substack{S\subseteq \Omega\setminus\{a\} \atop|S|=s-1}}\left[ v(S \cup\{a\})-v(S) \right] \\
            &=\mathbb{E}_{\substack{S\subseteq \Omega\setminus\{a\} \atop|S|=s-1}}\left[\mathbb{E}_{b\in \Omega\setminus \{S\cup \{a\}\}} \left[ v(S\cup\{b\} \cup\{a\})-v(S\cup\{b\}) - v(S \cup\{a\})+v(S)  \right] \right] \\
            &=\mathbb{E}_{b\in\Omega\setminus\{a\}} \mathbb{E}_{S\subseteq \Omega\setminus\{a, b\} \atop |S|=s-1}\left[v(S\cup\{b\} \cup\{a\})-v(S\cup\{b\}) - v(S \cup\{a\})+v(S)\right] \\
            &=\mathbb{E}_{b\in\Omega\setminus\{a\}} \left[ I^{(s-1)}_{ab}\right].
        \end{split}
    \end{equation}

    In this way, according to Eq.~(\ref{eq:decompose-multi-order-shapley-value}) and Eq.~(\ref{eq:diff-multi-order-shapley-value}), we have
    \begin{equation}
        \begin{split}
            v(\Omega)-v(\emptyset) &=  \frac{1}{n} \sum_{a\in \Omega} \sum_{s=0}^{n-1} \phi^{(s)}(a\mid \Omega) - \phi^{(s-1)}(a\mid \Omega) + \phi^{(s-1)}(a\mid \Omega) - \phi^{(s-2)}(a\mid \Omega) + \dots - \phi^{(0)}(a\mid \Omega) + \phi^{(0)}(a\mid \Omega) \\
            &= \frac{1}{n} \sum_{a\in \Omega} \sum_{s=0}^{n-1} \left[\mathbb{E}_{b\in\Omega\setminus\{a\}} \left[ I^{(s-1)}_{ab}\right] + \mathbb{E}_{b\in\Omega\setminus\{a\}} \left[ I^{(s-2)}_{ab}\right] + \dots + \mathbb{E}_{b\in\Omega\setminus\{a\}} \left[ I^{(0)}_{ab}\right] + \phi^{(0)}(a\mid \Omega)\right] \\
            &= \frac{1}{n} \sum_{a\in \Omega} \sum_{s=0}^{n-1} \left[\sum_{k=0}^{s-1}\left[\mathbb{E}_{b\in\Omega\setminus\{a\}} \left[ I^{(k)}_{ab}\right] \right] + \phi^{(0)}(a\mid \Omega)\right] \\
            &= \sum_{a\in \Omega} \phi^{(0)}(a\mid \Omega) + \frac{1}{n (n-1)} \sum_{a\in \Omega} \sum_{b\in\Omega\setminus\{a\}}\sum_{s=0}^{n-1} \sum_{k=0}^{s-1} \left[ I^{(k)}_{ab}\right] \\
            &= \sum_{a\in \Omega} \phi^{(0)}(a\mid \Omega) + \frac{1}{n (n-1)} \sum_{a\in \Omega} \sum_{b\in\Omega\setminus\{a\}}\sum_{s=0}^{n-2} (n-1-s) \left[ I^{(s)}_{ab}\right] \\
            &= \sum_{a\in\Omega}[v(\{a\})-v(\emptyset)] + \sum_{a\ne b\in \Omega}\sum_{s=0}^{n-1} \frac{n-1-s}{n(n-1)}I^{(s)}_{ab}.
        \end{split}
    \end{equation}
\end{proof}

\subsection{Sum of the pairwise interactions}
\label{append:sum of interaction}
According to Appendix~\ref{append:consistent-interaction}, the pairwise interaction between the two perturbation units $a$ and $b$ is given as follows.
\begin{equation}
    I_{a b}(\delta) = \phi_{a, \text{w/}\, b} - \phi_{a, \text{w/o}\, b},
\end{equation}
where $\phi_{a, \text{w/}\, b}$ represents the importance of $\delta_a$, when $\delta_b$ is always present; and $\phi_{a, \text{w/o}\, b}$ represents the importance of $\delta_a$, when $\delta_b$ is always absent.

The sum of pairwise interactions is given as follows.
\begin{equation}
    \sum_{a, b \in \Omega}\left[ I_{a b}(\delta)\right] =  \sum_{a \in \Omega} \left[v(\Omega)-v(\Omega\setminus \{a\}) - v(\{a\}) + v(\emptyset)\right], \label{eq:expectation}
\end{equation}which is proved as follows.
$ $\newline

\begin{proof}
The sum of the p interaction can be written as follows.
\begin{align}
    \sum_{a, b \in \Omega}\left[ I_{a b}(\delta)\right]  = \sum_{a}\left\{\sum_{b \in \Omega\setminus \{a\}} \left[\phi_{b, w/\, a} - \phi_{b, w/o\, a} \right]\right\}.
\end{align}
According to the \textbf{\textit{efficiency property}} of Shapley values (see Appendix~\ref{append:shapley-axiom} for details):
\begin{align}
   &\sum_{b\in \Omega\setminus\{a\}}\phi_{b, w/\, a} = v(\Omega) - v{\{a\}}, \\
   &\sum_{b\in \Omega\setminus\{a\}}\phi_{b, w/o\, a} = v(\Omega\setminus \{a\}) - v(\emptyset).
\end{align}
In this way,
\begin{align}
    \sum_{a, b \in \Omega}\left[ I_{a b}(\delta)\right] =  \sum_{a \in \Omega} \left[v(\Omega)-v(\Omega\setminus \{a\}) - v(\{a\}) + v(\emptyset)\right].
\end{align}
\end{proof}

\section{Dynamics of adversarial perturbations.} \label{append:delta-closed-form}

In this section, we prove Lemma~\ref{lemma:delta} in Section~\ref{subsec:multi-step vs single-step}, which analyzes the dynamics of the adversarial perturbation via a closed-form solution.

    \textbf{General form of the multi-step attack.}
    Given an input sample  $x\in\mathbb{R}^n$ and a DNN $h$, we use  $\textit{Loss}(x' = x+\delta) = \textit{Loss}(h(x+\delta),y)$  to denote the classification loss~\textit{w.r.t.} the adversarial example $x' = x+\delta$.
    In general, the adversarial perturbation of the $m$-step attack can be written as follows.
     \begin{equation} \label{eq:multi-step-attack-ori}
        \delta^{(m)}_{\text{general}} \triangleq  \sum_{t=0}^{m-1} \alpha_t\nabla_{x'} \textit{Loss}(x'=x+\delta^{(t)}_{\text{general}}).
    \end{equation}
    Here, $\alpha_t$ represents the step size at step $t$, which depends on the normalization of the gradient.
    For example, in the $L_2$ attack, the step size is determined by the $L_2$ norm of the gradient,~\textit{i.e.}, $ \alpha_t=\alpha/ \|\nabla_{x'} \textit{Loss}(x'=x+\delta^{(t)})\|_2$, where $\alpha$ is a small constant.

    \textbf{Multi-step attack without the normalization effect.}
    To simplify our analysis, we ignore the normalization effects in the $L_2$ and $L_\infty$ attacks.
    In this way, the adversarial perturbation generated by the $m$-step attack can be written as follows.
    \begin{equation} 
    \begin{aligned}
     \label{eq:multi-step-attack}
        \delta^{(m)} & \triangleq \sum_{t=0}^{m-1} \alpha \nabla_{x'} \textit{Loss}(x'=x+\delta^{(t)})
     \end{aligned}
    \end{equation}
    Here, for the multi-step attack without the normalization effect, we use a small constant $\alpha$ to approximate step size $\alpha_t$ .
    This is because, when the step number $m$ is  sufficiently large, and the size of each step $\alpha_t$ is sufficiently small, Eq.~\eqref{eq:multi-step-attack} can be approximated using Eq.~\eqref{eq:multi-step-attack-ori}.

Furthermore, we define the change of the perturbation with the multi-step attack at each step $t$ as follows.
\begin{equation} \label{eq:delta_multi}
    \Delta x^{(t)} \triangleq \alpha\cdot \nabla_{x'}   \textit{Loss}(x' = x + \delta^{(t-1)}).
\end{equation}

In this way,
the perturbation $\delta^{(m)}$ of the $m$-step attack in Eq.~\eqref{eq:multi-step-attack}
can be written as follows.
\begin{equation} \label{eq:multi_update}
    \delta^{(m)} = \sum_{t=1}^{m} \Delta x^{(t)}.
\end{equation}

Then, in order to derive the closed-form solution to the adversarial perturbation $\delta^{(m)}$ in Eq.~\eqref{eq:multi-step-attack}, we use the second-order Taylor expansion to decompose the loss function as follows.
Here, $g\triangleq \nabla_x \textit{Loss}(x)$ represents the gradient of the loss function~\textit{w.r.t.} the input sample $x$, and $H \triangleq \nabla_x^2 \textit{Loss}(x)$ denotes the Hessian matrix of the loss function~\textit{w.r.t.} the input sample $x$.
$R_2(\delta^{(m)})$ is referred to as terms of higher than the second order in the Taylor series~\textit{w.r.t} the adversarial perturbation $\delta^{(m)}$.
\begin{equation}
 \label{eq:loss_taylor}
 \begin{split}
     \textit{Loss}(x'=x+\delta^{(m)}) &= \textit{Loss}(x) + (x'-x)^T g + \frac{1}{2} (x'-x) ^T H (x'-x) + R_2(x'-x).\\
     &=\textit{Loss}(x) + (\delta^{(m)})^T g + \frac{1}{2} (\delta^{(m)}) ^T H \delta^{(m)} + R_2(\delta^{(m)}).
 \end{split}
\end{equation}

\begin{assumption}
\label{assumption:taylor}
 Because the adversarial perturbation $\delta^{(m)}$ is typically imperceptibly small, we assume that $R_2(\delta^{(m)})$ in Eq.~\eqref{eq:loss_taylor} can be neglected,~\textit{i.e.} $R_2(\delta^{(m)})\approx 0$.
\end{assumption}

Based on Assumption~\ref{assumption:taylor}, the Taylor expansion of the loss function in Eq.~\eqref{eq:loss_taylor} can be approximated as
\begin{equation}
 \label{eq:loss_taylor_1}
\begin{aligned}
\textit{Loss}(x'=x+\delta^{(m)}) & = \textit{Loss}(x) + (\delta^{(m)})^T g + \frac{1}{2} (\delta^{(m)}) ^T H \delta^{(m)} + R_2(\delta^{(m)})
\\
&\approx \textit{Loss}(x) + (\delta^{(m)})^T g + \frac{1}{2} (\delta^{(m)}) ^T H \delta^{(m)}.
\end{aligned}
\end{equation}

In this way, based on Eq.~\eqref{eq:loss_taylor_1}, the gradient of the loss function~\textit{w.r.t} the adversarial example $x' = x+\delta^{(m)}$  can be represented as
\begin{equation}
 \label{eq:grad_taylor}
\begin{aligned}
&\nabla_{x'} \textit{Loss}(x'=x+\delta^{(m)}) \\
& =\nabla_{x'}\left[ \textit{Loss}(x) + (x'-x)^T g + \frac{1}{2} (x'-x) ^T H (x'-x) + R_2(x'-x) \right]
\quad // \quad \text{According to Eq.~\eqref{eq:loss_taylor}}
\\
& \approx \nabla_{x'}\left[ \textit{Loss}(x) + (x'-x)^T g + \frac{1}{2} (x'-x) ^T H (x'-x)  \right]
\quad // \quad \text{According to Eq.~\eqref{eq:loss_taylor_1}}
\\
&= g + H (x'-x)
\\
&= g + H \delta^{(m)}
\end{aligned}
\end{equation}

\begin{lemma}
\label{lemma:induction}
    Based on Assumption~\ref{assumption:taylor}, the adversarial perturbation $\Delta x^{(t)}$ generated at the $t$-th step attack can be written as $\Delta x^{(t)} =\alpha \left[I + \alpha H \right]^{t-1} g$.
    \end{lemma}
 \begin{proof}
 We prove Lemma~\ref{lemma:induction} by induction.

\textit{Base case}:
When  $t=1$, we have   $\Delta x^{(1)} = \alpha\cdot g = \alpha \left[I + \alpha H \right]^{0} g$.

\textit{Inductive step}:
For $1<t'<t$, assuming $\Delta x^{(t')}= \alpha \left[I + \alpha H \right]^{t'-1} g$, then at step $t$, we have
  \begin{align}
  \label{eq:lemma_0}
                \Delta x^{(t)} &= \alpha \cdot \nabla_{x'} \textit{Loss}(x'=x+\delta^{(t-1)})
                 \quad // \quad \text{According to Eq.~(\ref{eq:delta_multi})} \nonumber
                 \\
                &= \alpha \cdot \left[g + H \delta^{(t-1)} \right]
                 \quad // \quad\text{According to Eq.~(\ref{eq:grad_taylor})} \nonumber
                \\
                 &= \alpha \cdot \left[g + H \sum_{t'=1}^{t-1}\Delta x^{(t')} \right]
                  \quad // \quad\text{According to Eq.~(\ref{eq:multi_update})} \nonumber
                 \\
                & =  \alpha\cdot \Bigg[g + \alpha \cdot H \bigg[ \sum\nolimits_{t'=1}^{t-1 }\left[I+\alpha H \right]^{t'-1} \bigg] g \Bigg] \nonumber  \\
                & = \alpha\cdot \Bigg[g + \alpha \cdot H T(x) \Bigg],
 \end{align}
where
\begin{equation}
\label{eq:lemma_1}
T(x) = \left[ \sum_{t'=1}^{t-1 }\left[I+\alpha H \right]^{t'-1} \right] g.
\end{equation}

Multiply both sides of Eq.~(\ref{eq:lemma_1}) by $(I+\alpha H)$, and we obtain
    \begin{align} \label{eq:lemma_2}
        (I+\alpha H)T(x)
        = \left[ \sum_{t'=2}^{t}\left[I+\alpha H \right]^{t'-1}\right] g.
    \end{align}

Subsequently, the difference between Eq.~\eqref{eq:lemma_1} and Eq.~\eqref{eq:lemma_2} can be written as
 \begin{equation}
\begin{aligned}
 \label{eq:lemma_3}
   (I+\alpha H)T(x) -  T(x)  &=  \left[ \sum_{t'=2}^{t}\left[I+\alpha H \right]^{t'-1}\right] g -  \left[ \sum_{t'=1}^{t-1 }\left[I+\alpha H \right]^{t'-1} \right] g
   \\
\Rightarrow \quad \alpha H T(x)    &= \left[ \left[I+\alpha H \right]^{t-1} -I \right] g.
 \end{aligned}
\end{equation}

Substituting Eq.~(\ref{eq:lemma_3}) back into Eq.~(\ref{eq:lemma_0}), we have

 \begin{equation}
\begin{aligned}
  \Delta x^{(t)} &=   \alpha\cdot \Bigg[g + \alpha \cdot H T(x) \Bigg]
  \\
  &=   \alpha\cdot \Bigg[g +  \left[ \left[I+\alpha H \right]^{t-1} -I \right] g \Bigg]
 \quad // \quad\text{According to Eq.~(\ref{eq:lemma_3})}
  \\
&= \alpha \left[I + \alpha H \right]^{t-1} g .
\end{aligned}
\end{equation}

\textit{Conclusion}: because both the base case and inductive step are proven to be true, we have $\forall t \ge 1, \Delta x^{(t)} =\alpha \left[I + \alpha H \right]^{t-1} g $.

Thus, Lemma~\ref{lemma:induction} is proven.
\end{proof}

\subsection{Proof of Lemma~\ref{lemma:delta} in Section~\ref{subsec:multi-step vs single-step}}
\label{lemmaappendix:delta}

Based on Lemma~\ref{lemma:induction}, we prove Lemma~\ref{lemma:delta} in Section~\ref{subsec:multi-step vs single-step}, which provides a closed-form solution to the adversarial perturbation of the multi-step attack.

\begin{lemmaappendix}
\label{lemmaappend:delta-close-form}
    Let us decompose the Hessian matrix ($H=\nabla^2_x \textit{Loss}(x)\in \mathbb{R}^{n\times n}$) via eigen-decomposition $H\!=\!V \Lambda V^{\text{T}}$ subject to $V=[v_1,v_2,\ldots,v_n], \Lambda=\text{diag}(\lambda_1,\lambda_2,\ldots,\lambda_n)$, where $\lambda_i \in \mathbb{R}$ and $v_n\in \mathbb{R}^{n}$ denote the $i$-th eigenvalue and eigenvector of $H$, respectively.
    Furthermore, $g\!=\!\sum_{i=1}^n \gamma_i v_i$  decomposes the gradient $g=\nabla_x \textit{Loss}(x)$ along the eigenvectors of $H$, where $\gamma_i = g^T v_i$.
    Then, when we attack a DNN via gradient ascent for $m$ steps with the step size $\alpha$, based on Assumption~\ref{assumption:taylor}, the generated adversarial perturbation is $\delta^{(m)} =V D^{(m)} \boldsymbol{\gamma}$, where $\boldsymbol{\gamma}=[\gamma_1,\gamma_2,\ldots,\gamma_n]^T$ and  $D^{(m)}\in \mathbb{R}^{n\times n}$ is a diagonal matrix. If $\lambda_i\ne0$, $D^{(m)}_{ii}=\frac{(\alpha\lambda_i+1)^m-1}{\lambda_i}$; otherwise, $D^{(m)}_{ii}=\alpha m$.
    \end{lemmaappendix}

    \begin{proof}
        According to Eq.~{(\ref{eq:multi_update})} and  Lemma~\ref{lemma:induction},
the adversarial perturbation of the $m$-step attack can be written as
        \begin{equation} \label{eq:delta-m}
            \begin{split}
                \delta^{(m)} &= \sum_{t=1}^{m} \Delta x^{(t)} \\
                &= \alpha \left[\sum_{t=1}^m  (I+\alpha H)^{t-1} \right]g.
            \end{split}
        \end{equation}

        Because the Hessian matrix $H \triangleq \nabla_x^2 \textit{Loss}(x)$ of the loss function~\textit{w.r.t.} the input sample is a real symmetric matrix,
        we can use eigen-decomposition to decompose the Hessian matrix $H$ as
        \begin{equation} \label{eq:H-decompose}
            H=V \Lambda V^{\text{T}},
        \end{equation}
        where $\Lambda=\text{diag}(\lambda_1,\lambda_2,\ldots,\lambda_n)$ is a diagonal matrix, whose diagonal elements are the corresponding eigenvalues, $\Lambda_{ii} = \lambda_i$.
        The square matrix $V=[v_1,v_2,\ldots,v_n] \in \mathbb{R}^{n \times n}$ contains $n$ mutually orthonormal eigenvectors, \emph{i.e.} $\forall i\neq k, v_i^Tv_k=0$; and $\forall i, v_i^T v_i=1$.
        Thus, $V$ is a orthonormal matrix, \emph{i.e.} $VV^T = V^T V =I$.

        In this way, based on Eq.~\eqref{eq:delta-m} and Eq.~\eqref{eq:H-decompose}, the adversarial perturbation $\delta^{(m)}$ can be re-written as
        \begin{equation}
        \label{eq:delta-m-closed-form}
            \begin{split}
                \delta^{(m)} &= \alpha \left[\sum_{t=1}^m  (I+\alpha H)^{t-1} \right]g \\
                &= \alpha \left[\sum_{t=1}^m  (VV^T+\alpha V \Lambda V^{T})^{t-1} \right]g
                 \quad // \quad\text{According to Eq.~\eqref{eq:H-decompose}}
                \\
                &= \alpha \left[\sum_{t=1}^m  \left[V (I + \alpha\Lambda)V^T \right]^{t-1} \right] g
                \\
                & = \alpha \left[ I +  V (I + \alpha\Lambda)V^T+  V \underbrace{(I + \alpha\Lambda) \overbrace{V^T V}^{=I} (I + \alpha\Lambda)}_{=(I+\alpha\Lambda)^2}V^T + \dots + V \underbrace{ (I + \alpha\Lambda)V^T\dots V (I + \alpha\Lambda)}_{=(I + \alpha\Lambda)^{m-1}}V^T  \right] g \\
                & = \alpha V\left[\sum_{t=1}^m  (I + \alpha\Lambda)^{t-1}  \right]V^T g.
            \end{split}
        \end{equation}

For simplicity, let {$D^{(m)} \in \mathbb{R}^{n\times n}$} denote the term {$\alpha \sum_{t=1}^m  (I + \alpha\Lambda)^{t-1}$} in Eq.~\eqref{eq:delta-m-closed-form},~\textit{i.e.,}
{$D^{(m)}=\alpha \sum_{t=1}^m  (I + \alpha\Lambda)^{t-1}$}.
Because both {$I$} and $\Lambda$ are diagonal matrices, matrix {$D^{(m)}$} is also a diagonal matrix, \emph{i.e.} $\forall i\ne j, D^{(m)}_{ij}=0.$
In this way, let us focus on the $i$-th diagonal element {$D^{(m)}_{ii} \in \mathbb{R}$}.
\begin{equation}
\begin{aligned}
\label{supp_eqn:diagonal_m_1}
D^{(m)}_{ii}&=\alpha(1+(1+\alpha \lambda_i)+\cdots+(1+\alpha \lambda_i)^{m-1})\\
&=\begin{cases}
        \alpha(1\times\frac{1-(1+\alpha\lambda_i)^m}{1-(1+\alpha\lambda_i)}) ,  & \text{if $\lambda_i\ne 0$} \\
        \alpha (1 + 1 + \dots + 1), & \text{if $\lambda_i=0$}
    \end{cases} \\
&=\begin{cases}
        \frac{(1+\alpha\lambda_i)^m-1}{\lambda_i} ,  & \text{if $\lambda_i\ne 0$} \\
        \lim_{\lambda_i\rightarrow 0} \frac{(1+\alpha\lambda_i)^m-1}{\lambda_i} =  \alpha m, & \text{ if $\lambda_i=0$}
    \end{cases} \\
\end{aligned}
\end{equation}

Note that if $\lim {\lambda_i \rightarrow 0}$, according to L'Hospital's Rule, $\lim_{\lambda_i \rightarrow 0} \frac{(1+\alpha\lambda_i)^m-1}{\lambda_i} =\lim_{\lambda_i \rightarrow 0} \frac{ d(1+\alpha\lambda_i)^m-1 ) / d \lambda_i}{ d\lambda_i /d\lambda_i }= \lim_{\lambda_i \rightarrow 0} \frac{\alpha m(1+\alpha\lambda_i)^{m-1}}{1} = \alpha m$.

Moreover, as $n$ eigenvectors $v_i$ of the Hessian matrix form a set of unit orthogonal basis, the gradient {$g= \nabla_x \textit{Loss}(x)\in\mathbb{R}^n$} can be represented as follows.
{$\gamma_i$} is the projection length of {$g$} on {$v_i$}.
\begin{equation} \label{eq:g-decompose}
                g = \nabla_x \textit{Loss}(x)
                = \sum_{i=1}^n \gamma_i v_i,
        \end{equation}
where $\gamma_i = g^T v_i$.
In this way, based on Eq.~(\ref{eq:g-decompose}), the adversarial perturbation $\delta^{(m)}$ in Eq.~\eqref{eq:delta-m-closed-form} can be written as follows, where
{$\boldsymbol{\gamma}=[\gamma_1,\gamma_2,\ldots,\gamma_n]^T$}.

        \begin{equation}
            \begin{split}
                \delta^{(m)}
                & =  \alpha V\left[\sum_{t=1}^m  (I + \alpha\Lambda)^{t-1}  \right]V^T g \\
                &=  V D^{(m)} V^T g
                \\
                &= V D^{(m)} [v_1,v_2,\ldots,v_n]^T \sum_{i=1}^n \gamma_i v_i
                \quad // \quad\text{According to Eq.~\eqref{eq:g-decompose}}
                \\
                &= V D^{(m)} \begin{bmatrix}
                    v_1^T \sum_{i=1}^n \gamma_i v_i \\
                    v_2^T \sum_{i=1}^n \gamma_i v_i \\
                    \vdots \\
                    v_n^T \sum_{i=1}^n \gamma_i v_i \\
                    \end{bmatrix} \\
                & = V D^{(m)} \boldsymbol{\gamma}.
            \end{split}
        \end{equation}

Thus, Lemma~\ref{lemma:delta} in Section~\ref{subsec:multi-step vs single-step} of the main paper is proven.
 \end{proof}

\subsection{Proof of Corollary~\ref{corollary:delta} in Section~\ref{subsec:multi-step vs single-step}}
\label{corollaryappendix:delta}

In this subsection, we prove Corollary~\ref{corollary:delta} in Section~\ref{subsec:multi-step vs single-step} of the main paper, which analyzes the adversarial perturbation $\delta^{(\infty)}$ generated by the infinite-step attack with an infinitesimal step size.

\begin{corollaryappendix}
 \label{corollaryappendix:delta-inf}
Let $\beta=\alpha m$ denote the total adversarial strength.
If we use the infinite-step adversarial attack with an infinitesimal step size to simplify the multi-step adversarial attack, \textit{i.e.}, $m\!\to\! \infty$ and $\alpha = \beta/m$ is infinitesimal,
then the adversarial perturbation can be represented as $\delta^{(\infty)}=V  D  \boldsymbol{\gamma}=\sum_{i=1}^n D_{ii} \gamma_i v_i$, where if $\lambda_i\ne0$, $D_{ii}=\frac{\exp(\beta\lambda_i)-1}{\lambda_i}$; otherwise, $D_{ii}=\beta$.
\end{corollaryappendix}

\begin{proof}
The adversarial perturbation $\delta^{(\infty)}$ generated by the infinite-step attack with an infinitesimal step size is defined as
\begin{equation}
\begin{aligned}
 \label{eq:delta-inf-def}
\delta^{(\infty)} &  \triangleq  \lim_{m\rightarrow\infty} \alpha\sum_{t=0}^{m-1}\nabla_{x'}\textit{Loss}(x' = x+\delta^{(t)})
\\
& =  \lim_{m\rightarrow\infty} \delta^{(m)}.
\end{aligned}
\end{equation}

Subsequently, based on Lemma~\ref{lemmaappend:delta-close-form}, the adversarial perturbation $\delta^{(\infty)}$ in Eq.~\eqref{eq:delta-inf-def} can be re-written as
\begin{equation}
\label{eq:delta-inf}
    \begin{split}
        \delta^{(\infty)} &= \lim_{m\rightarrow\infty}\delta^{(m)} \\
        &=\lim_{m\rightarrow\infty}V D^{(m)} \boldsymbol{\gamma}
        \quad // \quad\text{According to Lemma~\ref{lemmaappend:delta-close-form}}
        \\
        &=V \left[\lim_{m\rightarrow\infty} D^{(m)} \right] \boldsymbol{\gamma}.
    \end{split}
\end{equation}

 For simplicity, we use $D\in\mathbb{R}^{n\times n}$ to denote $\lim_{m\rightarrow\infty} D^{(m)}$,~\textit{i.e.,}
 {$D = \lim_{m\rightarrow\infty} D^{(m)}= \lim_{m\rightarrow\infty} \alpha \sum_{t=1}^m  (I + \alpha\Lambda)^{t-1}$}.
 Because matrix $D^{(m)}$ is a diagonal matrix, matrix $D$ is also a diagonal matrix, \emph{i.e.} $\forall i\ne j, D_{ij}=0.$
 Then, to simplify Eq.~\eqref{eq:delta-inf}, let us focus on the $i$-th diagonal element $D_{ii} \in \mathbb{R}$.
\begin{equation}
\label{eq:Dii-inf}
\begin{aligned}
D_{ii} &= \lim_{m\rightarrow +\infty} D^{(m)}_{ii} \\
&=\lim_{m\rightarrow +\infty} \big[\alpha(1+(1+\alpha \lambda_i)+\cdots+(1+\alpha \lambda_i)^{m-1})\big]\\
&=\begin{cases}
        \lim_{m\rightarrow +\infty} \frac{(1+\alpha\lambda_i)^m-1}{\lambda_i}  ,  & \text{if $\lambda_i\ne 0$} \\
        \lim_{m\rightarrow +\infty}\alpha (1 + 1 + \dots + 1), & \text{if $\lambda_i=0$}
    \end{cases} \\
&=\begin{cases}
         \frac{\lim_{m\rightarrow +\infty} (1+\frac{\beta\lambda_i}{m})^m-1}{\lambda_i} ,  & \text{if $\lambda_i\ne 0$} \\
         \lim_{m\rightarrow +\infty} \alpha m, & \text{if $\lambda_i=0$}
    \end{cases} \\
&=\begin{cases}
        \frac{\exp(\beta\lambda_i)-1}{\lambda_i} ,  & \text{if $\lambda_i\ne 0$} \\
        \lim_{\lambda_i\rightarrow 0} \frac{\exp(\beta\lambda_i)-1}{\lambda_i}  =  \beta, & \text{ if $\lambda_i=0$}
    \end{cases}
\end{aligned}
\end{equation}

Note that if $\lim {\lambda_i \rightarrow 0}$, according to L'Hospital's Rule, $\lim_{\lambda_i \rightarrow 0} \frac{\exp(\beta\lambda_i)-1}{\lambda_i} =\lim_{\lambda_i \rightarrow 0} \frac{ d(\exp(\beta\lambda_i)-1 ) / d \lambda_i}{ d\lambda_i /d\lambda_i }= \lim_{\lambda_i \rightarrow 0} \frac{\beta \exp(\beta\lambda_i)}{1} = \beta$.

In this way, based on Eq.~\eqref{eq:Dii-inf}, the adversarial perturbation $\delta^{(\infty)}$ in Eq.~\eqref{eq:delta-inf} can be further represented as follows.
\begin{equation}
\begin{aligned}
\delta^{(\infty)}
&= V \lim_{m\rightarrow\infty} D^{(m)} \; \boldsymbol{\gamma}
\\
&= V  D \; \boldsymbol{\gamma}
\\
& = \sum_{i=1}^n D_{ii} \gamma_i v_i
\end{aligned}
\end{equation}

Thus, Corollary~\ref{corollary:delta} in Section~\ref{subsec:multi-step vs single-step} of the main paper is proven.
\end{proof}

\section{Relationship between interactions and adversarial perturbations} \label{append:interaction}

In this section, we prove Lemma~\ref{lemma:interaction} in Section~\ref{subsec:multi-step vs single-step}, which provides a closed-form solution for interactions between  perturbation units.

Before introducing the interaction between adversarial perturbations, we first focus on the attacking utility $v(\Omega)=\textit{Loss}(x'=x+\delta)-\textit{Loss}(x)$.
The interaction is defined based on the attacking utility.
Specifically, we consider the adversarial attack as a game and use the Shapley value in game theory~\cite{shapley1953value} to analyze the attacking utility of adversarial perturbations $v(\Omega)$.
We can consider the perturbation in each pixel (\emph{i.e.} a perturbation unit) as a player in the game, and the attacking utility $v(\Omega) = \textit{Loss}(x'=x+\delta)-\textit{Loss}(x)$ can be considered as the overall reward won by the collaboration of all perturbation units.

Given an adversarial perturbation $\delta\in \mathbb{R}^{n}$, let $\Omega\!=\!\{1, 2, \dots, n\}$ denote the set of all players (perturbation units).
Let $\delta^{(S)}\!\in\!\mathbb{R}^n$ denote the perturbation map that  contains only a subset of perturbation units $S\subseteq \Omega$, as follows.
\begin{equation}
\label{eq:delta_i_s}
    \delta^{(S)}_a = \begin{cases}
        \delta_a, \qquad &\text{if $a\in S$},\\
        0, \qquad &\text{if $a\notin S$}.
    \end{cases}
\end{equation}

Then, the attacking utility $v(S)$ of perturbation units in $S$ is defined as the loss increase when only perturbation units in $S$ are added to the input.
\begin{equation}
\label{eq:attack_utility}
        v(S)\triangleq \textit{Loss}(x+\delta^{(S)})-\textit{Loss}(x).
    \end{equation}

In this way, the overall attacking utility of all perturbation units corresponds to
\begin{equation}
v(\Omega)=\textit{Loss}(x'=x+\delta)-\textit{Loss}(x).
 \end{equation}

Based on the attacking utility, we focus on interactions inside the adversarial perturbation.
This is because, in the multi-step attack, perturbation units in the adversarial perturbation do not work independently to contribute to the attacking utility.
Instead, different perturbation units may interact with each other to form certain patterns.
Then, the interaction between two perturbation units $(a,b)$ measures the compositional contribution of the interaction between $(a,b)$ to the overall attacking utility.
\textit{I.e.,} the interaction $I_{ab}$ measures the difference between (1) the joint contribution of two units when they are perturbed together and (2) the sum of their individual contributions~\cite{grab1999An}, as follows.
\begin{equation}
 \label{eq:interaction_appendix}
    I_{a b}(\delta) \triangleq \phi(S_{a b}|\Omega') - \left[\phi(a|\Omega\setminus\{b\}) + \phi(b|{\Omega\setminus\{a\}})\right],
\end{equation}
where we consider the coalition of $\{a,b\}$ as a singleton player in the game, which is either perturbed together or not perturbed together. 
The singleton player is denoted as $S_{ab}=\{a,b\}$.
In this case, the set of all perturbation units $\Omega'=\Omega \cup S_{ab}\setminus \{a,b\}$, and the Shapley value (numerical contribution) of  $(a,b)$ to the attacking utility when they work together is quantified as $\phi(S_{a b}|\Omega')$.
Moreover, if we do not consider perturbation units $a$ and $b$ as a singleton player, then
$\phi(a|{\Omega\setminus\{b\}})$ denotes the Shapley value (numerical contribution) of the perturbation unit $a$, when the unit $a$ works individually and the unit $b$ is always absent.
Similarly, $\phi(b|{\Omega\setminus\{a\}})$ denotes the Shapley value (numerical contribution) of the perturbation unit $b$ when the unit $a$ is always absent.

Essentially, the above definition of the interaction reflects the relationship like ``$1+ 1 > 2$'' between two perturbation units.
Specifically, the positive interaction $I_{a b} > 0$ implies that perturbation units $a$ and $b$ cooperate to boost the attacking utility (\emph{i.e.}, ``$1+ 1 > 2$''). 
By contrast, the negative interaction $I_{a b}<0$ implies that perturbation units $a$ and $b$ conflict, thereby weakening the attacking utility (\emph{i.e.}, ``$1+ 1 < 2$'').
When $I_{ab}\approx 0$, it means no interaction exists between perturbation units $a$ and $b$ (\emph{i.e.}, ``$1+ 1\approx 2$'').

As proven in Appendix~\ref{append:consistent-interaction}, the interaction $I_{a b}$ in Eq.~\eqref{eq:interaction_appendix} can be calculated as follows.
\begin{equation}
\begin{aligned}
 \label{eq:shapley-interaction-index}
    I_{ab}(\delta) &  \triangleq \phi(S_{a b}|\Omega') - \left[\phi(a|\Omega\setminus\{b\}) + \phi(b|{\Omega\setminus\{a\}})\right]
    \\
    &= \underbrace{\mathbb{E}_{S \subseteq \Omega \backslash\{a, b\}} \left[v(S \cup\{a, b\})-v(S)\right]}_{\phi(S_{a b}|\Omega')}
    -  \underbrace{\mathbb{E}_{S \subseteq \Omega \backslash\{a, b\}}\left[ v(S \cup\{a\}) -v(S)\right]}_{\phi(a|\Omega\setminus\{b\})}
    - \underbrace{\mathbb{E}_{S \subseteq \Omega \backslash\{a, b\}}\left[  v(S \cup\{b\}) -v(S)\right]}_{\phi(b|{\Omega\setminus\{a\}})}.
\end{aligned}
\end{equation}

\begin{lemmaappendix}[the closed-form solution to interactions]
\label{lemmaappendix:interaction-closed-form}
    Based on Assumption~\ref{assumption:taylor}, for the adversarial perturbation $\delta$, the sum of pairwise interactions between perturbation units is {$\sum_{a, b\in \Omega} I_{ab}(\delta)=\delta^T H \delta$. Moreover, according to Corollary~\ref{corollaryappendix:delta-inf}, if the adversarial perturbation  is given as the result of the infinite-step adversarial
    attack with an infinitesimal step size $\delta^{(\infty)}$ in Eq.~(\ref{eq:delta-inf-def}), then the sum of pairwise interactions inside is given as $\sum_{a,b\in \Omega} I_{ab}(\delta^{(\infty)})=\sum_{i=1}^n \lambda_{i}(D_{ii}\gamma_{i})^2$.}
    \end{lemmaappendix}

 \begin{proof}
According to Eq.~\eqref{eq:shapley-interaction-index}, the pairwise interaction between perturbation units $I_{ab}(\delta)$ is calculated as
\begin{equation}
\begin{aligned}
\label{eq:shapley-interaction-index_1}
 I_{ab}(\delta)
&= \mathbb{E}_{S \subseteq \Omega \backslash\{a, b\}} \bigg[v(S \cup\{a, b\})-v(S)\bigg]
    -  \mathbb{E}_{S \subseteq \Omega \backslash\{a, b\}} \bigg[v(S \cup\{a\})-v(S) \bigg]
    - \mathbb{E}_{S \subseteq \Omega \backslash\{a, b\}}\bigg[ v(S \cup\{b\})-v(S)\bigg]
    \\
&=  \mathbb{E}_{S \subseteq \Omega \backslash\{a, b\}} \bigg[v(S \cup\{a, b\})- v(S \cup\{a\})-v(S \cup\{b\})+v(S)\bigg]
\end{aligned}
\end{equation}

In order to simplify Eq.~\eqref{eq:shapley-interaction-index_1}, let us focus on the attacking utility $v(S)$ defined in Eq.~\eqref{eq:attack_utility}.
 According to Assumption~\ref{assumption:taylor} and Eq.~\eqref{eq:loss_taylor_1}, the classification loss can be approximated via second-order Taylor expansion as
 $\textit{Loss}(x+\delta)= \textit{Loss}(x) +  \delta^T g  + \frac{1}{2}\delta^T H \delta$.
Here,  $g\triangleq \nabla_x \textit{Loss}(x)$ represents the gradient of the loss function~\textit{w.r.t.} the input sample $x$, and $H \triangleq \nabla_x^2 \textit{Loss}(x)$ denotes the Hessian matrix of the loss function~\textit{w.r.t.} the input sample $x$.
Thus, the attacking utility $v(S)$ defined in Eq.~\eqref{eq:attack_utility} can be re-written as
\begin{equation} \label{eq:v-s-taylor}
    \begin{split}
        v(S) & \triangleq \textit{Loss}(x+\delta^{(S)})-\textit{Loss}(x) \\
        &=\textit{Loss}(x) + (\delta^{(S)})^T g +\frac{1}{2}(\delta^{(S)})^T H \delta^{(S)} - \textit{Loss}(x)
        \quad // \quad \text{According to Eq.~\eqref{eq:loss_taylor_1}}
        \\
        &=\sum_{a\in \Omega} g_{a} \delta_{a}^{(S)} + \frac{1}{2} \sum_{a,b\in \Omega} \delta_{a}^{(S)}  H_{a b} \delta_{b}^{(S)}
         \\
         &=\sum_{a\in S} g_{a} \delta_{a} + \frac{1}{2} \sum_{a, b\in S} \delta_{a}  H_{a b} \delta_{b}. \quad // \quad \text{According to Eq.~\eqref{eq:delta_i_s}}
    \end{split}
\end{equation}

In this way, based on Eq.~\eqref{eq:v-s-taylor}, the pairwise interaction between perturbation units $I_{ab}(\delta)$ in Eq.~\eqref{eq:shapley-interaction-index_1} can be simplified as
\begin{align} \label{eq:shapley-interaction-index_2}
 I_{ab}(\delta) &= \mathbb{E}_{S \subseteq \Omega \backslash\{a, b\}} \bigg[v(S \cup\{a, b\})- v(S \cup\{a\})-v(S \cup\{b\})+v(S)\bigg] \nonumber
 \\
&=  \mathbb{E}_{S \subseteq \Omega \backslash\{a, b\}} \left[ \sum_{a'\in S\cup\{a, b\}} g_{a'} \delta_{a'} + \frac{1}{2} \sum_{a',b'\in S\cup\{a, b\}} \delta_{a'}  H_{a' b'} \delta_{b'}\right]
\quad // \quad \text{According to Eq.~\eqref{eq:v-s-taylor}} \nonumber
 \\
& \quad -  \mathbb{E}_{S \subseteq \Omega \backslash\{a, b\}} \left[\sum_{a'\in S\cup\{a\}} g_{a'} \delta_{a'} + \frac{1}{2} \sum_{a', b'\in S\cup\{a\}} \delta_{a'}  H_{a' b'} \delta_{b'}\right] \nonumber\\
& \quad -  \mathbb{E}_{S \subseteq \Omega \backslash\{a, b\}} \left[\sum_{a'\in S\cup\{b\}} g_{a'} \delta_{a'} + \frac{1}{2} \sum_{a', b'\in S\cup\{b\}} \delta_{a'}  H_{a' b'} \delta_{b'} \right] \nonumber\\
& \quad +  \mathbb{E}_{S \subseteq \Omega \backslash\{a, b\}} \left[\sum_{a'\in S} g_{a'}(x) \delta_{a'} + \frac{1}{2} \sum_{a', b'\in S} \delta_{a'}  H_{a' b'} \delta_{b'}\right] \nonumber\\
&= \frac{1}{2} \cdot \mathop{\mathbb{E}} \limits_{S \subseteq \Omega \backslash\{a, b\}}  \Bigg[ \sum_{a',b'\in S\cup\{a, b\}} \delta_{a'}  H_{a' b'} \delta_{b'}
         - \sum_{a', b'\in S\cup\{a\}} \delta_{a'}  H_{a' b'} \delta_{b'}
         -\sum_{a', b'\in S\cup\{b\}} \delta_{a'}  H_{a' b'} \delta_{b'}
         + \sum_{a', b'\in S} \delta_{a'}  H_{a' b'} \delta_{b'}\Bigg]
         \nonumber\\
&=\frac{1}{2} \cdot  \mathop{\mathbb{E}} \limits_{S \subseteq \Omega \backslash\{a, b\}}
 \Bigg[ \sum_{a',b'\in S\cup\{a\}} \delta_{a'}  H_{a' b'} \delta_{b'} +  \sum_{a'\in S\cup\{a\}} \delta_{a'}  H_{a' b} \delta_{b} + \sum_{b'\in S\cup\{a\}} \delta_{b}  H_{b b'} \delta_{b'} +   \delta_{b}  H_{b b} \delta_{b}\Bigg]
 \nonumber\\
& \quad -\frac{1}{2} \cdot  \mathop{\mathbb{E}} \limits_{S \subseteq \Omega \backslash\{a, b\}} \bigg[\sum_{a', b'\in S\cup\{a\}} \delta_{a'}  H_{a' b'} \delta_{b'}\bigg]
\nonumber\\
 &\quad  -\frac{1}{2} \cdot  \mathop{\mathbb{E}} \limits_{S \subseteq \Omega \backslash\{a, b\}} \Bigg [\sum_{a',b'\in S} \delta_{a'}  H_{a' b'} \delta_{b'} +  \sum_{a'\in S } \delta_{a'}  H_{a' b} \delta_{b} + \sum_{ b'\in S} \delta_{b}  H_{b b'} \delta_{b'} +   \delta_{b}  H_{b b} \delta_{b} \Bigg]
 \nonumber\\
&\quad  +\frac{1}{2} \cdot  \mathop{\mathbb{E}} \limits_{S \subseteq \Omega \backslash\{a, b\}}  \bigg[\sum_{a', b'\in S} \delta_{a'}  H_{a' b'} \delta_{b'} \bigg]
\nonumber\\
&= \frac{1}{2} \cdot  \mathop{\mathbb{E}} \limits_{S \subseteq \Omega \backslash\{a, b\}}
\Bigg[  \sum_{a\in S\cup\{a\} } \delta_{a'}  H_{a' b} \delta_{b} + \sum_{ b'\in S\cup\{a\}} \delta_{b}  H_{b b'} \delta_{b'}
 - \sum_{a\in S } \delta_{a'}  H_{a' b} \delta_{b}- \sum_{ b'\in S} \delta_{b}  H_{b b'} \delta_{b'} \Bigg]
 \nonumber\\
&=  \mathbb{E}_{S \subseteq \Omega \backslash\{a, b\}} \bigg[\delta_a H_{ab} \delta_b \bigg]
\nonumber\\
&= \sum\nolimits_{S \subseteq \Omega \backslash\{a, b\}} \frac{|S| !(n-|S|-2) !}{(n-1) !}  \bigg[\delta_a H_{ab} \delta_b \bigg]
\nonumber\\
&= \sum\nolimits_{s=0}^{n-2} \sum\nolimits_{\overset{S\subseteq\Omega\setminus\{a, b\},} {|S|=s}}\frac{s !(n-s-2) !}{(n-1) !} \delta_a H_{ab} \delta_b
 \nonumber\\
&=  \sum\nolimits_{s=0}^{n-2} \frac{(n-2) !}{ s ! (n-s-2) !}  \frac{s !(n-s-2) !}{(n-1) !} \delta_a H_{ab} \delta_b
 \nonumber\\
&= \delta_a H_{ab} \delta_b,
\end{align}

Note that  we can extend the definition of the interaction in Eq.~\eqref{eq:shapley-interaction-index_2} to measure the interactive influence of a unit $a$ on itself as $I_{aa}= \delta_a H_{aa} \delta_a$.
Specifically, if we add an additional perturbation $\delta_a$ on the adversarial image $x'=x+\delta$, then the interactive influence $I_{aa}$ measures the additional increase of the importance of the perturbation unit $a$, caused by the additional perturbation $\delta_a$.

Based on Eq.~\eqref{eq:shapley-interaction-index_2}, the sum of pairwise interactions between perturbation units in Lemma~\ref{lemmaappendix:interaction-closed-form} can be represented as
\begin{align}
\label{eq:sum-of-pairwise-interaction}
    \sum_{a,b\in \Omega} I_{ab}(\delta)&= \sum_{a,b\in \Omega} \delta_{a} H_{ab} \delta_b = \delta^T H \delta.
\end{align}

Furthermore, let us focus on the interactions of the adversarial perturbation $\delta^{(\infty)}$, which is generated by the infinite-step attack with an infinitesimal step size.
Based on Eq.~\eqref{eq:sum-of-pairwise-interaction}, the sum of pairwise interactions inside $\delta^{(\infty)}$ can be written as
 \begin{equation}
 \label{eq:infi_inter}
 \sum_{a,b\in \Omega} I_{ab}(\delta^{(\infty)})=(\delta^{(\infty)})^T H \delta^{(\infty)}.
 \end{equation}

Here, $H \triangleq \nabla_x^2 \textit{Loss}(x)$ denotes the Hessian matrix of the loss function~\textit{w.r.t.} the input sample $x$.
According to Corollary~\ref{corollaryappendix:delta-inf}, the adversarial perturbation $\delta^{(\infty)}$ can be represented as follows.
\begin{equation}
 \label{eq:infi_inter_1}
        \delta^{(\infty)} = \sum_{i=1}^n D_{ii} \gamma_{i} v_{i},
    \end{equation}
    where if $\lambda_i\ne0$, $D_{ii}=\frac{\exp(\beta\lambda_i)-1}{\lambda_i}$; otherwise, $D_{ii}=\beta$.
$v_i$ denotes the eigenvector corresponding to the $i$-th eigenvalue of the Hessian matrix $H$, and $\gamma_i=g^T v_i$ is referred to as the projection of the gradient $g$ on the eigenvector $v_i$.

Then, based on Eq.~\eqref{eq:infi_inter_1}, the sum of pairwise interactions of the adversarial perturbation $\delta^{(\infty)}$ in  Eq.~\eqref{eq:infi_inter} can be further simplified as
    \begin{align} \label{eq:delta-inf-interaction}
            \sum_{a,b\in \Omega} I_{ab}(\delta^{(\infty)})&=(\delta^{(\infty)})^T H \delta^{(\infty)} \nonumber\\
            &= (\sum_{i=1}^n D_{ii} \gamma_{i} v_{i}^T)(\sum_{j=1}^n \lambda_{j} v_{j} v_{j}^T) (\sum_{k=1}^n D_{kk} \gamma_{k} v_{k}) \nonumber
             \quad // \quad \text{According to Eq.~\eqref{eq:H-decompose} and Eq.~\eqref{eq:infi_inter_1}}
            \\
            &=\left[\sum_{i=1}^n \underbrace{ (D_{ii} \gamma_{i} v_{i}^T)( \lambda_{i} v_{i} v_{i}^T) }_{\text{$ v_{i}^T v_{i}=1$}}+ \sum_{1\le i,j\le n \atop i\ne j} \underbrace{(D_{ii} \gamma_{i} v_{i}^T)( \lambda_{j} v_{j} v_{j}^T)}_{\text{if $i\ne j, v_{i}^T v_{j}=0$}} \right] (\sum_{k=1}^n D_{kk} \gamma_{k} v_{k}) \nonumber \\
            &=(\sum_{i=1}^n \lambda_{i} D_{ii} \gamma_{i} v_{i}^T ) (\sum_{k=1}^n D_{kk} \gamma_{k} v_{k}) \nonumber\\
            &=\sum_{i=1}^n \underbrace{ (\lambda_{i} D_{ii} \gamma_{i} v_{i}^T)( D_{ii} \gamma_{i} v_{i})}_{\text{$ v_{i}^T v_{i}=1$}}+ \sum_{1\le i, k \le n \atop i\ne k} \underbrace{ (\lambda_{i} D_{ii} \gamma_{i} v_{i}^T)( D_{kk} \gamma_{k} v_{k})}_{\text{if $i\ne k, v_{i}^T v_{k}=0$}} \nonumber \\
            &= \sum_{i=1}^n \lambda_{i}(D_{ii}\gamma_{i})^2.
        \end{align}
\end{proof}

Thus, Lemma~\ref{lemma:interaction} in Section~\ref{subsec:multi-step vs single-step} is proven.

\section{Proof of the positive semi-definiteness of the Hessian matrix}
\label{appendix:positive-semi-definite}

In this section, we prove Corollary~\ref{corollary:semi-definite} in Section~\ref{subsec:multi-step vs single-step}, which shows that if the DNN is a ReLU network, the Hessian matrix of the loss function~\textit{w.r.t.} the input sample is positive semi-definite.
In fact, the positive semi-definiteness of the Hessian matrix $H=\nabla_x^2 \textit{Loss}(x)$ has been proven in~\cite{yao2018hessian}.

Specifically, we consider an input sample $x\in\mathbb{R}^n$, and a ReLU network with a total of $d$ linear layers.
We use $z$ to denote the output of this ReLU network, which is formulated as follows.
 \begin{equation}
 \label{eq:network_output}
        z = W_d^T \Sigma_{d-1}( W_{d-1}^T \Sigma_{d-2}( \cdots (W_1^{T} x + b_1) \cdots) + b_{d-1}) +  b_{d},
\end{equation}
where $W_l \in \mathbb{R}^{n_{l-1} \times n_{l}}$ and $b_l\in \mathbb{R}^{n_l}$ represent the weight and bias of the $l$-th linear layer, respectively.
The diagonal matrix $\Sigma_l = [\sigma_{11},\sigma_{22},\cdots, \sigma_{n_{l} n_{l}}] \in \mathbb{R}^{n_{l} \times n_{l}}$ represents the activation states of the $l$-th ReLU layer, where $\sigma_{ii} \in\{0, 1\}$.

\begin{assumption} \label{assumption:activation}
        Because the DNN is highly non-linear, the change of the activation states $\Sigma_l$ is unpredictable for analysis.
        To simplify our analysis, we assume that either the adversarial perturbation does not significantly change the activation states in each layer, or changes of activation states have various effects on the  following gradient $\tilde{W}_x$ that are usually cancel each other out.
        Thus, the ReLU network can be approximately regarded as piecewise linear, and the forward propagation can be approximated as $z \approx \tilde{W}_x^T x + b \in\mathbb{R}^{c}$, where $\tilde{W}_x^T =  W_d^T \Sigma_{d-1} W_{d-1}^T \dots \Sigma_{1} W_{1}^T  \in\mathbb{R}^{c\times n}$.
    \end{assumption}

\begin{lemma}
\label{lemma:hessian-closed-form}
Based on Assumption~\ref{assumption:activation}, the Hessian matrix of the loss function~\textit{w.r.t.} the input sample $H \triangleq \nabla_x^2 \textit{Loss}(x)$ can be represented as $H = g_z \nabla_z^2 \textit{Loss}(x) g_z^T$, where $g_z = \nabla_x z$ denotes the gradient of the network output $z$~\textit{w.r.t.} the input sample $x$.
\end{lemma}

   \begin{proof}
According to the chain rule, the Hessian matrix $H \triangleq \nabla_x^2 \textit{Loss}(x)$ can be re-written as
        \begin{align}
            \label{eq:hessain_re}
                H &= \frac{\partial^2 \textit{Loss}(x)}{\partial x \partial x^T} \nonumber \\
                &= \frac{\partial (\frac{\partial \textit{Loss}(x)}{\partial z^T } \frac{\partial z}{\partial x^T} )^T}{ \partial x^T} \nonumber \\
                &= \underbrace{\frac{(\frac{\partial z}{\partial x^T})^T}{\partial x^T}}_{=0} \frac{\partial \textit{Loss}(x)}{\partial z} + (\frac{\partial z}{\partial x^T})^T \frac{\partial (\frac{\partial \textit{Loss}}{\partial z})}{\partial x^T} \qquad // \text{According to Assumption~\ref{assumption:activation}} \nonumber \\
                &= (\frac{\partial z}{\partial x^T})^T \frac{\partial (\frac{\partial \textit{Loss}}{\partial z})}{\partial z^T} \frac{\partial z}{\partial x^T}\nonumber \\
                &= (\frac{\partial z}{\partial x^T})^T \frac{\partial^2 \textit{Loss}(x)}{\partial z \partial z^T} \frac{\partial z}{\partial x^T} \nonumber \\
                &= g_z \nabla_z^2 \textit{Loss}(x) (g_z)^T.
        \end{align}

Thus, Lemma~\ref{lemma:hessian-closed-form} is proven.
\end{proof}
$ $\newline

We aim to prove that the Hessian matrix $H=\nabla_x^2 \textit{Loss}(x)$ is positive semi-definite under both multi-category classification and binary classification.

    \textbf{Multi-category classification.}
    The classification has $c$ categories.
    Given the network output $z\in\mathbb{R}^c$ and the ground-truth label $y\in \{1, 2, \dots, c \}$, the loss function  is formulated as a cross-entropy loss upon  the softmax function, as follows.
    \begin{equation} \label{eq:softmax-ce}
        \begin{split}
            p_i &= \frac{\exp(z_i)}{\sum_{j=1}^c \exp(z_j)} \\
            \textit{Loss}(x) &= -\sum_{i=1}^c p^*_i \log(p_i)
        \end{split}
    \end{equation}
    where if $i=y, p^*_i =1$; otherwise $p^*_i=0$.

    \textbf{Binary classification.}
    There are two categories in binary classification.
    Given an output $z\in\mathbb{R}$ and the ground-truth label $y\in \{0, 1\}$, the loss function  is formulated as a cross-entropy loss upon  the sigmoid function, as follows.
    Note that the output $z\in\mathbb{R}$ in binary classification is a scalar.
    \begin{equation} \label{eq:sigmoid-ce}
        \begin{split}
            p(y=1\mid x) &= \text{Sigmoid}(z) =  \frac{\exp(z)}{1 + \exp(z)} \\
            p(y=0\mid x) &= 1- p(y=1\mid x) = \frac{1}{1 + \exp(z)} \\
            \textit{Loss}(x) &= -p_0^*\log(p(y=0\mid x)) -p_1^*\log(p(y=1\mid x)) \\
            &= -yz + \log(1 + \exp(z))
        \end{split}
    \end{equation}
    where $p^*_0=1-y$ and $p^*_1=y$.

\begin{lemma}[Gershgorin circle theorem]
\label{lemma:circle-theorem}
        Let $A\in\mathbb{C}^{n\times n}$,  $R_i=\sum_{j\ne i} |A_{ij}|$.
    Each eigenvalue of $A$ lies within at least one of the Gershgorin discs $D(A_{ii},R_{i})$, where the Gershgorin disc is defined as $D(A_{ii}, R_{i})= \{z\in C\mid |z-A_{ii}|\le R_i\}$
    \end{lemma}

\begin{corollaryappendix}
    \label{corollaryappendix:semi-definite}
    According to \cite{yao2018hessian}, if the loss function is formulated as the cross-entropy upon the softmax function defined in Eq.~(\ref{eq:softmax-ce}) or upon the sigmoid function defined in Eq.~(\ref{eq:sigmoid-ce}), then the Hessian matrix $H=\nabla_x^2 \textit{Loss}(x)$ is positive semi-definite.
    Therefore, $\sum_{a, b\in\Omega} I_{ab}(\delta) \ge 0$.
    \end{corollaryappendix}
\begin{proof}

\textbf{Multi-category classification.}
We first prove that the Hessian matrix is positive semi-definite in the scenario of multi-category classification with a softmax function.
Based on Assumption~\ref{assumption:activation}, the network can be approximated as $z \approx \tilde{W}_x^T x + b$, the gradient $g_z$ of the network output $z$~\textit{w.r.t.} the input sample $x$ can be simplified as
 \begin{equation}
 \label{eq:g_z_re}
        \begin{split}
        g_z &\triangleq \nabla_x z \\
        &\approx \tilde{W}_x
        = (W_d^T \Sigma_{d-1} W_{d-1}^T \dots \Sigma_{1} W^T_{1})^T.
   \end{split}
    \end{equation}

Then, based on Lemma~\ref{lemma:hessian-closed-form}, the Hessian matrix $H=\nabla_x^2 \textit{Loss}(x)$ can be written as
\begin{equation}
    \begin{split}
    \label{eq:hessian_re_1}
        H &= \nabla_x^2 \textit{Loss}(x) \\
        &= g_z \nabla_z^2 \textit{Loss}(x) g_z^T
        \quad // \quad \text{According to Eq.~\eqref{eq:hessain_re}}
         \\
        &= \tilde{W}_x \nabla_z^2 \textit{Loss}(x) \tilde{W}_x^T.
        \quad // \quad \text{According to Eq.~\eqref{eq:g_z_re}}
    \end{split}
    \end{equation}

To further simplify Eq.~\eqref{eq:hessian_re_1}, let us focus on each element $(\nabla_z^2 \textit{Loss}(x))_{ij}$ of matrix $\nabla_z^2 \textit{Loss}(x)$ in Eq.~\eqref{eq:hessian_re_1}.
\begin{equation} \label{eq:loss-to-z}
        \begin{split}
            (\nabla_z^2 \textit{Loss}(x))_{ij} &= \frac{\partial^2 \textit{Loss}(x)}{ \partial z_i \partial z_j} \\
            &=  \frac{\partial (\frac{\partial\textit{Loss}(x) }{\partial z_i})}{\partial z_j} \\
            &= \frac{\partial (\sum_{k=1}^c \frac{\partial\textit{Loss}(x) }{\partial p_k}\frac{\partial p_k}{\partial z_i})}{\partial z_j} \\
            &= \frac{\partial ( \frac{\partial\textit{Loss}(x) }{\partial p_y}\frac{\partial p_y}{\partial z_i})}{\partial z_j}
             \quad// \quad  \forall k\ne y, \frac{\partial\textit{Loss}(x) }{\partial p_k}=0 \\
            &= \frac{\partial ( -\frac{1}{p_y}\frac{\partial p_y}{\partial z_i})}{\partial z_j}.
        \end{split}
    \end{equation}

Here, the term $\frac{\partial p_y}{\partial z_i}$ in Eq.~\eqref{eq:loss-to-z} can be further simplified as follows.
Specifically, when the category $i$ equals to the ground-truth label $y$,~\textit{i.e.,} $i=y$, we have
\begin{equation} \label{eq:dp-dz-y}
        \begin{split}
            \frac{\partial p_y}{\partial z_{y}}
            &=\frac{\partial \bigg(\frac{\exp(z_y)}{\sum_{k=1}^c\exp{(z_{k})}}\bigg)} {\partial z_{y}}
            \\
            &= \frac{\exp(z_y)(\sum_{k=1}^c \exp(z_k)) -\exp(z_y)\exp(z_y) }{(\sum_{k=1}^c \exp(z_k))^2} \\
            &= \frac{\exp(z_y)}{\sum_{k=1}^c \exp(z_k)} \left(1 - \frac{\exp(z_y)}{\sum_{k=1}^c \exp(z_k)} \right) \\
            & =p_y (1- p_y).
        \end{split}
    \end{equation}

For the other case, when the category $i$ does not equal to the ground-truth label $y$,~\textit{i.e.,} $i\ne y$, we have
 \begin{equation} \label{eq:dp-dz-i}
        \begin{split}
            \frac{\partial p_y}{\partial z_{i}}
            &=\frac{\partial \bigg(\frac{\exp(z_y)}{\sum_{k=1}^c\exp{(z_{k})}}\bigg)} {\partial z_{i}}
            \\
            &= \frac{-\exp(z_y)\exp(z_i) }{(\sum_{k=1}^c \exp(z_k))^2} \\
            &= -\frac{\exp(z_y)}{\sum_{k=1}^c \exp(z_k)}\frac{\exp(z_i)}{\sum_{k=1}^c \exp(z_k)} \\
            & = -p_y p_i.
        \end{split}
    \end{equation}

Based on Eq.~(\ref{eq:dp-dz-y}) and Eq.~(\ref{eq:dp-dz-i}), each element $(\nabla_z^2 \textit{Loss}(x))_{ij}$ of matrix $\nabla_z^2 \textit{Loss}(x)$ in Eq.~\eqref{eq:loss-to-z} can be represented as follows.
        \begin{align}
        \label{eq:loss_z_re}
            (\nabla_z^2 \textit{Loss}(x))_{ij} &=\frac{\partial^2 \textit{Loss}(x)}{ \partial z_i \partial z_j} \\
                &= \frac{\partial ( -\frac{1}{p_y}\frac{\partial p_y}{\partial z_i})}{\partial z_j}
                \quad // \quad \text{According to Eq.~\eqref{eq:loss-to-z}.}
               \nonumber \\
                &=\begin{cases}
                    \frac{\partial (p_y-1)}{\partial z_j}, &\text{if $i=y$},
                     \quad // \quad \text{According to Eq.~\eqref{eq:dp-dz-y}.}
                     \\
                    \frac{\partial p_i}{\partial z_j}, &\text{if $i\ne y $.}
                    \quad // \quad \text{According to Eq.~\eqref{eq:dp-dz-i}.}
                \end{cases} \nonumber \\
                &=
                \begin{cases}
                    p_y(1-p_y), &\text{if $i=y, j=i$}, \\
                    -p_y p_j, &\text{if $i=y, j\ne i$}, \\
                    p_i(1-p_i), &\text{if $i\ne y, j=i $}, \\
                    -p_i p_j, &\text{if $i\ne y, j\ne i $}. \\
                \end{cases} \nonumber \\
                &=
                \begin{cases}
                    p_i(1-p_i), &\text{if $i=j $}, \\
                    -p_i p_j, &\text{if $i \ne j $}. \\
                \end{cases} 
        \end{align}

In order to prove that the Hessian matrix $H = \nabla_x^2 \textit{Loss}(x) = \tilde{W}_x \nabla_z^2 \textit{Loss}(x) \tilde{W}_x^T$ is positive semi-definite, we first prove that matrix $\nabla_z^2 \textit{Loss}(x)$ is positive semi-definite.
Because $\nabla_z^2 \textit{Loss}(x)$ is a real symmetric matrix, we can prove the positive semi-definiteness of matrix $\nabla_z^2 \textit{Loss}(x)$, by proving that all eigenvalues of matrix $\nabla_z^2 \textit{Loss}(x)$  are non-negative.
To this end, we use Lemma~\ref{lemma:circle-theorem} (Gershgorin circle theorem) to estimate the bounds of such eigenvalues.

Lemma~\ref{lemma:circle-theorem} indicates that each eigenvalue of $\nabla_z^2 \textit{Loss}(x)$ lies within at least one of the Gershgorin discs $D((\nabla_z^2 \textit{Loss}(x))_{ii}, R_i)$, where $R_i = \sum_{j\ne i} |(\nabla_z^2 \textit{Loss}(x))_{ij}|= \sum_{j\ne i} p_i p_j = p_i (\sum_{j\ne i} p_j) = p_i (1-p_i)$, based on Eq.~\eqref{eq:loss_z_re}.
Thus, for each eigenvalue $\lambda_i'$ of $\nabla_z^2 \textit{Loss}(x)$, we have $0\le \lambda_i' \le  \max_j 2p_j (1-p_j)$.
That is, each eigenvalue $\lambda_i'$ of $\nabla_z^2 \textit{Loss}(x)$ is non-negative, therefore matrix $\nabla_z^2 \textit{Loss}(x)$ is positive semi-definite.

In this way, to prove that the Hessian matrix $H=\nabla_x^2 \textit{Loss}(x)$ is positive semi-definite, for any vector $a\in \mathbb{R}^n$, we have
\begin{equation}
        \begin{split}
            a^T H a &= a^T \tilde{W}_x \nabla_z^2 \textit{Loss}(x) \tilde{W}_x^T a
            \quad // \quad \text{According to Eq.~\eqref{eq:hessian_re_1}}
            \\
            & = ( \tilde{W}_x^T a)^T \nabla_z^2 \textit{Loss}(x) ( \tilde{W}_x^T a) \ge 0.
             \quad// \quad \text{Because $\nabla_z^2 \textit{Loss}(x)$ is positive semi-definite.}
        \end{split}
    \end{equation}

Thus, for multi-category classification, if the loss function is formulated as the cross-entropy with a softmax function, then we have proven that the Hessian matrix $H=\nabla_x^2 \textit{Loss}(x)$ is positive semi-definite.

    \textbf{Binary classification.}
    We then prove that the Hessian matrix is positive semi-definite in binary classification scenario with a sigmoid function.
    Based on Assumption~\ref{assumption:activation}, for binary classification, the forward propagation of the ReLU network can be approximated as follows.
    \begin{equation}
        z \approx \tilde{w}_x^T x + b.
    \end{equation}

    Note that in binary classification, $\tilde{w}\in\mathbb{R}^n$ is a vector and $z\in\mathbb{R}$ is a scalar.
    In this way, the gradient $g_z$ of the network output $z$~\textit{w.r.t.} the input sample $x$ can be simplified as
     \begin{equation}
     \label{eq:g_z_re-binary}
            \begin{split}
            g_z &\triangleq \nabla_x z \\
            &\approx \tilde{w}_x
            = (W_d^T \Sigma_{d-1} W_{d-1}^T \dots \Sigma_{1} W^T_{1})^T.
       \end{split}
        \end{equation}

    According to Lemma~\ref{lemma:hessian-closed-form}, the Hessian matrix $H=\nabla_x^2 \textit{Loss}(x)$ is given as follows.
    \begin{equation}
    \begin{split}
        H = \nabla_x^2 \textit{Loss}(x) &= g_z \nabla_z^2 \textit{Loss}(x) g_z^T
        \quad // \quad \text{According to Eq.~\eqref{eq:hessain_re}}
        \\
        &= \tilde{w}_x \nabla_z^2 \textit{Loss}(x) \tilde{w}_x^T.
        \quad // \quad \text{According to Eq.~\eqref{eq:g_z_re-binary}}
    \end{split}
    \end{equation}

    Note that in binary classification, the loss function is formulated as the cross-entropy loss upon the sigmoid function.
    Thus, according to Eq.~(\ref{eq:sigmoid-ce}), $\nabla_z^2 \textit{Loss}(x)$ can be derived as follows.
    \begin{equation}
        \begin{split}
            \nabla_z^2 \textit{Loss}(x) &= \frac{\partial^2 \textit{Loss}(x)}{\partial z^2} \\
            & =\frac{\partial (\frac{\partial \textit{Loss}(x)}{\partial z})}{\partial z} \\
            & =\frac{\partial (-y + \frac{\exp(z)}{1 + \exp(z)})}{\partial z} \\
            &=\frac{\exp(z)}{(1+\exp(z))^2} > 0.
        \end{split}
    \end{equation}

    Then, given $H=\nabla_x^2 \textit{Loss}(x)$, for any vector $a\in \mathbb{R}^n$, we have
    \begin{equation}
        \begin{split}
            a^T H a &= a^T \tilde{w}_x \nabla_z^2 \textit{Loss}(x) \tilde{w}_x^T a \\
            & =  \frac{\exp(z)}{(1+\exp(z))^2} ( \tilde{w}_x^T a)^2 \ge 0.
        \end{split}
    \end{equation}

    Thus, for binary classification, if the loss function is formulated as the cross-entropy with the sigmoid function, then we have proven the Hessian matrix $H=\nabla_x^2 \textit{Loss}(x)$ is positive semi-definite.

        \textbf{The sum of pairwise interactions.}
        According to Lemma~\ref{lemmaappendix:interaction-closed-form}, given the adversarial perturbation $\delta$, the sum of pairwise interactions inside $\delta$ is given as follows.
        \begin{equation}
            \sum_{a,b\in \Omega} I_{ab}(\delta) = \delta^T H \delta.
        \end{equation}

        As proven above, $H=\nabla_x^2 \textit{Loss}(x)$ is positive semi-definite.
        Thus, $\sum_{a,b\in \Omega} I_{ab}(\delta) = \delta^T H \delta \ge 0$.

\end{proof}

\section{Multi-step attack generating perturbations with larger interactions than the single-step attack}
\label{append:single vs multi}

In this section, we prove Theorem~\ref{theorem: multi-step have large interactions} in Section~\ref{subsec:multi-step vs single-step}, which demonstrates that the sum of pairwise interactions for the multi-step attack is larger than that for the single-step attack.

Before the proof of Theorem~\ref{theorem: multi-step have large interactions}, we first revisit the multi-step attack and the single-step attack.

\subsection{Multi-step attack}

\textbf{Perturbation generated by the multi-step attack.}
We revisit the formulation of the adversarial perturbation $\delta^{(\textrm{multi})}$ of the multi-step attack, which is defined in Appendix~\ref{append:delta-closed-form}.
    \begin{equation} \label{eq:delta-multi-inf}
        \begin{split}
            &\delta^{(\textrm{multi})} \triangleq  \delta^{(\infty)} = \sum_{i=1}^n D_{ii} \gamma_i v_i, 
            \quad // \quad \text{According to Corollary~\ref{corollaryappendix:delta-inf}}
            \\
            D_{ii} &= \begin{cases}
                \frac{\exp(\beta\lambda_i)-1}{\lambda_i}, &\text{ if $\lambda_i\ne0$} \\
                \lim_{\lambda_i\rightarrow 0} \frac{\exp(\beta\lambda_i)-1}{\lambda_i}  =  \beta, & \text{ if $\lambda_i=0$}
            \end{cases}
        \end{split}
    \end{equation}
    where $\beta=\alpha m$ denotes the adversarial strength.
    $\lambda_i$ denotes the $i$-th eigenvalue of the Hessian matrix $H=\nabla_x^2 \textit{Loss}(x)$, and $v_i$ indicates the eigenvector corresponding to the $i$-th eigenvalue $\lambda_i$.
       $\gamma_i = g^T v_i$ represents the projection of the gradient $g=\nabla_x \textit{Loss}(x)$ on the eigenvector $v_i$.

       In fact, this adversarial perturbation $\delta^{(\textrm{multi})}$ is generated byout normalization effects.
       However, as discussed in  Appendix~\ref{append:delta-closed-form}, 
       {the perturbation $\delta^{(\textrm{multi})}$ well approximates the perturbation $\delta_{\textrm{general}}$ (defined in Eq.~\eqref{eq:multi-step-attack-ori}), because the step size $\alpha$ is sufficiently small and the step number $m$ is sufficiently large.}

       \textbf{Pairwise interaction inside $\delta^{(\text{multi})}$.}
       Based on the multi-step attack $\delta^{(\text{multi})}$, the interaction between each pair of perturbation units $(a,b)$ can be computed as 
       \begin{equation}
       I_{ab}(\delta^{(\textrm{multi})}) = \delta^{(\textrm{multi})}_a H_{ab} \delta^{(\textrm{multi})}_b.
       \quad // \quad \text{According to Eq.~\eqref{eq:shapley-interaction-index_2} in Appendix~\ref{append:interaction}}.
       \end{equation}
       Furthermore, we sum the interactions among all pairs of perturbation units $(a, b)$ as follows.
       \begin{equation} \label{eq:interaction-multi}
            \sum_{a, b\in\Omega}[{I_{a b}(\delta^{(\text{multi})})}] = \sum_{i=1}^n \lambda_{i}(D_{ii}\gamma_{i})^2.
            \quad // \quad \text{According to Eq.~\eqref{eq:delta-inf-interaction}}.
        \end{equation}

    \subsection{Single-step attack}
    \textbf{Perturbation generated by the single-step attack.}
     Given an input sample $x\in\mathbb{R}^n$, the corresponding adversarial perturbation of the single-step attack generated by the gradient ascent is defined as follows.
    \begin{equation} \label{eq:delta-single-eta}
        \begin{split}
            \delta^{(\textrm{single})} &\triangleq \eta \nabla_x \textit{Loss}(x)  \\
            &=  \sum_{i=1}^n \eta \gamma_i v_i.
            {\quad // \quad \text{According to Eq.~\eqref{eq:g-decompose}}}
        \end{split}
    \end{equation}
    where $\eta$ denotes the step size for the single-step attack.
    
    \textbf{Pairwise interaction inside $\delta^{(\text{single})}$.}
    Based on the perturbation of the single-step attack $\delta^{(\text{single})}$, the interaction between each pair of perturbation units $(a,b)$ can be computed as 
    \begin{equation}
    I_{ab}(\delta^{(\textrm{single})}) = \delta^{(\textrm{single})}_a H_{ab} \delta^{(\textrm{single})}_b.
     \quad // \quad \text{According to Eq.~\eqref{eq:shapley-interaction-index_2}}.
    \end{equation}
    In this scenario,  the sum of interactions among all pairs of perturbation units $(a, b)$ is given as follows.
 
    \begin{align} \label{eq:interaction-single}
        &\sum_{a, b\in\Omega}[{I_{a b}(\delta^{(\text{single})})}] \nonumber \\
        &= (\delta^{(\text{single})})^T H \delta^{(\text{single})} \nonumber\\
        &= ( \sum_{i=1}^n\eta \gamma_{i}v_{i})^T (\sum_{j=1}^n \lambda_{j}v_{j}v_{j}^T) ( \sum_{k=1}^n \eta\gamma_{k}v_{k})
            \quad // \quad \text{According to Eq.~\eqref{eq:interaction-single-def}}
            \nonumber
            \\
        &=\left[\sum_{i=1}^n \underbrace{ (\eta \gamma_{i} v_{i}^T)( \lambda_{i} v_{i} v_{i}^T) }_{\text{$ v_{i}^T v_{i}=1$}}+ \sum_{1\le i, j \le n \atop i\ne j} \underbrace{(\eta \gamma_{i} v_{i}^T)( \lambda_{j} v_{j} v_{j}^T)}_{\text{if $i\ne j, v_{i}^T v_{j}=0$}} \right] (\sum_{k=1}^n \eta \gamma_{k} v_{k}) \nonumber \\
        &=(\sum_{i=1}^n \lambda_{i} \eta \gamma_{i} v_{i}^T ) (\sum_{k=1}^n\eta \gamma_{k} v_{k}) \nonumber\\
        &=\sum_{i=1}^n \underbrace{ (\lambda_{i} \eta \gamma_{i} v_{i}^T)(\eta \gamma_{i} v_{i})}_{\text{$ v_{i}^T v_{i}=1$}}+ \sum_{1\le i, k\le n \atop i\ne k} \underbrace{ (\lambda_{i} \eta\gamma_{i} v_{i}^T)( \eta \gamma_{k} v_{k})}_{\text{if $i\ne k, v_{i}^T v_{k}=0$}}\nonumber \\
        &=  \sum_{i=1}^n \lambda_{i} (\eta \gamma_i)^2.
    \end{align}

    Thus, Theorem~\ref{theorem: multi-step have large interactions} focuses on comparisons between
    interactions $\sum_{a, b\in\Omega}[{I_{a b}(\delta^{(\text{multi})})}]$ for the multi-step attack and interactions $\sum_{a, b\in\Omega}[{I_{a b}(\delta^{(\text{single})})}]$ for the single-step attack.

    \textbf{Fair comparisons between the multi-step attack and the single-step attack.}
    In order to conduct fair comparisons between  $ \sum_{a, b\in\Omega}[{I_{a b}(\delta^{(\text{multi})})}]$ based on perturbations $\delta^{(\text{multi})}$
     and $ \sum_{a, b\in\Omega}[{I_{a b}(\delta^{(\text{single})})}]$ based on perturbations $\delta^{(\text{single})}$, the magnitudes of such perturbations should be controlled at the same level.
    Specifically, we consider two perspectives to constrain the magnitude of the perturbation $\delta^{(\text{single})}$.
    
    \textbf{Perspective 1.} 
    We set the adversarial strength of the single-step attack to be the same as that of the multi-step attack, \emph{i.e.} $\eta=\beta$.

   \textbf{Perspective 2.} 
     We set the $L_2$ norm of the perturbation $\delta^{(\text{single})}$ to be the same as that of the perturbation $\delta^{(\text{multi})}$,~\textit{i.e.,} $\|\delta^{(\text{single})}\|_2 =\|\delta^{(\text{multi})}\|_2$, \emph{i.e.} $\eta = \frac{\sqrt{\sum_{i=1}^n (D_{ii} \gamma_i)^2}}{\sqrt{\sum_{i=1}^n ( \gamma_i)^2}}$.

     \subsection{Proof of Theorem~\ref{theoremappendix: multi-step have large interactions}}

    \begin{theoremappendix}
    \label{theoremappendix: multi-step have large interactions}
        {If the loss function for attacking $\textit{Loss}(x)$  is formulated as the cross-entropy loss in Eq.~\eqref{eq:softmax-ce}},
        we use the infinite-step adversarial attack with an infinitesimal step size to simplify  the multi-step attack.
        The sum of pairwise interactions~\textit{w.r.t.} the perturbation $\delta^{(\text{multi})}$ of the multi-step attack is larger than that of the perturbation $\delta^{(\text{single})}$ of the single-step attack,
       $\sum_{a, b\in\Omega}[{I_{a b}(\delta^{(\text{multi})})}]\ge\sum_{a, b\in\Omega}[{I_{a b}(\delta^{(\text{single})})}]$.
    \end{theoremappendix}

    \begin{proof}
     We prove Theorem~\ref{theoremappendix: multi-step have large interactions} based on perturbations generated from the above two perspectives.

    \textbf{Perspective 1.}
    We generate the perturbation of the single-step attack by setting its adversarial strength to be the same as that of the multi-step attack,~\textit{i.e.,} $\eta=\beta$.
     To this end, the perturbation $\delta^{(\textrm{single})}$ of the single-step attack can be represented as
        \begin{equation} \label{eq:interaction-single-def}
            \begin{split}
                \delta^{(\textrm{single})} = \beta \nabla_x \textit{Loss}(x) =  \sum_{i=1}^n \beta \gamma_i v_i.
            \end{split}
        \end{equation}

        In this way, according to Eq.~\eqref{eq:interaction-single}, the sum of pairwise interactions inside $\delta^{(\textrm{single})} $ is given as follows.
        \begin{equation}\label{eq:interaction-single-beta}
            \sum_{a, b\in\Omega}[{I_{a b}(\delta^{(\text{single})})}] =\sum_{i=1}^n \lambda_{i} (\beta \gamma_i)^2.
        \end{equation}

        From this perspective, the difference between interactions $\sum_{a, b\in\Omega}[{I_{a b}(\delta^{(\text{multi})})}]$ for the multi-step attack and interactions $\sum_{a, b\in\Omega}[{I_{a b}(\delta^{(\text{single})})}]$ can be written as
        \begin{equation}
        \label{eq:interaction-diff_1}
            \begin{split}
                &\sum_{a, b\in\Omega}[{I_{a b}(\delta^{(\text{multi})})}]-\sum_{a, b\in\Omega}[{I_{a b}(\delta^{(\text{single})})}]
                \\
                &= \sum_{i=1}^n \lambda_{i} \left[(D_{ii}^2 - \beta^2) \gamma_i^2 \right]
                   \quad // \quad \text{According to Eq.~(\ref{eq:interaction-multi}) and Eq.~(\ref{eq:interaction-single-beta})}
                  \\
                &= \sum_{i=1}^n \lambda_{i} \left[ (D_{ii} + \beta)(D_{ii} - \beta) \gamma_i^2 \right].
            \end{split}
        \end{equation}

        Here, this difference is proven to be positive, because each term in Eq.~\eqref{eq:interaction-diff_1} is positive.
        The proof is as follows.

        \textbullet{ $\lambda_i \ge 0$.} According to Corollary~\ref{corollaryappendix:semi-definite}, the Hessian matrix $H=\nabla_x^2 \textit{Loss}(x)$ is positive semi-definite.
        Thus, all eigenvalues of the Hessian matrix $H$ are non-negative, \emph{i.e.} $\forall i, \lambda_i\ge 0$.

         \textbullet{ $ \gamma_i^2 \ge 0$.}

        \textbullet{ $ D_{ii} + \beta \ge 0$.} According to Corollary~\ref{corollaryappendix:delta-inf}, the value of $D_{ii}$ is given as follows.
        \begin{equation}
        \label{eq:interaction-diff_2_before}
            D_{ii} = \begin{cases}
                \beta, &\text{if $\lambda_i=0$;} \\
                \frac{\exp(\beta \lambda_i) - 1}{\lambda_i},  &\text{if $\lambda_i > 0$.}
            \end{cases}
        \end{equation}

        Note that $\beta=\alpha m > 0$, and $\lambda_i > 0$.
        In order to prove that $D_{ii}$ in Eq.~\eqref{eq:interaction-diff_2_before} is positive, we use the Taylor series to decompose the term $D_{ii}=\frac{\exp(\beta\lambda_i)-1}{\lambda_i}$ as follows.
        \begin{equation}
            \begin{split}
            \label{eq:interaction-diff_2}
                D_{ii}&=\frac{\exp(\beta\lambda_i)-1}{\lambda_i} \\
                & = \frac{\sum_{t=0}^{+\infty}\frac{(\beta \lambda_i)^t}{t!} -1}{\lambda_i} \\
                & = \sum_{t=1}^{+\infty}\frac{\beta^t \lambda_i^{t-1}}{t!} \ge 0 
                \quad // \quad \beta>0, \lambda_i \ge 0.
            \end{split}
        \end{equation}

        Thus, $D_{ii} + \beta \ge 0$, because both $\beta >0$ and $D_{ii} \ge 0$.

        \textbullet{ $ D_{ii} - \beta \ge 0$.} According to Eq.~\eqref{eq:interaction-diff_2},
        \begin{equation}
            D_{ii} - \beta = \begin{cases}
                0, &\text{if $\lambda_i=0$;} \\
                \sum_{t=1}^{+\infty}\frac{\beta^t \lambda_i^{t-1}}{t!} - \beta =\sum_{t=2}^{+\infty}\frac{\beta^t \lambda_i^{t-1}}{t!} \ge 0,  &\text{if $\lambda_i > 0$.}
            \end{cases}
        \end{equation}

        In summary, 
        \begin{equation}
            \begin{split}
                &\sum_{a, b\in\Omega}[{I_{a b}(\delta^{(\text{multi})})}]-\sum_{a, b\in\Omega}[{I_{a b}(\delta^{(\text{single})})}]
                \\
                &= \sum_{i=1}^n \underbrace{\lambda_{i}}_{\ge 0} \left[\underbrace{(D_{ii} - \beta)}_{\ge 0} \underbrace{(D_{ii} + \beta)}_{\ge 0} \underbrace{\gamma_i^2 }_{\ge 0}\right] \ge 0.
             \end{split}
        \end{equation}
        
        Thus, from the first perspective of constraining the strength of the single-step attack, we have proven that the interactions $\sum_{a, b\in\Omega}[{I_{a b}(\delta^{(\text{multi})})}]$ for the multi-step attack is larger than interactions $\sum_{a, b\in\Omega}[{I_{a b}(\delta^{(\text{single})})}]$ for the single step.

\textbf{Perspective 2.}
From the second perspective of constraining the magnitude of perturbation, we generate the perturbation $\delta^{(\text{single})}$ of the single-step attack by setting its $L_2$ norm to be the same as that of the multi-step attack,~\textit{i.e.,} $\|\delta^{(\text{single})}\|_2 =\|\delta^{(\text{multi})}\|_2$.
To this end, the step size $\eta$ to generate the perturbation of the single-step attack can be calculated as follows.
         \begin{equation} \label{eq:same-norm}
            \begin{split}
                &\|\delta^{(\text{single})}\|_2 =\|\delta^{(\text{multi})}\|_2
                \\
                \Rightarrow&\quad \sqrt{\sum_{i=1}^n (\eta \gamma_i)^2} = \sqrt{\sum_{i=1}^n (D_{ii} \gamma_i)^2}
                 \quad // \quad \text{According to Eq.~(\ref{eq:delta-multi-inf}) and Eq.~(\ref{eq:delta-single-eta}).}
                \\
                \Rightarrow&\quad \eta = \frac{\sqrt{\sum_{i=1}^n (D_{ii} \gamma_i)^2}}{\sqrt{\sum_{i=1}^n ( \gamma_i)^2}}.
            \end{split}
        \end{equation}

         In this perspective, based on Eq.~\eqref{eq:interaction-single}, the sum of pairwise interactions of the perturbation $\delta^{(\text{single})}$ can be represented as
          \begin{equation}
          \label{eq:inter-single-1}
            \begin{split}
                &\sum_{a, b\in\Omega}[{I_{a b}(\delta^{(\text{single})})}] \\
                &= (\delta^{(\text{single})})^T H \delta^{(\text{single})}\\
                &= \sum_{i=1}^n \lambda_i (\eta \gamma_i)^2.
            \end{split}
        \end{equation}

        Thus, based on Eq.~\eqref{eq:interaction-multi} and Eq.~\eqref{eq:inter-single-1}, the difference  
        between interactions $\sum_{a, b\in\Omega}[{I_{a b}(\delta^{(\text{multi})})}]$ for the multi-step attack and interactions $\sum_{a, b\in\Omega}[{I_{a b}(\delta^{(\text{single})})}] $ for the single-step attack can be represented as
        \begin{equation} \label{eq:multi-single-interaction-diff}
            \begin{split}
                &\sum_{a, b\in\Omega}[{I_{a b}(\delta^{(\text{multi})})}]-\sum_{a, b\in\Omega}[{I_{a b}(\delta^{(\text{single})})}] \\
                &= \sum_{i=1}^n \lambda_i \left[(D_{ii}\gamma_i)^2 - (\eta \gamma_i)^2 \right]. \\
            \end{split}
        \end{equation}

        Without loss of generality, let the eigenvalue $\lambda_i$ of the Hessian matrix $H=\nabla_x^2 \textit{Loss}(x)$ monotonically decrease with the index $i$,~\textit{i.e.,}
        $ \lambda_1 \ge  \lambda_2 \ge \dots \ge \lambda_n \ge 0$.
        The proof of $\forall i, \lambda_i\ge 0$ is demonstrated in Appendix~\ref{appendix:positive-semi-definite}.
        From the second perspective, we admit that we cannot fully prove that the sum of pairwise interactions~\textit{w.r.t.} the perturbation $\delta^{(\text{multi})}$ being greater than that of the perturbation $\delta^{(\text{single})}$,~\textit{i.e.,} Eq.~\eqref{eq:multi-single-interaction-diff} is positive.
        However, we can roughly consider that $\left[(D_{ii}\gamma_i)^2 - (\eta \gamma_i)^2 \right]$ also monotonically decreases with the index $i$.
        This is because $(D_{ii})^2 =(\frac{\exp(\beta\lambda_i)-1}{\lambda_i})^{2}$ rapidly decreases with the $\lambda_i$ in an exponential manner.
     In other words, we have
     \begin{equation}
        D_{11}^2\gamma_1^2 - \eta^2\gamma_1^2 \ge  D_{22}^2\gamma_2^2 - \eta^2\gamma_2^2 \ge \dots \ge  D_{nn}^2\gamma_n^2 - \eta^2\gamma_n^2
     \end{equation}

       Based on this, we can use Chebyshev's sum inequality to prove that  the sum of pairwise interactions~\textit{w.r.t.} the perturbation $\delta^{(\text{multi})}$ is likely to be greater than that of the perturbation $\delta^{(\text{single})}$.

     \begin{lemma}
     \label{lemma:csinequality}
                (Chebyshev's sum inequality) If $a_1 \ge a_2 \ge \dots \ge a_n$ and $b_1 \ge b_2 \ge \dots \ge b_n$, then $\frac{1}{n}\sum_{k=1}^n a_k b_k \ge \left(\frac{1}{n}\sum_{k=1}^n a_k \right)\left(\frac{1}{n}\sum_{k=1}^n b_k \right) $.
            \end{lemma}

            \begin{align}
                &\sum_{a, b\in\Omega}[{I_{a b}(\delta^{(\text{multi})})}]-\sum_{a, b\in\Omega}[{I_{a b}(\delta^{(\text{single})})}] \nonumber \\
                &= \sum_{k=1}^n \lambda_i \left[(D_{ii}\gamma_i)^2 - (\eta \gamma_i)^2 \right] \nonumber \\
                &\ge n \left(\frac{1}{n}\sum_{i=1}^n \lambda_i\right)  \left[\frac{1}{n} \sum_{i=1}^n \gamma_i^2\left( D_{ii}^2 - \eta^2 \right) \right]
                 \quad // \quad \text{Chebyshev's sum inequality} \nonumber
                 \\
                &= \left(\frac{1}{n}\sum_{i=1}^n \lambda_i\right)  \left[ \|\delta^{(\text{multi})}\|^{2}_2 -\|\delta^{(\text{single})}\|^{2}_2 \right]
                \quad // \quad \text{According to Eq.~(\ref{eq:delta-multi-inf}) and Eq.~(\ref{eq:delta-single-eta})} \nonumber
                 \\
                &=0.
                \quad // \quad \text{According to Eq.~(\ref{eq:same-norm})}.
            \end{align}

       Thus, from the second perspective of constraining the $L_2$ norm of the adversarial perturbation, interactions  $\sum_{a, b\in\Omega}[{I_{a b}(\delta^{(\text{multi})})}]$ for the multi-step attack is very likely to be larger than interactions $\sum_{a, b\in\Omega}[{I_{a b}(\delta^{(\text{single})})}]$ for the single-step attack.

       In summary, Theorem~\ref{theorem: multi-step have large interactions} in Section~\ref{subsec:multi-step vs single-step} is proven from the above two perspectives.
        
    \end{proof}

\section{Proof that more balanced prediction probabilities can result in smaller interactions in the single-step attack} \label{append:balance}

In this section, we aim to prove that given a input sample $x$, if the prediction probabilities of $x$ are more balanced, then adversarial perturbations generated by the single-step attack exhibit smaller interactions.
Before the proof, we introduce a lemma.

\subsection{Proof of Lemma~\ref{lemma:g-H-in-p-form}} \label{append:g-H-in-p-form} 
In this subsection, given an input $x$ and a DNN using ReLU activations, we use $\textit{Loss}(x)$ to denote the classification loss.
 If the classification loss $\textit{Loss}(x)$  is formulated as the cross-entropy loss in Eq.~\eqref{eq:softmax-ce}, then we can derive the closed-form solution to the gradient $\nabla_x \textit{Loss}(x)$ and the Hessian matrix $\nabla_x^2 \textit{Loss}(x)$.

Let the classification have $c$ categories.
Given the network output $z\in\mathbb{R}^c$ and the ground-truth label $y\in \{1, 2, \dots, c \}$, the loss function  is formulated as a cross-entropy loss upon  the softmax function, as follows.
\begin{equation}  \label{eq:softmax-in-balance}
    \begin{split}
        p_i &= \frac{\exp(z_i)}{\sum_{j=1}^c \exp(z_j)} \\
        \textit{Loss}(x) &= -\sum_{i=1}^c p^*_i \log(p_i)
    \end{split}
\end{equation}
where if $i=y, p^*_i =1$; otherwise $p^*_i=0$.

\begin{lemma} \label{lemma:g-H-in-p-form}
Based on Assumption~\ref{assumption:activation}, if the classification loss $\textit{Loss}(x)$  is formulated as the cross-entropy loss in Eq.~\eqref{eq:softmax-ce}, then the output of the ReLU network can be written as 
    $z = \tilde{W}_x^T x + b\in\mathbb{R}^c$,  where $\tilde{W}_x^T =  W_d^T \Sigma_{d-1} W_{d-1}^T \dots \Sigma_{1} W_{1}^T \in \mathbb{R}^{c\times n}$.
    Then, given the input $x$,  the gradient \emph{w.r.t} the input and the Hessian matrix \emph{w.r.t} the input are given as $  g = \nabla_x \textit{Loss}(x)  = \tilde{W}_x (p- Y)$  and  $H = \nabla_x^2 \textit{Loss}(x) =  \tilde{W}_x (\text{diag}(p) - p p^T) \tilde{W}_x^T$, respectively. Here, $p\in [0, 1]^c$ denotes prediction probabilities of the $c$ categories, and $Y\in \{0, 1 \}^c$ is a one-hot vector, where $Y_y = 1$; and $\forall i\ne y, Y_i = 0$.
\end{lemma}

\begin{proof}
According to the chain rule, the gradient $g = \nabla_x \textit{Loss}(x)$ can be re-written as follows.
\begin{align}
    g &= \nabla_x \textit{Loss}(x) \nonumber \\
    &= \left( \frac{\partial \textit{Loss}}{\partial z^T} \frac{\partial z}{\partial x^T}  \right)^T \nonumber \\
    &= \left( \frac{\partial \textit{Loss}}{\partial z^T} \tilde{W}_x^T \right)^T \nonumber \\
    &= \tilde{W}_x  \frac{\partial \textit{Loss}}{\partial z},
\end{align}
where
\begin{align}
    \frac{\partial \textit{Loss}}{\partial z_i} &= \sum_{k=1}^{c} \frac{\partial \textit{Loss}}{\partial p_k} \frac{\partial p_k}{\partial z_i} \nonumber \\ 
    &= \frac{\partial\textit{Loss}}{\partial p_y}\frac{\partial p_y}{\partial z_i}
             \quad// \quad  \forall k\ne y, \frac{\partial\textit{Loss}(x) }{\partial p_k}=0 \nonumber \\
    &= -\frac{1}{p_y} \frac{\partial p_y}{\partial z_i} \nonumber \\
    &=
    \begin{cases}
        p_y - 1, \quad &\text{if $i = y$},  \quad // \quad \text{According to Eq.~\eqref{eq:dp-dz-y}}\\
        p_i, \quad &\text{if $i \ne y$}, \quad // \quad \text{According to Eq.~\eqref{eq:dp-dz-i}}\\
    \end{cases} \nonumber
\end{align}

In this way, $g$ can be written as follows.
\begin{align}
    g_= \nabla_x \textit{Loss}(x)  = \tilde{W}_x (p- Y),
\end{align}
where $Y\in \{0, 1 \}^c$ is a one-hot vector, where $Y_y = 1$; and $\forall i\ne y, Y_i = 0$.

According to Lemma~\ref{lemma:hessian-closed-form}, the Hessian matrix $H_{\theta} = \nabla_x^2 \textit{Loss}_{\theta}(x)$ can be written as
\begin{equation}
     H = (\frac{\partial z}{\partial x^T})^T \frac{\partial^2 \textit{Loss}}{\partial z \partial z^T} \frac{\partial z}{\partial x^T}
\end{equation}

Because $z = \tilde{W}_x^T x + b$, we have $\frac{\partial z}{\partial x^T} = \tilde{W}_x^T$.
Moreover, according to Eq.~\eqref{eq:loss_z_re}, 
\begin{equation}
     (\nabla_z^2 \textit{Loss}(x))_{ij} =
                \begin{cases}
                    p_i(1-p_i), &\text{if $i=j $}, \\
                    -p_i p_j, &\text{if $i \ne j $}. \\
                \end{cases} 
\end{equation}

Thus, 
\begin{equation}
    \frac{\partial^2 \textit{Loss}}{\partial z \partial z^T} = {diag}(p) - p p^T
\end{equation}

In this way, $H$ can be written as follows.
\begin{align}
    H = \nabla_x^2 \textit{Loss}(x)  = \tilde{W}_x ({diag}(p) - p p^T) \tilde{W}_x^T
\end{align}

Thus, the lemma is proven.
\end{proof}
$ $\newline

\subsection{Proof of Theorem~\ref{theoremappendix:balance}}
Given an input $x$ and a DNN $\theta$ using ReLU activations, we use $\textit{Loss}_{\theta}(x)$ to denote the classification loss.
Note that according to Assumption~\ref{assumption:activation}, given the input $x$, the output before the softmax operation of the $i$-th category is given as  $z_i = (\tilde{W}_x)_i^T x + b_i$, where $(\tilde{W}_x)_i$ denotes the $i$-th column of $\tilde{W}_x$.
Therefore, we consider the class information of the $i$-th category to be encoded by $(\tilde{W}_x)_i$. 
In this study, we roughly consider that different categories share little information; therefore, we can assume that the weight vectors of different categories $(\tilde{W}_x)_i$ are orthogonal to each other.
This assumption is partially supported in previous studies~\cite{papyan2020prevalence}.
Moreover, we assume that  the $(\tilde{W}_x)_i$ of different categories have similar strengths.
In this way, we make the following assumption.

\begin{assumption} \label{assum:ortho-weights}
Given the input $x$ and the ReLU network, the network output can be written as $z = \tilde{W}_x^T x + b$.
We assume that $\tilde{W}_x^T  \tilde{W}_x  = \kappa\cdot I$, where $\kappa \in \mathbb{R}$ is a constant scalar.
\end{assumption}

Then, the adversarial perturbation generated by the single-step attack on the input $x$ is given as follows.
\begin{equation}
    \delta^{(\text{single})}_{\theta} = \eta g_{\theta},
\end{equation}
where $\eta$ denotes the step size, and $g_{\theta} = \nabla_x \textit{Loss}_{\theta}(x)$.
According to Lemma~\ref{lemma:interaction}, the sum of pairwise interactions of perturbation units $(a, b)$ is given as follows.
\begin{align}
    \sum_{a, b\in\Omega}[{I_{a b}(\delta_{\theta}^{(\text{single})})}] &= (\delta_{\theta}^{(\text{single})})^T H_{\theta} \delta_{\theta}^{(\text{single})},
\end{align}
where $H_{\theta} = \nabla_x^2 \textit{Loss}_{\theta}(x)$.

\begin{definition} \label{def:balance}
We are given the input $x$ and two DNNs parameterized by $\theta$ and $\theta'$, respectively. 
    We use $p \in [0, 1]^c$  and $q \in [0, 1]^c$ to denote two probability distributions of the prediction over $c$ categories on the same input sample $x$, which are predicted by the DNN $\theta$ and $\theta'$, respectively.
We assume $p_y=q_y\ge 0.5$ for the ground-truth category $y$ of the input sample $x$. 
For the other $c-1$ class probabilities, without loss of generality, let $p_{i_1}\ge p_{i_2} \ge \dots \ge p_{i_{c-1}}$ and $q_{j_1} \ge q_{j_2} \ge \dots \ge q_{j_{c-1}}$. If $\forall 1\le k \le c-2$, $q_{j_{k}} - q_{j_{k+1}} \le p_{i_{k} } - p_{i_{k+1}}$, then we say that $q$ is more balanced than $p$.
\end{definition}

In Definition~\ref{def:balance}, we discuss the probability distributions estimated using two DNNs. 
We consider the two DNNs to have the same classification ability, \emph{i.e.} we assume $p_y = q_y$, which is the condition required to compare their interactions further. 
Otherwise, it is unfair to compare pairwise interactions generated on two different DNNs, if we do not constrain the discrimination power of the two DNNs into the same level. 
Therefore, we assume that the input can be correctly classified by both DNNs with high confidence. 
Then, the main difference between these two DNNs is that the probability distribution of $q$ over the remaining $c-1$ classes is more balanced than that of $p$ over the remaining $c-1$ classes.

In this subsection, we compare the interaction of perturbations generated on these two DNNs, as shown in the following theorem.

\begin{theoremappendix} \label{theoremappendix:balance}
In multi-category classification,  let the loss function for attacking $\textit{Loss}(x)$  be formulated as the cross-entropy loss in Eq.~\eqref{eq:softmax-ce}.
    We are given the input $x$ and two DNNs parameterized by $\theta$ and $\theta'$, respectively. 
    We use $p \in [0, 1]^c$  and $q \in [0, 1]^c$ to denote the probability of $x$ that belongs to each of the $c$ categories, which are predicted by the DNN $\theta$ and $\theta'$, respectively.
    Let the two DNNs both satisfy Assumption~\ref{assum:ortho-weights} and Definition~\ref{def:balance}.
    Then, if $q$ is more balanced than $p$, then we have $\sum_{a, b\in\Omega}[{I_{a b}(\delta_{\theta'}^{(\text{single})})}] \le \sum_{a, b\in\Omega}[{I_{a b}(\delta_{\theta}^{(\text{single})})}]$.
\end{theoremappendix}

\begin{proof}
According to Lemma~\ref{lemma:interaction}, the sum of pairwise interactions of perturbation units $(a, b)$ is given as follows.
\begin{align} \label{eq:p-interction}
    \sum_{a, b\in\Omega}[{I_{a b}(\delta_{\theta}^{(\text{single})})}] &= (\delta_{\theta}^{(\text{single})})^T H_{\theta} \delta_{\theta}^{(\text{single})}  \nonumber \\
    &=\eta^2 \left[ \tilde{W}_x (p- Y)\right]^T \tilde{W}_x (\text{diag}(p) - p p^T)  \tilde{W}_x^T \left[ \tilde{W}_x (p- Y)\right] \quad//\quad \text{According to Lemma~\ref{lemma:g-H-in-p-form}}  \nonumber\\
    &= \eta^2 (p-Y)^T \tilde{W}_x^T  \tilde{W}_x (\text{diag}(p) - p p^T)  \tilde{W}_x^T \tilde{W}_x  (p- Y)  \nonumber\\
    &= \eta^2 \kappa^2 (p-Y)^T  (\text{diag}(p) - p p^T)   (p- Y)  \quad//\quad \text{According to Assumption~\ref{assum:ortho-weights}}  \nonumber \\
    &= \eta^2\kappa^2 \left[ (p-Y)^T\text{diag}(p) (p- Y)-   (p- Y)^T p p^T   (p- Y)  \right]  \nonumber \\
    &= \eta^2\kappa^2 \left[ (p-Y)^T\text{diag}(p) (p- Y) - \left(  p^T   (p- Y)\right)^2   \right]  \nonumber \\
    &= \eta^2\kappa^2 \left[ \sum_{i=1}^c p_i (p_i - Y_i)^2 - \left( \sum_{i=1}^c p_i (p_i - Y_i)\right)^2   \right]  \nonumber \\
    &= \eta^2\kappa^2 \left[ \sum_{i\ne y} p_i^3 + p_y (p_y - 1)^2  - \left( \sum_{i\ne y} p_i^2 + p_y (p_y -1) \right)^2  \right]
\end{align}

Similarly, we have 
\begin{align} \label{eq:q-interction}
    \sum_{a, b\in\Omega}[{I_{a b}(\delta_{\theta'}^{(\text{single})})}] &= (\delta_{\theta'}^{(\text{single})})^T H_{\theta} \delta_{\theta}^{(\text{single})} \nonumber \\
    &= \eta^2\kappa^2 \left[ \sum_{i\ne y} q_i^3 + q_y (q_y - 1)^2  - \left( \sum_{i\ne y} q_i^2 + q_y (q_y -1) \right)^2  \right]
\end{align}

Note that $p_y = q_y$, and for the other $c-1$ class probabilities, without loss of generality, let $p_{i_1}\ge p_{i_2} \ge \dots \ge p_{i_{c-1}}$ and $q_{j_1} \ge q_{j_2} \ge \dots \ge q_{j_{c-1}}$.
In this way, the difference between $\sum_{a, b\in\Omega}[{I_{a b}(\delta_{\theta}^{(\text{single})})}]$ and $\sum_{a, b\in\Omega}[{I_{a b}(\delta_{\theta'}^{(\text{single})})}]$ is given as follows.
\begin{align} \label{eq:p-q-interaction-diff}
&\sum_{a, b\in\Omega}[{I_{a b}(\delta_{\theta}^{(\text{single})})}] - \sum_{a, b\in\Omega}[{I_{a b}(\delta_{\theta'}^{(\text{single})})}] \nonumber \\
    &= \eta^2\kappa^2 \left\{\sum_{k = 1}^{c-1} p_{i_k}^3 + p_y(p_y-1)^2 - \left( \sum_{k = 1}^{c-1} p_{i_k}^2 + p_y (p_y -1) \right)^2 - \left[\sum_{k = 1}^{c-1} q_{j_k}^3 +q_y(q_y-1)^2  - \left(\sum_{k = 1}^{c-1}q_{j_k}^2  + q_y (q_y -1) \right)^2 \right] \right\} \nonumber\\
    &= \eta^2\kappa^2 \left\{\sum_{k = 1}^{c-1} \left( p_{i_k}^3 - q_{j_k}^3 \right) + \underbrace{p_y(p_y-1)^2 - q_y(q_y-1)^2}_{=0} - \left[\left( \sum_{k = 1}^{c-1}p_{i_k}^2 + p_y (p_y -1) \right)^2 - \left(\sum_{k = 1}^{c-1} q_{j_k}^2 + p_y (p_y -1) \right)^2  \right]\right\} \nonumber \\
    &=  \eta^2\kappa^2\left\{\underbrace{ \sum_{k = 1}^{c-1} \left( p_{i_k} - q_{j_k} \right) \left(p_{i_k}^2 + p_{i_k}q_{j_k} + q_{j_k}^2 \right)}_{\text{(I)} } - \underbrace{\left(  \sum_{k = 1}^{c-1} (p_{i_k}^2 - q_{j_k}^2) \right)}_{\text{(II)} }  \underbrace{\left(  \sum_{k = 1}^{c-1} p_{i_k}^2 + p_y (p_y -1) +  \sum_{k = 1}^{c-1}q_{j_k}^2 + q_y (q_y -1) \right)}_{\text{(III)} }\right\}
\end{align}

\textbullet{ Term (I) $\ge 0$. } 
Because $q_{j_1}\ge q_{j_2} \ge \dots \ge q_{j_{c-1}} \ge 0$, then $q_{j_1}^2\ge q_{j_2}^2 \ge \dots \ge q_{j_{c-1}}^2 \ge 0$.
Because $p_{i_1}\ge p_{i_2} \ge \dots \ge p_{i_{c-1}}\ge 0$, then we have $p_{i_1}^2\ge p_{i_2}^2 \ge \dots \ge p_{i_{c-1}}^2\ge 0$ and $p_{i_1} q_{j_1}\ge p_{i_2} q_{j_2} \ge \dots \ge p_{i_{c-1}} q_{j_{c-1}} \ge 0$.
Therefore, $p_{i_k}^2 + p_{i_k}q_{j_k} + q_{j_k}^2 \ge p_{i_{k+1}}^2 + p_{i_{k+1}}q_{j_{k+1}} + q_{j_{k+1}}^2$.
Moreover, according to Definition~\ref{def:balance}, if $g_{\theta'}$ is more balanced than $g_{\theta}$, then $q_{j_{k}} - q_{j_{k+1}} \le p_{i_{k} } - p_{i_{k+1}}$.
In other words, we have $p_{i_{k} } - q_{j_{k}} \ge p_{i_{k+1}}- q_{j_{k+1}}$ as follows.
\begin{equation} \label{eq:p-q-order}
    p_{i_1} -q_{j_1}\ge p_{i_2} -q_{j_2}\ge \dots \ge p_{i_c} -q_{j_c}
\end{equation}

So according to Lemma~\ref{lemma:csinequality} (Chebyshev's sum inequality), we have
\begin{equation}
    \text{(I)}=\sum_{k = 1}^{c-1} \left( p_{i_k} - q_{j_k} \right) \left(p_{i_k}^2 + p_{i_k}q_{j_k} + q_{j_k}^2 \right) \ge  \frac{1}{c-1}  \underbrace{\left(  \sum_{k = 1}^{c-1} \left( p_{i_k} - q_{j_k} \right)\right)}_{=0}  \underbrace{\left( \sum_{k = 1}^{c-1} \left(p_{i_k}^2 + p_{i_k} q_{j_k} + q_{j_k}^2 \right) \right)}_{\ge 0} = 0
\end{equation}

\textbullet{ Term (II) $\ge 0$. } 
Because $p_{i_1}\ge p_{i_2} \ge \dots \ge p_{i_{c-1}} \ge 0$ and $q_{j_1}\ge q_{j_2} \ge \dots \ge q_{j_{c-1}} \ge 0$, we have $p_{i_1} + q_{j_1}\ge p_{i_2} + q_{j_2} \ge \dots \ge p_{i_{c-1}} + q_{j_{c-1}} \ge 0$.
So according to and Eq.~\eqref{eq:p-q-order} and Lemma~\ref{lemma:csinequality} (Chebyshev's sum inequality), we have
\begin{equation}
    \text{(II)}=\sum_{k = 1}^{c-1} (p_{i_k}^2 - q_{j_k}^2) =   \sum_{k = 1}^{c-1} (p_{i_k} - q_{j_k}) (p_{i_k} + q_{j_k}) \ge  \frac{1}{c-1}  \underbrace{\left(  \sum_{k = 1}^{c-1} \left( p_{i_k} - q_{j_k} \right)\right)}_{=0}  \underbrace{\left( \sum_{k = 1}^{c-1} \left(p_{i_k} + q_{j_k} \right) \right)}_{\ge 0} = 0
\end{equation}

\textbullet{ Term (III) $\le 0$. } 
Let us focus on $ \sum_{k = 1}^{c-1} p_{i_k}^2 + p_y (p_y -1)$.
$\forall  p_{i_k}$, we have
\begin{align}
    p_{i_k} p_{i_k} \le (\max_{k} p_{i_k}) p_{i_k} \le (1-p_y) p_{i_k}
\end{align}

In this way, 
\begin{align}
    \sum_{k = 1}^{c-1} p_{i_k}^2 \le (1-p_y)  \sum_{k = 1}^{c-1} p_{i_k}  = (1-p_y)^2
\end{align}

Then, $ \sum_{k = 1}^{c-1} p_{i_k}^2 + p_y (p_y -1)$ is bounded as follows.
\begin{align}
    \sum_{k = 1}^{c-1} p_{i_k}^2 + p_y (p_y -1)  & \le (1-p_y)^2 + p_y (p_y-1) \nonumber \\
    & = ( p_y^2 -2p_y + 1) + (p_y^2 - p_y) \nonumber \\
    & = 2p_y^2 -3 p_y + 1 \le 0 \quad // \quad \text{Because $0.5 \le p_y \le 1$}
\end{align}

Similarly, we have 
\begin{equation}
    \sum_{k = 1}^{c-1}q_{j_k}^2 + q_y (q_y -1) \le 0
\end{equation}

Thus, (III) $\le 0$.

In summary,  we have
\begin{align} 
&\sum_{a, b\in\Omega}[{I_{a b}(\delta^{(\text{single})}_{\theta})}] - \sum_{a, b\in\Omega}[{I_{a b}(\delta^{(\text{single})}_{\theta’})}] \nonumber \\
    &=  \underbrace{\eta^2\kappa^2}_{\ge 0}\left\{\underbrace{ \sum_{k = 1}^{c-1} \left( p_{i_k} - q_{j_k} \right) \left(p_{i_k}^2 + p_{i_k}q_{j_k} + q_{j_k}^2 \right)}_{\text{(I)} \ge 0 } - \underbrace{\left(  \sum_{k = 1}^{c-1} (p_{i_k}^2 - q_{j_k}^2) \right)}_{\text{(II)} \ge 0 }  \underbrace{\left(  \sum_{k = 1}^{c-1} p_{i_k}^2 + p_y (p_y -1) +  \sum_{k = 1}^{c-1}q_{j_k}^2 + q_y (q_y -1) \right)}_{\text{(III)} \le 0 }\right\} \nonumber \\
    & \ge 0
\end{align}

Thus, the theorem is proven.
\end{proof}

\section{Proof that the MI Attack generates perturbations with smaller interactions than those of the multi-step attack}
\label{append:mi}

In this section, we prove Proposition~\ref{pro:mi} in Section~\ref{sec:explaining}, which shows that perturbations of the MI Attack~\cite{mim} exhibit smaller interactions than those of the multi-step attack.

Before the proof of Proposition~\ref{pro:mi}, we revisit the MI Attack.

\subsection{MI Attack}
    The MI Attack~\cite{mim} boosts adversarial transferability by incorporating  the momentum of gradients in attacking as follows.
    \begin{equation} \label{eq:ori-mi}
        \begin{split}
            p^{(t)}_{\text{mi-ori}} &=  p^{(t-1)}_{\text{mi-ori}} + \frac{\nabla_{x'}\textit{Loss}(x'=x + \delta^{(t-1)}_{\text{mi-ori}})}{\|\nabla_{x'}\textit{Loss}(x'=x + \delta^{(t-1)}_{\text{mi-ori}}) \|_1}, \\
            \delta^{(t)}_{\text{mi-ori}} &= \operatorname{clip}\left\{\delta^{(t-1)}_{\text{mi-ori}} + \eta\cdot \operatorname{sign}(p^{(t)}_{\text{mi-ori}})  \right\},
        \end{split}
    \end{equation}
    where $\delta^{(t)}_{\text{mi-ori}} $ denotes the adversarial perturbation generated by the MI Attack after $t$ steps, $t \ge 1$. $p^{(0)}_{\text{mi-ori}}$ is initialized as $\mathbf{0}$.

For the MI Attack in Eq.~(\ref{eq:ori-mi}), the normalization of the gradient $\nabla_{x'}\textit{Loss}(x'=x + \delta^{(t-1)}_{\text{mi-ori}})$, sign operation, and clip operation are used to constrain the adversarial perturbation $\delta^{(t)}_{\text{mi-ori}}$.
In fact, the sign operation and clip operation, as well as the normalization of the gradient, are common operations in generating adversarial perturbations.
Therefore, the distinctive technique in the MI Attack is the use of the momentum of gradients $p^{(t)}_{\text{mi-ori}}$, instead of the raw gradient $\nabla_{x'}\textit{Loss}(x'=x + \delta^{(t-1)}_{\text{mi-ori}})$, to update the perturbation. 
Thus, to simplify the proof, we ignore the normalization of the gradient, the sign operation and the clip operation, without hurting the trustworthiness of the analysis of the MI Attack.
In this way, the MI Attack can be simplified as follows.

    \begin{equation} \label{eq:revised-mi}
        \begin{split}
            p^{(t)}_{\text{mi}} &=  p^{(t-1)}_{\text{mi}} + {\nabla_{x'}\textit{Loss}(x' = x + \delta^{(t-1)}_{\text{mi}})},\\
            \delta^{(t)}_{\text{mi}} &= \delta^{(t-1)}_{\text{mi}} + \alpha\cdot \underbrace{\text{norm}( p^{(t)}_{\text{mi}})}_{g^{(t)}_{\text{mi}}}.
        \end{split}
    \end{equation}
    
Here we define the normalized accumulated gradient $p^{(t)}_{\text{mi}}$ as $g^{(t)}_{\text{mi}} = \frac{p^{(t)}_{\text{mi}}}{t}$.
Note that the sign operation, $\text{sign}(p^{(t)}_{\text{mi}})$, can be considered as a type of normalization on $p^{(t)}_{\text{mi}}$.
Thus, we need use a normalization on $p^{(t)}_{\text{mi}}$, in order to avoid the increasing gradient strength in the attack owing to the accumulation of gradient.
Such a gradient $g^{(t)}_{\text{mi}}$ can be further represented as follows.

        \begin{align} \label{eq:our-revised-mi}
            g_{\text{mi}}^{(t)} = \frac{ p^{(t)}_{\text{mi}}}{t}&= \frac{\sum_{t'=1}^{t} \nabla_{x'}\textit{Loss}(x'=x + \delta^{(t'-1)}_{\text{mi}})}{t} \nonumber \\
            &= \frac{\sum_{t'=1}^{t-1} \nabla_{x'}\textit{Loss}(x'=x + \delta^{(t'-1)}_{\text{mi}}) +  \nabla_{x'}\textit{Loss}(x'=x+\delta^{(t-1)}_{\textrm{mi}})}{t}  \nonumber\\
            &= \frac{(t-1) g_{\textrm{mi}}^{(t-1)} +  \nabla_{x'}\textit{Loss}(x'=x+\delta^{(t-1)}_{\textrm{mi}})}{t} 
           \quad // \quad \text{According to $g_{\text{mi}}^{(t-1)} = { p^{(t-1)}_{\text{mi}}}/{(t-1)}$}
             \nonumber\\
            &=\mu g_{\textrm{mi}}^{(t-1)} + {(1-\mu)} \nabla_{x'}\textit{Loss}(x'=x+\delta^{(t-1)}_{\textrm{mi}}),
        \end{align}
    where $t$ denotes the step, and $\mu={(t-1)} / {t}$.

Moreover, the perturbation generated by the $m$-step MI Attack in Eq.~\eqref{eq:revised-mi} can be re-written as 
\begin{equation} 
\begin{aligned}
\label{eq:mi_update}
 \delta_{\textrm{mi}}^{(m)} 
&\triangleq  \delta^{(m-1)}_{\text{mi}} + \alpha\cdot g^{(m)}_{\text{mi}}
\quad //\quad \text{According to Eq.~\eqref{eq:revised-mi}}
 \\
 &=  \delta^{(m-2)}_{\text{mi}} + \alpha\cdot g^{(m-1)}_{\text{mi}} + \alpha\cdot g^{(m)}_{\text{mi}}
 \\
 & \dots
 \\
 &= \sum_{t=1}^{m} \alpha g^{(t)}_{\textrm{mi}}.
 \end{aligned}
\end{equation}

\begin{lemma} \label{lemma:mi-form}
    Based on Assumption~\ref{assumption:taylor}, the gradient $ g^{(t)}_{\textrm{mi}}$ 
    \textit{w.r.t.} the $t$-step MI Attack
    can be written as $g^{(t)}_{\textrm{mi}} = \sum_{s=0}^{t-1} a_{s}^{(t)} H^s g$, where $H=\nabla_x^2 \textit{Loss}(x)$ denotes the Hessian matrix of the loss function~\textit{w.r.t} the input sample $x$, the coefficient $\forall s,t, \,a_{s}^{(t)} \ge 0$, and the step number $1\le t \le m$. 
\end{lemma}

\begin{proof}
    We  prove Lemma~\ref{lemma:mi-form} by induction.

    \textit{Base case}:
    When the step $t=1$, we have   
    \begin{equation}
    \begin{aligned}
    g^{(1)}_{\textrm{mi}} 
   & =\frac{1-1}{1} g_{\textrm{mi}}^{(0)} + \frac{1}{1} \nabla_{x'}\textit{Loss}(x'=x+\delta^{(0)}_{\textrm{mi}})
    \quad // \quad \text{According to Eq.~\eqref{eq:our-revised-mi}}
   \\
   & = \nabla_{x}\textit{Loss}(x)
   \\ 
    & = g = a_0^{(1)} H^0 g,
    \end{aligned}
    \end{equation}
where $a_0^{(1)}=1\ge 0$ and $H^0=I$.

    \textit{Inductive step}:
    For $1< t' \le t-1$, we assume $g^{(t')}_{\textrm{mi}} = \sum_{s=0}^{t'-1} a_{s}^{(t')} H^s g$, where $a_{s}^{(t')} \ge 0$.
    Thus, at the step $t$, we have
    \begin{align}
        g_{\text{mi}}^{(t)} &= \mu g_{\textrm{mi}}^{(t-1)} + {(1-\mu)} \nabla_{x'}\textit{Loss}(x'=x+\delta^{(t-1)}_{\textrm{mi}}) 
        \quad // \quad \text{According to Eq.~\eqref{eq:our-revised-mi}}
         \nonumber\\
        &= \frac{t-1}{t}  g_{\textrm{mi}}^{(t-1)} +  \frac{1}{t}\nabla_{x'}\textit{Loss}(x'=x+\delta^{(t-1)}_{\textrm{mi}}) 
        \nonumber\\
        &= \frac{t-1}{t}  g_{\textrm{mi}}^{(t-1)} +  \frac{1}{t} \left[ g + H\delta^{(t-1)}_{\textrm{mi}}  \right] \quad // \quad \text{According to Eq.~\eqref{eq:grad_taylor}} 
         \nonumber\\
        &= \frac{t-1}{t}  g_{\textrm{mi}}^{(t-1)} +  \frac{1}{t} \left[ g + H\sum_{t'=1}^{t-1} \alpha g^{(t')}_{\text{mi}}  \right] 
        \quad // \quad \text{According to Eq.~\eqref{eq:mi_update}} 
        \nonumber\\
        &= \frac{t-1}{t}   \sum_{s=0}^{t-2} a_{s}^{(t-1)} H^s g +  \frac{1}{t} \left[ g +\alpha \sum_{t'=1}^{t-1}  \sum_{s=0}^{t'-1} a_{s}^{(t')} H^{s+1} g  \right] 
         \nonumber\\
        &= \frac{t-1}{t}   \sum_{s=0}^{t-2} a_{s}^{(t-1)} H^s g +  \frac{1}{t} \left[ g +\alpha \sum_{t'=1}^{t-1}  \sum_{s'=1}^{t'} a_{s'-1}^{(t')} H^{s'} g  \right] 
        \quad // \quad s'=s+1
         \nonumber\\
        &= \frac{t-1}{t}   \sum_{s=0}^{t-2} a_{s}^{(t-1)} H^s g +  \frac{1}{t} \left[ g +\alpha  \sum_{s'=1}^{t-1} \sum_{t'=s'}^{t-1}  a_{s'-1}^{(t')} H^{s'} g  \right] 
        \quad // \quad \text{Swapping the order of summations }
        \nonumber\\
        &= \frac{t-1}{t} \left[g + \sum_{s=1}^{t-2} a_{s}^{(t-1)} H^s g  \right]+ \frac{1}{t} \left[ g +\alpha  \sum_{s'=1}^{t-1} \sum_{t'=s'}^{t-1}  a_{s'-1}^{(t')} H^{s'} g  \right]
        \nonumber\\
        & = g +  \sum_{s=1}^{t-2} \left[ \frac{t-1}{t} a_{s}^{(t-1)}  + \frac{1}{t}  \sum_{t'=s}^{t-1}  \alpha\cdot a_{s-1}^{(t')}  \right] H^s g + \frac{\alpha }{t} \cdot a_{t-2}^{(t-1)} H^{t-1} g   
        \nonumber\\
         & =H^0 g +  \sum_{s=1}^{t-2} \left[ \frac{t-1}{t} a_{s}^{(t-1)}  + \frac{1}{t}  \sum_{t'=s}^{t-1}  \alpha\cdot a_{s-1}^{(t')}  \right] H^s g + \frac{\alpha }{t} \cdot a_{t-2}^{(t-1)} H^{t-1} g   
        \nonumber\\
        & =a_{0}^{(t)} H^0 g +  \sum_{s=1}^{t-2} a_{s}^{(t)} H^s g + a_{t-1}^{(t)} H^{t-1} g  
        \nonumber\\
        & =  \sum_{s=0}^{t-1} a_{s}^{(t)} H^s g,
    \end{align}
   
    where
    \begin{equation}
        a_{s}^{(t)} =
        \begin{cases}
            1 \ge 0, \quad &\text{if $s=0$} \\
            \frac{t-1}{t} a_{s}^{(t-1)}  + \frac{1}{t}  \sum_{t'=s}^{t-1}  \alpha a_{s-1}^{(t')} \ge 0, \quad &\text{if $0 < s < t-1$} \\
            \frac{\alpha }{t} a_{t-2}^{(t-1)} \ge 0, \quad &\text{if $s=t-1$}
        \end{cases}
    \end{equation}
    \textit{Conclusion}: Because both the base case and the inductive step have been proven to be true, we have $\forall 1\le t \le m,  g^{(t)}_{\textrm{mi}} = \sum_{s=0}^{t-1} a_{s}^{(t)} H^s g$, where $a_{s}^{(t)} \ge 0$.

    Thus, Lemma~\ref{lemma:mi-form} is proven.
\end{proof}

\begin{corollary} \label{corollary:mi-eigen-form}
    Let $\lambda_i \in \mathbb{R}$ and $v_i\in \mathbb{R}^{n}$ denote the $i$-th eigenvalue and its corresponding eigenvector of the Hessian matrix $H=\nabla_x^2 \textit{Loss}(x)$, respectively.
    Based on Lemma~\ref{lemma:mi-form},  the gradient $ g^{(t)}_{\textrm{mi}}$~\textit{w.r.t.} the MI Attack in the $t$-th step can be re-written as  $g^{(t)}_{\textrm{mi}} = \sum_{i=1}^n A^{(t)}_{ii} \gamma_i v_i$, where $A^{(t)}_{ii} = \sum_{s=0}^{t-1} a_{s}^{(t)} \lambda_i^s \ge 0$, and $\gamma_i = g^T v_i \in \mathbb{R}$ denotes the projection of the gradient $g = \nabla_{x} \textit{Loss}(x)$.
\end{corollary}

\begin{proof}
    According to  Lemma~\ref{lemma:mi-form}, the gradient $ g^{(t)}_{\textrm{mi}}$ can be written as
    \begin{align} \label{eq:mi-eigen-form}
        g^{(t)}_{\textrm{mi}} &= \sum_{s=0}^{t-1} a_{s}^{(t)} H^s g \nonumber 
        \quad // \quad\text{According to Lemma~\ref{lemma:mi-form}}
        \\
        &= \sum_{s=0}^{t-1} a_{s}^{(t)} \left[V \Lambda V^T \right]^s g 
        \quad // \quad\text{According to Eq.~\eqref{eq:H-decompose}} \nonumber\\
        &= V \left[\sum_{s=0}^{t-1} a_{s}^{(t)} \Lambda^s \right] V^T g.
       \quad // \quad  a_{s}^{(t)} \in \mathbb{R}.
    \end{align}

    For simplicity, let {$A^{(t)} \in \mathbb{R}^{n\times n}$} denote the term $\sum_{s=0}^{t-1} a_{s}^{(t)} \Lambda^s$ in Eq.~\eqref{eq:mi-eigen-form},~\textit{i.e.,}
{$A^{(t)}=\sum_{s=0}^{t-1} a_{s}^{(t)} \Lambda^s$}.
Because  $\Lambda$ is a diagonal matrix, matrix {{$A^{(t)}$}} is also a diagonal matrix, {\emph{i.e.} $\forall i\ne j, A^{(t)}_{ij}=0$, and $A^{(t)}_{ii} = \sum_{s=0}^{t-1} a_{s}^{(t)} \lambda_i^s  $}.

In this way, the gradient $ g^{(t)}_{\textrm{mi}}$~\textit{w.r.t.} the $t$-th step MI Attack in Eq.~\eqref{eq:mi-eigen-form} can be further written as
\begin{align}
    g^{(t)}_{\textrm{mi}} &=  V  A^{(t)} V^T g \nonumber \\
     &= [v_1,v_2,\ldots,v_n]
    \text{diag}(A^{(t)}_{11}, A^{(t)}_{22}, \dots, A^{(t)}_{nn})
      [v_1,v_2,\ldots,v_n]^T \sum_{i=1}^n \gamma_i v_i \nonumber \\
&=
\left[v_1 A^{(t)}_{11}, v_2 A^{(t)}_{22},\ldots,v_n A^{(t)}_{nn} \right]
\begin{bmatrix}
     \gamma_1 \\
     \gamma_2 \\
    \vdots \\
    \gamma_n \\
    \end{bmatrix} \nonumber\\
& = \sum_{i=1}^n A_{ii}^{(t)} \gamma_i v_i.
\end{align}

Here, the term $A^{(t)}_{ii} = \sum_{s=0}^{t-1} a_{s}^{(t)} \lambda_i^s \ge 0$ is positive, because both the term $a_{s}^{(t)} \ge 0$ (according to Lemma~\ref{lemma:mi-form}) and the eigenvalue $\lambda_i \ge 0$ (according to Corollary~\ref{corollaryappendix:semi-definite}) are positive.

Thus, Corollary~\ref{corollary:mi-eigen-form} is proven.
\end{proof}
$ $\newline

\textbf{The perturbation generated by the MI Attack.}
According to Corollary~\ref{corollary:mi-eigen-form}, 
the perturbation generated by the $m$-step MI Attack in Eq.~\eqref{eq:mi_update} can be written as
\begin{equation} \label{eq:mi-perturbation}
    \begin{split}
        \delta^{(\text{mi})}  \triangleq  \delta_{\textrm{mi}}^{(m)} &= \sum_{t=1}^m \alpha g^{(t)}_{\textrm{mi}}
        \quad // \quad\text{According to Eq.~\eqref{eq:mi_update}}
         \\
        &= \sum_{t=1}^m \alpha  \sum_{i=1}^n A_{ii}^{(t)} \gamma_i v_i 
          \quad // \quad\text{According to Corollary~\ref{corollary:mi-eigen-form}}
        \\
        &= \sum_{i=1}^n R_{ii}^{(m)} \gamma_i v_i,
    \end{split}
\end{equation}
where $R_{ii}^{(m)} =  \sum_{t=1}^m \alpha A_{ii}^{(t)}$.

\textbf{The pairwise interactions inside perturbation $ \delta^{(\text{mi})}$.}
Based on perturbations of the MI Attack $ \delta^{(\text{mi})}$, the interaction between each pair of perturbation units  can be computed as
\begin{equation} \label{eq: mi-pairwise-interaction}
 I_{ab}(\delta^{(\text{mi})}) =(\delta^{(\text{mi})})_{a}^T H_{ab}\delta^{(\text{mi})}_{b}.
 \quad // \quad \text{According to Eq.~\eqref{eq:shapley-interaction-index_2} in Appendix~\ref{append:interaction}}.
 \end{equation}
 
 In this scenario, the sum of interactions among all pairs of perturbation units $(a,b)$ in $\delta^{(\text{mi})}$ is given as follows.
    \begin{equation} \label{eq:mi_interaction}
        \begin{aligned}
            \sum_{a,b\in \Omega} I_{ab}(\delta^{(\text{mi})})&=(\delta^{(\text{mi})})^T H\delta^{(\text{mi})} \\
            &= (\sum_{i=1}^n R^{(m)}_{ii} \gamma_{i} v_{i}^T)(\sum_{j=1}^n \lambda_{j} v_{j} v_{j}^T) (\sum_{k=1}^n  R^{(m)}_{kk} \gamma_{k} v_{k})
             \quad // \quad \text{According to Eq.~\eqref{eq:mi-perturbation} and Eq.~\eqref{eq:H-decompose} }
            \\
            &=\left[\sum_{i=1}^n \underbrace{ ( R^{(m)}_{ii} \gamma_{i} v_{i}^T)( \lambda_{i} v_{i} v_{i}^T) }_{\text{$ v_{i}^T v_{i}=1$}}+ \sum_{1\le i,j\le n \atop i\ne j} \underbrace{( R^{(m)}_{ii} \gamma_{i} v_{i}^T)( \lambda_{j} v_{j} v_{j}^T)}_{\text{if $i\ne j, v_{i}^T v_{j}=0$}} \right] (\sum_{k=1}^n R^{(m)}_{kk} \gamma_{k} v_{k}) \\
            &=(\sum_{i=1}^n \lambda_{i}  R^{(m)}_{ii} \gamma_{i} v_{i}^T ) (\sum_{k=1}^n  R^{(m)}_{kk} \gamma_{k} v_{k})\\
            &=\sum_{i=1}^n \underbrace{ (\lambda_i  R^{(m)}_{ii} \gamma_{i} v_{i}^T)(  R^{(t)}_{ii} \gamma_{i} v_{i})}_{\text{$ v_{i}^T v_{i}=1$}}+ \sum_{1\le i, k \le n \atop i\ne k} \underbrace{ (\lambda_{i}  R^{(m)}_{ii} \gamma_{i} v_{i}^T)(  R^{(m)}_{kk} \gamma_{k} v_{k})}_{\text{if $i\ne k, v_{i}^T v_{k}=0$}} \\
            &= \sum_{i=1}^n \lambda_i( R^{(m)}_{ii}\gamma_{i})^2.
        \end{aligned}
    \end{equation}

\subsection{Multi-step attack}
In this subsection, we revisit the perturbation of the multi-step attack, which is defined in Appendix~\ref{append:delta-closed-form}.

According to Eq.~\eqref{eq:multi-step-attack}, the perturbation generated by the $m$-step attack can be represented as 
\begin{equation} \label{eq:multi-step-attack-in-mi}
\delta^{(\text{multi})}  \triangleq  \delta^{(m)} 
= \sum_{t=1}^m \alpha g^{(t)}_{\textrm{multi}}.
\end{equation}

Here, the gradient for the $t$-th step attack is defined as 
\begin{align}
    g^{(t)}_{\textrm{multi}} \triangleq \nabla_{x'}   \textit{Loss}(x' = x + \delta^{(t-1)}).
\end{align}

\begin{corollary} \label{corollary:multi-eigen-form}
    Let $\lambda_i \in \mathbb{R}$ and $v_n\in \mathbb{R}^{n}$ denote the $i$-th eigenvalue and its corresponding eigenvector of the Hessian matrix $H=\nabla_x^2 \textit{Loss}(x)$, respectively.
    Based on Lemma~\ref{lemma:induction},  the gradient $ g^{(t)}_{\textrm{mi}}$ for the $t$-th step attack
    can be written as  $g^{(t)}_{\textrm{multi}} = \sum_{i=1}^n B^{(t)}_{ii} \gamma_i v_i$, where $B^{(t)}_{ii} = (1+\lambda_i)^{t-1} \in \mathbb{R}$, and $\gamma_i = g^T v_i\in\mathbb{R}$ represents the projection of the gradient {$g=\nabla_x \textit{Loss}(x)$} on the eigenvector $v_i$.
\end{corollary}
\begin{proof}
    According to Eq.~\eqref{eq:delta_multi}, 
    the update of the perturbation at each step $t$ is represented as
    \begin{equation}
        \begin{split}
        \label{eq:g_pre}
            \Delta x^{(t)} \triangleq \alpha\cdot \nabla_{x'}   \textit{Loss}(x' = x + \delta^{(t-1)}) 
            &= \alpha\cdot \left[I + \alpha H \right]^{t-1}g
            \quad // \quad\text{According to Lemma~\ref{lemma:induction}}\\
            \Rightarrow \quad \nabla_{x'}   \textit{Loss}(x' = x + \delta^{(t-1)}) &= \left[I + \alpha H \right]^{t-1}g.
        \end{split}
    \end{equation}

    In this way, according to Eq.~\eqref{eq:g_pre}, the gradient $ g^{(t)}_{\textrm{mi}}$ for the $t$-th step attack can be written as
       \begin{equation}
        \begin{split}
        \label{eq:g_pre1}
            g^{(t)}_{\textrm{multi}} \triangleq \nabla_{x'}   \textit{Loss}(x' = x + \delta^{(t-1)}) & = \left[I + \alpha H \right]^{t-1}g
             \quad // \quad\text{According to Eq.~\eqref{eq:g_pre}}
             \\
            &=  \left[VV^T + \alpha V \Lambda V^T \right]^{t-1}g \quad // \quad\text{According to Eq.~\eqref{eq:H-decompose}} \\
            &= \left[ V \left(  I + \alpha \Lambda \right) V^T\right]^{t-1} g \\
            &= V \left(  I + \alpha \Lambda \right)^{t-1} V^T g. \\
        \end{split}
    \end{equation}

    For simplicity, let $B^{(t)} \in \mathbb{R}^{n\times n}$ denote the term $\left(  I + \alpha \Lambda \right)^{t-1}$ in Eq.~\eqref{eq:g_pre1},~\textit{i.e.,}
    $B^{(t)}=\left(  I + \alpha \Lambda \right)^{t-1}$.
    Because $\Lambda$ is a diagonal matrix, matrix $B^{(t)}$ is also a diagonal matrix, \emph{i.e.} $\forall i\ne j, B^{(t)}_{ij}=0$, and $B^{(t)}_{ii} =(1 + \alpha \lambda_i)^{t-1}$.

    Then, the gradient $ g^{(t)}_{\textrm{mi}}$ for the $t$-th step attack in Eq.~\eqref{eq:g_pre1} can be further re-written as
    \begin{align} \label{eq:g-multi-decompose}
        g^{(t)}_{\textrm{multi}} &=  V  B^{(t)} V^T g \nonumber \\
         &= [v_1,v_2,\ldots,v_n]
        \text{diag}(B^{(t)}_{11}, B^{(t)}_{22}, \dots, B^{(t)}_{nn})
          [v_1,v_2,\ldots,v_n]^T \sum_{i=1}^n \gamma_i v_i \nonumber \\
    &=
    \left[v_1 B^{(t)}_{11}, v_2 B^{(t)}_{22},\ldots,v_n B^{(t)}_{nn} \right]
    \begin{bmatrix}
         \gamma_1 \\
         \gamma_2 \\
        \vdots \\
        \gamma_n \\
        \end{bmatrix} \nonumber\\
    & = \sum_{i=1}^n B_{ii}^{(t)} \gamma_i v_i.
    \end{align}
    Thus, Corollary~\ref{corollary:multi-eigen-form} is proven.
\end{proof}
$\newline$

\textbf{The perturbation generated by the multi-step attack.}
According to Lemma~\ref{lemma:induction}, 
the perturbation generated by the $m$-step attack in Eq.~\eqref{eq:multi-step-attack-in-mi} can be written as
\begin{equation}
    \begin{split}
        \delta^{(\text{multi})}  \triangleq  \delta^{(m)} &= \sum_{t=1}^m \alpha g^{(t)}_{\textrm{multi}} 
        \quad // \quad\text{According to Eq.~\eqref{eq:g-multi-decompose}}
        \\
        &= \sum_{t=1}^m \alpha  \sum_{i=1}^n B_{ii}^{(t)} \gamma_i v_i 
         \quad // \quad\text{According to Corollary~\ref{corollary:multi-eigen-form}}
         \\
        &= \sum_{i=1}^n D_{ii}^{(m)} \gamma_i v_i,
    \end{split}
\end{equation}
where$D_{ii}^{(m)} =  \sum_{t=1}^m \alpha B_{ii}^{(t)}=\frac{(1 + \alpha \lambda_i)^m - 1}{\lambda_i}$.
$\newline$

 \textbf{Pairwise interactions inside $\delta^{(\text{multi})}$.}
       Based on the perturbation $\delta^{(\text{multi})}$ of the multi-step attack, the interaction between each pair of perturbation units $(a,b)$ can be computed as 
       \begin{equation}  \label{eq: multi-pairwise-interaction}
       I_{ab}(\delta^{(\textrm{multi})}) = \delta^{(\textrm{multi})}_a H_{ab} \delta^{(\textrm{multi})}_b.
       \quad // \quad \text{According to Eq.~\eqref{eq:shapley-interaction-index_2} in Appendix~\ref{append:interaction}}.
       \end{equation}
       Furthermore, we sum the interactions among all pairs of perturbation units $(a, b)$ as follows.
       \begin{equation} \label{eq:interaction-multi_1}
            \sum_{a, b\in\Omega}[{I_{a b}(\delta^{(\text{multi})})}] = \sum_{i=1}^n \lambda_{i}(D^{(m)}_{ii}\gamma_{i})^2.
            \quad // \quad \text{According to Eq.~\eqref{eq:delta-inf-interaction}}.
        \end{equation}

\subsection{Proof of Proposition~\ref{propositionappendix:mi}}
In order to fairly compare the MI Attack and the multi-step attack, we assume that the adversarial strength $\beta=\alpha m$ is the same.
Then, we prove the following proposition, which shows that perturbations generated by the MI Attack exhibit smaller interactions than perturbations generated by the multi-step attack.
\begin{propositionappendix}[Perturbations of the MI Attack exhibit smaller interactions than perturbations of the multi-step attack]
\label{propositionappendix:mi}
    {If the classification loss  for attacking $\textit{Loss}(x)$  is formulated as the cross-entropy loss in Eq.~\eqref{eq:softmax-ce}}, then we have $\sum_{a, b\in\Omega}[{I_{a b}(\delta^{(\text{mi})})}]\le\sum_{a, b\in\Omega}[{I_{a b}(\delta^{(\text{multi})})}]$.
\end{propositionappendix}

\begin{proof}
    According to Eq.~\eqref{eq:mi_interaction} and Eq.~\eqref{eq:interaction-multi_1}, the difference between interactions $\sum_{a, b\in\Omega}[{I_{a b}(\delta^{(\text{multi})})}]$ for the multi-step attack and interactions $\sum_{a, b\in\Omega}[{I_{a b}(\delta^{(\text{mi})})}]$ for the MI Attack is represented as
    \begin{align} \label{eq:diff-multi-mi}
        &\sum_{a, b\in\Omega}[{I_{a b}(\delta^{(\text{multi})})}]-\sum_{a, b\in\Omega}[{I_{a b}(\delta^{(\text{mi})})}] \nonumber \\
        &= (\delta^{(\text{multi})})^T H (\delta^{(\text{multi})})- (\delta^{(\text{mi})})^T H (\delta^{(\text{mi})})  
         \quad // \quad \text{According to Eq.~\eqref{eq: mi-pairwise-interaction} and Eq.~\eqref{eq: multi-pairwise-interaction}}
        \nonumber \\
    &= \sum_{i=1}^n \lambda_i (D_{ii}^{(m)} \gamma_i)^2 - \sum_{i=1}^n \lambda_i (R_{ii}^{(m)} \gamma_i)^2  
    \quad // \quad \text{According to Eq.~\eqref{eq:mi_interaction} and Eq.~\eqref{eq:interaction-multi_1}}
    \nonumber \\
    &= \sum_{i=1}^n \lambda_i \gamma_i^2 (D_{ii}^{(m)} -R_{ii}^{(m)} ) (D_{ii}^{(m)} + R_{ii}^{(m)}).
\end{align}

    Here, we prove that the difference is non-negative, because each term in Eq.~\eqref{eq:diff-multi-mi} is proven to be non-negative, as follows.

    \textbullet{ $\lambda_i \ge 0$.} According to Corollary~\ref{corollaryappendix:semi-definite}, the Hessian matrix $H=\nabla_x^2 \textit{Loss}(x)$ is positive semi-definite.
    Thus, all eigenvalues of the Hessian matrix $H$ are non-negative, \emph{i.e.} $\forall i, \lambda_i\ge 0$.

    \textbullet{ $ \gamma_i^2 \ge 0$.}

    \textbullet{ $D_{ii}^{(m)} +R_{ii}^{(m)} \ge 0$.} 
    This is because both the term $D_{ii}^{(m)} \ge 0$ and the term $R_{ii}^{(m)} \ge 0$ are positive.
    Specifically, according to Corollary~\ref{corollary:multi-eigen-form}, the term $D_{ii}^{(m)}=\alpha\sum_{t=1}^{m}  B_{ii}^{(t)} = \alpha\sum_{t=1}^{m}  (1  +\alpha \lambda_i)^{t-1} \ge 0$ is positive.
    Besides, according to Lemma~\ref{corollary:mi-eigen-form}, the term $R_{ii}^{(m)} = \alpha\sum_{t=1}^{m}  A_{ii}^{(t)} \ge 0$ is positive, because $A_{ii}^{(t)} \ge 0$.
    Thus, $D_{ii}^{(m)} +R_{ii}^{(m)} \ge 0$.

    \textbullet{ $D_{ii}^{(m)} -R_{ii}^{(m)} \ge 0$.} 
     The term $D_{ii}^{(m)} -R_{ii}^{(m)} = \alpha\sum_{t=1}^{m}(  B_{ii}^{(t)} -  A_{ii}^{(t)} ) \ge 0$ is positive, because the term $ B_{ii}^{(t)} -  A_{ii}^{(t)} \ge 0$ is positive, which is proven in Lemma~\ref{lemma:r-d} as follows.

\begin{lemma} \label{lemma:r-d}
    Given  $g^{(t)}_{\textrm{mi}} = \sum_{i=1}^n A^{(t)}_{ii} \gamma_i v_i$ and  $g^{(t)}_{\textrm{multi}} = \sum_{i=1}^n B^{(t)}_{ii} \gamma_i v_i$, we have $\forall i, A^{(t)}_{ii} \le B^{(t)}_{ii}$.
\end{lemma}

\begin{proof}
    We prove Lemma~\ref{lemma:r-d} by induction.

    \textit{Base case}:
    When the step $t=1$, we have   $g^{(1)}_{\textrm{mi}} = g^{(1)}_{\textrm{multi}}  = g$.
    Hence, $\forall i, A^{(1)}_{ii} = B^{(1)}_{ii}$.

    \textit{Inductive step}:
    For $1< t' < t$, we assume $\forall i, A^{(t')}_{ii} \le B^{(t')}_{ii} = (1 + \alpha \lambda_i)^{t'-1}$.
    Thus, at the step $t$, according to Eq.~\eqref{eq:our-revised-mi}, the gradient for the MI Attack can be re-written as
    \begin{align}
    \label{eq:g_mi_new}
        g_{\text{mi}}^{(t)} &= \mu g_{\textrm{mi}}^{(t-1)} + {(1-\mu)} \nabla_{x'}\textit{Loss}(x'=x+\delta^{(t-1)}_{\textrm{mi}}) 
        \quad // \quad \text{According to Eq.~\eqref{eq:our-revised-mi}}
        \nonumber \\
        &= \frac{t-1}{t}  g_{\textrm{mi}}^{(t-1)} +  \frac{1}{t}\nabla_{x'}\textit{Loss}(x'=x+\delta^{(t-1)}_{\textrm{mi}}) \nonumber \\
        &= \frac{t-1}{t}  g_{\textrm{mi}}^{(t-1)} +  \frac{1}{t} \left[ g + H\delta^{(t-1)}_{\textrm{mi}}  \right] \quad // \quad \text{According to Eq.~\eqref{eq:grad_taylor}} \nonumber \\
        &= \frac{t-1}{t}  g_{\textrm{mi}}^{(t-1)} +  \frac{1}{t} \left[ g + H\sum_{t'=1}^{t-1} \alpha \cdot g^{(t')}_{\text{mi}}  \right] \quad // \quad \text{According to Eq.~\eqref{eq:mi_update}} \nonumber \\
        &= \frac{t-1}{t}  g_{\textrm{mi}}^{(t-1)} +  \frac{1}{t} \left[ \sum_{i=1}^n \gamma_i v_i + \left(\sum_{i=1}^n \lambda_i v_i v_i^T\right)\sum_{t'=1}^{t-1} \alpha\cdot  g^{(t')}_{\text{mi}}  \right] \quad // \quad \text{According to Eq.~\eqref{eq:H-decompose} and Eq.~\eqref{eq:g-decompose} } \nonumber \\
        &= \frac{t-1}{t}  \sum_{i=1}^{n} A_{ii}^{(t-1)}\gamma_i v_i  +  \frac{1}{t} \left[ \sum_{i=1}^n \gamma_i v_i + \left(\sum_{i=1}^n \lambda_i v_i v_i^T\right)\sum_{t'=1}^{t-1} \alpha \sum_{i=1}^{n} A_{ii}^{(t')}\gamma_i v_i  \right]  \quad // \quad \text{According to Corollary~\ref{corollary:mi-eigen-form}} \nonumber \\
        &= \frac{t-1}{t}  \sum_{i=1}^{n} A_{ii}^{(t-1)}\gamma_i v_i  +  \frac{1}{t} \left[ \sum_{i=1}^n \gamma_i v_i +\sum_{t'=1}^{t-1} \alpha \sum_{i=1}^{n} \lambda_i A_{ii}^{(t')}\gamma_i v_i  \right] \nonumber \\
        &=  \sum_{i=1}^{n}   \frac{1}{t}\left[  (t-1) A_{ii}^{(t-1)} + 1 + \sum_{t'=1}^{t-1} \alpha \lambda_i A_{ii}^{(t')} \right]\gamma_i v_i.
    \end{align}

  Moreover, according to Corollary~\ref{corollary:mi-eigen-form}, the gradient for the $t$-step MI Attack can also be represented as $ g_{\text{mi}}^{(t)} = \sum_{i=1}^{n}  A_{ii}^{(t)}  \gamma_i v_i$.
  Hence, based on Corollary~\ref{corollary:mi-eigen-form} and Eq.~\eqref{eq:g_mi_new}, we obtain
   \begin{align}
   g_{\text{mi}}^{(t)} =  \sum_{i=1}^{n}  A_{ii}^{(t)}\gamma_i v_i 
   &= \sum_{i=1}^{n}   \frac{1}{t}\left[  (t-1) A_{ii}^{(t-1)} + 1 + \sum_{t'=1}^{t-1} \alpha \lambda_i A_{ii}^{(t')} \right]\gamma_i v_i 
   \quad // \quad \text{According to Eq.~\eqref{eq:g_mi_new} and Corollary~\ref{corollary:mi-eigen-form}}
   \nonumber \\
   \Rightarrow \quad A_{ii}^{(t)} &=  \frac{1}{t}  \left[   (t-1) A_{ii}^{(t-1)} + 1 + \sum_{t'=1}^{t-1} \alpha \lambda_i A_{ii}^{(t')} \right] 
   \nonumber \\
    &\le \frac{1}{t}  \left[   (t-1) B_{ii}^{(t-1)} + 1 + \sum_{t'=1}^{t-1} \alpha \lambda_i B_{ii}^{(t')} \right] \quad // \quad \text{For $1< t' < t$, $\forall i, A^{(t')}_{ii} \le B^{(t')}_{ii}$; \; $\alpha, \lambda_i \ge 0$. }
     \nonumber \\
    &= \frac{1}{t}  \left[   (t-1) (1+\alpha \lambda_i)^{t-2} + 1 + \sum_{t'=1}^{t-1} \alpha \lambda_i (1  + \alpha \lambda_i)^{t'-1} \right] \quad // \quad \text{$B^{(t')}_{ii} = (1 + \alpha \lambda_i)^{t'-1}$. }
     \nonumber \\
    &= \frac{1}{t}  \left[   (t-1) (1+\alpha \lambda_i)^{t-2} + 1 +  \alpha \lambda_i \frac{1 - (1+\alpha \lambda_i)^{t-1}}{-\alpha \lambda_i} \right] \nonumber \\
    &= (1+\alpha \lambda_i)^{t-2} \left[ \frac{t-1}{t} + \frac{1+\alpha \lambda_i}{t} \right]  \nonumber \\
    &= (1+\alpha \lambda_i)^{t-2} \left[ 1 + \frac{\alpha \lambda_i}{t} \right]  \nonumber \\
    &\le (1+\alpha \lambda_i)^{t-2} \left[1 + \alpha \lambda_i \right]
     \nonumber \\
    &= (1+\alpha \lambda_i)^{t-1} = B_{ii}^{(t)}.
   \end{align}

    \textit{Conclusion}: Because both the base case and the inductive step are proven to be true, we have $\forall i, A^{(t)}_{ii} \le B^{(t)}_{ii} = (1 + \alpha \lambda_i)^{t-1}$.

    Thus, Lemma~\ref{lemma:r-d} is proven.
\end{proof}

$ $\newline
In summary, the difference between interactions $\sum_{a, b\in\Omega}[{I_{a b}(\delta^{(\text{multi})})}]$ for the multi-step attack and interactions $\sum_{a, b\in\Omega}[{I_{a b}(\delta^{(\text{mi})})}]$ for the MI Attack in Eq.~\eqref{eq:diff-multi-mi} is proven to be positive.

\begin{align}
        &\sum_{a, b\in\Omega}[{I_{a b}(\delta^{(\text{multi})})}]-\sum_{a, b\in\Omega}[{I_{a b}(\delta^{(\text{mi})})}] \nonumber \\
    &= \sum_{i=1}^n \underbrace{\lambda_i}_{\ge 0} \underbrace{\gamma_i^2}_{\ge 0} \underbrace{(\alpha\sum_{t=1}^{m}  B_{ii}^{(t)} -  A_{ii}^{(t)} )}_{\ge 0 } \underbrace{ (\alpha\sum_{t=1}^{m}  B_{ii}^{(t)} +  A_{ii}^{(t)} )}_{\ge 0}  \nonumber \\
    & \ge 0
\end{align}

Thus, Proposition~\ref{pro:mi} in Section~\ref{sec:explaining} is proven.
\end{proof}

\section{Proof that the LinBP Attack generates perturbations with smaller interactions than those of the multi-step attack}
\label{append:lbp}

In this section, we prove Proposition~\ref{theoremappendix:linbp} in Section~\ref{sec:explaining}, which shows that perturbations of the LinBP Attack~\cite{guo2020backpropagating}  exhibit smaller interactions than those of the multi-step attack.

Before the proof of Proposition~\ref{theoremappendix:linbp}, we first revisit the LinBP Attack.

\subsection{LinBP Attack}

The LinBP Attack~\cite{guo2020backpropagating} boosts the adversarial transferability via  \textit{linear back-propagation}.
In contrast to standard back-propagation, linear back-propagation is referred to as the process in which the gradient back-propagates linearly as if  no activation function is encountered.
In addition, in the LinBP Attack, the forward propagation is conducted in the same way as the traditional forward propagation.

\textbf{Forward propagation.}
In \cite{guo2020backpropagating}, Guo~\textit{et al} consider a simplified stacked neural network, without bias terms.
Specifically, given an input sample $x\in\mathbb{R}^n$, the network output $z$ of a DNN with $d$ layers is given as follows.
\begin{equation}
    z = W_d^T \sigma(W_{d-1}^T\dots \sigma( W_1^T x)),
\end{equation}
where $\sigma(\cdot)$ denotes the activation function.

If ReLU is selected as the activation, as discussed in Appendix~\ref{appendix:positive-semi-definite}, the network output $z$ can be re-written as follows.
\begin{equation}
    z = W_d^T \Sigma_{d-1} W_{d-1}^T  \dots \Sigma_{1} W_1^T x,
\end{equation}
where $\Sigma_l = diag\bigg((\Sigma_l)_{11},(\Sigma_l)_{22},\cdots,(\Sigma_l)_{n_l n_l}\bigg) \in \mathbb{R}^{n_l\times n_l}$ represents the activation state of the $l$-th ReLU layer, $(\Sigma_l)_{ii}\in\{0, 1\}$.

\textbf{Standard back-propagation.}
Given the classification loss  $\textit{Loss}(x)$, the gradient of the loss function $\nabla_x \textit{Loss}(x)$~\textit{w.r.t.} the input sample $x$ is computed as follows.
\begin{equation} \label{eq:std-bp}
\begin{aligned}
    \nabla_x \textit{Loss}(x) &= \left(\frac{\partial \textit{Loss}}{\partial z^T} \frac{\partial z} {\partial x^T}\right)^T
    \\
    & = \left(\frac{\partial \textit{Loss}}{\partial z^T} W_d^T \Sigma_{d-1} W_{d-1}^T  \dots \Sigma_{1} W_1^T \right)^T.
\end{aligned}
\end{equation}

\textbf{Linear back-propagation.}
Compared with standard back-propagation, linear back-propagation in the LinBP Attack~\cite{guo2020backpropagating},  is computed by skipping all ReLU functions between the $k$-th  and the $d$-th layer.
In this way, the gradient computed with linear back-propagation is given as follows.
\begin{equation} \label{eq:linear-bp}
\begin{aligned}
    \tilde{\nabla}_x \textit{Loss} \triangleq  \left(\frac{\partial \textit{Loss}}{\partial z^T} \left(\prod_{l=k}^{d} W_{l}^T \right) \left(\prod_{l=1}^{k-1} \Sigma_{l} W_{l}^T\right) \right)^T.
\end{aligned}
\end{equation}

\textbf{LinBP Attack.}
Based on linear back-propagation, the adversarial perturbation generated after $t$-step LinBP Attack~\cite{guo2020backpropagating} is represented as follows.
\begin{equation} 
\begin{aligned}
\label{eq:ori-linbp}
    \delta^{(t)}_{\text{lbp-ori}} =& \operatorname{clip}\left\{\delta^{(t-1)}_{\text{lbp-ori}} + \alpha \cdot \operatorname{sign}(g^{(t-1)}_{\text{lbp-ori}})\right\},
    \\
    g^{(t-1)}_{\text{lbp-ori}} &\triangleq \tilde{\nabla}_{x'} \textit{Loss}(x'=x + \delta^{(t-1)}_{\text{lbp-ori}}).
 \end{aligned}
\end{equation}

For the LinBP Attack in Eq.~\eqref{eq:ori-linbp}, both the sign operation and the clip operation are used to constrain the adversarial perturbation $\delta^{(t)}_{\text{lbp-ori}}$.
In fact, the sign operation, and the clip operation are common operations in generating adversarial perturbations.
Therefore, the distinctive technique in the LinBP Attack is linear back-propagation $g^{(t-1)}_{\text{lbp-ori}}$, which replaces the standard backward propagation to craft the perturbation.
Thus, to simplify the proof, we ignore the sign operation and the clip operation, without hurting the trustworthiness of the analysis of the LinBP Attack.
In this way, we focus on a simplified version of the LinBP Attack, as follows.
\begin{equation} 
\begin{aligned}
\label{eq:linbp}
    \delta^{(t)}_{\text{lbp}} &= \delta^{(t-1)}_{\text{lbp}} + \alpha \cdot g^{(t-1)}_{\text{lbp}}
    \\
    g^{(t-1)}_{\text{lbp}} &\triangleq \tilde{\nabla}_{x'} \textit{Loss}(x'=x + \delta^{(t-1)}_{\text{lbp}}).
\end{aligned}
\end{equation}

\textbf{Perturbation generated by the LinBP Attack.}
We use the infinite-step adversarial attack with an infinitesimal step size to simplify the LinBP Attack.
Let $\beta = \alpha m$.
Based on Eq.~\eqref{eq:linbp}, the perturbation generated by the LinBP Attack can be calculated as
\begin{equation} \label{eq:lbp-infty}
    \delta^{(\text{lbp})} =\lim_{m\rightarrow +\infty} \delta_{\textrm{lbp}}^{(m)}.
\end{equation}

Note that the eigenvectors $\{v_1,v_2,\ldots,v_n\}$ of the Hessian matrix $H=\nabla_x^2 \textit{Loss}(x)$ can be regarded as an orthonormal basis.
In addition, note that $H$ is calculated for the gradient in Eq.~\eqref{eq:std-bp}, instead of the gradient in Eq.~\eqref{eq:linear-bp}.
Moreover, the Hessian matrix calculated based on Eq.~\eqref{eq:linear-bp} is not the real Hessian matrix of LinBP, because the forward propagation and back propagation in LinBP correspond to two different functions.
In this way, the perturbation $\delta^{(\text{lbp})}$ of the $m$-step LinBP Attack in Eq.~\eqref{eq:lbp-infty} can be re-written as
\begin{equation} \label{eq:perturbation-lbp}
    \delta^{(\text{lbp})} = \sum_{k=1}^n \gamma^{(\text{lbp})}_k v_k,
\end{equation}
where $\gamma^{(\text{lbp})}_k = (\delta^{(\text{lbp})})^T v_k \in \mathbb{R}$.

\textbf{Pairwise interactions inside perturbation $\delta^{(\text{lbp})}$.}
Based on perturbations of the LinBP Attack $\delta^{(\text{lbp})}$, the interaction between each pair of perturbation units $(a,b)$ can be computed as
\begin{equation} \label{eq:pair-inte-linbp}
I_{ab}(\delta^{(\text{lbp})}) = \delta^{(\text{lbp})}_{a} H_{ab} \delta^{(\text{lbp})}_{b}.
\quad // \quad \text{According to Eq.~\eqref{eq:shapley-interaction-index_2} in Appendix~\ref{append:interaction}}.
\end{equation}

In this scenario, the sum of interactions among all pairs of perturbation units $(a,b)$ in $\delta^{(\text{lbp})}$ is given as follows.
\begin{equation}  \label{eq:lbp-interaction}
    \begin{split}
        &\sum_{a, b\in \Omega} I_{ab}(\delta^{(\text{lbp})}) \\
        &= (\delta^{(\text{lbp})})^T H \delta^{(\text{lbp})}
        \quad // \quad \text{According to Eq.~\eqref{eq:pair-inte-linbp}}
        \\
        &= \left( \sum_{i=1}^n \gamma^{(\text{lbp})}_{i} v_{i}^T\right) \left(\sum_{j=1}^n \lambda_j v_j v_j^T \right) \left( \sum_{k=1}^n \gamma^{(\text{lbp})}_{k} v_{k}\right)
        \quad // \quad \text{According to Eq.~\eqref{eq:H-decompose} and Eq.~\eqref{eq:perturbation-lbp}}
        \\
        &=\left[\sum_{i=1}^n \underbrace{ ( \gamma^{(\text{lbp})}_{i} v_{i}^T)( \lambda_{i} v_{i} v_{i}^T) }_{\text{$ v_{i}^T v_{i}=1$}}+ \sum_{1\le i,j\le n \atop i\ne j} \underbrace{( \gamma^{(\text{lbp})}_{i} v_{i}^T)( \lambda_{j} v_{j} v_{j}^T)}_{\text{if $i\ne j, v_{i}^T v_{j}=0$}} \right] (\sum_{k=1}^n \gamma^{(\text{lbp})}_{k} v_{k}) \\
        &=(\sum_{i=1}^n \lambda_{i}\gamma^{(\text{lbp})}_{i} v_{i}^T ) ( \sum_{k=1}^n \gamma^{(\text{lbp})}_{k} v_{k})\\
        &=\sum_{i=1}^n \underbrace{ ( \lambda_{i}\gamma^{(\text{lbp})}_{i} v_{i}^T)( \gamma^{(\text{lbp})}_{i} v_{i})}_{\text{$ v_{i}^T v_{i}=1$}}+ 
        \sum_{1\le i, k \le n \atop i\ne k} \underbrace{ ( \lambda_{i}\gamma^{(\text{lbp})}_{i} v_{i}^T)(\gamma^{(\text{lbp})}_{k} v_{k})}_{\text{if $i\ne k, v_{i}^T v_{k}=0$}} \\
            &= \sum_{i=1}^n \lambda_{i} (\gamma^{(\text{lbp})}_{i})^2.
    \end{split}
\end{equation}

\subsection{Multi-step attack}
In this subsection, we revisit the perturbation of the multi-step attack, which is defined in Appendix~\ref{append:delta-closed-form}.

\textbf{Perturbation generated by the multi-step attack.}
According to Eq.~\eqref{eq:multi-step-attack}, the perturbation generated by the $\infty$-step attack can be represented as 
\begin{equation} 
\label{eq:multi-in-linbp-1}
\delta^{(\text{multi})}  \triangleq  \delta^{(\infty)} 
=  \sum_{i=1}^n D_{ii} \gamma_i v_i.
\end{equation}
where $\beta=\alpha m$; $\lambda_{i}$ and $v_i$ denote the $i$-th eigenvalue of the Hessian matrix $H=\nabla_x^2 \textit{Loss}(x)$ and its corresponding eigenvector, respectively.
If $\lambda_i\ne0$, $D_{ii}=\frac{\exp(\beta\lambda_i)-1}{\lambda_i}$; otherwise, $D_{ii}=\beta$. 
$\gamma_i = g^T v_i \in \mathbb{R}$ represents the projection of the gradient $g=\nabla_x \textit{Loss}(x)$ on the eigenvector $v_i$.
$\newline$

 \textbf{Pairwise interaction inside $\delta^{(\text{multi})}$.}
       Based on the perturbation $\delta^{(\text{multi})}$ of the multi-step attack, the interaction between each pair of perturbation units $(a,b)$ can be computed as 
       \begin{equation}
       I_{ab}(\delta^{(\textrm{multi})}) = \delta^{(\textrm{multi})}_a H_{ab} \delta^{(\textrm{multi})}_b.
       \quad // \quad \text{According to Eq.~\eqref{eq:shapley-interaction-index_2} in Appendix~\ref{append:interaction}}.
       \end{equation}
       Furthermore, we sum the interactions among all pairs of perturbation units $(a, b)$ as follows.
       \begin{equation} \label{eq:interaction-multi-in-lbp-1}
            \sum_{a, b\in\Omega}[{I_{a b}(\delta^{(\text{multi})})}] = \sum_{i=1}^n \lambda_{i}(D_{ii}\gamma_{i})^2.
            \quad // \quad \text{According to Eq.~\eqref{eq:delta-inf-interaction}}.
        \end{equation}

\textbf{Fair comparisons between the LinBP Attack and the multi-step attack.}
In order to conduct fair comparisons between interactions $\sum_{a, b\in\Omega}[{I_{a b}(\delta^{(\text{multi})})}]$ based on perturbations $\delta^{(\text{multi})}$ and interactions 
$\sum_{a, b\in\Omega}[{I_{a b}(\delta^{(\text{lbp})})}]$ based on perturbations $\delta^{(\text{lbp})}$, we should control the magnitudes of these two perturbations to be the same level.
Specifically, we control the $L_2$ norm of the perturbation $\delta^{(\text{lbp})}$ to be the same as that of the perturbation $\delta^{(\textrm{multi})}$,~\textit{i.e.,} $\|\delta^{(\textrm{multi})}\|_2 = \| \delta^{(\text{lbp})}\|_2$.
Note that the $L_2$ norm of the perturbation $\delta^{(\text{multi})}$ is determined by the adversarial strength $\beta = \alpha m$.
In this way, we can control $\|\delta^{(\textrm{multi})}\|_2 = \| \delta^{(\text{lbp})}\|_2$ by setting the adversarial strength $\beta = \alpha m$ of the multi-step attack to a proper value.
$ $\newline

\subsection{Proof of Proposition~\ref{theoremappendix:linbp}}

In this subsection, we prove Proposition~\ref{theoremappendix:linbp} in Section~\ref{sec:explaining} under the condition that $\|\delta^{(\textrm{multi})}\|_2 = \| \delta^{(\text{lbp})}\|_2$. 
To simplify the story, we prove Proposition~\ref{theoremappendix:linbp} based on binary classification.
In this scenario, we have the following lemma.
$ $\newline

\begin{lemma} \label{lemma:H-binary}
    Based on Assumption~\ref{assumption:activation}, if the loss function is formulated as the cross-entropy loss on the sigmoid function, then the forward propagation can be written as $z =\tilde{w}_x^T x + b\in\mathbb{R}$, where $\tilde{w}_x^T =  W_d^T \Sigma_{d-1} W_{d-1}^T \dots \Sigma_{1} W_{1}^T$.
    Then, the Hessian matrix $H=\nabla^2_x \textit{Loss}(x)$ has a single positive eigenvalue and $n-1$ zero eigenvalues.
    Without loss of generality, let $\lambda_1 > 0$ and $\forall 2\le i \le n, \lambda_i=0$.
    The value of $\lambda_1$  is given as $\lambda_1 = \frac{\partial^2 \textit{Loss}(x)}{\partial z^2}\|\tilde{w}_x\|_2^2$.
    The eigenvector $v_1$ is give as $v_1=\frac{\tilde{w}_x}{\|\tilde{w}_x\|_2}$.
\end{lemma}
\begin{proof}
According to Lemma~\ref{lemma:hessian-closed-form}, the Hessian matrix $H=\nabla_x^2 \textit{Loss}(x)$ is given as follows.
    \begin{equation}
    \begin{split}
    \label{eq:hessian-in-lbp-1}
        H = \nabla_x^2 \textit{Loss}(x) &= g_z \nabla_z^2 \textit{Loss}(x) g_z^T
        \\
        &= \frac{\partial z}{\partial x} \frac{\partial^2 \textit{Loss}(x)}{\partial z^2} (\frac{\partial z}{\partial x})^T
        \\
        &= \tilde{w}_x \frac{\partial^2 \textit{Loss}(x)}{\partial z^2} \tilde{w}_x^T.
    \end{split}
    \end{equation}
    
    Note that $\frac{\partial^2 \textit{Loss}(x)}{\partial z^2}\in\mathbb{R}$ is a scalar, and $\tilde{w}_x\in\mathbb{R}^n$ is a vector.
Thus, if $\tilde{w}_x \ne\mathbf{0}$, then $\text{rank}(H)=1$.
Moreover, according to Corollary~\ref{corollaryappendix:semi-definite}, $H$ is positive semi-definite.
Thus, $H$ has a single positive eigenvalue and $n-1$ zero eigenvalues.

Then, based on Eq.~\eqref{eq:hessian-in-lbp-1}, eigenvalue $\lambda_1$ and its corresponding eigenvector $v_1$ of $H$ can be represented as
\begin{equation}
\begin{aligned}
\label{eq:hessian_lambda_v_lbp}
    H &= \tilde{w}_x\frac{\partial^2 \textit{Loss}(x)}{\partial z^2} \tilde{w}_x^T \\
    \Rightarrow \quad H \frac{\tilde{w}_x}{\|\tilde{w}_x\|_2}& = \tilde{w}_x \frac{\partial^2 \textit{Loss}(x)}{\partial z^2} \tilde{w}_x^T \frac{\tilde{w}_x}{\|\tilde{w}_x\|_2} =  \underbrace{\frac{\partial^2 \textit{Loss}(x)}{\partial z^2}\|\tilde{w}_x\|_2^2}_{\lambda_1} \underbrace{\frac{\tilde{w}_x}{\|\tilde{w}_x\|_2}}_{v_1}.
    \\
    \Rightarrow \quad 
    \lambda_1 &= \frac{\partial^2 \textit{Loss}(x)}{\partial z^2}\|\tilde{w}_x\|_2^2,\;\;
    v_1 = \frac{\tilde{w}_x}{\|\tilde{w}_x\|_2}.
\end{aligned}    
\end{equation}
\end{proof}
$ $\newline

\begin{propositionappendix}[Perturbations of the LinBP Attack exhibit smaller interactions than those of the multi-step attack]
    \label{theoremappendix:linbp}
    {If the classification loss  for attacking $\textit{Loss}(x)$  is formulated as the cross-entropy loss on the sigmoid function for binary classification in Eq.~\eqref{eq:sigmoid-ce}}, we use the infinite-step adversarial attack with an infinitesimal step size to simplify the LinBP Attack and the multi-step attack.
  Then under the condition that $\|\delta^{(\textrm{multi})}\|_2 = \| \delta^{(\text{lbp})}\|_2$,  we have $\sum_{a, b\in\Omega}[{I_{a b}(\delta^{(\text{lbp})})}]\le\sum_{a, b\in\Omega}[{I_{a b}(\delta^{(\text{multi})})}]$.
\end{propositionappendix}
$ $\newline

\begin{proof}
We prove Proposition~\ref{theoremappendix:linbp} based on binary classification.
In this scenario, according to Lemma~\ref{lemma:H-binary} the Hessian matrix $H=\nabla^2_x \textit{Loss}(x)$ has a single positive eigenvalue and $n-1$ zero eigenvalues.
Without loss of generality, let $\lambda_1 > 0$ and $\forall 2\le i \le n, \lambda_i=0$.
Then, the Hessian matrix $H$ can be written as
\begin{equation}
    H = \lambda_1 v_1 v_1^T.
\end{equation}
where $v_1$ denotes the eigenvector corresponding to the eigenvalue $\lambda_1$.

In this way, the interactions $\sum_{a, b\in \Omega} I_{ab}(\delta^{(\textrm{multi})})$ for the multi-step attack in Eq.~\eqref{eq:interaction-multi-in-lbp-1} can be re-written as
\begin{equation} \label{eq:multi-interaction-in-lbp}
    \begin{split}
        \sum_{a, b\in \Omega} I_{ab}(\delta^{(\textrm{multi})})
        & = \sum_{i=1}^n \lambda_{i}(D_{ii}\gamma_{i})^2
         \quad // \quad \text{According to Eq.~\eqref{eq:interaction-multi-in-lbp-1}}
        \\
        & =  \sum_{i=1}^n \lambda_{i}(\frac{\exp(\beta\lambda_i)-1}{\lambda_i}\gamma_{i})^2
        \\
        &= \lambda_1 \left( \frac{\exp(\beta\lambda_1)-1}{\lambda_1} \gamma_1\right)^2.
    \end{split}
\end{equation}

Similarly, the interactions $\sum_{a, b\in \Omega} I_{ab}(\delta^{(\textrm{lbp})})$ for the LinBP Attack in Eq.~\eqref{eq:lbp-interaction} can be re-written as
\begin{equation}  \label{eq:lbp-interaction-binary}
    \begin{split}
        \sum_{a, b\in \Omega} I_{ab}(\delta^{(\text{lbp})}) 
        & =  \sum_{i=1}^n \lambda_{i} (\gamma^{(\text{lbp})}_{i})^2
        \quad // \quad \text{According to Eq.~\eqref{eq:lbp-interaction}}
        \\
        & =\lambda_1 (\gamma^{(\text{lbp})}_1)^2.
    \end{split}
\end{equation}

Thus, according to Eq.~\eqref{eq:multi-interaction-in-lbp} and Eq.~\eqref{eq:lbp-interaction-binary}, the difference between interactions $\sum_{a, b\in \Omega} I_{ab}(\delta^{(\textrm{multi})})$ for the multi-step attack and interactions $\sum_{a, b\in \Omega} I_{ab}(\delta^{(\textrm{lbp})})$ for the LinBP Attack is computed as
\begin{equation} \label{eq:interaction-multi-lbp}
    \begin{split}
        &\sum_{a, b\in \Omega} I_{ab}(\delta^{(\textrm{multi})})- \sum_{a, b\in \Omega} I_{ab}(\delta^{(\text{lbp})}) \\
        & = \lambda_1 (\frac{\exp(\beta\lambda_1)-1}{\lambda_1} \gamma_1)^2 - \lambda_1 (\gamma^{(\text{lbp})}_1)^2
        \\
        & = \lambda_1 \left[(\frac{\exp(\beta\lambda_1)-1}{\lambda_1} \gamma_1)^2 - (\gamma^{(\text{lbp})}_1)^2 \right].
    \end{split}
\end{equation}

Here, we prove that the difference is non-negative, because each term in Eq.~\eqref{eq:interaction-multi-lbp} is proven to be non-negative, as follows.

\textbullet{ $\lambda_1 \ge 0$.} According to Corollary~\ref{corollaryappendix:semi-definite}, the Hessian matrix $H=\nabla_x^2 \textit{Loss}(x)$ is positive semi-definite.
Thus, the eigenvalue $\lambda_1$ is non-negative,~\emph{i.e.} $\lambda_1 \ge 0$.
$ $\newline    

\textbullet{ $\left[(\frac{\exp(\beta\lambda_1)-1}{\lambda_1} \gamma)^2 - (\gamma^{(\text{lbp})}_1)^2 \right] \ge 0$,} which is proven as follows.
Note that in order to ensure the fair comparison between the LinBP Attack and the multi-step attack, we generated perturbations $\delta^{(\text{lbp})}$ of the LinBP Attack by controling its $L_2$ norm to be the same as that of the multi-step attack,~\textit{i.e.,} $\|\delta^{(\textrm{multi})}\|_2 = \|\delta^{(\text{lbp})}\|_2 $.
We can control $\|\delta^{(\textrm{multi})}\|_2 = \| \delta^{(\text{lbp})}\|_2$ by setting the adversarial strength $\beta = \alpha m$ of the multi-step attack to a proper value.
\begin{equation} \label{eq:multi-lbp-norm-equal}
    \begin{split}
        \|\delta^{(\textrm{multi})}\|_2 &= \|\delta^{(\text{lbp})}\|_2 
        \\
        \Rightarrow\quad \| \sum_{i=1}^n D_{ii} \gamma_i v_i \|_2 &=\| \sum_{k=1}^n \gamma^{(\text{lbp})}_k v_k\|_2 
         \quad // \quad \text{According to Eq.~\eqref{eq:multi-in-linbp-1} and Eq.~\eqref{eq:perturbation-lbp}}
        \\
        \Rightarrow\quad \| \sum_{i=1}^n  \frac{\exp(\beta\lambda_i)-1}{\lambda_i} \gamma_i v_i \|_2 &=\| \sum_{k=1}^n \gamma^{(\text{lbp})}_k v_k\|_2 
        \\
       \Rightarrow\quad \| \sum_{i=1}^n  \frac{\exp(\beta\lambda_i)-1}{\lambda_i} \gamma_i v_i \|_2 &=\sqrt{\sum_{k=1}^n (\gamma^{(\text{lbp})}_k)^2}. 
        \quad // \quad v_k^{T}v_k=1.
     \end{split}
\end{equation}

\begin{lemma} \label{lemma:H-g-binary}
    Based on Assumption~\ref{assumption:activation}, if the loss function is formulated as the cross-entropy loss on the sigmoid function, then the forward propagation can be written as $z =\tilde{w}_x^T x + b\in\mathbb{R}$, where $\tilde{w}_x^T =  W_d^T \Sigma_{d-1} W_{d-1}^T \dots \Sigma_{1} W_{1}^T$.
    Then, the gradient $g=\nabla_x \textit{Loss}(x)$ can be represented by $g=\gamma_1 v_1$, where $\gamma_1=\|\tilde{w}_x\|_2 \frac{\partial \textit{Loss}(x)}{\partial z}$.
\end{lemma}
\begin{proof}
According to Lemma~\ref{lemma:H-binary}, $v_1 = \frac{\tilde{w}_x}{\|\tilde{w}_x\|_2}$.
In this way, the gradient $g=\nabla_x \textit{Loss}(x)$ is given as follows.
\begin{equation}
    \begin{split}
    \label{eq:gradient-in-lbp-1}
        g &= \nabla_x \textit{Loss}(x) \\
        &=\frac{\partial \textit{Loss}(x)}{\partial z} \frac{\partial z}{\partial x} \\
        &= \frac{\partial \textit{Loss}(x)}{\partial z}\tilde{w}_x \\
        &= \frac{\partial \textit{Loss}(x)}{\partial z} \|\tilde{w}_x\|_2  v_1
         \quad // \quad v_1 = \frac{\tilde{w}_x}{\|\tilde{w}_x\|_2} \text{in Eq.~\eqref{eq:hessian_lambda_v_lbp}}
        \\
        &=\gamma_1 v_1.
    \end{split}
\end{equation}
where $\gamma_1=\|\tilde{w}_x\|_2 \frac{\partial \textit{Loss}(x)}{\partial z}$.
\end{proof}
$ $\newline

Thus, based on Lemma~\ref{lemma:H-g-binary}, $\forall i\ne 1$, the projection length $\gamma_i$ can be represented as
\begin{equation}
\begin{aligned}
\label{eq:gamma-in-lbp}
\gamma_i=g^T v_i = \gamma_1 v_1^T v_i=0.
\end{aligned}
\end{equation}

Then, based on Eq.~\eqref{eq:gamma-in-lbp}, Eq.~\eqref{eq:multi-lbp-norm-equal} can be further simplified as follows.
\begin{equation} \label{eq:multi-lbp-norm-equal-1}
    \begin{split}
        \|\delta^{(\textrm{multi})}\|_2 &= \|\delta^{(\text{lbp})}\|_2 
        \\
        \Rightarrow\quad \| \sum_{i=1}^n  \frac{\exp(\beta\lambda_i)-1}{\lambda_i} \gamma_i v_i \|_2^2 &=\sum_{k=1}^n (\gamma^{(\text{lbp})}_k)^2 
        \quad // \quad \text{According to Eq.~\eqref{eq:multi-lbp-norm-equal}}
        \\
        \Rightarrow\quad \left[ \frac{\exp(\beta\lambda_1)-1}{\lambda_1}\gamma_1\right]^2 &= \sum_{k=1}^n (\gamma^{(\text{lbp})}_k)^2
         \quad // \quad \text{According to Eq.~\eqref{eq:gamma-in-lbp}}
        \\
    \Rightarrow \quad \left[ \frac{\exp(\beta\lambda_1)-1}{\lambda_1}\gamma_1\right]^2 -(\gamma^{(\text{lbp})}_1)^2 &= \sum_{k=1}^n (\gamma^{(\text{lbp})}_k)^2- (\gamma^{(\text{lbp})}_1)^2  
    =  \sum_{k=2}^n (\gamma^{(\text{lbp})}_k)^2   \ge 0.
     \end{split}
\end{equation}

In summary, interactions $\sum_{a, b\in \Omega} I_{ab}(\delta^{(\textrm{multi})})$ for the multi-step attack are larger than interactions $\sum_{a, b\in \Omega} I_{ab}(\delta^{(\textrm{lbp})})$ for the LinBP Attack,~\textit{i.e.,} Eq.~\eqref{eq:interaction-multi-lbp} is positive.
\begin{equation}
    \begin{split}
        &\sum_{a, b\in \Omega} I_{ab}(\delta^{(\textrm{multi})})- \sum_{a, b\in \Omega} I_{ab}(\delta^{(\text{lbp})}) \\
        & = \underbrace{\lambda_1}_{> 0} \underbrace{ \left[ (\frac{\exp(\beta\lambda_1)-1}{\lambda_1} \gamma)^2 - (\gamma^{(\text{lbp})}_1)^2 \right]}_{\ge 0} \ge 0
    \end{split}
\end{equation}

Thus, Proposition~\ref{theoremappendix:linbp} in Section~\ref{sec:explaining} is proven.
\end{proof}

\section{Proof that perturbations generated by the multi-step attack on the adversarially-trained DNN have smaller interactions  than those generated on the normally-trained DNN}
\label{append:at}
Spring \emph{et al.}~\cite{springer2021adversarial} find that the adversarial perturbation generated on the adversarially-trained DNN is more transferable than the adversarial perturbation generated on the normally-trained DNN.
In this section, we prove that the interaction inside the adversarial perturbation generated on the adversarially-trained DNN is smaller than  the interaction inside the adversarial perturbation generated on the normally-trained DNN.

\textbf{Adversarial training.}
Adversarial training~\cite{pgd2018} is the most widely-used method to defend adversarial attacking.
Given a DNN parameterized by $\theta$ and an input $x$, we use $\textit{Loss}_\theta (x)$ to denote the classification loss.
The adversarial attack aims to fool the DNN by adding a human-imperceptible perturbation on the input $x$, as follows.
\begin{equation} \label{eq:adv-attack}
    \delta_{\theta} = \arg\max_{\|\delta\|_p \le \epsilon} \textit{Loss}_\theta(x + \delta)
\end{equation}
where $\|\cdot\|_p$ is the $L_p$ norm constraint.

The adversarial training aims to defend against the adversarial attack, which is formulated as a min-max game~\cite{pgd2018} as follows.
\begin{equation} \label{eq:adv-training}
    \theta^{(\text{adv})} =\arg\min_\theta \mathbb{E}_x\left[\textit{Loss}_\theta(x + \delta_\theta) \right]=\arg\min_\theta \mathbb{E}_x \left[\max_{\|\delta\|_p \le \epsilon} \textit{Loss}_\theta(x + \delta) \right]
\end{equation}

\textbf{Normal training.}
The aim of the standard training is to minimize the classification loss, which is formulated as follows.
\begin{equation} \label{eq:std-training}
    \theta^{(\text{nor})}=\arg\min_\theta \mathbb{E}_x \left[\textit{Loss}_\theta(x ) \right]
\end{equation}

$ $\newline

\begin{propositionappendix}[Perturbations generated by the multi-step attack on the adversarially-trained DNN have smaller interactions  than perturbations generated on the normally-trained DNN]
    \label{theoremappendix:at}
    Let $\delta_{\theta^{(\text{adv})}}=\arg\max_{\|\delta\|_p \le \epsilon} \textit{Loss}_{\theta^{\text{(adv)}}}(x + \delta)$ and $\delta_{\theta^{(\text{nor})}}=\arg\max_{\|\delta\|_p \le \epsilon} \textit{Loss}_{\theta^{\text{(nor)}}}(x + \delta)$ be perturbations generated on the adversarially-trained DNN and perturbations generated on the normally-trained DNN, respectively.
    Then, we have $\mathbb{E}_x \left[\sum_{a, b\in\Omega}I_{ab}\left(\delta_{\theta^{(\text{adv})}}\right) \right] \le  \mathbb{E}_x \left[\sum_{a, b\in\Omega}I_{ab}\left(\delta_{\theta^{(\text{nor})}}\right) \right]$.
  \end{propositionappendix}

\begin{proof}
Based on Taylor series and Eq.~{(\ref{eq:adv-training})}, adversarial training can be written as follows.
\begin{align} \label{eq:adv-train-taylor}
    \theta^{(\text{adv})} &= \arg\min_\theta \mathbb{E}_x \left[\max_{\|\delta\|_p \le \epsilon} \textit{Loss}_\theta(x + \delta) \right] \nonumber \\
    &=\arg\min_\theta \mathbb{E}_x\left[\textit{Loss}_\theta(x + \delta_\theta) \right] \nonumber\\
    &\approx\arg\min_\theta \mathbb{E}_x\left[\textit{Loss}_\theta(x) + (\delta_\theta)^T g_\theta + \frac{1}{2} (\delta_\theta)^T H_\theta \delta_\theta \right] \qquad // \text{According to Assumption~\ref{assumption:taylor}}
\end{align}
where $g_\theta = \nabla_x \textit{Loss}_\theta(x )$ and $H_\theta = \nabla_x^2 \textit{Loss}_\theta(x )$.

According to Eq.\eqref{eq:std-training}, compared with standard training $\theta^{(\text{nor})}=\arg\min_\theta \mathbb{E}_x \left[\textit{Loss}_\theta(x ) \right]$, adversarial training penalizes $\frac{1}{2} (\delta_\theta)^T H_\theta \delta_\theta$.
Therefore, compared with standard training, it is more likely that
\begin{dmath} \label{eq:adv-loss-effect}
 \mathbb{E}_x \left[ \frac{1}{2} (\delta_{\theta^{(\text{adv})}})^T H_{\theta^{(\text{adv})}} \delta_{\theta^{(\text{adv})}}\right] \le  \mathbb{E}_x \left[ \frac{1}{2} (\delta_{\theta^{(\text{nor})}})^T H_{\theta^{(\text{nor})}} \delta_{\theta^{(\text{nor})}}\right]
\end{dmath}

Note that, according to Lemma~\ref{lemmaappendix:interaction-closed-form}, the sum of pairwise interactions inside the adversarial perturbation $\delta_{\theta^{(\text{adv})}}$  and $\delta_{\theta^{(\text{nor})}}$ are given as follows.
\begin{equation} \label{eq:interaction-adv-std}
    \begin{split}
        &\sum_{a, b\in\Omega}I_{ab}\left(\delta_{\theta^{(\text{adv})}}\right) = (\delta_{\theta^{(\text{adv})}})^T H_{\theta^{(\text{adv})}} \delta_{\theta^{(\text{adv})}} \\
        &\sum_{a, b\in\Omega}I_{ab}\left(\delta_{\theta^{(\text{nor})}}\right) = (\delta_{\theta^{(\text{nor})}})^T H_{\theta^{(\text{nor})}} \delta_{\theta^{(\text{nor})}}
    \end{split}
\end{equation}
where $H_{\theta^{(\text{adv})}}=\nabla_x^2 \textit{Loss}_{\theta^{(\text{adv})}}(x)$ and $H_{\theta^{(\text{nor})}}=\nabla_x^2 \textit{Loss}_{\theta^{(\text{nor})}}(x)$ represent the Hessian matrices of the adversarially and  normally-trained DNNs on the input $x$, respectively.

Thus, according to Eq.~(\ref{eq:adv-loss-effect}) and Eq.~(\ref{eq:interaction-adv-std}), we have
\begin{equation}
    \begin{split}
        &\mathbb{E}_x \left[ \frac{1}{2} (\delta_{\theta^{(\text{adv})}})^T H_{\theta^{(\text{adv})}} \delta_{\theta^{(\text{adv})}}\right] \le  \mathbb{E}_x \left[ \frac{1}{2} (\delta_{\theta^{(\text{nor})}})^T H_{\theta^{(\text{nor})}} \delta_{\theta^{(\text{nor})}}\right] \\
        \Rightarrow & \quad \mathbb{E}_x \left[\sum_{a, b\in\Omega}I_{ab}\left(\delta_{\theta^{(\text{adv})}}\right) \right] \le  \mathbb{E}_x \left[\sum_{a, b\in\Omega}I_{ab}\left(\delta_{\theta^{(\text{nor})}}\right) \right].
    \end{split}
\end{equation}
\end{proof}

\section{Proof of the PI Attack generates perturbations with smaller interactions than the multi-step attack}
\label{append:pi}

In this section, we prove Proposition~\ref{pro:pi} in Section~\ref{sec:explaining}, which shows that perturbations of the PI Attack~\cite{gao2020patch} exhibit smaller interactions than those of the multi-step attack.

Before the proof of Proposition~\ref{pro:pi}, we first revisit the PI Attack.

\subsection{PI Attack}
The PI Attack~\cite{gao2020patch} boosts the transferability by distributing the perturbations out of the norm constraint $\epsilon$ to neighboring pixels.
In particular, for a standard $L_\infty$ attack, if a perturbation unit $\delta_a$ exceeds the $\epsilon$-ball, $\delta_a$ will be clipped.
\emph{I.e.} for each perturbation unit $a$, if $|\delta_a|> \epsilon$, then $\delta_a^{\text{clip}}=\operatorname{clip}(\delta_a) = \epsilon\cdot \operatorname{sign}(\delta_a)$.
However, Gao \emph{et al.}~\cite{gao2020patch} claim that such a direct clip loses some gradient information.
Hence, the proposed PI Attack uniformly distributes the perturbation out of the norm $C_a\triangleq \delta_a-\delta_a^{\text{clip}}$  to neighboring pixels, instead of directly clipping the perturbation exceeding the $\epsilon$-ball.
In this way, the PI Attack~\cite{gao2020patch} re-distributes the gradient of the classification loss to generate the adversarial perturbation.
\begin{equation}
\label{eq:pi-def}
    \delta^{(t)}_{\text{pi-ori}}=\operatorname{clip}\{ \delta^{(t-1)}_{\text{pi-ori}} + \alpha \cdot \operatorname{sign}(\nabla_{x'} \textit{Loss}(x'=x + \delta^{(t-1)}_{\text{pi-ori}}))
    +\eta_{\text{pi}} \cdot \operatorname{sign}(W_{p} * C^{(t-1)}_{\text{pi-ori}})\},
\end{equation}
where $W_{p}$ denotes an uniform project kernel, and $C^{(t-1)}_{\text{pi-ori}}$ represents the  perturbation out of the norm at the $t-1$
step.
$\eta_{\text{pi}}\in\mathbb{R}$ represents the projection factor, and $*$ denotes the convolution operation.

The PI Attack is originally defined as Eq.~\eqref{eq:pi-def}, where the sign operation and the clip operation are used to constrain the adversarial perturbation $\delta^{(t)}_{\text{pi-ori}}$.
In fact, the sign operation and the clip operation are common operations in generating adversarial perturbations.
Therefore, the distinctive technique in the PI Attack is to uniformly distribute the perturbation out of the norm $C^{(t-1)}_{\text{pi-ori}}$.
Thus, to simplify the proof, we ignore the sign operation and the clip operation, without hurting the trustworthiness of the analysis of the PI Attack. 
In this way, we focus on a simplified version of the PI Attack, which uniformly distributes the perturbation out of the norm on each pixel to its K surrounding pixels.
\begin{equation} \label{eq:pi-generalized}
    \delta^{(t)}_{\textrm{pi}}= \delta^{(t-1)}_{\textrm{pi}} + \alpha\cdot A \nabla_{x'} \textit{Loss}(x'=x + \delta^{(t-1)}_{\textrm{pi}}),
\end{equation}
where $A\in\mathbb{R}^{n\times n}$ is given as follows.
\begin{dmath} \label{eq:matrix-A}
A=\left[\begin{array}{cccc}
    1-\tau_{1} & \frac{\tau_{2}}{K} & \ldots & \frac{\tau_{n}}{K} \\
    \frac{\tau_{1}}{K} & 1-\tau_{2} & \ldots & 0 \\
    0 & \frac{\tau_{2}}{K} & \ldots & \frac{\tau_{n}}{K} \\
    \vdots & \vdots & \ddots & \vdots \\
    \frac{\tau_{1}}{K} & 0 & \ldots & 1-\tau_{n}
    \end{array}\right]
\end{dmath}

If the perturbation unit $a$ does not exceeds the norm, then $\tau_a = 0$.
Otherwise, $C_{\text{pi}, a}^{(t-1)} = \text{clip}(\delta^{(t-1)}_{\textrm{pi}, a} + \alpha g_{\text{pi, a}}^{(t-1)}) \ne 0$, where $g_{\text{pi, a}}^{(t-1)} =  \nabla_x \textit{Loss}(x + \delta^{(t-1)}_{\textrm{pi}})$.
In this way, $\tau_a = C_{\text{pi}, a}^{(t-1)} / (\alpha g_{\text{pi, a}}^{(t-1)} )$.
Thus, matrix $A$ uniformly assigns the perturbation out of the norm to surrounding pixels, where $0 \le \tau_a < 1$ represents the strength of the perturbation out of the norm. 
Generally speaking, there are only four neighbors or eight neighbors for each image pixel, thus $K=4$ or $8$. Here, we might as well ignore the pixels at the edge of the image.

Generally speaking, in the early stage of attacking, adversarial perturbations are usually weak, and often do not reach the constraint for perturbations.
Hence, we divide the PI Attack into two stages.
Specifically, in the first stage, most perturbation units do not exceed the $\epsilon$-ball and are updated via straightforward gradient ascent as $\delta^{(t)}_{\textrm{pi}} = \delta^{(t-1)}_{\textrm{pi}} + \alpha\cdot \nabla_x \textit{Loss}(x + \delta^{(t-1)}_{\textrm{pi}})$.
In the second stage, many perturbations exceed the $\epsilon$-ball and are updated with Eq.~(\ref{eq:pi-generalized}).
\begin{equation} \label{eq:pi-final}
    \delta^{(t)}_{\textrm{pi}} =
    \begin{cases}
        \delta^{(t-1)}_{\textrm{pi}} + \alpha\cdot \nabla_x \textit{Loss}(x + \delta^{(t-1)}_{\textrm{pi}}),  & \text{if $t \le m_1$} \\
        \delta^{(t-1)}_{\textrm{pi}} + \alpha\cdot A \nabla_x \textit{Loss}(x + \delta^{(t-1)}_{\textrm{pi}}), & \text{if $ m_1 < t \le m_1+m_2$}
        \end{cases}
\end{equation}
Here, $m_1$ denotes the step in which the perturbations begin to exceed the constraint.
The total number of steps is referred to as $m_1 + m_2$.

In fact, we can tentatively consider $A$ remains relatively consistent after enough attacking steps, because all perturbation units that reach the norm of epsilon will probably stay on the surface of the epsilon ball in all subsequent attacking steps. 
Therefore, we ignore the minor change of $A$ in the following proof for simplicity. In this scenario, the perturbation of the PI Attack updated in each step can be calculated as
\begin{equation} \label{eq:pi-update}
    \Delta x^{(t)}_{\text{pi}}= \delta^{(t)}_{\textrm{pi}} -\delta^{(t-1)}_{\textrm{pi}} =
    \begin{cases}
        \alpha\cdot \nabla_x \textit{Loss}(x + \delta^{(t-1)}_{\textrm{pi}}),  & \text{if $t \le m_1$} \\
        \alpha\cdot A \nabla_x \textit{Loss}(x + \delta^{(t-1)}_{\textrm{pi}}). & \text{if $ m_1 < t \le m_1+m_2 $}
        \end{cases}
\end{equation}

Then, the adversarial perturbation generated after two stages of attacks by the PI Attack can be represented as follows.

\begin{equation} \label{eq:delta-pi}
    \delta^{(t)}_{\textrm{pi}} = \sum_{t'=1}^{t'} \Delta x^{(t')}_{\text{pi}}, \qquad\Delta x^{(t')}_{\text{pi}} = 
    \begin{cases}
        \Delta x^{(t')}_{\text{pi-first}},  & \text{if $t' \le m_1$} \\
        \Delta x^{(s')}_{\text{pi-second}}. & \text{$s' = t'-m_1$, if $ m_1 < t' \le m_1+m_2 $}
        \end{cases}
\end{equation}
where 
\begin{align}
    \Delta x^{(t)}_{\text{pi-first}}&=\alpha \left[ I + \alpha H\right]^{t-1}g \quad//\quad \text{According to Lemma~\ref{lemma:induction}}, \nonumber \\
    \Delta x^{(s)}_{\text{pi-second}}&=\left[ I + \alpha A H\right]^{s -1} \Delta x^{( 1)}_{\text{pi-second}}. \quad//\quad \text{According to Lemma~\ref{lemma:pi-induction}} 
\end{align}

$ $\newline

\begin{lemma} \label{lemma:pi-induction}
    Based on Assumption~\ref{assumption:taylor}, the perturbation of the PI Attack updated at the $s$-th step in the second stage of attacking in Eq.~(\ref{eq:delta-pi}) can be written as $\Delta x^{(s)}_{\text{pi-second}} =\left[ I + \alpha A H\right]^{s -1} \Delta x^{( 1)}_{\text{pi-second}}$.
    \end{lemma}

\begin{proof}
    According to Eq.~(\ref{eq:pi-update}), if $ 1 \le s \le m_2$, let $t = s + m_1$. Then, we have
    \begin{align} \label{eq:Delta-x-pi}
        \Delta x^{(s)}_{\text{pi-second}} &= 
        \Delta x^{(t)}_{\text{pi}} \nonumber\\
        &=\alpha \cdot A \nabla_{x'} \textit{Loss}(x'=x + \delta^{(t-1)}_{\textrm{pi}})
        \quad //\quad \text{According to Eq.~(\ref{eq:pi-update})}
        \nonumber \\
        &=  \alpha \cdot A \left[g + H\delta^{(t-1)}_{\textrm{pi}} \right] \quad //\quad\text{According to Assumption~\ref{assumption:taylor}}\nonumber \\
        &=  \alpha \cdot A \left[g + H\sum_{t'=1}^{t-1} \Delta x^{(t')}_{\text{pi}} \right]. \quad //\quad\text{According to Eq.~(\ref{eq:delta-pi})}
    \end{align}

Based on Eq.~(\ref{eq:Delta-x-pi}), $\Delta x^{(s)}_{\text{pi-second}}$ can be further written as
\begin{align} \label{eq:Delta-x-pi-recursive}
    \Delta x^{(s)}_{\text{pi-second}} &=\alpha \cdot A \left[g + H\sum_{t'=1}^{t-1} \Delta x^{(t')}_{\text{pi}} \right] \nonumber\\
        &= \alpha\cdot A \left[g + H\sum_{t'=1}^{t-2} \Delta x^{(t')}_{\text{pi}} \right] + \alpha \cdot A H\Delta x^{(t-1)}_{\text{pi} }\nonumber\\
        & = \Delta x^{(t-1)}_{\text{pi}}+ \alpha\cdot AH \Delta x^{(t-1)}_{\text{pi}}
        \quad //\quad\text{According to Eq.~(\ref{eq:Delta-x-pi})}
        \nonumber\\
        & = \left[ I + \alpha\cdot A H\right] \Delta x^{(t-1)}_{\text{pi}} \nonumber\\
        &= \left[ I + \alpha\cdot A H\right] \Delta x^{(s-1)}_{\text{pi-second}}.
\end{align}

In this way, based on Eq.~(\ref{eq:Delta-x-pi-recursive}),  $\forall 1 < s \le m_2$, $ \Delta x^{(s)}_{\text{pi-second}} $ can be represented as 
\begin{equation} \label{eq:Delta-x-pi-final}
    \begin{split}
    \Delta x^{(s)}_{\text{pi-second}} &=  \left[ I + \alpha \cdot A H\right] \Delta x^{(s-1)}_{\text{pi-second}}
    \quad //\quad\text{According to Eq.~(\ref{eq:Delta-x-pi-recursive})}
    \\
    &= \left[ I + \alpha\cdot A H\right] \left[ I + \alpha\cdot A H\right] \Delta x^{(s-2)}_{\text{pi-second}}
    \quad //\quad\text{According to Eq.~(\ref{eq:Delta-x-pi-recursive})}
    \\
    &= \left[ I + \alpha\cdot A H\right]^{2} \Delta x^{(s-2)}_{\text{pi-second}}
    \\
    &\vdots\\
    &=\left[ I + \alpha\cdot A H\right]^{s-1}\Delta x^{( 1)}_{\text{pi-second}}.
    \end{split}
\end{equation}
Thus, Lemma~\ref{lemma:pi-induction} is proven.
\end{proof}
$ $\newline

\textbf{Perturbation generated by the PI Attack.}
Let we focus on the adversarial perturbation of the PI Attack based on binary classification.
\emph{I.e.} we formulate the loss function $\textit{Loss}(x)$ as the cross-entropy loss with a sigmoid function.
\begin{equation} \label{eq:sigmoid-ce-in-linbp}
    \begin{split}
        p(y=1\mid x) &= \text{Sigmoid}(z) =  \frac{\exp(z)}{1 + \exp(z)} \\
        p(y=0\mid x) &= 1- p(y=1\mid x) = \frac{1}{1 + \exp(z)} \\
        \textit{Loss}(x) &= -p_0^*\log(p(y=0\mid x)) -p_1^*\log(p(y=1\mid x)) \\
        &= -yz + \log(1 + \exp(z)),
    \end{split}
\end{equation}
where the label $y\in \{0, 1\}$; $p^*_0=1-y$;  $p^*_1=y$.

In the  binary classification scenario, the adversarial perturbation of the PI Attack can be calculated as follows.
 We use the infinite-step adversarial attack with an infinitesimal step size to simplify the PI Attack.
Given $\beta_1 = \alpha m_1$ and $\beta_2 = \alpha m_2$, we have
\begin{align} \label{eq:pi-perturbation-definition}
    \delta^{(\text{pi})} = \delta_{\text{pi}}^{(\text{first}, \infty)} + \delta_{\text{pi}}^{(\text{second}, \infty)}, \nonumber \\
    \delta_{\text{pi}}^{(\text{first}, \infty)} = \lim_{m_1 \rightarrow \infty }\sum_{t=1}^{m_1} \Delta x_{\text{pi-first}}^{(t)}, \qquad \text{s.t.} \quad \alpha m_1 = \beta_1 \nonumber \\
    \delta_{\text{pi}}^{(\text{second}, \infty)} = \lim_{m_2 \rightarrow \infty }\sum_{s=1}^{m_2} \Delta x_{\text{pi-second}}^{(s)}. \qquad \text{s.t.} \quad \alpha m_2 = \beta_2
\end{align}

$ $\newline

\begin{lemma} \label{lemma:pi-perturbation}
    If the loss function is formulated as the cross-entropy loss on the sigmoid function, we use the infinite-step adversarial attack with an infinitesimal step size to simplify the PI Attack.
    Then, the perturbation generated by the PI Attack defined in Eq.~(\ref{eq:pi-perturbation-definition}) is given as $\delta^{(\text{pi})}=\gamma_1\left[\frac{\exp \left(\beta_{1} \lambda_1\right)-1}{\lambda_1}\right] v_1 + \gamma_1^{\prime}\left[\frac{\exp \left(\beta_{2} \lambda_1 v_1^{T} A v_1\right)-1}{\lambda v_1^{T} A v_1}\right] A v_1$, where $\beta_1=\alpha m_1$, $\beta_2 = \alpha m_2$, $\gamma_1=g^T v_1$ and $\gamma_1^\prime=\exp(\beta_1\lambda_1)\gamma_1$.
\end{lemma}

\begin{proof}
    Let us first focus on the adversarial perturbation generated in the first stage $ \delta_{\text{pi}}^{(\text{first}, \infty)}$ defined in Eq.~\eqref{eq:pi-perturbation-definition}.
    According to Corollary~\ref{corollaryappendix:delta-inf}, 
    if we use the infinite-step adversarial attack with an infinitesimal step size, then the perturbation of the PI Attack generated in the first stage can be written as
\begin{equation}
\label{eq:pi-first-stage-1}
        \delta_{\text{pi}}^{(\text{first}, \infty)}= \sum_{i=1}^n D_{ii} \gamma_i v_i,
    \end{equation}    
  where if $\lambda_i\ne0$, $D_{ii}=\frac{\exp(\beta_1\lambda_i)-1}{\lambda_i}$; otherwise, $D_{ii}=\beta_1$.    
  $\lambda_i$ and $v_i$ denote the $i$-th eigenvalue of the Hessian matrix $H=\nabla_x^2 \textit{Loss}(x)$ and its corresponding eigenvector, respectively.
  $\gamma_i = g^T v_i \in \mathbb{R}$ denotes the projection of the gradient $g=\nabla_x \textit{Loss}(x)$ on the eigenvector $v_i$.

Furthermore, in binary classification, if the loss function is formulated as the cross-entropy loss on the sigmoid function,  according to Lemma~\ref{lemma:H-binary}, the Hessian matrix only has a single positive eigenvalue and $n-1$ zero eigenvalues.
Without loss of generality, let $\lambda_1 > 0$ and $\forall 2\le i \le n, \lambda_i=0$.
Thus, the Hessian matrix can be re-written as $H=\lambda_1 v_1 v_1^T$.
Moreover, according to Lemma~\ref{lemma:H-g-binary}, the gradient can be represented as $g = \nabla_x \textit{Loss}(x)=\gamma_1 v_1$.
Thus, $\forall i\ne 1$, the projection $\gamma_i$ can be written as $\gamma_i=g^T v_i = \gamma_1 v_1^T v_i=0$.
    
In this scenario, the perturbation $\delta_{\text{pi}}^{(\text{first}, \infty)}$ of the PI Attack generated in the first stage can be written as
\begin{equation} 
\begin{aligned}
\label{eq:pi-first}
        \delta_{\text{pi}}^{(\text{first}, \infty)} 
        &= \sum_{i=1}^n D_{ii} \gamma_i v_i
        \quad // \quad \text{According to Eq.~\eqref{eq:pi-first-stage-1}}
        \\
        &= \sum_{i=1}^n \frac{\exp(\beta_1\lambda_i)-1}{\lambda_i}  \gamma_i v_i
        \\
        &= \frac{\exp(\beta_1\lambda_1 )-1}{\lambda_1} \gamma_1 v_1.
          \quad // \quad  \forall i\ne 1,\; \gamma_i=0
\end{aligned}
\end{equation}

Then, based on Eq.~(\ref{eq:pi-first}), the gradient of the loss function~\textit{w.r.t.} the adversarial example $x' = x + \delta_{\text{pi}}^{(\text{first}, \infty)}$ can be written as
    \begin{align}
    \label{eq:grad-in-pi}
        g' &\triangleq \nabla_{x'} \textit{Loss}(x'=x + \delta_{\text{pi}}^{(\text{first}, \infty)} ) \nonumber \\
        &= g + H  \delta_{\text{pi}}^{(\text{first}, \infty)} 
        \quad //\quad \text{According to Assumption~\ref{assumption:taylor}}\nonumber \\
        &= \gamma_1 v_1 + \lambda_1 v_1 v_1^T  \frac{\exp(\beta_1\lambda_1 )-1}{\lambda_1} \gamma_1 v_1 
        \quad // \quad g =\gamma_1 v_1,  H=\lambda_1 v_1 v_1^T \;\text{in Lemma~\ref{lemma:H-binary} and Lemma~\ref{lemma:H-g-binary}}
        \nonumber \\
        &= \gamma_1 v_1 +  \left[\exp(\beta_1\lambda_1 )-1 \right] \gamma_1 v_1 \nonumber\\
        &=\exp(\beta_1\lambda_1 )\gamma_1 v_1 \nonumber \\
            &= \gamma'_1 v_1,
    \end{align}
    where $\gamma'_1=\exp(\beta_1 \lambda_1) \gamma_1\in\mathbb{R}$.

    In this way, based on Eq.~\eqref{eq:pi-update} and Eq.~\eqref{eq:grad-in-pi}, $\Delta x^{( 1)}_{\text{pi-second}} $ can be written as
    \begin{equation} \label{eq:delta-m-1}
        \begin{split}
            \Delta x^{( 1)}_{\text{pi-second}} &= \alpha\cdot A \nabla_{x'} \textit{Loss}(x'=x + \delta_{\text{pi}}^{(\text{first}, \infty)} )
             \quad //\quad\text{According to Eq.~(\ref{eq:pi-update})}
            \\
            &= \alpha\cdot \gamma'_1\cdot  A v_1.
             \quad //\quad\text{According to Eq.~(\ref{eq:grad-in-pi})}
        \end{split}
    \end{equation}

     Then, we focus on the adversarial perturbation generated in the second stage $\delta^{(\text{second}, m_2)}_{\text{pi}}$.
     For simplicity, $\forall m_1< t \le m_1+m_2$, let $s = t - m_1$ denote the step in the second stage.
    According to Lemma~\ref{lemma:pi-induction}, the perturbation of the PI Attack generated in the second stage can be written as
    \begin{align} \label{eq:pi-seoncd}
            \delta^{(\text{second}, m_2)}_{\text{pi}} &= \sum_{s=1}^{m_2} \Delta x^{(s)}_{\text{pi-second}}
             \quad //\quad\text{According to Eq.~\eqref{eq:pi-perturbation-definition}}
            \nonumber \\
            &= \sum_{s=1}^{m_2}\left[ I + \alpha A H\right]^{s-1}\Delta x^{( 1)}_{\text{pi-second}}
             \quad //\quad\text{According to Lemma~\ref{lemma:pi-induction}}
            \nonumber \\
            &= \sum_{s=1}^{m_2}\left[ \sum_{s'=0}^{s-1}\mybinom{s-1}{s'}(\alpha AH)^{s'}\right] \Delta x^{( 1)}_{\text{pi-second}}
            \nonumber \\
            &= \sum_{s=1}^{m_2}\left[ \sum_{s'=0}^{s-1}\mybinom{s-1}{s'}(\alpha A \lambda_1 v_1 v_1^T)^{s'}\right] \Delta x^{( 1)}_{\text{pi-second}} 
             \quad //\quad\text{According to Lemma~\ref{lemma:H-binary}}
            \nonumber\\
           &= \sum_{s=1}^{m_2}\left[ \sum_{s'=0}^{s-1}\mybinom{s-1}{s'}(\alpha A \lambda_1 v_1 v_1^T)^{s'} \right]\alpha \gamma'_1 A v_1 
           \quad //\quad\text{According to Eq.~(\ref{eq:delta-m-1}})
           \nonumber \\
            &= \sum_{s=1}^{m_2}\left[ I+\sum_{s'=1}^{s-1}\mybinom{s-1}{s'} (\alpha A \lambda_1 v_1 v_1^T)^{s'}\right]\alpha \gamma'_1 A v_1 
            \nonumber \\
            &= \sum_{s=1}^{m_2}\left[ I+\sum_{s'=1}^{s-1}\mybinom{s-1}{s'}(\alpha \lambda_1)^{s'} Av_1\; v_1^TAv_1\; v_1^T \cdots Av_1\; v_1^T\right]\alpha \gamma'_1 A v_1 
            \nonumber \\
            &= \sum_{s=1}^{m_2}\left[ I+\sum_{s'=1}^{s-1}\mybinom{s-1}{s'}(\alpha \lambda_1)^{s'} Av_1 (v_1^TAv_1)^{s'-1} v_1^T\right]\alpha \gamma'_1 A v_1 
            \nonumber \\
            &= \alpha \gamma'_1 \sum_{s=1}^{m_2}\left[ Av_1 + \sum_{s'=1}^{s-1}\mybinom{s-1}{s'}(\alpha \lambda_1 v_1^TAv_1)^{s'} Av_1\right] 
            \quad // \quad  v_1^TAv_1\in\mathbb{R}
            \nonumber \\
            &= \alpha \gamma'_1\sum_{s=1}^{m_2} \left[ 1 + \sum_{s'=1}^{s-1}\mybinom{s-1}{s'}(\alpha \lambda_1 v_1^TAv_1)^{s'} \right]Av_1 
            \nonumber\\
            &= \alpha \gamma'_1\sum_{s=1}^{m_2} \left( 1 + \alpha \lambda_1 v_1^TAv_1 \right)^{s-1}Av_1
             \nonumber\\
            &= \alpha \gamma'_1 \left[ \frac{(1 + \alpha \lambda_1 v^T_1 A v_1)^{m_2}-1}{\alpha \lambda_1 v^T_1 A v_1}\right]Av_1
            \nonumber \\
            &=  \gamma'_1 \left[ \frac{(1 + \frac{\beta_2}{m_2} \lambda_1 v^T_1 A v_1)^{m_2}-1}{\lambda_1 v^T_1 A v_1}\right]Av_1
            \quad // \quad \beta_2 = \alpha m_2 
    \end{align}
  
  In this way, 
  \begin{align}
      \delta^{(\text{second}, \infty)}_{\text{pi}} &= \lim_{m_2\rightarrow +\infty} \delta^{(\text{second}, m_2)}_{\text{pi}} \nonumber \\
      &=\lim_{m_2\rightarrow +\infty}   \gamma'_1 \left[ \frac{(1 + \frac{\beta_2}{m_2} \lambda_1 v^T_1 A v_1)^{m_2}-1}{\lambda_1 v^T_1 A v_1}\right]Av_1 \nonumber\\
      &=\gamma'_1 \left[ \frac{\exp(\beta_2\lambda_1 v^T_1 A v_1)-1}{\lambda_1 v^T_1 A v_1}\right]Av_1
  \end{align}

    Hence, based on Eq.~(\ref{eq:pi-first}) and Eq.~(\ref{eq:pi-seoncd}), the adversarial perturbation $ \delta^{(\text{pi})} $ of the PI Attack in Eq.~\eqref{eq:pi-perturbation-definition} can be written as
    \begin{align}
        \delta^{(\text{pi})}  &= \delta_{\text{pi}}^{(\text{first}, \infty)} + \delta^{(\text{second}, \infty)}_{\text{pi}} 
        \quad //\quad\text{According to Eq.~\eqref{eq:pi-perturbation-definition}}
        \\
            &=\gamma_1\left[\frac{\exp \left(\beta_{1} \lambda_1\right)-1}{\lambda_1}\right] v_1+\gamma_1^{\prime}\left[\frac{\exp \left(\beta_{2} \lambda_1 v_1^{T} A v_1\right)-1}{\lambda_1 v_1^{T} A v_1}\right] A v_1
    \end{align}
    where $\beta_1=\alpha m_1$, $\beta_2 = \alpha m_2$, and $\gamma^\prime_1=\exp(\beta_1\lambda_1)\gamma_1\in\mathbb{R}$.
Thus, Lemma~\ref{lemma:pi-perturbation} is proven.
\end{proof}
$ $\newline

\textbf{Pairwise interactions inside perturbation $\delta^{(\text{pi})}$}.
Based on the perturbations $\delta^{(\text{pi})}$ of the PI Attack,  the interaction between each pair of perturbation units $(a,b)$ can be computed as 
 \begin{equation}  
 \label{eq: pi-interaction-in-pi}
 I_{ab}(\delta^{(\textrm{pi})}) = \delta^{(\textrm{pi})}_a H_{ab} \delta^{(\textrm{pi})}_b.
       \quad // \quad \text{According to Eq.~\eqref{eq:shapley-interaction-index_2} in Appendix~\ref{append:interaction}}.
  \end{equation}

Furthermore, we sum the interactions among all pairs of perturbation units $(a, b)$ as follows.

\begin{small}
\begin{equation} 
\begin{aligned}
\label{eq:pi-interaction}
    \sum_{a, b\in \Omega} I_{ab}(\delta^{(\textrm{pi})} ) 
    &= (\delta^{(\textrm{pi})} )^T H \delta^{(\textrm{pi})}  
    \\
    &= \left(\gamma_1\left[\frac{\exp \left(\beta_{1} \lambda_1\right)-1}{\lambda_1}\right] v_1^T+\gamma_1^{\prime}\left[\frac{\exp \left(\beta_{2} \lambda v_1^{T} A v_1\right)-1}{\lambda v_1^{T} A v_1}\right]  v_1^T A^T\right) 
    \left(\lambda_1 v_1 v_1^T \right) 
    \delta^{(\textrm{pi})}  
     \\
    &= \lambda_1\left(\gamma_1\left[\frac{\exp \left(\beta_{1} \lambda_1\right)-1}{\lambda_1}\right] +\gamma_1^{\prime}\left[\frac{\exp \left(\beta_{2} \lambda v_1^{T} A v_1\right)-1}{\lambda v_1^{T} A v_1}\right] v_1^T A^T v_1 \right) 
    v_1^T  
    \delta^{(\textrm{pi})} 
\end{aligned}
\end{equation}
\end{small}     

Because the term $v_1^{T} A v_1\in\mathbb{R}$ in Eq.~\eqref{eq:pi-interaction} is a scalar, then we have $v_1^{T} A v_1 = (v_1^{T} A v_1)^{T} = v_1^T A^T v_1$.
In this way, interactions $\sum_{a, b\in \Omega} I_{ab}(\delta^{(\textrm{pi})} )$ in Eq.~\eqref{eq:pi-interaction} can be simplified as

\begin{small}
\begin{equation} 
\begin{aligned}
\label{eq:pi-interaction-1}
    &\sum_{a, b\in \Omega} I_{ab}(\delta^{(\textrm{pi})} ) 
    \\
    &= \lambda_1\left(\gamma_1\left[\frac{\exp \left(\beta_{1} \lambda_1\right)-1}{\lambda_1}\right] +\gamma_1^{\prime}\left[\frac{\exp \left(\beta_{2} \lambda v_1^{T} A v_1\right)-1}{\lambda v_1^{T} A v_1}\right] v_1^T A^T v_1 \right) 
    v_1^T  
    \delta^{(\textrm{pi})}       
      \\
      &= \lambda_1  \left( \gamma_1\frac{\exp(\beta_1\lambda_1) -1}{\lambda_1} + \gamma'_1 \frac{\exp(\beta_2\lambda_1 v_1^T A v_1) -1}{\lambda_1} \right)
       v_1^T  
    \delta^{(\textrm{pi})} 
     \\
     &= \lambda_1  \left( \gamma_1\frac{\exp(\beta_1\lambda_1) -1}{\lambda_1} + \gamma'_1 \frac{\exp(\beta_2\lambda_1 v_1^T A v_1) -1}{\lambda_1} \right)
     \left(\gamma_1\left[\frac{\exp \left(\beta_{1} \lambda_1\right)-1}{\lambda_1}\right] v_1^T   v_1+\gamma_1^{\prime}\left[\frac{\exp \left(\beta_{2} \lambda v_1^{T} A v_1\right)-1}{\lambda v_1^{T} A v_1}\right]  v_1^T   A v_1 \right)
     \\
    &= \lambda_1  \left( \gamma_1\frac{\exp(\beta_1\lambda_1) -1}{\lambda_1} + \gamma'_1 \frac{\exp(\beta_2\lambda_1 v_1^T A v_1) -1}{\lambda_1} \right)^2
     \\
    &=\lambda_1  \left( \gamma_1\frac{\exp(\beta_1\lambda_1) -1}{\lambda_1} + \gamma_1 \frac{\exp(\beta_2\lambda_1 v_1^T A v_1  + \beta_1 \lambda_1) -\exp(\beta_1\lambda_1)}{\lambda_1} \right)^2  
    \quad // \quad \gamma^\prime_1=\exp(\beta_1\lambda_1)\gamma_1
    \\
    &=\lambda_1 \left(  \gamma_1\frac{\exp(\beta_2\lambda_1 v_1^T A v_1  + \beta_1 \lambda_1) -1}{\lambda_1}\right)^2.
\end{aligned}
\end{equation}
\end{small}

\subsection{Multi-step attack}
In this subsection, we revisit the perturbation of the multi-step attack, which is defined in Appendix~\ref{append:delta-closed-form}.

\textbf{Perturbation generated by the multi-step attack.}
Note that the total number of attacking steps of the PI Attack is $m_1 + m_2$.
Thus, in order to ensure fair comparisons between the PI Attack and the multi-step attack, we also use $m_1 + m_2$ steps to craft perturbations of the multi-step attack.
Let $\beta_1=\alpha m_1$ and $\beta_2=\alpha m_2$.
Then, according to Corollary~\ref{corollaryappendix:delta-inf}, the adversarial perturbation $\delta^{(\textrm{multi})}$ of the multi-step attack can be represented as
\begin{equation}
\label{eq:delta-multi-inf-in-pi-1}
    \delta^{(\textrm{multi})}\triangleq \delta^{(\infty)} = \sum_{i=1}^n D_{ii} \gamma_i v_i,
\end{equation}
where if $\lambda_i\ne0$, $D_{ii}=\frac{\exp((\beta_1 + \beta_2)\lambda_i)-1}{\lambda_i}$; otherwise, $D_{ii}=\beta_1 + \beta_2$. 
$\gamma_i = g^T v_i \in\mathbb{R}$ represents the projection of the gradient $g=\nabla_x \textit{Loss}(x)$ on the eigenvector $v_i$.

In this section, because our goal is to compare the interaction between perturbation units generated by the PI and multi-step attacks, we require the multi-step attack to follow the same experimental setting as the PI Attack. That is, we conduct the multi-step attack on the DNN for binary classification, like in the PI Attack. In this way, according to Eq.~\eqref{eq:gamma-in-lbp}, the Hessian matrix $H=\nabla_x^2 \textit{Loss}(x)$ has only one positive eigenvalues and $n-1$ zero eigenvalues.
Moreover, we have $\forall i\ne 1, \gamma_i=g^T v_i = \gamma_1 v_1^T v_i=0$.
In this way, the adversarial perturbation $\delta^{(\textrm{multi})}$ of the multi-step attack in Eq.~\eqref{eq:delta-multi-inf-in-pi-1} can be re-written as
\begin{equation} 
\begin{aligned}
\label{eq:delta-multi-inf-in-pi}
\delta^{(\textrm{multi})}&=    \sum_{i=1}^n \frac{\exp((\beta_1 + \beta_2)\lambda_i)-1}{\lambda_i} \gamma_i v_i
\\
&=\frac{\exp((\beta_1 + \beta_2)\lambda_1)-1}{\lambda_1}\gamma_1 v_1,
\quad // \quad \forall i\ne 1,\; \gamma_i=0.
\end{aligned}
\end{equation}
where $\lambda_1$ is referred to as the only positive eigenvalue of $H=\nabla_x^2 \textit{Loss}(x)$, and $v_1$ indicates its corresponding eigenvector.
$ $\newline

\textbf{Pairwise interactions inside perturbation $\delta^{(\text{multi})}$}.
Based on the perturbations $\delta^{(\text{multi})}$ of the multi-step attack, the interaction between each pair of perturbation units $(a,b)$ can be computed as 
       \begin{equation}  \label{eq: multi-pairwise-interaction-in-pi}
       I_{ab}(\delta^{(\textrm{multi})}) = \delta^{(\textrm{multi})}_a H_{ab} \delta^{(\textrm{multi})}_b.
       \quad // \quad \text{According to Eq.~\eqref{eq:shapley-interaction-index_2} in Appendix~\ref{append:interaction}}.
       \end{equation}

       Furthermore, we sum the interactions among all pairs of perturbation units $(a, b)$ as follows.
       \begin{small}
       \begin{equation} \label{eq:multi-interaction-in-pi}
        \begin{split}
            &\sum_{a, b\in \Omega} I_{ab}(\delta^{(\textrm{multi})}) \\
            &= (\delta^{(\textrm{multi})})^T H \delta^{(\textrm{multi})} \\
            &= \left(\frac{\exp((\beta_1 + \beta_2)\lambda_1)-1}{\lambda_1} \gamma_1 v_1^T\right) \left(\lambda_1 v_1 v_1^T \right) \left(\frac{\exp((\beta_1 + \beta_2)\lambda_1)-1}{\lambda_1} \gamma_1 v_1\right)\quad // H=\lambda_1 v_1 v_1^T, \text{according to Lemma~\ref{lemma:H-g-binary}} \\
            &= \left(\frac{\exp((\beta_1 + \beta_2)\lambda_1)-1}{\lambda_1} \gamma_1 \lambda_1 v_1^T \right) \left(\frac{\exp((\beta_1 + \beta_2)\lambda_1)-1}{\lambda_1} \gamma_1 v_1\right) \qquad // \qquad v_1^T v_1=1\\
            &= \lambda_1 \left( \frac{\exp((\beta_1 + \beta_2)\lambda_1)-1}{\lambda_1} \gamma_1\right)^2.
        \end{split}
    \end{equation}\end{small}
$ $\newline

\subsection{Proof of Proposition~\ref{theoremappendix:pi}}

We assume the value of $\tau_a$ in Eq.~\eqref{eq:matrix-A} is under Assumption~\ref{assumption:tau-mono}.
Moreover, under Assumption~\ref{assumption:mean-v}, we consider the gradient of all pixels can be approximated by the small-scale gradient of adjacent pixels. 
  In binary classification scenario, if the adversarial strength $\beta=\alpha m$ is the same for the PI Attack and the multi-step attack, we prove the following proposition, which shows that PI Attack decreases the interaction between perturbation units.

\begin{propositionappendix}[Perturbations of the PI Attack exhibit smaller interactions than those of the multi-step attack]
    \label{theoremappendix:pi}
    In binary classification, if the classification loss for attacking  is formulated as the cross-entropy loss on the sigmoid function in Eq.~\eqref{eq:sigmoid-ce}, we use the infinite-step adversarial attack with an infinitesimal step size to simplify the PI Attack and the multi-step attack.
  If the adversarial strength $\beta=\alpha m$ is the same for the PI Attack and the multi-step attack, then, under Assumptions~\ref{assumption:tau-mono}~and~\ref{assumption:mean-v}, we have  $\sum_{a, b\in\Omega}[I_{a b}(\delta^{(\textrm{pi})})] \le \sum_{a, b\in\Omega}[I_{ab}(\delta^{(\textrm{multi})})]$.
  \end{propositionappendix}
$ $\newline

\begin{proof}
According to Eq.~(\ref{eq:multi-interaction-in-pi}) and Eq.~(\ref{eq:pi-interaction-1}), the difference between interactions $\sum_{a, b\in \Omega} I_{ab}(\delta^{(\textrm{multi})})$ for the multi-step attack and interactions $\sum_{a, b\in \Omega} I_{ab}(\delta^{(\textrm{pi})} )$ for the PI Attack  can be written as

\begin{small}
\begin{equation}
\begin{aligned}
 \label{eq:diff-interaction-pi}
    &\sum_{a, b\in \Omega} I_{ab}(\delta^{(\textrm{multi})}) -\sum_{a, b\in \Omega} I_{ab}(\delta^{(\textrm{pi})} )  \\
    &= \lambda_1 \left( \gamma_1\frac{\exp((\beta_1 + \beta_2)\lambda_1)-1}{\lambda_1} \right)^2 - \lambda_1 \left(  \gamma_1\frac{\exp(\beta_2\lambda_1 v_1^T A v_1  + \beta_1 \lambda_1) -1}{\lambda_1}\right)^2 
    \quad // \quad \text{According to Eq.~\eqref{eq:multi-interaction-in-pi} and Eq.~\eqref{eq:pi-interaction-1}}
    \\
    &=  \frac{\gamma_1^2}{\lambda_1} \left\{ \bigg[ \exp(\beta_1\lambda_1) \left[ \exp(\beta_2\lambda_1) + \exp(\beta_2\lambda_1 v_1^T A v_1  ) \right]-2\bigg]  \bigg[ \exp(\beta_1\lambda_1) \left[ \exp(\beta_2\lambda_1) - \exp(\beta_2\lambda_1 v_1^T A v_1  ) \right] \bigg]  \right\}
\end{aligned}
\end{equation}
\end{small}

    Then, in order to prove that Eq.~\eqref{eq:diff-interaction-pi} is non-negative, we focus on the  range of $v_1^T A v_1$.

    \textbullet\;\textbf{The upper bound of  $v_1^T A v_1$.} The value of $v_1^T A v_1$ can be expressed as follows.
    \begin{align}\label{eq:vAv}
        v_1^T A v_1 &= [v_{1, 1}, v_{1, 2} \dots, v_{1, n}]
        \left[\begin{array}{cccc}
            1-\tau_{1} & \frac{\tau_{2}}{K} & \ldots & \frac{\tau_{n}}{K} \\
            \frac{\tau_{1}}{K} & 1-\tau_{2} & \ldots & 0 \\
            0 & \frac{\tau_{2}}{K} & \ldots & \frac{\tau_{n}}{K} \\
            \vdots & \vdots & \ddots & \vdots \\
            \frac{\tau_{1}}{K} & 0 & \ldots & 1-\tau_{n}
            \end{array}\right]
            \begin{bmatrix}
                v_{1, 1} \\
                v_{1, 2} \\
                \vdots \\
                v_{1, n} \\
            \end{bmatrix} \nonumber
            \\
            &= [v_{1, 1}, v_{1, 2} \dots, v_{1, n}]
            \begin{bmatrix}
                (1-\tau_1) v_{1, 1}  + \frac{1}{K} \sum_{j\in N_1} \tau_j v_{1, j} \\
                \vdots \\
                (1-\tau_i) v_{1, i}  + \frac{1}{K} \sum_{j\in N_i} \tau_j v_{1, j} \\
                \vdots \\
                (1-\tau_n) v_{1, n}  + \frac{1}{K} \sum_{j\in N_n} \tau_j v_{1, j} \\
            \end{bmatrix} \nonumber
            \\
            &=\sum_{i=1}^n (1-\tau_i) v_{1, i}^2 + \frac{1}{K}\sum_{i=1}^n\sum_{j\in N_i}\tau_i v_{1, i} v_{1, j} \nonumber \\
            &=1 - \left[\sum_{i=1}^n \tau_i v_{1, i}^2 -\sum_{i=1}^n\tau_i v_{1, i} \frac{1}{K}\sum_{j\in N_i} v_{1, j}\right].
            \quad // \quad \sum_{i=1}^n v_{1, i}^2 = 1
    \end{align}
Here, $N_i$ denotes a set of $K$ pixels closest to the pixel $i$.

        Furthermore, we use the Chebyshev's sum inequality in Lemma~\ref{lemma:csinequality} to determine the upper bound of Eq.~\eqref{eq:vAv}.
        In the second stage of the PI Attack,  we can consider that the perturbation on some pixels has not yet reached the epsilon ball, then $\tau_i=0$, and the gradient strength on this part of the pixels is often very small, that is, $|g_i|$ is small.
        In another part of the pixels with large values of $|g_i|$, they have often reached the edge of the epsilon ball, then in each subsequent steps, the perturbations is simply clipped to the spherical surface of the epsilon-ball, then we can roughly consider $\tau_i \propto |g_i|$, because the perturbation out of the epsilon-ball $\delta_i-\delta_i^{\text{clip}} $ is required to distributed to neighboring pixels.
        Moreover, in binary classification, if the loss is formulated as the cross-entropy loss on a sigmoid function, we have $g = \gamma_1 v_1$.
        In this way, the above claim can be roughly summarized as the following assumption in an ideal case.
        
    \begin{assumption} \label{assumption:tau-mono}
        We assume that $\tau_i(|v_{1, i}|)\in [0, 1]$ is a monotonically increasing function \emph{w.r.t.} $|v_{1, i}|$.
        Thus, if $v_{1, 1} \ge v_{1, 2} \ge \dots \ge v_{1, n}$, then we have $\tau_1v_{1, 1} \ge \tau_2v_{1, 2} \ge\dots \ge \tau_n v_{1, n}$.
    \end{assumption}
    
    Without loss of generality,  we assume $v_{1, 1} \ge v_{1, 2} \ge \dots \ge v_{1, n}$.
    Then, according to Assumption~\ref{assumption:tau-mono}, we have $\tau_1v_{1, 1} \ge \tau_2v_{1, 2} \ge\dots \ge \tau_n v_{1, n}$.
    In this way, applying Lemma~\ref{lemma:csinequality}~(Chebyshev's sum inequality), we obtain
    
    \begin{equation} \label{eq:vAv-csinequality}
        \begin{split}
            &\frac{1}{n}\sum_{i=1}^n \tau_i v_{1, i}^2 \ge \left(\frac{1}{n}\sum_{i=1}^n \tau_i v_{1, i} \right) \left(\frac{1}{n}\sum_{j=1}^n  v_{1, j} \right) \\
            \Rightarrow& \sum_{i=1}^n \tau_i v_{1, i}^2 - \left(\sum_{i=1}^n \tau_i v_{1, i} \right)\left(\frac{1}{n}\sum_{j=1}^n  v_{1, j} \right) \ge 0.
        \end{split}
    \end{equation}

    Moreover, for most pixels in an image, we can assume that mean value of the gradient of each pixel can be approximated as the mean value of gradients on the neighboring pixels of the $i$-th pixel.
    Moreover, in binary classification, if the loss is formulated as the cross-entropy loss on a sigmoid function, we have $g = \gamma_1 v_1$.
    \begin{assumption} \label{assumption:mean-v}
    We consider there are 4 neighbors or 8 neighbors for each image pixel, \emph{i.e.} $K=4$ or $8$.
        We assume that $\frac{1}{n}\sum_{i=1}^n  v_{1, i}\approx  \frac{1}{K}\sum_{j\in N_i} v_{1, j}$, where $N_i$ denotes the neighboring dimension of $v_{1, i}$.
    \end{assumption}
   
    Thus, based on Assumption~\ref{assumption:mean-v}, we have
    \begin{equation}
    \label{eq:vAv-csinequality-1}
    \sum_{i=1}^n \tau_i v_{1, i}^2 - \left(\sum_{i=1}^n \tau_i v_{1, i} \right)\left(\frac{1}{n}\sum_{j=1}^n  v_{1, j} \right)
    \approx
     \sum_{i=1}^n \tau_i v_{1, i}^2 - \sum_{i=1}^n \tau_i v_{1, i} \frac{1}{K}\sum_{j\in N_i} v_{1, j} 
         \ge 0
    \end{equation}

    In this way, based on Eq.~\eqref{eq:vAv-csinequality-1} and  Eq.~\eqref{eq:vAv}, the upper bound of $v_1^T A v_1$ can be represented as
    \begin{equation} 
    \begin{aligned}
    \label{eq:vAv-right}
        v_1^T A v_1 &=1 - \left[\sum_{i=1}^n \tau_i v_{1, i}^2 -\sum_{i=1}^n\tau_iv_{1, i} \frac{1}{K}\sum_{j\in N_i} v_{1, j}\right]
        \quad // \quad \text{According to Eq.~\eqref{eq:vAv}}
        \\        
        & \le 1-0
         \quad // \quad \text{According to  Eq.~\eqref{eq:vAv-csinequality-1}}
        \\
        &= 1.
   \end{aligned}
    \end{equation}

    \textbullet\; \textbf{The lower bound of  $v_1^T A v_1$.}
    Because the term $v_1^T A v_1\in\mathbb{R}$ is a scalar, we have $v_1^T A v_1 = (v_1^T A v_1)^T =v_1^T A^T v_1 $.
    Note that matrix $A$ is not a symmetric matrix, we construct a real symmetric matrix $B = A + A^T$ to derive the lower bound of the term $v_1^T A v_1$.
    To this end, we first use the Rayleigh quotient to derive the range of the term $v_1^T B v_1$.

    Specifically, the Rayleigh quotient of matrix $B$ is defined as $\frac{q^T B q}{q^T q}$, where $q\in\mathbb{R}^n$ is a non-zero vector.
    The  Rayleigh quotient is bounded by the smallest eigenvalue of matrix $B$ and the largest eigenvalue as follows.
    \begin{equation}
    \label{eq:Rayleigh quotient}
       \min_k \lambda^{(B)}_k\le  \frac{q^T B q}{q^T q} \le \max_k \lambda^{(B)}_k
    \end{equation}
    where $\lambda^{(B)}_k$ denotes the $k$-th eigenvalue of matrix $B$.
    According to Lemma~\ref{lemma:circle-theorem}~(Gershgorin circle theorem), each eigenvalue of matrix $B$ lies within at least one of the Gershgorin discs $D(B_{ii}, R_i)$, where $B_{ii}=A_{ii}+A_{ii}^{T}=2(1-\tau_i)$ and $R_i = \sum_{j\ne i} |B_{ij}|= \tau_i + \frac{\sum_{j\in N_i} \tau_j}{K}$.
    Note that $0\le \tau_i \le 1$.
    Thus, each eigenvalue of $B$ is bounded as follows.

    \begin{align}
    \label{eq:eigenvalue-in-pi-att}
        \min_i (B_{ii} - R_i) \le &\lambda^{(B)}_k \le \max_i (B_{ii} + R_i) \nonumber \\
        \Rightarrow\quad \min_i (2 - 3\tau_i - \frac{\sum_{j\in N_i} \tau_j}{K})  \le  &\lambda^{(B)}_k  \le \max_i (2 -\tau_i +\frac{\sum_{j\in N_i} \tau_j}{K})\nonumber \\
        \Rightarrow\quad  -2 \le &\lambda^{(B)}_k \le 3
    \end{align}

    Thus, based on Eq.~\eqref{eq:Rayleigh quotient} and Eq.~\eqref{eq:eigenvalue-in-pi-att}, we obtain
    \begin{align}
    \label{eq:vbv-1}
        v_1^T B v_1 &\ge (\min_k \lambda^{(B)}_k) v_1^T v_1 \ge -2. \quad //\quad v_1^T v_1=1
    \end{align}

    Note that $v_1^T B v_1 = v_1^T (A+A^{T}) v_1=v_1^T A v_1+(v_1^T A v_1)^{T}=2 v_1^T A v_1\in\mathbb{R}$.
    Then, based on Eq.~\eqref{eq:vbv-1}, we obtain the lower bound of the term $v_1^T A v_1$.
    \begin{align} \label{eq:vAv-left}
        v_1^T A v_1 &\ge -1.
    \end{align}

    Hence, combining Eq.~(\ref{eq:vAv-right}) and Eq.~(\ref{eq:vAv-left}), the range of the term $v_1^T A v_1$ can be represented as
    \begin{equation} \label{eq:vAv-inequality}
        -1 \le v_1^T A v_1 \le 1.
    \end{equation}

    Note that $\lambda_1, \beta_1, \beta_2 > 0$.
    According to Eq.~(\ref{eq:diff-interaction-pi}) and Eq.~(\ref{eq:vAv-inequality}), we have proven interactions $\sum_{a, b\in \Omega} I_{ab}(\delta^{(\textrm{multi})})$ for the multi-step attack are larger than interactions $\sum_{a, b\in \Omega} I_{ab}(\delta^{(\textrm{pi})})$ for the PI Attack.
    \begin{align}
        &\sum_{a, b\in \Omega} I_{ab}(\delta^{(\textrm{multi})}) -\sum_{a, b\in \Omega} I_{ab}(\delta^{(\textrm{pi})} ) \nonumber  \\
        &=\frac{\gamma_1^2}{\lambda_1} \left\{ \left[ \exp(\beta_1\lambda_1) \left[ \exp(\beta_2\lambda_1) + \exp(\beta_2\lambda_1 v_1^T A v_1  ) \right]-2\right]  \left[ \exp(\beta_1\lambda_1) \left[ \exp(\beta_2\lambda_1) - \exp(\beta_2\lambda_1 v_1^T A v_1  ) \right] \right]  \right\}
        \nonumber \\
        &\qquad \qquad\qquad\qquad// \quad \text{According to  Eq.~\eqref{eq:diff-interaction-pi}}
        \nonumber \\
        &= \frac{\gamma_1^2}{\lambda_1} \left\{ \left[ \exp(\beta_1\lambda_1) \left[ \sum_{t=0}^{+\infty} \frac{(\beta_2\lambda_1)^t}{t!} + \sum_{t=0}^{+\infty} \frac{(\beta_2\lambda_1 v_1^T A v_1)^t}{t!} \right]-2\right]  \left[ \exp(\beta_1\lambda_1) \bigg[ \exp(\beta_2\lambda_1) - \exp(\beta_2\lambda_1 v_1^T A v_1  ) \bigg] \right]  \right\}\nonumber \\
        &= \frac{\gamma_1^2}{\lambda_1} \left\{ \left[ \exp(\beta_1\lambda_1) \left[ 1 + \sum_{t=1}^{+\infty} \frac{(\beta_2\lambda_1)^t}{t!} + 1 + \sum_{t=1}^{+\infty} \frac{(\beta_2\lambda_1 v_1^T A v_1)^t}{t!} \right]-2\right]  \left[ \exp(\beta_1\lambda_1) \left[ \exp(\beta_2\lambda_1) - \exp(\beta_2\lambda_1 v_1^T A v_1  ) \right] \right]  \right\}\nonumber \\
        &= \underbrace{\frac{\gamma_1^2}{\lambda_1}}_{\ge 0} \left\{ \left[ \underbrace{\exp(\beta_1\lambda_1)}_{> 0} \underbrace{ \left[ \sum_{t=1}^{+\infty} \frac{ (\beta_2 \lambda_1)^t (1 + (v_1^T A v_1)^t)}{t!} \right]}_{\ge 0} + \underbrace{2 (\exp(\beta_1\lambda_1)-1)}_{\ge 0}\right]  \left[ \underbrace{\exp(\beta_1\lambda_1) }_{> 0} \underbrace{\left[ \exp(\beta_2\lambda_1) - \exp(\beta_2\lambda_1 v_1^T A v_1  ) \right]}_{\ge 0} \right]  \right\} \nonumber \\
        & \ge 0.
        \quad // \quad \text{According to  Eq.~\eqref{eq:vAv-inequality}}
    \end{align}

Thus, Proposition~\ref{theoremappendix:pi} is proven.
\end{proof}

\section{Proof that the perturbation generated by the single-step attack on the DNN trained  by the IA Attack exhibit smaller interactions than that generated by the single-step attack on the normally-trained DNN}
\label{append:ia}
\subsection{IA Attack}
It is found in~\cite{zhu2022rethinking} that out-of-distribution (OOD) adversarial examples exhibit high adversarial transferability.
Based on this finding, given a normally-traiend DNN, the IA Attack~\cite{zhu2022rethinking} is proposed to finetune it, in order to enable it to generate more OOD adversarial examples.
Note that the IA Attack is not an attacking method for directly generating adversarial examples.
Instead, the IA Attack finetunes a normally-trained DNN, and then generates adversarial examples on the finetuned DNN.

Given the input $x$ with label $y$ and a DNN parameterized by $\theta$, we use $p_{\theta}(y\mid {x})$ to denote the probability that $x$ is classified to the ground-truth category $y$.
The objective of the IA Attack is given as follows. 
\begin{equation} \label{eq:ori-obj-ia}
    \theta^{\text{(ia)}}  = \arg\max_{\theta}   \mathbb{E}_{x}\left[ -\nabla_{x}\cdot\frac{\nabla_{x} \log ( p_{\theta}(y\mid x))}{\left\|\nabla_{x} \log ( p_{\theta}(y\mid x))\right\|_{2}} \right]
\end{equation}

Note that according to~\cite{zhu2022rethinking}, Eq.\eqref{eq:ori-obj-ia} does not simply compute the higher-order derivative on the normalized gradient $\frac{\nabla_{x} \log ( p_{\theta}(y \mid x))}{\left\|\nabla_{x} \log ( p_{\theta}(y \mid x))\right\|_{2}}$, in a similar way of computing a Hessian matrix.
Instead, $\nabla_{x}\cdot\frac{\nabla_{x} \log ( p_{\theta}(y \mid x))}{\left\|\nabla_{x} \log ( p_{\theta}(y \mid x))\right\|_{2}}$  is equivalent to $(\frac{\partial }{\partial x_1},\frac{\partial }{\partial x_2}, \dots, \frac{\partial }{\partial x_n} ) \frac{\nabla_{x} \log ( p_{\theta}(y \mid x))}{\left\|\nabla_{x} \log ( p_{\theta}(y \mid x))\right\|_{2}}  = \sum_{i=1}^n \frac{\partial }{\partial x_i} ((\frac{\nabla_{x} \log ( p_{\theta}(y \mid x))}{\left\|\nabla_{x} \log ( p_{\theta}(y \mid x))\right\|_{2}})_i)$.

We use  $\theta^{\text{(ia)}}$ to denote the DNN {trained} by the IA Attack.
Note that in~\cite{zhu2022rethinking}, the DNN uses SoftPlus activation functions. Here, to simplify our analysis, we conduct our analysis on the DNN with ReLU activation functions.

Then, we compare adversarial perturbations generated by the single-step attack on the DNN {trained} by the IA Attack with those generated by the single-step attack on the normally-trained DNN.
Note that we denote the the classification loss for attacking computed on the DNN $\theta$ as $\textit{Loss}_{\theta} (x)$.

\textbf{Single-step attack on the  DNN {trained} by the IA Attack.}
We generate the adversarial perturbation by conducting the single-step attack on the DNN $\theta^{\text{(ia)}}$ {trained} by the IA Attack.
In this way, the perturbation is expressed as follows.
\begin{align}
    \delta^{(\text{single})}_{\theta^{(\text{ia})}} \triangleq \eta g_{\theta^{\text{(ia)}}}
\end{align}
where $\eta \in \mathbb{R}$ denotes the step size, and $g_{\theta^{\text{(ia)}}} = \nabla_x \textit{Loss}_{\theta^{\text{(ia)}}} (x)$ denotes the gradient \emph{w.r.t.} the input sample on the DNN $\theta^{\text{(ia)}}$ {trained} by the IA Attack.
Then, the sum of pairwise interactions inside $\delta^{(\text{single})}_{\theta^{(\text{ia})}}$ is given as follows.
\begin{align}
    \sum_{a, b\in\Omega} \left[I_{ab}\left(\delta^{(\text{single})}_{\theta^{(\text{ia})}}\right)\right] = (\delta^{(\text{single})}_{\theta^{(\text{ia})}})^T H_{\theta^{(\text{ia})}} \delta^{(\text{single})}_{\theta^{(\text{ia})}} = \eta^2 (g_{\theta^{(\text{ia})}})^T H_{\theta^{(\text{ia})}} g_{\theta^{(\text{ia})}}
\end{align}

\textbf{Single-step attack on the normally-trained DNN.}
We use $\theta^{\text{(nor)}}$ to denote the normally-trained DNN.
Similarly, we generate the adversarial perturbation by conducting the single-step attack on the  normally-trained DNN $\theta^{\text{(nor)}}$.
In this way, the perturbation is given as follows.
\begin{align}
    \delta^{(\text{single})}_{\theta^{\text{(nor)}}} \triangleq \eta g_{\theta^{\text{(nor)}}}
\end{align}
 where $g_{\theta^{\text{(nor)}}} = \nabla_x \textit{Loss}_{\theta^{\text{(nor)}}} (x)$ denotes the gradient \emph{w.r.t.} the input sample on the normally-trained DNN $\theta^{\text{(nor)}}$.
Then, the sum of pairwise interactions inside $\delta^{(\text{single})}_{\theta^{\text{(nor)}}}$ is given as follows.
\begin{align}
    \sum_{a, b\in\Omega}\left[I_{ab} \left(\delta^{(\text{single})}_{\theta^{\text{(nor)}}}\right)\right] = (\delta^{(\text{single})}_{\theta^{\text{(nor)}}})^T H_{\theta^{\text{(nor)}}} \delta^{(\text{single})}_{\theta^{\text{(nor)}}} = \eta^2 (g_{\theta^{\text{(nor)}}})^T H_{\theta^{\text{(nor)}}} g_{\theta^{\text{(nor)}}}
\end{align}

In the following subsections, under two major conditions that (1) the prediction probability of the ground truth category $p_{\theta^{\text{(nor)}}}(y\mid x)$ and $p_{\theta^{\text{(nor)}}}(y \mid x)$ are the same, and both are similar and confident enough; (2) the vectors corresponding to different categories are orthogonal to each other, we can prove that $\sum_{a, b\in\Omega} \left[I_{ab}\left(\delta^{(\text{single})}_{\theta^{(\text{ia})}}\right)\right] \le \sum_{a, b\in\Omega}\left[I_{ab} \left(\delta^{(\text{single})}_{\theta^{\text{(nor)}}}\right)\right]$.
Before the proof, we first introduce some lemmas.

\subsection{Proof of Lemma~\ref{lemma:same-eigenvalue}} \label{sec:proof-lemma-same-eigenvalue}
In this subsection, we derive the relationship between the eigenvalue of $H = \frac{\partial^2\textit{Loss}}{\partial x \partial x^T}$ and the eigenvalue of $H_z = \frac{\partial^2\textit{Loss}}{\partial z \partial z^T}$, where $x$ is the input of the DNN, and $z\in \mathbb{R}^c$ denote the feature input into the softmax layer to generate the classification probability in the DNN.

In this subsection, we assume that the vectors corresponding to different categories in thd DNN trained by IA Attack are mutually orthogonal.
The reasons for this are as follows.
According to Assumption~\ref{assumption:activation}, given the input $x$, the output before the softmax operation of the $i$-th category is given as  $z_i = (\tilde{W}_x)_i^T x + b_i$, where $(\tilde{W}_x)_i$ denotes the $i$-th column of $\tilde{W}_x$.
Thus, we can consider the class information of the $i$-th category is encoded by $(\tilde{W}_x)_i$. 
In this study, we roughly consider that different categories share very little information, so we can assume that weight vectors of different categories $(\tilde{W}_x)_i$ are orthogonal to each other.
This assumption has been partially supported in previous studies~\cite{papyan2020prevalence}.
Moreover, we assume that  $(\tilde{W}_x)_i$ of different categories have similar strength.

\begin{lemma} \label{lemma:same-eigenvalue}
We are given $H = \tilde{W}_x H_z  \tilde{W}_x^T$, where $H_z = \frac{\partial^2\textit{Loss}}{\partial z \partial z^T}$.
        We use $\lambda_i$ ($1\le i \le n$) and $\lambda_j'$ ($1\le j \le c$), to denote the eigenvalue of $H$ and $ H_z$, respectively.
        Based on Assumption~\ref{assum:ortho-weights}, we have $\forall 1\le i\le c, \lambda_i = \kappa\lambda_i'$; $\forall i > c, \lambda_i=0$.
\end{lemma}

\begin{proof}
Because $H_z$ is a symmetric matrix, it can be decomposed into 
\begin{equation} \label{eq:H-z-decompose}
        H_z=V' \Lambda' {V'}^{\text{T}},
\end{equation}
where $\Lambda=\text{diag}(\lambda_1',\lambda_2',\ldots,\lambda_c')$ is a diagonal matrix, whose diagonal elements are the corresponding eigenvalues, $\Lambda_{ii}' = \lambda_i'$.
The square matrix $V'=[v_1',v_2',\ldots,v_c'] \in \mathbb{R}^{c \times c}$ contains $c$ mutually orthonormal eigenvectors, \emph{i.e.} $\forall i\neq j, {v_i'}^Tv_j'=0$; and $\forall i, {v_i'}^T v_i'=1$.
Thus, $V'$ is an orthonormal matrix, \emph{i.e.} $V' {V'}^T = {V'}^T V =I$.

In this way, $H$ can be represented as follows.
\begin{align}
    H &= \tilde{W}_x H_z \tilde{W}_x^T \nonumber \\
    &= \tilde{W}_x V' \Lambda' {V'}^{\text{T}} \tilde{W}_x^T \nonumber \\
    &= U \Lambda' U^T,
\end{align}
where we use $U\in\mathbb{R}^{n\times c}$ to denote $\tilde{W}_x V'$ for simplicity.
In this way, we have
\begin{align}
    H &= U \Lambda' U^T \nonumber \\
    &= [u_1, u_2, \dots, u_c] \Lambda'
    \begin{bmatrix}
        u_1^T \\
        u_2^T \\
        \vdots \\
        u_n^T
    \end{bmatrix} \nonumber \\
    &= [\lambda_1' u_1,\lambda_2' u_2, \dots, \lambda_c' u_c]
    \begin{bmatrix}
        u_1^T \\
        u_2^T \\
        \vdots \\
        u_n^T
    \end{bmatrix} \nonumber \\
    &= \sum_{j=1}^c \lambda_j' u_j u_j^T
\end{align}
where $u_j \in \mathbb{R}^n$

Based on Assumption~\ref{assum:ortho-weights}, we have 
    \begin{equation}
        \tilde{W}_x^T \tilde{W}_x = \kappa\cdot I.
    \end{equation}

Thus, 
\begin{equation}
    U^T U =  {V'}^T \tilde{W}_x^T\tilde{W}_xV' = \kappa\cdot I,
\end{equation}
which means that $\forall j'\neq j, u_{j'}^T u_j=0$; and $\forall j, u_j^T u_j = \kappa$.

We now prove that $u_i$ and $\kappa\lambda_i'$ are the eigenvector and eigenvalue of $H$, respectively.
\begin{align}
    \forall 1\le i\le c, \quad H u_i &= (\sum_{j=1}^c \lambda_j' u_j u_j^T) u_i = \kappa\lambda_i' u_i
\end{align}

Thus,  we have $\forall 1\le i\le c, \lambda_i = \kappa\lambda_i'$.
Moreover, because $\text{rank}(H) \le \min (\text{rank}(H_z\in\mathbb{R}^{c\times c}), \text{rank}(\tilde{W}_x\in\mathbb{R}^{n\times c})) \le c$, we have $\forall i > c, \lambda_i=0$.
Thus, this lemma is proven.
\end{proof}
$ $\newline

\subsection{Proof of Lemma~\ref{lemma:ia-essence}}
In this subsection, we prove that the objective function of the IA Attack on the specific sample $x$ can be roughly understood as the minimization of $\sum_{i\ne y} p_i^2$, where $p_i$ is the prediction probability of $x$ belonging to category $i$ as Eq.~\eqref{eq:softmax-ce}.
In this way, we simply investigate the pairwise interaction between perturbation units, which are generated by a DNN on the specific sample $x$.
Let us simplify the story by focusing on the case of simply applying the distinctive loss function of the IA Attack in Eq.~\eqref{eq:ori-obj-ia} to the input sample $x$.
As mentioned above, we use Assumption~\ref{assum:ortho-weights} and Definition~\ref{def:balance} to ensure the fair comparison between normally-trained DNNs and DNNs trained by the IA Attack.

\begin{lemma} \label{lemma:ia-essence}
Let the DNN trained by the IA Attack satisfies Assumption~\ref{assum:ortho-weights} and Definition~\ref{def:balance}.
If the classification loss  for attacking $\textit{Loss}_{\theta}(x)$  is formulated as the cross-entropy loss in Eq.~\eqref{eq:softmax-ce},
the objective function of the IA Attack on the specific sample $x$ can be roughly understood as the minimization of $\sum_{i\ne y} p_i^2$, where $p_i$ is the prediction probability of $x$ for belonging to category $i$ as Eq.~\eqref{eq:softmax-ce}.
\end{lemma}

\begin{proof}
    Given the input $x\in\mathbb{R}^n$ and a DNN parameterized by $\theta$,  the cross-entropy loss for classification can be formulated as $\textit{Loss}_{\theta}(x) = -\log\left( p_{\theta}(y\mid {x}) \right)$.
    Thus, the objective of the IA Attack in Eq.~\eqref{eq:ori-obj-ia} on the specific sample $x$ can be re-written as follows.
    \begin{align}
        &\max_{\theta}  -\nabla_{x}\cdot\frac{\nabla_{x} \log ( p_{\theta}(y\mid x))}{\left\|\nabla_{x} \log ( p_{\theta}(y \mid x))\right\|_{2}} \nonumber \\
\Rightarrow & \max_{\theta} \nabla_{x}\cdot\frac{\nabla_{x} (-\log ( p_{\theta}(y \mid x)))}{\left\|\nabla_{x} (-\log ( p_{\theta}(y \mid x)))\right\|_{2}} \nonumber \\
\Rightarrow &\max_{\theta} \nabla_{x}\cdot\frac{\nabla_{x} \textit{Loss}_{\theta}(x)}{\left\|\nabla_{x} \textit{Loss}_{\theta}(x)\right\|_{2}}  \nonumber \\
\Rightarrow & \max_{\theta}\nabla_{x}\cdot\frac{g_{\theta}}{\left\|g_{\theta}\right\|_{2}},
    \end{align}
    where $g_{\theta} \triangleq \nabla_{x} \textit{Loss}_{\theta}(x)$.
    To simplify the notation, we omit the subscript $\theta$.
    Note that according to~\cite{zhu2022rethinking}, here, $\nabla_{x}\cdot\frac{g}{\left\|g\right\|_{2}}$ is not to compute a matrix of second derivatives.
    Instead,
    \begin{align} \label{eq:partial_g_norm_g_partial_x}
         \nabla_{x}\cdot\frac{g}{\left\|g\right\|_{2}} &\triangleq  (\frac{\partial }{\partial x_1},\frac{\partial }{\partial x_2}, \dots, \frac{\partial }{\partial x_n} ) (\frac{g_1}{\|g\|_2}, \frac{g_2}{\|g\|_2}, \dots, \frac{g_n}{\|g\|_2})^T  \nonumber \\
    &= \sum_{i=1}^n \frac{\partial}{\partial x_i}(\frac{g_i}{\|g\|_2}) \nonumber \\
    &= \sum_{i=1}^n(\frac{\frac{\partial g_i}{\partial x_i} \|g\|_2 - g_i\frac{\partial \|g\|_2}{\partial x_i}}{\|g\|_2^2}),
    \end{align}
    where $\frac{\partial \|g\|_2}{\partial x_i}$ is computed as follows.
    \begin{align} \label{eq:partial_norm_g_partial_x}
        \frac{\partial \|g\|_2}{\partial x_i} &= \frac{\partial}{\partial x_i} \left( \sqrt{\sum_{j=1}^n g_j^2}\right), \nonumber \\
    &= \frac{1}{2} \frac{1}{ \sqrt{\sum_{j=1}^n g_j^2}} \left[\frac{\partial}{\partial x_i} \left(\sum_{j=1}^n g_j^2 \right) \right] \nonumber \\
    &= \frac{1}{2} \frac{1}{ \sqrt{\sum_{j=1}^n g_j^2}} \left[ \left(\sum_{j=1}^n \frac{\partial  g_j^2}{\partial x_i} \right) \right] \nonumber \\
    &= \frac{1}{2} \frac{1}{ \sqrt{\sum_{j=1}^n g_j^2}} \left[ \left(\sum_{j=1}^n 2 g_j \frac{\partial  g_j}{\partial x_i} \right) \right] \nonumber \\
    &= \frac{\sum_{j=1}^n g_j \frac{\partial g_j}{\partial x_i}}{\|g\|_2}
    \end{align}
    
    Substituting Eq.~(\ref{eq:partial_norm_g_partial_x}) back to Eq.~(\ref{eq:partial_g_norm_g_partial_x}), we have
    \begin{align} \label{eq:our-ia-obj}
        & \nabla_{x}\cdot\frac{g}{\left\|g\right\|_{2}}  \nonumber \\
    &=\sum_{i=1}^n(\frac{\frac{\partial g_i}{\partial x_i} \|g\|_2 - g_i \frac{\sum_{j=1}^n g_j \frac{\partial g_j}{\partial x_i}}{\|g\|_2}}{\|g\|_2^2})  \nonumber \\
    &= \sum_{i=1}^n(\frac{H_{ii}\|g\|_2^2 - {\sum_{j=1}^n  g_i g_j H_{ij}}}{\|g\|_2^3} ) \nonumber \\
    &= (\frac{\sum_{i=1}^n H_{ii}\|g\|_2^2 - \sum_{i=1}^n{\sum_{j=1}^n  g_i g_j H_{ij}}}{\|g\|_2^3} ) \nonumber \\
    &= \frac{ \text{trace}(H) g^T g - g^T H g}{\|g\|_2^3}  \nonumber \\
    &= \frac{  g^T \left(\text{trace}(H) I - H \right) g}{\|g\|_2^3}  
    \end{align}
    
    In order to investigate the condition when Eq.~\eqref{eq:our-ia-obj} reaches its maximum,
    we first use the  Rayleigh quotient, as defined in Eq.~\eqref{eq:Rayleigh quotient}, to estimate the value range of Eq.~\eqref{eq:our-ia-obj}.
    We consider $Q=\text{trace}(H) I - H$, which is a symmetric matrix.
    Then, according to Eq.~\eqref{eq:Rayleigh quotient}, we have $ \lambda_{\text{min}}^{(Q)} g^T g\le g^T Q g \le  \lambda_{\text{max}}^{(Q)}g^T g$, where $\lambda_{\text{min}}^{(Q)}$ and $\lambda_{\text{max}}^{(Q)}$ represent the minimal and the maximal eigenvalues of $Q$, respectively.
    Note that according to Eq.~\eqref{eq:H-decompose}, we have $H = V \Lambda V^T$, where $V V^T = I$.
    In this way, $Q = \text{trace}(H) I - H = \text{trace}(H) V V^T - V\Lambda V^T = V (\text{trace}(H)I - \Lambda) V^T$.
    Thus, the $i$-th eigenvalue of $Q$ is given as $\lambda_{i}^{(Q)} = \text{trace}(H) - \lambda_i$, where $\lambda_i$ is the $i$-th eigenvalue of $H$.
    In this way, we have
    \begin{align} \label{eq:aai-p-form}
         &\frac{(\text{trace}(H) - \max_i \lambda_i) g^T g}{ \|g\|_2^3} = \frac{\lambda_{\text{min}}^{(Q)} g^T g}{ \|g\|_2^3} \le \frac{  g^T Q g}{\|g\|_2^3} \le \frac{\lambda_{\text{max}}^{(Q)} g^T g}{ \|g\|_2^3} = \frac{(\text{trace}(H) - \min_i \lambda_i) g^T g}{ \|g\|_2^3} \nonumber \\
         &\Rightarrow  \frac{(\text{trace}(H) - \max_i \lambda_i) }{ (g^T g)^{\frac{1}{2}}} \le \frac{  g^T \left(\text{trace}(H) I - H \right) g}{(g^T g)^{\frac{3}{2}}} \le \frac{(\text{trace}(H) - \min_i \lambda_i) }{ (g^T g)^{\frac{1}{2}}}
    \end{align}

    Moreover, if the loss function for classification $Loss(x)$ is formulated as the cross-entropy loss, according to Lemma~\ref{lemma:g-H-in-p-form}, the Hessian matrix $H = \nabla_x^2 \textit{Loss}(x)$ and the gradient $g = \nabla_x \textit{Loss}(x) $ are given as 
    \begin{align}
        H &= \tilde{W}_x \frac{\partial^2 L}{\partial z \partial z^T } \tilde{W}_x^T =  \tilde{W}_x (\text{diag}(p) - p p^T) \tilde{W}_x^T, \label{eq:ia-H-form} \\
        g &= \tilde{W}_x \frac{\partial L}{\partial z}  = \tilde{W}_x (p- Y), \label{eq:ia-g-form}
    \end{align}
    where $p\in [0, 1]^{c}$ denotes the probabilities of the prediction for $c$ categories, and $Y\in\{0, 1\}^{c}$ is a one-hot vector that $Y_y = 1; \forall i\ne y, Y_i = 0$.

    Moreover, based on Assumption~\ref{assum:ortho-weights}, we have 
    \begin{equation} \label{eq:ia-w-ortho}
        \tilde{W}_x^T \tilde{W}_x = \kappa \cdot I,
    \end{equation}
    where $\kappa\in\mathbb{R}$ is a constant scalar.

    Based on Lemma~\ref{lemma:g-H-in-p-form} and Lemma~\ref{lemma:same-eigenvalue}, we can derive values of  $\text{trace}(H)$ and $g^T g$, and determine the value range of $\lambda_i$ as follows.
    
    \textbullet{ The value of $\text{trace}(H)$.} 
    Note that for a square matrix $H\in\mathbb{R}^{n\times n}$, we have $\text{trace}(H) = \sum_{i=1}^n \lambda_i$.
    According to  Lemma~\ref{lemma:g-H-in-p-form} and Lemma~\ref{lemma:same-eigenvalue}, we have
    \begin{equation}
        \text{trace}(H) = \sum_{i=1}^n \lambda_i =\sum_{i=1}^c \kappa\lambda_i'= \kappa\cdot\text{trace}(\frac{\partial^2 \textit{Loss}}{\partial z \partial z^T }) = \kappa\cdot \text{trace}(\text{diag}(p) - p p^T) = \kappa \left(\sum_{i=1}^c (p_i - p_i^2) \right) = \kappa \left(1 - \sum_{i=1}^c p_i^2 \right).
    \end{equation}
    
    \textbullet{ The value of $g^T g$.} According to  Lemma~\ref{lemma:g-H-in-p-form} and Assumption~\ref{assum:ortho-weights}, we have
    \begin{equation}
        g^T g  = (p-Y)^T  \tilde{W}_x^T \tilde{W}_x (p- Y) = \kappa (p-Y)^T (p- Y) =  \kappa\left(\sum_{i\ne y} p_i^2 + (p_y-1)^2 \right).
    \end{equation}
    
    \textbullet{ The value range of $\lambda_i$.} As proven in Appendix~\ref{appendix:positive-semi-definite}, based on Lemma~\ref{lemma:circle-theorem} (Gershgorin circle theorem), we have $0\le \lambda_i' \le  \max_j 2p_j (1-p_j)$, where $\lambda'_i$ is the eigenvalue of $\frac{\partial^2 \textit{Loss}}{\partial z \partial z^T}$. According to Lemma~\ref{lemma:same-eigenvalue}, we have $\forall 1\le i\le c, 0\le \lambda_i = \kappa \lambda_i' \le  \kappa (\max_j 2p_j (1-p_j))$.
     According to Definition~\ref{def:balance}, we assume that the input sample is correctly classified with high confidence that $p_y \ge 0.5$. 
    In this way, $\forall j\ne y, p_j \le 1- p_y\le 0.5$.
    Thus, $\max_j 2p_j (1-p_j) \le 2 p_y (1-p_y)$.
    In this way, we have 
    \begin{align}
        &\forall i, \lambda_i \ge 0 \Rightarrow \min_i \lambda_i\ge 0 \Rightarrow -\min_i \lambda_i \le 0 \\
        &\forall i, \lambda_i \le  \kappa (\max_j 2p_j (1-p_j)) \le 2\kappa p_y (1-p_y)       \Rightarrow  \max_i \lambda_i\le 2 \kappa p_y(1-p_y)    \Rightarrow -\max_i  \lambda_i \ge 2\kappa p_y (p_y - 1)
    \end{align}
    
    In this way, Eq.~\eqref{eq:aai-p-form} can be written as follows.
    \begin{equation} \label{eq:bound-of-ia}
       {\sqrt{\kappa}}\cdot\frac{1 - \sum_{i\ne y} p_i^2 -p_y^2 + 2p_y (p_y-1) }{( \sum_{i\ne y} p_i^2 + (p_y-1)^2)^{\frac{1}{2}}} \le \frac{  g^T \left(\text{trace}(H) I - H \right) g}{(g^T g)^{\frac{3}{2}}} \le {\sqrt{\kappa}}\cdot\frac{1 - \sum_{i\ne y} p_i^2 -p_y^2 }{( \sum_{i\ne y} p_i^2 + (p_y-1)^2)^{\frac{1}{2}}}
    \end{equation}

    Then, we prove that both the lower bound and the upper bound in Eq.\eqref{eq:bound-of-ia} monotonically decrease along with the increase of $\sum_{i\ne y} p_i^2$.
    
    \textbullet{ The lower bound is monotonically decreasing with $\sum_{i\ne y} p_i^2$. We use $e$ to denote $\sum_{i\ne y} p_i^2$, \emph{i.e.} $e = \sum_{i\ne y} p_i^2 \ge 0 \in\mathbb{R}$.
    Let $f(e) = \frac{1-e -p_y^2 + 2p_y (p_y-1)}{( e + (p_y-1)^2)^{\frac{1}{2}}}$.}
    The derivative of $f(e)$ is given as 
    \begin{align}
        \frac{df(e)}{de} &= \frac{-\sqrt{e + (p_y-1)^2} - \frac{1}{2}\frac{1 - e -p_y^2 + 2p_y (p_y-1)}{\sqrt{e + (p_y-1)^2}}}{ e + (p_y-1)^2} \nonumber \\
        &= \frac{-(e + (p_y-1)^2) - \frac{1}{2}\left(1 - e -p_y^2 + 2p_y (p_y-1) \right)}{ \left(e + (p_y-1)^2 \right)^{\frac{3}{2}}} \nonumber \\
        &= \frac{-\frac{1}{2} \left[ 2 \left(e + (p_y-1)^2 \right) + 1 - e-p_y^2 + 2p_y (p_y-1)\right]}{ \left(e + (p_y-1)^2 \right)^{\frac{3}{2}}} \nonumber \\
        &= \frac{-\frac{1}{2} \left[ 2e + 2(p_y^2 -2p_y + 1)  + 1 - e-p_y^2 + 2p_y (p_y-1) \right]}{ \left(e + (p_y-1)^2 \right)^{\frac{3}{2}}} \nonumber \\
        &= \frac{-\frac{1}{2} \left[ e + 3 (p_y-1)^2\right]}{ \left(e + (p_y-1)^2 \right)^{\frac{3}{2}}} \le 0
    \end{align}
    
    \textbullet{ The upper bound is monotonically decreasing with $\sum_{i\ne y} p_i^2$. We use $e$ to denote $\sum_{i\ne y} p_i^2$, \emph{i.e.} $e = \sum_{i\ne y} p_i^2 \ge 0 \in\mathbb{R}$.
    Let $g(e) = \frac{1-e -p_y^2 }{( e + (p_y-1)^2)^{\frac{1}{2}}}$.}
    The derivative of $g(e)$ is given as 
    \begin{align}
        \frac{dg(e)}{de} &= \frac{-\sqrt{e + (p_y-1)^2} - \frac{1}{2}\frac{1 - e -p_y^2 }{\sqrt{e + (p_y-1)^2}}}{ e + (p_y-1)^2} \nonumber \\
        &= \frac{-(e + (p_y-1)^2) - \frac{1}{2}\left(1 - e -p_y^2 \right)}{ \left(e + (p_y-1)^2 \right)^{\frac{3}{2}}} \nonumber \\
        &= \frac{-\frac{1}{2} \left[ 2 \left(e + (p_y-1)^2 \right) + 1 - e-p_y^2 \right]}{ \left(e + (p_y-1)^2 \right)^{\frac{3}{2}}} \nonumber \\
        &= \frac{-\frac{1}{2} \left[ 2e + 2(p_y^2 -2p_y + 1)  + 1 - e-p_y^2 \right]}{ \left(e + (p_y-1)^2 \right)^{\frac{3}{2}}} \nonumber \\
        &= \frac{-\frac{1}{2} \left[ e + (p_y-1)(p_y-3) \right]}{ \left(e + (p_y-1)^2 \right)^{\frac{3}{2}}} \le 0 
        \quad // \quad \text{Given $p_y\in[0, 1]$, $(p_y-1)(p_y-3)\in [0, 3] \ge 0$}
    \end{align}
    
    Because both the lower and upper bounds of $\frac{  g^T \left(\text{trace}(H) I - H \right) g}{(g^T g)^{\frac{1}{2}}}$ are
    monotonically decreasing with $\sum_{i\ne y} p_i^2$, the term $\frac{  g^T \left(\text{trace}(H) I - H \right) g}{(g^T g)^{\frac{1}{2}}}$ can be maximized by minimizing $\sum_{i\ne y} p_i^2$, as follows.
    
    \begin{equation} \label{eq:sum-p-square}
        \min \sum_{i\ne y} p_i^2 \quad \text{s.t.} \quad \sum_{i\ne y} p_i = 1- p_y
    \end{equation}
    
    Thus, the lemma is proven.
\end{proof}

\subsection{Proof of Proposition~\ref{propositionappendix:ia}}
In the this subsection, under two major conditions that (1) the prediction probability of the ground truth category $p_{\theta^{\text{(nor)}}}(y\mid x)$ and $p_{\theta^{\text{(nor)}}}(y \mid x)$ are the same, and both are similar and confident enough; (2) the vectors corresponding to different categories are orthogonal to each other, we can prove that $\sum_{a, b\in\Omega} \left[I_{ab}\left(\delta^{(\text{single})}_{\theta^{(\text{ia})}}\right)\right] \le \sum_{a, b\in\Omega}\left[I_{ab} \left(\delta^{(\text{single})}_{\theta^{\text{(nor)}}}\right)\right]$.

\begin{propositionappendix}[Perturbations generated by the single-step attack on the DNN trained by the IA Attack exhibit smaller interactions  than those generated by the single-step attack on the normally-trained DNN.]
    \label{propositionappendix:ia}
    We are given $\theta^{(\text{ia})}$ as the DNN trained by the IA Attack, and $\theta^{(\text{nor})}$ as the normally-trained DNN.
    Let the two DNNs both satisfy Assumption~\ref{assum:ortho-weights} and Definition~\ref{def:balance}.
    {If the classification loss  for attacking $\textit{Loss}_{\theta}(x)$  is formulated as the cross-entropy loss in Eq.~\eqref{eq:softmax-ce}}, then we have $\sum_{a, b\in\Omega} \left[I_{ab}\left(\delta^{(\text{single})}_{\theta^{(\text{ia})}}\right)\right] \le  \sum_{a, b\in\Omega} \left[I_{ab}\left(\delta^{(\text{single})}_{\theta^{\text{(nor)}}}\right) \right]$.
  \end{propositionappendix}

\begin{proof}
The basic flowchart of the proof is as follows. 
We first prove that compared with the normally-trained DNN, the DNN trained by the IA Attack typically realizes more balance predictions on different categories (except for the ground-truth category) than the normally-trained DNN. 
In other words, the DNN trained by the IA Attack usually makes classification scores of different categories on the sample x more similar to each other. Then, according to Theorem~\ref{theoremappendix:balance}, we can prove that perturbations generated by the single-step attack on the DNN trained by the IA Attack have smaller interactions  than perturbations generated by the single-step attack on the normally-trained DNN.

    According to Lemma~\ref{lemma:ia-essence}, the objective function of the IA Attack on the specific sample $x$ can be roughly understood as the minimization of $\sum_{i\ne y} p_i^2$, where $p_i$ is the prediction probability for category $i$.
    In this way, we can prove that, given an input sample $x$, the IA Attack  essentially pushes the prediction probabilities of $x$ to be more balanced.

    Based on Cauchy–Schwarz inequality, we can prove $\sum_{i\ne y} p_i^2$ reaches its minimal value when prediction probabilities of all categories except for the ground-truth category are equal, \emph{i.e.} $\forall i\ne y, p_i = \frac{1-p_y}{c-1}$.

    \begin{lemma}[Cauchy–Schwarz inequality] $\left( \sum_{i=1}^n a_i^2 \right) \left( \sum_{i=1}^n b_i^2 \right) \ge \left( \sum_{i=1}^n a_i b_i\right)^2$
    \end{lemma}

    Let $a_i =1$ and $b_i=p_i$, we have
    \begin{equation}
    \begin{split}
        &(c-1) \left( \sum_{i\ne y} p_i^2 \right) \ge \left( \sum_{i\ne y} p_i\right)^2 \\
    \Rightarrow & \sum_{i\ne y} p_i^2 \ge \frac{1}{c-1 } \left( 1-p_y\right)^2,
    \end{split}
\end{equation}
where the equality holds if $\forall i\ne y, p_i = \frac{1-p_y}{c-1}$.
In other words, $\sum_{i\ne y} p_i^2$ is minimized when prediction probabilities of other categories are equal, except for the ground truth category.

Given the input $x$, we use $p^{\text{(ia)}}(x)$ to denote the prediction probability predicted by the DNN $\theta^{\text{(ia)}}$.
In summary, the IA Attack  essentially pushes the prediction probabilities of $x$ towards the uniform distribution of prediction probabilities, \emph{i.e.}, towards the ideal case that $\forall i,j\ne y,  p^{\text{(ia, ideal)}}_i(x) = p^{\text{(ia, ideal)}}_j(x)=\frac{1-p_y}{c-1}$.
To this end, we can rough consider $p^{\text{(ia)}}(x)$ is the more balanced prediction probability than $p^{\text{(nor)}}(x)$, which satisfies the requirement in Definition~\ref{def:balance}. 
In this way, according to Theorem~\ref{theoremappendix:balance}, we have

\begin{align}
\sum_{a, b\in\Omega} \left[{I_{a b}(\delta^{(\text{single})}_{\theta^{(\text{nor})}})} \right] - \sum_{a, b\in\Omega}\left[{I_{a b}(\delta^{(\text{single})}_{\theta^{(\text{ia})}})}\right] \ge 0
\end{align}

Thus, the proposition is proven.
\end{proof}

\section{Proof that the RAP Attack generates perturbations with smaller interactions than the multi-step attack}
\label{append:rap}
In this section, we prove Proposition~\ref{pro:rap} in Section~\ref{sec:explaining}, which shows that perturbations of the RAP Attack~\cite{qin2022boosting} exhibit smaller interactions than those of the multi-step attack.
Before the proof of Proposition~\ref{pro:rap}, we revisit the RAP Attack.

\subsection{RAP Attack}

The RAP Attack aims to generate adversarial perturbations located in the flat region of the loss landscape.
The objective of the untargeted RAP Attack~\cite{qin2022boosting} can be formulated as follows.
\begin{equation} \label{eq:ojb-rap}
    \underset{\|\delta_{\text{rap-ori}}\|_p \le \epsilon}{\text{max}}\,\, \underset{\|r_{\text{rap-ori}}\|_p \le \epsilon_r}{\text{min}} \,\, \textit{Loss} (x  + r_{\text{rap-ori}} + \delta_{\text{rap-ori}} ),
\end{equation}

The objective defined in Eq.~\eqref{eq:ojb-rap} is a max-min optimization problem, which can be solved by iteratively optimizing the inner minimization and the outer maximization.

In~\cite{qin2022boosting}, $r_{\text{rap-ori}}$ and $\delta_{\text{rap-ori}}$ are iteratively updated. 
In each step, it first updates $r_{\text{rap-ori}}$ for $m_r$ steps, and then updates $\delta_{\text{rap-ori}}$ for one step.
Specifically, at the $t$-th step of the outer maximization, the inner minimization is iteratively solved as follows.
\begin{equation} \label{eq:ori-r-update}
    r^{(t)}_{\text{rap-ori},s} = \operatorname{clip} \left(r^{(t)}_{\text{rap-ori},s-1} - \eta \cdot \operatorname{sign}(\nabla_{x'} \textit{Loss}(x'=x+\delta_{\text{rap-ori}}^{(t-1)} + r^{(t)}_{\text{rap-ori},s-1})) \right).
\end{equation}
where $s$ indicates the $s$-th step of the gradient descent for the inner minimization for updating $r_{\text{rap-ori}}$.

Without loss of generality, in each iteration of the outer maximization, we  assume that the inner minimization is optimized for  $m_r$ steps, thereby $r_{\text{rap-ori}}^{(t)} \triangleq r^{(t)}_{\text{rap-ori},m_r} $.
In this way, the adversarial perturbation $\delta^{(t)}_{\text{rap-ori}} $ generated after the $t$-step RAP Attack is formulated as
\begin{equation}  \label{eq:ori-delta-update}
    \delta^{(t)}_{\text{rap-ori}} = \operatorname{clip} \left(\delta^{(t-1)}_{\text{rap-ori}} + \eta \cdot \operatorname{sign}(\nabla_{x'} \textit{Loss}(x'=x+\delta_{\text{rap-ori}}^{(t-1)} + r_{\text{rap-ori}}^{(t)})) \right).
\end{equation}

The RAP Attack  is originally defined as Eq.~\eqref{eq:ori-r-update} and Eq.~\eqref{eq:ori-delta-update}, where the sign operation and the clip operation are used to constrain the adversarial perturbation $\delta^{(t)}_{\text{rap-ori}} $.
In fact, the sign operation and the clip operation are common operations in generating adversarial perturbations.
Therefore, the distinctive technique in the RAP Attack is to generate adversarial perturbation via max-min optimization.
Thus, to simplify the proof, we ignore the sign operation and the clip operation, without hurting the trustworthiness of the analysis of the RAP Attack.
In this way, we focus on a simplified version of the RAP Attack, as follows.

\begin{equation} \label{eq:gd-r}
    r^{(t)}_{s} = r^{(t)}_{s-1} - \alpha \cdot \nabla_{x'} \textit{Loss}(x'=x+\delta^{(t-1)}_{\text{rap}} + r^{(t)}_{s-1})
\end{equation}

\begin{equation}  \label{eq:ga-d}
    \delta^{(t)}_{\text{rap}} =\delta^{(t-1)}_{\text{rap}} + \alpha \cdot\nabla_{x'} \textit{Loss}(x'=x+\delta^{(t-1)}_{\text{rap}} + r^{(t)}).
\end{equation}
Here, $r^{(t)} \triangleq r^{(t)}_{m_r}$, and $m_r$ denotes the step number for the inner minimization.
$ $\newline

\textbf{The inner minimization.}
Given the $t$-th step of the outer maximization, the perturbation generated at the $s$-th step of the inner minimization can be calculated as follows.
\begin{equation} \label{eq:delta-r}
    \begin{split}
        \Delta r^{(t)}_s &=  - \alpha \nabla_{x'} \textit{Loss} (x'=x  + \delta^{(t-1)}_{\text{rap}} + r^{(t)}_{s-1} ). \\
    \end{split}
\end{equation}

Then, the perturbation $r^{(t)}_s$ in Eq.~\eqref{eq:gd-r}
can be written as follows.
\begin{equation} \label{eq:r-multi-update}
    r^{(t)}_s = \sum_{s'=1}^{s} \Delta r^{(t)}_{s'}.
\end{equation}

\begin{lemma} \label{lemma:rap-induction}
    Based on Assumption~\ref{assumption:taylor}, the update of the perturbation via gradient descent at the $t$-th step defined in Eq.~\eqref{eq:delta-r}
    can be written as $\Delta r^{(t)}_s = \left[ I - \alpha H\right]^{s-1}\Delta r^{(t)}_1 $.
    \end{lemma}

\begin{proof}
Given $g = \nabla_x \textit{Loss}(x)$ and $H =\nabla_x^2 \textit{Loss}(x) $, based on Assumption~\ref{assumption:taylor}, we have
\begin{equation} \label{eq:delta-r-closed-form}
    \begin{split}
        \Delta r^{(t)}_s
        &= -\alpha \nabla_{x'} \textit{Loss}(x'=x + \delta^{(t-1)}_{\text{rap}} + r^{(t)}_{s-1}) \quad // \quad \text{According to Eq.~\eqref{eq:gd-r}} \\
        &= -\alpha \left[g + H\delta^{(t-1)}_{\text{rap}} +  H  r^{(t)}_{s-1} \right] \quad // \quad \text{According to Eq.~\eqref{eq:grad_taylor}} \\
        &=-\alpha \left[g  + H\delta^{(t-1)}_{\text{rap}} + H \sum_{s'=1}^{s-1} \Delta r^{(t)}_{s'} \right] \quad // \quad \text{According to Eq.~\eqref{eq:r-multi-update}} \\
    \end{split}
\end{equation}

For simplicity, let $g' = g  + H\delta^{(t-1)}_{\text{rap}}$.
According to Eq.~(\ref{eq:delta-r-closed-form}), we can obtain the recursive solution to $\Delta r^{(t)}$ as follows.
\begin{equation}
    \begin{aligned}
    \label{eq: recursive_r_t}
        \Delta r^{(t)}_s &=-\alpha \left[g' + H \sum_{s'=1}^{s-1} \Delta r^{(t)}_{s'} \right] \\
        &= -\alpha \left[g’ + H \sum_{s'=1}^{s-2} \Delta r^{(t)}_{s'} \right] - \alpha H \Delta r^{(t)}_{s-1} \\
        &= \Delta r^{(t)}_{s-1} - \alpha H \Delta r^{(t)}_{s-1} 
        \quad // \quad \text{According to Eq.~\eqref{eq:delta-r-closed-form}}
        \\
        &= \left[I - \alpha H\right]\Delta r^{(t)}_{s-1}.
    \end{aligned}
\end{equation}

In this way, based on Eq.~\eqref{eq: recursive_r_t}, the perturbation $\Delta r^{(t)}_s$ can be further simplified as
\begin{equation}
    \begin{aligned}
        \Delta r^{(t)}_s &= \left[I - \alpha H\right]\Delta r^{(t)}_{s-1}\\
        &= \left[I - \alpha H\right]\left[I - \alpha H\right] \Delta r^{(t)}_{s-2}
        \quad // \quad \text{According to Eq.~\eqref{eq: recursive_r_t}}
        \\
        &= \left[I - \alpha H\right]^2 \Delta r^{(t)}_{s-2}\\
        &\vdots \\
        &= \left[I - \alpha H\right]^{s-1} \Delta r^{(t)}_1.
    \end{aligned}
\end{equation}

Thus, Lemma~\ref{lemma:rap-induction} is proven.
\end{proof}
$ $\newline

Without loss of generality, we assume that the inner minimization is optimized for  $m_r$ steps.
Then, the perturbation $r^{(t)}$ in Eq.~\eqref{eq:ga-d} can be represented as
\begin{equation} \label{eq:r-t}
    r^{(t)} \triangleq \sum_{s=1}^{m_r} \Delta r^{(t)}_s.
\end{equation}

\begin{lemma} \label{lemma:r-closed-form}
    Based on Assumption~\ref{assumption:taylor}, $r^{(t)}$ can be written as $r^{(t)} = V R V^T \Delta r^{(t)}_1$, where  $R\in \mathbb{R}^{n\times n}$ is a diagonal matrix. If $\lambda_i\ne0$, $R_{ii}=\frac{1-(1-\alpha\lambda_i)^{m_r}}{\alpha\lambda_i}$; otherwise, $R_{ii}= {m_r}$.
    $\Delta r^{(t)}_1=- \alpha \nabla_{x'} \textit{Loss} (x'=x  + \delta^{(t-1)}_{\text{rap}} ) $.
\end{lemma}

\begin{proof}
    According to Eq.~\eqref{eq:r-t}, the perturbation $r^{(t)}$ can be expressed as follows.
    \begin{align} \label{eq:r-t-closed-form}
        r^{(t)} &= \sum_{s=1}^{m_r} \Delta r^{(t)}_s \nonumber \\
        &=  \sum_{s=1}^{m_r} \left[I - \alpha H\right]^{s-1} \Delta r^{(t)}_1 \quad// \quad \text{According to Lemma~\ref{lemma:rap-induction}} \nonumber \\
        &=  \sum_{s=1}^{m_r} \left[V V^T - \alpha V \Lambda V^T\right]^{s-1} \Delta r^{(t)}_1 \quad // \quad\text{According to Eq.~\eqref{eq:H-decompose}} \nonumber  \\
        &=  \sum_{s=1}^{m_r} \left[V \left(I -\alpha \Lambda \right) V^T\right]^{s-1} \Delta r^{(t)}_1 
        \nonumber \\
        &=  \sum_{s=1}^{m_r} \left[V \left(I -\alpha \Lambda \right) V^T \;  V \left(I -\alpha \Lambda \right) V^T \cdots V \left(I -\alpha \Lambda \right) V^T     \right] \Delta r^{(t)}_1 
        \nonumber \\
        &=  \sum_{s=1}^{m_r} \left[V \left(I -\alpha \Lambda \right)^{s-1} V^T\right] \Delta r^{(t)}_1
        \quad // \quad V^T V = I
         \nonumber \\
        &= V  \left[ \sum_{s=1}^{m_r}\left(I -\alpha \Lambda \right)^{s-1} \right] V^T \Delta r^{(t)}_1.
    \end{align}

For simplicity, let {$R \in \mathbb{R}^{n\times n}$} denote the term {$ \sum_{s=1}^{m_r}\left(I -\alpha \Lambda \right)^{s-1} $} in Eq.~\eqref{eq:r-t-closed-form},~\textit{i.e.,}
{$R= \sum_{s=1}^{m_r}\left(I -\alpha \Lambda \right)^{s-1}$}.
Because both {$I$} and $\Lambda$ are diagonal matrices, matrix {$R$} is also a diagonal matrix, \emph{i.e.} $\forall i\ne j, R_{ij}=0.$
In this way, let us focus on the $i$-th diagonal element $R_{ii} \in \mathbb{R}$.
\begin{equation}
\begin{aligned}
\label{eq:diagonal_r}
R_{ii}&=(1+(1-\alpha \lambda_i)+\cdots+(1-\alpha \lambda_i)^{m_r-1})
\quad // \quad \text{According to Eq.~\eqref{eq:r-t-closed-form}}
\\
&=\begin{cases}
        (1\times\frac{1-(1-\alpha\lambda_i)^{m_r}}{1-(1-\alpha\lambda_i)}) ,  & \text{if $\lambda_i\ne 0$} \\
         (1 + 1 + \dots + 1), & \text{if $\lambda_i=0$}
    \end{cases} \\
&=\begin{cases}
        \frac{1-(1-\alpha\lambda_i)^{m_r}}{\alpha\lambda_i} ,  & \text{if $\lambda_i\ne 0$} \\
        \lim_{\lambda_i \rightarrow 0} \frac{1-(1-\alpha\lambda_i)^{m_r}}{\lambda_i} = m_r, & \text{if $\lambda_i=0$}
    \end{cases} \\
\end{aligned}
\end{equation}

Note that if $\lim {\lambda_i \rightarrow 0}$, according to L'Hospital's Rule, $\lim_{\lambda_i \rightarrow 0} \frac{1-(1-\alpha\lambda_i)^{m_r}}{\alpha \lambda_i} =\lim_{\lambda_i \rightarrow 0} \frac{ d(1-(1-\alpha\lambda_i)^{m_r} ) / d \lambda_i}{ d\; (\alpha \lambda_i) /d\lambda_i }= \lim_{\lambda_i \rightarrow 0} \frac{\alpha m_r(1-\alpha\lambda_i)^{m_r-1}}{\alpha} = m_r$.

Thus, Lemma~\ref{lemma:r-closed-form} is proven.
\end{proof}

$ $\newline

\textbf{The outer maximization.}
After obtaining the perturbation $r^{(t)}$, the final adversarial perturbation generated by the RAP Attack can be obtained with Eq.~{(\ref{eq:ga-d})} with gradient ascent.

In this way, the perturbation with the RAP Attack generated at each step $t$ can be calculated as follows.
\begin{equation} \label{eq:rap-update}
    \Delta x^{(t)}_{\text{rap}} = \alpha\cdot \nabla_{x'}   \textit{Loss}(x' = x + r^{(t)} + \delta^{(t-1)}_{\text{rap}}).
\end{equation}

Without loss of generality, we assume that the outer maximization is optimized for $m$ step in total.
Hence, the final adversarial perturbation $\delta^{(m)}_{\text{rap}}$  of the RAP Attack can be represented as
\begin{equation} \label{eq:delta-rap-0}
    \delta^{(m)}_{\text{rap}} = \sum_{t=1}^{m} \Delta x^{(t)}_{\text{rap}}.
\end{equation}

\begin{lemma} \label{lemma:delta-x-rap-closed-form}
    We are given $H = \nabla_x^2 \textit{Loss}(x) = V\lambda V^T$ and $g=\nabla_x \textit{Loss}(x)$.
    Based on Assumption~\ref{assumption:taylor}, the update of the perturbation via gradient descent at $t$-th step in Eq.~\eqref{eq:rap-update}
    can be written as $\Delta x^{(t)}_{\text{rap}} =  V G^{(t)}  V^T g $, where $G^{(t)}\in\mathbb{R}^{n\times n}$ is a diagonal matrix, and $G^{(t)}_{ii}= \alpha \big(1 + \alpha \lambda_i (1-\alpha \lambda_i)^{m_r}\big)^{t-1} (1-\alpha \lambda_i)^{m_r}$
\end{lemma}

\begin{proof}
    According to Eq.~\eqref{eq:rap-update}, the perturbation $\Delta x^{(t)}_{\text{rap}}$ generated at the $t$-th step can be re-written as
    \begin{align} \label{eq:delta-x-rap-t-pre}
        \Delta x^{(t)}_{\text{rap}} &= \alpha\cdot \nabla_{x'}   \textit{Loss}(x' = x + r^{(t)} + \delta^{(t-1)}_{\text{rap}}) \nonumber \\
        &= \alpha \left[g + H \left(r^{(t)} + \delta^{(t-1)}_{\text{rap}} \right) \right]  \quad // \quad \text{According to Eq.~\eqref{eq:grad_taylor}} \nonumber \\
        &= \alpha \left[g + H \left(V R V^T \Delta r^{(t)}_1 + \delta^{(t-1)}_{\text{rap}} \right) \right]  \quad // \quad \text{According to Lemma~\ref{lemma:r-closed-form}} \nonumber \\
        &= \alpha \left[g + H \left(V R V^T \left(- \alpha \nabla_{x'} \textit{Loss} (x'=x  + \delta^{(t-1)}_{\text{rap}} )  \right) + \delta^{(t-1)}_{\text{rap}} \right) \right]  \quad // \quad \text{According to Eq.~\eqref{eq:delta-r}}  \nonumber \\
        &= \alpha \left[g + H \bigg(- \alpha V R V^T \left(g + H\delta^{(t-1)}_{\text{rap}}  \right) + \delta^{(t-1)}_{\text{rap}} \bigg) \right] 
         \quad // \quad \text{According to Eq.~\eqref{eq:grad_taylor}} \nonumber \\
        &= \alpha \left[g - \alpha H V R V^T g + H \delta^{(t-1)}_{\text{rap}} - \alpha H V R V^T H \delta^{(t-1)}_{\text{rap}} \right]   \nonumber \\
        &= \alpha \left[ \left(I- \alpha H V R V^T \right)g  + \left(I- \alpha H V R V^T  \right) H\delta^{(t-1)}_{\text{rap}}\right]   \nonumber \\
        &= \alpha \left[ \left(I- \alpha H V R V^T \right) \left(g +  H\delta^{(t-1)}_{\text{rap}} \right)\right]
    \end{align}
    where $R\in \mathbb{R}^{n\times n}$ is a diagonal matrix, which is defined in Eq.~\eqref{eq:diagonal_r}.

    Thus, according to Eq.~\eqref{eq:delta-x-rap-t-pre}, the perturbation $\Delta x^{(t-1)}_{\text{rap}}$ generated at the $(t-1)$-th step can be re-written as
    \begin{align}  \label{eq:delta-x-rap-t-1}
        \Delta x^{(t-1)}_{\text{rap}} &=  \alpha \left[ \left(I- \alpha H V R V^T \right) \left(g +  H\delta^{(t-2)}_{\text{rap}} \right)\right].
    \end{align}

    Then, the difference between Eq.~\eqref{eq:delta-x-rap-t-pre} and Eq.~\eqref{eq:delta-x-rap-t-1} can be written as
    \begin{align}
    \label{eq:delta-x-rap-t-0}
        \Delta x^{(t)}_{\text{rap}} - \Delta x^{(t-1)}_{\text{rap}} &=   \alpha \left[ \left(I- \alpha H V R V^T \right) \left(g +  H\delta^{(t-1)}_{\text{rap}} \right)\right]- \alpha \left[ \left(I- \alpha H V R V^T \right) \left(g +  H\delta^{(t-2)}_{\text{rap}} \right)\right] \nonumber \\
        \Rightarrow \Delta x^{(t)}_{\text{rap}} - \Delta x^{(t-1)}_{\text{rap}} &= \alpha \left[ \left(I- \alpha H V R V^T \right) H \left(\delta^{(t-1)}_{\text{rap}} - \delta^{(t-2)}_{\text{rap}} \right)\right] \nonumber \\
         \Rightarrow \Delta x^{(t)}_{\text{rap}} - \Delta x^{(t-1)}_{\text{rap}} &= \alpha \left[ \left(I- \alpha H V R V^T \right) H \left(\sum_{t'=1}^{t-1}\Delta x^{(t')}_{\text{rap}} -\sum_{t'=1}^{t-2}\Delta x^{(t')}_{\text{rap}} \right)\right]
         \quad // \quad \text{According to Eq.~\eqref{eq:delta-rap-0}}
          \nonumber \\
         \Rightarrow \Delta x^{(t)}_{\text{rap}} - \Delta x^{(t-1)}_{\text{rap}} &= \alpha \left[ \left(I- \alpha H V R V^T \right) \left( H \Delta x^{(t-1)}_{\text{rap}}\right)\right] 
          \nonumber \\
        \Rightarrow \Delta x^{(t)}_{\text{rap}}  &= \left[ \alpha \left[ \left(I- \alpha H V R V^T \right) H\right] + I\right] \Delta x^{(t-1)}_{\text{rap}} 
        \nonumber \\
        \Rightarrow \Delta x^{(t)}_{\text{rap}}  &=  \left[ I + \alpha H - \alpha^2 H V R V^T H\right] \Delta x^{(t-1)}_{\text{rap}}.
    \end{align}

    In this way, based on Eq.~\eqref{eq:delta-x-rap-t-0}, the perturbation $\Delta x^{(t)}_{\text{rap}}$ generated at the $t$-th step can be further simplified as
    \begin{align}  \label{eq:delta-x-rap-t}
        \Delta x^{(t)}_{\text{rap}} &=  \left[ I + \alpha H - \alpha^2 H V R V^T H\right] \Delta x^{(t-1)}_{\text{rap}} 
        \nonumber \\
        &=  \left[ I + \alpha H - \alpha^2 H V R V^T H\right]  \left[ I + \alpha H - \alpha^2 H V R V^T H\right] \Delta x^{(t-2)}_{\text{rap}} 
       \quad // \quad \text{According to Eq.~\eqref{eq:delta-x-rap-t-0}}
        \nonumber \\
        &=  \left[ I + \alpha H - \alpha^2 H V R V^T H\right]^2 \Delta x^{(t-2)}_{\text{rap}} 
        \nonumber \\
        &\vdots \nonumber \\
        &=  \left[ I + \alpha H - \alpha^2 H V R V^T H\right]^{t-1} \Delta x^{(1)}_{\text{rap}}.
    \end{align}

    Moreover, according to Eq.~\eqref{eq:rap-update}, the perturbation $\Delta x^{(1)}_{\text{rap}}$ generated at the first step can be written as
    \begin{align} \label{eq:delta-x-rap-1}
        \Delta x^{(1)}_{\text{rap}} &= \alpha\cdot \nabla_{x'}   \textit{Loss}(x' = x + r^{(1)}) \nonumber \\
        &= \alpha \left[ g + H r^{(1)} \right]  \quad // \quad \text{According to Eq.~\eqref{eq:grad_taylor}} \nonumber \\
        &= \alpha \left[ g + H V R V^T \Delta r^{(1)}_1 \right] \quad // \quad \text{According to Lemma~\ref{lemma:r-closed-form}} \nonumber \\
        &= \alpha \left[ g + H V R V^T\left(- \alpha \nabla_{x} \textit{Loss} (x)  \right) \right] \quad // \quad \text{According to Eq.~\eqref{eq:delta-r}}  \nonumber \\
        &= \alpha \left[ I -\alpha H V R V^T  \right] g.
    \end{align}

    According to Eq.~\eqref{eq:delta-x-rap-t} and Eq.~\eqref{eq:delta-x-rap-1},  the perturbation $\Delta x^{(t)}_{\text{rap}}$ can be written as
    \begin{align} \label{eq:delta-x-rap-t-closed-form}
        \Delta x^{(t)}_{\text{rap}} &=  \left[ I + \alpha H - \alpha^2 H V R V^T H\right]^{t-1} \Delta x^{(1)}_{\text{rap}} 
         \quad // \quad \text{According to Eq.~\eqref{eq:delta-x-rap-t}} 
        \nonumber \\
        &= \alpha\left[ I + \alpha H - \alpha^2 H V R V^T H\right]^{t-1}  \left[ I -\alpha H V R V^T  \right] g 
         \quad // \quad \text{According to Eq.~\eqref{eq:delta-x-rap-1}} 
         \nonumber\\
        &= \alpha\left[  V V^T + \alpha V \Lambda V^T - \alpha^2 V \Lambda V^T V R V^T V \Lambda V^T\right]^{t-1}  \left[ V V^T -\alpha V \Lambda V^T V R V^T  \right] g \quad // \quad\text{According to Eq.~\eqref{eq:H-decompose}}  \nonumber \\
        &= \alpha\left[  V V^T + \alpha V \Lambda V^T - \alpha^2 V \Lambda  R  \Lambda V^T\right]^{t-1}  \left[ V V^T -\alpha V \Lambda R V^T  \right] g   \nonumber \\
        &= \alpha\left[  V \left(I +\alpha \Lambda -\alpha^2 \Lambda  R  \Lambda  \right) V^T \right]^{t-1}  \left[ V \left(I - \alpha \Lambda R \right) V^T  \right] g   \nonumber \\
        &= \alpha V\left[  I +\alpha \Lambda -\alpha^2 \Lambda  R  \Lambda  \right]^{t-1}V^T V \left[  I - \alpha \Lambda R  \right]  V^T g   \nonumber \\
        &= \alpha V\left[  I +\alpha \Lambda -\alpha^2 \Lambda  R  \Lambda  \right]^{t-1} \left[  I - \alpha \Lambda R  \right]  V^T g.
    \end{align}

    For simplicity, let $G^{(t)} \in \mathbb{R}^{n\times n}$ denote the term $\alpha\left[  I +\alpha \Lambda -\alpha^2 \Lambda  R  \Lambda  \right]^{t-1} \left[  I - \alpha \Lambda R  \right]$ in Eq.~\eqref{eq:delta-x-rap-t-closed-form},~\textit{i.e.,}
$G^{(t)}=\alpha\left[  I +\alpha \Lambda -\alpha^2 \Lambda  R  \Lambda  \right]^{t-1} \left[  I - \alpha \Lambda R  \right]$.
In this way,  the perturbation $\Delta x^{(t)}_{\text{rap}}$ in Eq.~\eqref{eq:delta-x-rap-t-closed-form} can be re-written as
\begin{align}
    \Delta x^{(t)}_{\text{rap}} &=  V G^{(t)}  V^T g.
\end{align}

Because both {$I$}, $\Lambda$ and $R$ are diagonal matrices, matrix {$G^{(t)}$} is also a diagonal matrix, \emph{i.e.} $\forall i\ne j, G^{(t)}_{ij}=0.$
In this way, let us focus on the $i$-th diagonal element {$G^{(t)}_{ii} \in \mathbb{R}$}.
\begin{align} \label{eq:diagonal-g}
    G^{(t)}_{ii} &=\alpha (1 + \alpha \lambda_i -\alpha^2 \lambda_{i}^2 R_{ii})^{t-1} (1-\alpha \lambda_i R_{ii}) \nonumber\\
    &= \begin{cases}
       \alpha (1 + \alpha \lambda_i -\alpha^2 \lambda_i^2 \frac{1-(1-\alpha\lambda_i)^{m_r}}{\alpha\lambda_i})^{t-1} (1-\alpha \lambda_i \frac{1-(1-\alpha\lambda_i)^{m_r}}{\alpha\lambda_i}), &\text{if $\lambda_i\ne 0$}\\
        \alpha, &\text{if $\lambda_i=0$}
    \end{cases} \quad // \quad \text{According to Eq.~\eqref{eq:diagonal_r}} \nonumber \\
    &= \alpha (1 + \alpha \lambda_i (1-\alpha \lambda_i)^{m_r})^{t-1} (1-\alpha \lambda_i)^{m_r}.
\end{align}

Thus, Lemma~\ref{lemma:delta-x-rap-closed-form} is proven.
\end{proof}
$ $\newline

\begin{lemma}
\label{lemma:delta-x-rap-1}
    Based on Assumption~\ref{assumption:taylor},
    the perturbation of the RAP Attack in Eq.~\eqref{eq:delta-rap-0} is given as $\delta^{(m)}_{\text{rap}}=\sum_{i=1}^n   (D_{\text{rap}}^{(m)})_{ii} \gamma_i v_i$, where if $\lambda_i\ne 0$, $(D_{\text{rap}}^{(m)})_{ii}=\frac{(1 + \alpha \lambda_i)^{m}-1}{\lambda_i} \left(1 - \alpha \lambda_i\right)^{m_r}$.
    Otherwise $\lambda_i=0$, $(D_{\text{rap}}^{(m)})_{ii}= \alpha m \left(1 - \alpha \lambda_i\right)^{m_r}$.
\end{lemma}

\begin{proof} 
    According to Eq.~\eqref{eq:delta-rap-0}, the perturbation $\delta^{(m)}_{\text{rap}}$ is given as follows.
    \begin{align} \label{eq:delta-rap-closed-form}
        \delta^{(m)}_{\text{rap}} &= \sum_{t=1}^{m}  \Delta x^{(t)}_{\text{rap}}
        \quad // \quad \text{According to Eq.~\eqref{eq:delta-rap-0}}
         \nonumber \\
        &= \sum_{t=1}^{m} V G^{(t)}  V^T g 
        \quad // \quad \text{According to Lemma~\ref{lemma:delta-x-rap-closed-form}} \nonumber \\
        &= V \left[\sum_{t=1}^{m} G^{(t)}  \right] V^T g,
    \end{align}
    where $G^{(t)}$ is a diagonal matrix defined in Eq.~\eqref{eq:diagonal-g}.

    For simplicity, let {$D^{(m)}_{\text{rap}} \in \mathbb{R}^{n\times n}$} denote the term {$\sum_{t=1}^{m} G^{(t)} $} in Eq.~\eqref{eq:delta-rap-closed-form},~\textit{i.e.,}
    { $D^{(m)}_{\text{rap}} = \sum_{t=1}^{m} G^{(t)}$}.
    Because  $G^{(t)}$ is a diagonal matrix, matrix {$D^{(m)}_{\text{rap}}$} is also a diagonal matrix, \emph{i.e.} $\forall i\ne j, (D^{(m)}_{\text{rap}})_{ij}=0.$
In this way, let us focus on the $i$-th diagonal element {$  (D^{(m)}_{\text{rap}})_{ii} \in \mathbb{R}$}.

    \begin{align}
    \label{eq:D-IN-RAP}
    (D_{\text{rap}}^{(m)})_{ii} 
    &=\sum_{t=1}^{m}  \alpha (1 + \alpha \lambda_i (1-\alpha \lambda_i)^{m_r})^{t-1} (1-\alpha \lambda_i)^{m_r} \nonumber
    \\
    &= \alpha \Bigg(1+ \bigg(1 + \alpha \lambda_i (1-\alpha \lambda_i)^{m_r}\bigg)^{1}+\cdots+\bigg(1 + \alpha \lambda_i (1-\alpha \lambda_i)^{m_r}\bigg)^{m-1} \Bigg) \left(1 - \alpha \lambda_i\right)^{m_r} \nonumber\\
    &=\begin{cases}
            \alpha(1\times\frac{1-(1+\alpha\lambda_i (1-\alpha \lambda_i)^{m_r})^m}{1-(1+\alpha\lambda_i (1-\alpha \lambda_i)^{m_r})})\left(1 - \alpha \lambda_i\right)^{m_r} ,  & \text{if $\lambda_i\ne 0$} \\
            \alpha (1 + 1 + \dots + 1), & \text{if $\lambda_i=0$}
        \end{cases}\nonumber \\
    &=\begin{cases}
            \frac{(1+\alpha\lambda_i (1-\alpha \lambda_i)^{m_r})^m-1}{\lambda_i (1-\alpha \lambda_i)^{m_r}} \left(1 - \alpha \lambda_i\right)^{m_r} ,  & \text{if $\lambda_i\ne 0$} \\
            \alpha m , & \text{if $\lambda_i=0$}
        \end{cases}\nonumber \\
        &=\begin{cases}
            \frac{(1+\alpha\lambda_i (1-\alpha \lambda_i)^{m_r})^m-1}{\lambda_i} ,  & \text{if $\lambda_i\ne 0$} \\
            \alpha m , & \text{if $\lambda_i=0$}
        \end{cases} 
    \end{align}
    Note that if $\lim {\lambda_i \rightarrow 0}$, according to L'Hospital's Rule, $\lim_{\lambda_i \rightarrow 0} \frac{(1+\alpha\lambda_i (1-\alpha \lambda_i)^{m_r})^m-1}{\lambda_i} =\lim_{\lambda_i \rightarrow 0} \frac{ d(1+\alpha\lambda_i (1-\alpha \lambda_i)^{m_r})^m-1 ) / d \lambda_i}{ d\; (\lambda_i) /d\lambda_i }= \lim_{\lambda_i \rightarrow 0} \frac{m (1 + \alpha \lambda_i (1-\alpha \lambda_i)^{m_r})^{m-1} (\alpha (1-\alpha \lambda_i)^{m_r}) (\alpha \lambda_i m_r (1-\alpha \lambda_i)^{m_r -1})}{1} = \alpha m$.

In this way, the perturbation $\delta^{(m)}_{\text{rap}}$ can be written as
\begin{align}
    \delta^{(m)}_{\text{rap}} &= V D_{\text{rap}}^{(m)} V^T g \nonumber \\
    &= [v_1,v_2,\ldots,v_n]
        \text{diag}( (D_{\text{rap}}^{(m)})_{11}, (D_{\text{rap}}^{(m)})_{22}, \dots, (D_{\text{rap}}^{(m)})_{nn})
          [v_1,v_2,\ldots,v_n]^T \sum_{i=1}^n \gamma_i v_i \nonumber \\
    &=
    \left[D_{\text{rap}}^{(m)})_{11} v_1 , (D_{\text{rap}}^{(m)})_{22} v_2 ,\ldots, (D_{\text{rap}}^{(m)})_{nn}v_n  \right]
    \begin{bmatrix}
         \gamma_1 \\
         \gamma_2 \\
        \vdots \\
        \gamma_n \\
        \end{bmatrix} \nonumber\\
    & = \sum_{i=1}^n (D_{\text{rap}}^{(m)})_{ii} \gamma_i v_i,
\end{align}
where $(D_{\text{rap}}^{(m)})_{ii}$ is defined in Eq.~\eqref{eq:D-IN-RAP}.

Thus, Lemma~\ref{lemma:delta-x-rap-1} is proven.

\end{proof}

$ $\newline

\begin{corollary} \label{corollary:delta-rap-inf}
    If we use the infinite-step adversarial attack with an infinitesimal step size to simplify the RAP Attack, \emph{i.e.}, $m_r, m\!\to\! +\infty$.
    Given $\beta_r=\alpha m_r$, $\beta=\alpha m$, then the adversarial perturbation $\delta^{(\infty)}_{\text{rap}}$ can be represented as $\delta^{(\infty)}_{\text{rap}}= \sum_{i=1}^n   D^{\text{(rap)}}_{ii} \gamma_i v_i$, where if $\lambda_i\ne0$, $D^{\text{(rap)}}_{ii}=\frac{\exp(\beta\lambda_i)-1}{\lambda_i} \exp(-\beta_r\lambda_i)$; otherwise, $D^{\text{(rap)}}_{ii}=\beta\exp(-\beta_r\lambda_i)$.
\end{corollary}

\begin{proof}
\begin{align}
\label{eq:d-for-rap}
    D^{\text{(rap)}}_{ii}&=\lim_{m_r\rightarrow\infty \atop m\rightarrow\infty} (D^{(m)}_{\text{rap}})_{ii} \nonumber\\
        &=\begin{cases}
           \lim_{m_r\rightarrow\infty \atop m\rightarrow\infty}  \frac{(1+\alpha\lambda_i (1-\alpha \lambda_i)^{m_r})^m-1}{\lambda_i},  & \text{if $\lambda_i\ne 0$} \\
            \lim_{m_r\rightarrow\infty \atop m\rightarrow\infty} \alpha m , & \text{if $\lambda_i=0$}
            \end{cases}  
             \quad // \quad \text{According to Eq.~\eqref{eq:D-IN-RAP}}
            \nonumber \\
        &= \begin{cases}
            \lim_{m_r\rightarrow\infty \atop m\rightarrow\infty}  \frac{(1+ \frac{\beta}{m}\lambda_i (1-\frac{\beta_r }{m_r} \lambda_i)^{m_r})^m-1}{\lambda_i},  & \text{if $\lambda_i\ne 0$} \\
            \lim_{m_r\rightarrow\infty \atop m\rightarrow\infty} \alpha m, & \text{if $\lambda_i=0$}
            \end{cases} \nonumber \\
        &= \begin{cases}
            \frac{\exp \left(\beta\lambda_i \exp(-\beta_r \lambda_i)\right)-1}{\lambda_i},  & \text{if $\lambda_i\ne 0$} \\
            \beta, & \text{if $\lambda_i=0$}
            \end{cases}
\end{align}

Thus, Corollary~\ref{corollary:delta-rap-inf} is proven.
\end{proof}

$ $\newline

\textbf{Perturbation generated by the RAP attack.}
Let $\lambda_i$ and $v_i$ denote the $i$-th eigenvalue and its corresponding  eigenvector of the Hessian matrix $H=\nabla_x^2 \textit{Loss}(x)$, respectively.
According to Corollary~\ref{corollary:delta-rap-inf}, the adversarial perturbation  $\delta^{(\textrm{rap})}$ generated by the infinite-step RAP Attack with an infinitesimal step size, $m_r\rightarrow\infty, m\rightarrow\infty, \alpha\rightarrow 0$ is formulated as follows, where $\beta_r=\alpha m_r, \beta=\alpha m$.
    \begin{equation} \label{eq:delta-rap}
        \delta^{(\textrm{rap})}\triangleq \delta^{(\infty)}_{\text{rap}}= \sum_{i=1}^n   D^{\text{(rap)}}_{ii} \gamma_i v_i
    \end{equation}
     where if $\lambda_i\ne0$, $D^{\text{(rap)}}_{ii}=\frac{\exp \left(\beta\lambda_i \exp(-\beta_r \lambda_i)\right)-1}{\lambda_i}$; otherwise, $D^{\text{(rap)}}_{ii}=\beta$.
$ $\newline

\textbf{Pairwise interactions inside perturbation $\delta^{(\textrm{rap})}$.}
Based on the perturbations $\delta^{(\textrm{rap})}$ of the RAP Attack,  the interaction between each pair of perturbation units $(a,b)$ can be computed as 
 \begin{equation}  
 \label{eq: rap-interaction-in-rap}
 I_{ab}(\delta^{(\textrm{rap})}) = \delta^{(\textrm{rap})}_a H_{ab} \delta^{(\textrm{rap})}_b.
       \quad // \quad \text{According to Eq.~\eqref{eq:shapley-interaction-index_2} in Appendix~\ref{append:interaction}}.
  \end{equation}

Furthermore, we sum the interactions among all pairs of perturbation units $(a, b)$ as follows.
\begin{align} \label{eq:rap_interaction}
    \sum_{a,b\in \Omega} I_{ab}(\delta^{(\text{rap})})&=(\delta^{(\text{rap})})^T H\delta^{(\text{rap})} \nonumber \\
    &= (\sum_{i=1}^n D^{\text{(rap)}}_{ii} \gamma_{i} v_{i}^T)(\sum_{j=1}^n \lambda_{j} v_{j} v_{j}^T) (\sum_{k=1}^n D^{\text{(rap)}}_{kk} \gamma_{k} v_{k})
     \quad // \quad \text{According to Eq.~\eqref{eq:H-decompose} and Eq.~\eqref{eq:delta-rap}}
    \nonumber \\
    &=\left[\sum_{j=1}^n \underbrace{ (D^{\text{(rap)}}_{jj} \gamma_{j} v_{j}^T)( \lambda_{j} v_{j} v_{j}^T) }_{\text{$ v_{j}^T v_{j}=1$}}+ \sum_{1\le i,j\le n \atop i\ne j} \underbrace{(D^{\text{(rap)}}_{ii} \gamma_{i} v_{i}^T)( \lambda_{j} v_{j} v_{j}^T)}_{\text{if $i\ne j, v_{i}^T v_{j}=0$}} \right] (\sum_{k=1}^n D^{\text{(rap)}}_{kk} \gamma_{k} v_{k}) 
    \nonumber \\
    &=(\sum_{j=1}^n \lambda_{j}  D^{\text{(rap)}}_{jj} \gamma_{j} v_{j}^T ) (\sum_{k=1}^n  D^{\text{(rap)}}_{kk} \gamma_{k} v_{k})
    \nonumber \\
    &=\sum_{k=1}^n \underbrace{ (\lambda_i D^{\text{(rap)}}_{kk} \gamma_{k} v_{k}^T)(  D^{\text{(rap)}}_{kk} \gamma_{k} v_{k})}_{\text{$ v_{k}^T v_{k}=1$}}+ \sum_{1\le j, k \le n \atop j\ne k} \underbrace{ (\lambda_{j} D^{\text{(rap)}}_{jj} \gamma_{j} v_{j}^T)( D^{\text{(rap)}}_{kk} \gamma_{k} v_{k})}_{\text{if $j\ne k, v_{j}^T v_{k}=0$}} 
    \nonumber \\
    &= \sum_{k=1}^n \lambda_k( D^{\text{(rap)}}_{kk}\gamma_{k})^2,
\end{align}
    where if $\lambda_i\ne0$, $D^{\text{(rap)}}_{kk}=\frac{\exp \left(\beta\lambda_i \exp(-\beta_r \lambda_i)\right)-1}{\lambda_i}$; otherwise, $D^{\text{(rap)}}_{kk}=\beta$.

\subsection{Multi-step attack}

In this subsection, we revisit the perturbation of the multi-step attack, which is defined in Appendix~\ref{append:delta-closed-form}.

\textbf{Perturbation generated by the multi-step attack.}
According to Eq.~\eqref{eq:multi-step-attack}, the perturbation generated by the $\infty$-step attack can be represented as 
\begin{equation} 
\label{eq:delta-multi-inf-in-rap}
\delta^{(\text{multi})}  \triangleq  \delta^{(\infty)} 
=  \sum_{i=1}^n D_{ii} \gamma_i v_i.
\end{equation}
where $\beta=\alpha m$; $\lambda_{i}$ and $v_i$ denote the $i$-th eigenvalue of the Hessian matrix $H=\nabla_x^2 \textit{Loss}(x)$ and its corresponding eigenvector, respectively.
If $\lambda_i\ne0$, $D_{ii}=\frac{\exp(\beta\lambda_i)-1}{\lambda_i}$; otherwise, $D_{ii}=\beta$. 
$\gamma_i = g^T v_i \in \mathbb{R}$ represents the projection of the gradient $g=\nabla_x \textit{Loss}(x)$ on the eigenvector $v_i$.

$ $\newline

 \textbf{Pairwise interaction inside $\delta^{(\text{multi})}$.}
       Based on the perturbation $\delta^{(\text{multi})}$ of the multi-step attack, the interaction between each pair of perturbation units $(a,b)$ can be computed as 
       \begin{equation}  \label{eq: multi-pairwise-interaction-in-rap}
       I_{ab}(\delta^{(\textrm{multi})}) = \delta^{(\textrm{multi})}_a H_{ab} \delta^{(\textrm{multi})}_b.
       \quad // \quad \text{According to Eq.~\eqref{eq:shapley-interaction-index_2} in Appendix~\ref{append:interaction}}.
       \end{equation}
       Furthermore, we sum the interactions among all pairs of perturbation units $(a, b)$ as follows.
       \begin{equation} \label{eq:interaction-multi-in-rap-1}
            \sum_{a, b\in\Omega}[{I_{a b}(\delta^{(\text{multi})})}] = \sum_{i=1}^n \lambda_{i}(D_{ii}\gamma_{i})^2.
            \quad // \quad \text{According to Eq.~\eqref{eq:delta-inf-interaction}}.
        \end{equation}

\subsection{Proof of Proposition~\ref{propositionappendix:rap}}
If the adversarial strength $\beta=\alpha m$ is the same for the RAP Attack and the multi-step attack, we prove the following proposition, which shows that RAP Attack decreases the interaction between perturbation units.

\begin{propositionappendix}[Perturbations of the RAP Attack exhibit smaller interactions than those of the multi-step attack]
    \label{propositionappendix:rap}
    {If the classification loss  for attacking $\textit{Loss}(x)$  is formulated as the cross-entropy loss in Eq.~\eqref{eq:softmax-ce}}
    and we use the infinite-step attack with an infinitesimal step size  to simplify both the RAP Attack and the multi-step attack,
    under $\beta = \alpha m$ and $\beta_r = \alpha m_r$, according to Eq.~\eqref{eq:d-for-rap} and Eq.~\eqref{eq:delta-multi-inf-in-rap}, then we have $\sum_{a, b\in\Omega}[{I_{a b}(\delta^{(\text{rap})})}]\le\sum_{a, b\in\Omega}[{I_{a b}(\delta^{(\text{multi})})}]$.
  \end{propositionappendix}

\begin{proof}
    According to Eq.~(\ref{eq:interaction-multi-in-rap-1}) and Eq.~(\ref{eq:rap_interaction}), the difference between interactions $\sum_{a, b\in\Omega}[{I_{a b}(\delta^{(\text{multi})})}]$ for the multi-step attack and interactions $\sum_{a, b\in\Omega}[{I_{a b}(\delta^{(\text{rap})})}]$ for the RAP Attack can be represented as
           \begin{equation}
            \begin{split}
            \label{eq:diff-in-rap-1}
                &\sum_{a, b\in\Omega}[{I_{a b}(\delta^{(\text{multi})})}]-\sum_{a, b\in\Omega}[{I_{a b}(\delta^{(\text{rap})})}] \\
                &= \sum_{k=1}^n \lambda_i \left[(D_{kk}^2 - (D_{kk}^{\text{(rap)}})^2) \gamma_k^2 \right] 
                \quad // \quad \text{According to Eq.~(\ref{eq:interaction-multi-in-rap-1}) and Eq.~(\ref{eq:rap_interaction})}
                \\
                &= \sum_{k=1}^n \lambda_i \left[(D_{kk} - D_{kk}^{\text{(rap)}}) (D_{kk} + D_{kk}^{\text{(rap)}}) \gamma_k^2 \right].
            \end{split}
        \end{equation}

    Here, we prove that the difference is non-negative, because each term in Eq.~\eqref{eq:diff-in-rap-1} is proven to be non-negative, as follows.

     \textbullet{ $\lambda_i \ge 0$.} According to Corollary~\ref{corollaryappendix:semi-definite}, the Hessian matrix $H=\nabla_x^2 \textit{Loss}(x)$ is positive semi-definite.
    Thus, all eigenvalues of the Hessian matrix $H$ are non-negative, \emph{i.e.} $\forall i, \lambda_i\ge 0$.

     \textbullet{ $ \gamma_k^2 \ge 0$.}

      \textbullet{ $D_{kk} + D_{kk}^{\text{(rap)}} \ge 0$,} because both $D_{kk}$  and $ D_{kk}^{\text{(rap)}}$ are proven to be positive, as follows.
      \begin{align}
      \label{eq:diff-rap-2}
        &0  \le \exp(-\beta_r \lambda_i) \le 1 
        \quad // \quad   \beta_r \ge 0, \lambda_i \ge 0
        \nonumber \\
        \Rightarrow &\qquad\beta \lambda_i \ge \beta \lambda_i \exp(-\beta_r \lambda_i) \ge 0
         \quad // \quad  \beta \ge 0
         \nonumber \\
        \Rightarrow & \qquad\exp(\beta \lambda_i) \ge \exp(\beta \lambda_i \exp(-\beta_r \lambda_i) ) \ge 1 \nonumber \\
        \Rightarrow & \qquad\exp(\beta \lambda_i)-1 \ge \exp(\beta \lambda_i \exp(-\beta_r \lambda_i) ) -1 \ge 0 \nonumber \\
        \Rightarrow & \qquad D_{kk} \ge D_{kk}^{(\text{rap})} \ge 0.
         \quad // \quad \text{According to Eq.~\eqref{eq:d-for-rap}}
    \end{align}
     
      \textbullet{ $D_{kk} - D_{kk}^{\text{(rap)}} \ge 0$,} according to Eq.~\eqref{eq:diff-rap-2}.
    $ $\newline 

In summary,  interactions $\sum_{a, b\in\Omega}[{I_{a b}(\delta^{(\text{multi})})}]$ for the multi-step attack are larger than interactions $\sum_{a, b\in\Omega}[{I_{a b}(\delta^{(\text{rap})})}]$ for the RAP Attack.
    \begin{equation}
        \begin{split}
            &\sum_{a, b\in\Omega}[{I_{a b}(\delta^{(\text{multi})})}]-\sum_{a, b\in\Omega}[{I_{a b}(\delta^{(\text{rap})})}] \\
            &= \sum_{k=1}^n \underbrace{\lambda_i}_{\ge 0} \left[ \underbrace{(D_{kk} - D_{kk}^{(\text{rap})})}_{\ge 0} \underbrace{(D_{kk} + D_{kk}^{(\text{rap})})}_{\ge 0} \underbrace{\gamma_k^2}_{\ge 0 } \right] \\
            &\ge 0.
        \end{split}
    \end{equation}
    
 Thus, Proposition~\ref{propositionappendix:rap} is proven.
\end{proof}

\section{Proof that the IL Attack generates perturbations with smaller interactions than those of the multi-step attack}
\label{append:il}

  In this subsection, we prove Proposition~\ref{pro:il} in Section~\ref{sec:explaining}, which shows that the Intermediate Level Attack (IL Attack)~\cite{huang2019enhancing} exhibits smaller interactions than those of the multi-step attack.
  Before the proof of Proposition~\ref{pro:il}, we first revisit the IL Attack.

   \subsection{IL Attack}
 
    The IL Attack~\cite{huang2019enhancing} boosts the transferability by fine-tuning an existing adversarial perturbation.
    The IL Attack has two stages. In the first stage, the IL Attack generates a relatively weak adversarial perturbation $\delta_{\textrm{base-ori}}$ in the same manner as the PGD Attack~\cite{pgd2018}. Then, in the second stage, the IL Attack generates another adversarial perturbation towards the direction of feature changes in the target layer, which are made by the previous perturbation $f_l(x+\delta_{\textrm{base-ori}}) - f_l(x)$, as follows.
    \begin{equation} \label{eq:obj-ia}
      \max_{\|\delta^{(\text{il-ori})} \|_p \le \epsilon } [f_l(x+\delta^{(\text{il-ori})}) - f_l(x) ]^T [f_l(x+\delta_{\textrm{base-ori}}) - f_l(x) ],
    \end{equation}
    where $f_l(x)$ denotes the output feature of the $l$-th layer in the neural network.

    Without loss of generality, we consider an input $x$ and a ReLU network.
    The intermediate feature of the $l$-th linear layer can be formulated as follows.
     \begin{equation}
    \label{eq:intermediate-feature}
    f_l(x) = W_l^T \Sigma_{l-1}( W_{l-1}^T \Sigma_{l-2}( \cdots (W_1^{T} x + b_1) \cdots) + b_{l-1}) +  b_{l} \in \mathbb{R}^{n_l},
    \end{equation}
    where $W_l \in \mathbb{R}^{n_{l-1} \times n_{l}}$ and $b_l\in \mathbb{R}^{n_l}$ represent the weight and bias of the $l$-th linear layer, respectively.
    The diagonal matrix $\Sigma_l = diag([\sigma_{11},\sigma_{22},\cdots, \sigma_{n_{l} n_{l}}]) \in \mathbb{R}^{n_{l} \times n_{l}}$ represents the activation states of the $l$-th ReLU layer, where $\sigma_{ii} \in\{0, 1\}$.

    Based on Assumption~\ref{assumption:activation}, the ReLU network can be approximately regarded as piecewise linear.
    In this scenario, the  intermediate feature of the $l$-th linear layer $f_l(x)$ in Eq.~\eqref{eq:intermediate-feature} can be approximated as follows.
      \begin{equation}
        f_l(x) \approx (\tilde{W}_x^{(l)})^T x + b^{(l)},
    \end{equation}
     where $(\tilde{W}_x^{(l)})^T =  W_l^T \Sigma_{l-1} W_{l-1}^T \dots \Sigma_{1} W_{1}^T$.

     In this way, the objective of Eq.~\eqref{eq:obj-ia} can be approximated as follows.
   
    \begin{small}
     \begin{equation} \label{eq:obj-ia-our}
        \begin{split}
            &\max_{\| {\delta^{(\text{il})}} \|_p \le \epsilon }  [f_l(x+{\delta^{(\text{il})}}) - f_l(x) ]^T [f_l(x+\delta_{\text{base}}) - f_l(x) ] 
             \\
            \Rightarrow&
            \max_{\|{\delta^{(\text{il})}}\|_p \le \epsilon }  [ \left((\tilde{W}_x^{(l)})^T (x+{\delta^{(\text{il})}}) + b^{(l)} \right)- \left( (\tilde{W}_x^{(l)})^T x + b^{(l)}\right) ]^T
             [ \left((\tilde{W}_x^{(l)})^T (x+{\delta_{\text{base}}}) + b^{(l)} \right)- \left( (\tilde{W}_x^{(l)})^T x + b^{(l)}\right) ] 
             \\
            \Rightarrow&  \max_{\|{\delta^{(\text{il})}}\|_p \le \epsilon }  [(\tilde{W}_x^{(l)})^T {\delta^{(\text{il})}} ]^T [(\tilde{W}_x^{(l)})^T{\delta_{\text{base}}} ].
        \end{split}
      \end{equation}
       \end{small}

      Eq.~(\ref{eq:obj-ia-our}) shows that the IL Attack pushes the feature $(\tilde{W}_x^{(l)})^T {\delta^{(\text{il})}}$ towards the direction of $(\tilde{W}_x^{(l)})^T\delta_{\textrm{base}}$, and increases its magnitude.
      Hence, we can roughly consider the perturbation $\delta^{(\text{il})}$ generated by the IL Attack  can be approximately calculated as
      \begin{equation} \label{eq:il-perturbation}
          \delta^{(\text{il})} \approx \eta \cdot \delta_{\textrm{base}},
      \end{equation}
      where $\eta\ge 1$.
       $ $\newline

        \textbf{Perturbation generated by the IL Attack.}
        Here, we use the multi-step attack as the baseline attack to generate the perturbation $\delta_{\textrm{base}}$.
        Specifically, we assume that the perturbation $\delta_{\textrm{base}}$ is generated by the multi-step attack, with steps being $m_2$ and the step size being $\alpha$.
        According to Corollary~\ref{corollaryappendix:delta-inf}, let $\beta_2 = \alpha m_2$ denote the total adversarial strength.
        
        Moreover, in~\cite{huang2019enhancing}, the $\delta_{\text{base}}$ is typically chosen as a  relatively weak perturbation. Therefore, here, we consider that 
        \begin{equation} \label{eq:beta2-beta1}
            \beta_2 \ll \beta_1
        \end{equation}
        where $\beta_1$ is the adversarial strength of the normal multi-step attack, which is introduced in Eq.~\eqref{eq:delta-multi-inf-in-il}.
        
        {Then, if we use infinite-step adversarial attack with an infinitesimal step size to simplify the multi-step attack,~\textit{i.e.,} $m_2\rightarrow\infty, \alpha\rightarrow 0$, the perturbation $\delta_{\textrm{base}}$ can be formulated as follows.}
          \begin{equation} \label{eq:delta-base-inf-in-il}
              \delta_{\text{base}}  = \sum_{i=1}^n D_{ii}^{(\text{base})} \gamma_i v_i.
               \quad // \quad \text{According to Corollary~\ref{corollaryappendix:delta-inf}}
          \end{equation}
          where if $\lambda_i\ne0$, $D_{ii}^{(\text{base})}=\frac{\exp(\beta_2\lambda_i)-1}{\lambda_i}$; otherwise, $D_{ii}^{(\text{base})}=\beta_2$.
          $\lambda_i$ and $v_i$ denote the $i$-th eigenvalue and its corresponding eigenvector of the Hessian matrix $H = \nabla_x^2 \textit{Loss}(x)$.
          $\gamma_i = g^T v_i \in \mathbb{R}$ represents the projection of the gradient $g=\nabla_x \textit{Loss}(x)$ on the eigenvector $v_i$.

          In this way, based on Eq.~\eqref{eq:delta-base-inf-in-il}, the perturbation $\delta^{(\text{il})}$ of the IL Attack in Eq.~\eqref{eq:il-perturbation} can be represented as
         \begin{equation} \label{eq:il-delta}
            \delta^{(\text{il})} = \eta \cdot \delta_{\textrm{base}} =  \sum_{i=1}^n D_{ii}^{(\text{il})} \gamma_i v_i.
         \end{equation}
         where if $\lambda_i\ne0$, $D_{ii}^{(\text{il})}=\eta\cdot\frac{\exp(\beta_2\lambda_i)-1}{\lambda_i}$; otherwise, $D_{ii}^{(\text{il})}=\eta\cdot\beta_2$.

        $ $\newline

    \textbf{Pairwise interactions inside perturbation $\delta^{(\text{il})}$.}
    Based on perturbations $\delta^{(\text{il})}$ of the IL Attack, the interaction between each pair of perturbations units $(a,b)$ can be computed as 
    \begin{equation}  
 I_{ab}(\delta^{(\textrm{il})}) = \delta^{(\textrm{il})}_a H_{ab} \delta^{(\textrm{il})}_b.
       \quad // \quad \text{According to Eq.~\eqref{eq:shapley-interaction-index_2} in Appendix~\ref{append:interaction}}.
  \end{equation}

Furthermore, we sum the interactions among all pairs of perturbation units $(a, b)$ as follows.
    \begin{equation} \label{eq:il_interaction}
        \begin{aligned}
            \sum_{a,b\in \Omega} I_{ab}(\delta^{(\text{il})})&=(\delta^{(\text{il})})^T H\delta^{(\text{il})} \\
            &= (\sum_{i=1}^n D^{\text{(il)}}_{ii} \gamma_{i} v_{i}^T)(\sum_{j=1}^n \lambda_{j} v_{j} v_{j}^T) (\sum_{k=1}^n D^{\text{(il)}}_{kk} \gamma_{k} v_{k})
             \quad // \quad \text{According to Eq.~\eqref{eq:H-decompose} and Eq.~\eqref{eq:il-delta}}
            \\
            &=\left[\sum_{i=1}^n \underbrace{ (D^{\text{(il)}}_{ii} \gamma_{i} v_{i}^T)( \lambda_{i} v_{i} v_{i}^T) }_{\text{$ v_{i}^T v_{i}=1$}}+ \sum_{1\le i,j\le n \atop i\ne j} \underbrace{(D^{\text{(il)}}_{ii} \gamma_{i} v_{i}^T)( \lambda_{j} v_{j} v_{j}^T)}_{\text{if $i\ne j, v_{i}^T v_{j}=0$}} \right] (\sum_{k=1}^n D^{\text{(il)}}_{kk} \gamma_{k} v_{k}) \\
            &=(\sum_{i=1}^n \lambda_{i}  D^{\text{(il)}}_{ii} \gamma_{i} v_{i}^T ) (\sum_{k=1}^n  D^{\text{(il)}}_{kk} \gamma_{k} v_{k})\\
            &=\sum_{i=1}^n \underbrace{ (\lambda_i D^{\text{(il)}}_{ii} \gamma_{i} v_{i}^T)(  D^{\text{(il)}}_{ii} \gamma_{i} v_{i})}_{\text{$ v_{i}^T v_{i}=1$}}+ \sum_{1\le i, k \le n \atop i\ne k} \underbrace{ (\lambda_{i} D^{\text{(il)}}_{ii} \gamma_{i} v_{i}^T)( D^{\text{(il)}}_{kk} \gamma_{k} v_{k})}_{\text{if $i\ne k, v_{i}^T v_{k}=0$}} \\
            &= \sum_{i=1}^n \lambda_i( D^{\text{(il)}}_{ii}\gamma_{i})^2.
        \end{aligned}
    \end{equation}
    where if $\lambda_i\ne0$, $D_{ii}^{(\text{il})}=\eta\frac{\exp(\beta_2\lambda_i)-1}{\lambda_i}$; otherwise, $D_{ii}^{(\text{il})}=\eta\beta_2$.

\subsection{Multi-step attack}

In this subsection, we revisit the perturbation of the multi-step attack, which is defined in Appendix~\ref{append:delta-closed-form}.

\textbf{Perturbation generated by the multi-step attack.}
Without loss of generality, we focus on the $m_1$-step attack with the step size $\alpha$.
Let $\beta_1=\alpha m_1$ denote the total adversarial strength for the multi-step attack.
Then,  according to Corollary~\ref{corollaryappendix:delta-inf}, if we use the infinite-step adversarial attack with an infinitesimal step size to simplify the multi-step attack,~\textit{i.e.,} $m_1\rightarrow\infty, \alpha\rightarrow 0$, the perturbation $\delta^{(\text{multi})}$ can be represented as
 \begin{equation} \label{eq:delta-multi-inf-in-il}
            \delta^{(\textrm{multi})}\triangleq \delta^{(\infty)} = \sum_{i=1}^n D_{ii} \gamma_i v_i.
        \end{equation}
   where if $\lambda_i\ne0$, $D_{ii}=\frac{\exp(\beta_1\lambda_i)-1}{\lambda_i}$; otherwise, $D_{ii}=\beta_1$.
   $\lambda_{i}$ and $v_i$ denote the $i$-th eigenvalue of the Hessian matrix $H=\nabla_x^2 \textit{Loss}(x)$ and its corresponding eigenvector, respectively. 
$\gamma_i = g^T v_i \in \mathbb{R}$ represents the projection of the gradient $g=\nabla_x \textit{Loss}(x)$ on the eigenvector $v_i$.

$ $\newline

 \textbf{Pairwise interaction inside $\delta^{(\text{multi})}$.}
       Based on the perturbation $\delta^{(\text{multi})}$ of the multi-step attack, the interaction between each pair of perturbation units $(a,b)$ can be computed as 
       \begin{equation}  \label{eq: multi-pairwise-interaction-in-il}
       I_{ab}(\delta^{(\textrm{multi})}) = \delta^{(\textrm{multi})}_a H_{ab} \delta^{(\textrm{multi})}_b.
       \quad // \quad \text{According to Eq.~\eqref{eq:shapley-interaction-index_2} in Appendix~\ref{append:interaction}}.
       \end{equation}
       Furthermore, we sum the interactions among all pairs of perturbation units $(a, b)$ as follows.
       \begin{equation} \label{eq:interaction-multi-in-il}
            \sum_{a, b\in\Omega}[{I_{a b}(\delta^{(\text{multi})})}] = \sum_{i=1}^n \lambda_{i}(D_{ii}\gamma_{i})^2.
            \quad // \quad \text{According to Eq.~\eqref{eq:delta-inf-interaction}}.
        \end{equation}

 \textbf{Fair comparisons between the multi-step attack and the IL Attack.}
         In order to conduct fair comparisons between $I_{ab}(\delta^{(\textrm{multi})})$ based on perturbations $\delta^{(\textrm{multi})}$ and $\sum_{a,b\in \Omega} I_{ab}(\delta^{(\text{il})})$ based on perturbations $\delta^{(\text{il})}$, the magnitudes of these perturbations should be controlled at the same level.
         Specifically, we control the $L_2$ norm of the perturbation $\delta^{(\text{il})}$ to be the same as that of the perturbation $\delta^{(\textrm{multi})}$,~\textit{i.e.,} 
         \begin{equation} \label{eq:il-same-norm}
            \|\delta^{(\text{multi})}\|_2 = \|\delta^{(\text{il})}\|_2.
         \end{equation}
         Note that according to Eq.~\eqref{eq:il-perturbation}, the  $L_2$ norm of the perturbation is determined by $\eta$.
         Hence, in implementation, we control $ \|\delta^{(\text{multi})}\|_2 = \|\delta^{(\text{il})}\|_2$ by setting $\eta$ of the IL Attack to be $\eta = \frac{\sqrt{\sum_{i=1}^n (D_{ii} \gamma_i)^2}}{\sqrt{\sum_{i=1}^n ( D_{ii}^{(\text{base})}\gamma_i)^2}}$.

\subsection{Proof of Proposition~\ref{propositionappendix:il}}

In this subsection, we prove Proposition~\ref{pro:il} in Section~\ref{sec:explaining} under the condition that $\|\delta^{(\textrm{multi})}\|_2 = \| \delta^{(\text{il})}\|_2$.

\begin{propositionappendix}[ Perturbations of the IL Attack exhibit smaller interactions than perturbations of the multi-step attack]
  \label{propositionappendix:il}
  {If the classification loss  for attacking $\textit{Loss}(x)$  is formulated as the cross-entropy loss in Eq.~\eqref{eq:softmax-ce}}, and we use the infinite-step attack with an infinitesimal step size to simplify both the IL Attack and the multi-step attack,
  then under the condition $\|\delta^{(\text{multi})}\|_2 = \|\delta^{(\text{il})}\|_2$, we have $\sum_{a, b\in\Omega}[{I_{a b}(\delta^{(\text{il})})}]\le\sum_{a, b\in\Omega}[{I_{a b}(\delta^{(\text{multi})})}]$.
\end{propositionappendix}

\begin{proof}
Based on Eq.~\eqref{eq:il_interaction} and Eq.~\eqref{eq:interaction-multi-in-il}, the difference  
        between interactions $\sum_{a, b\in\Omega}[{I_{a b}(\delta^{(\text{multi})})}]$ for the multi-step attack and interactions $\sum_{a, b\in\Omega}[{I_{a b}(\delta^{(\text{il})})}] $ for the IL attack can be represented as
        \begin{equation} \label{eq:multi-il-interaction-diff}
            \begin{split}
                &\sum_{a, b\in\Omega}[{I_{a b}(\delta^{(\text{multi})})}]-\sum_{a, b\in\Omega}[{I_{a b}(\delta^{(\text{il})})}] \\
                &= \sum_{i=1}^n \lambda_i \left[(D_{ii}\gamma_i)^2 -( D^{\text{(il)}}_{ii} \gamma_i)^2 \right]. \\
            \end{split}
        \end{equation}
    
        Without loss of generality, let the eigenvalue $\lambda_i$ of the Hessian matrix $H=\nabla_x^2 \textit{Loss}(x)$ monotonically decrease with the index $i$,~\textit{i.e.,}
        $ \lambda_1 \ge  \lambda_2 \ge \dots \ge \lambda_n \ge 0$.
        According to Corollary~\ref{corollaryappendix:semi-definite}, the Hessian matrix $H=\nabla_x^2 \textit{Loss}(x)$ is positive semi-definite.
    Thus, all eigenvalues of the Hessian matrix $H$ are non-negative, \emph{i.e.} $\forall i, \lambda_i\ge 0$.

        We admit that we cannot fully prove that interactions $\sum_{a, b\in\Omega}[{I_{a b}(\delta^{(\text{multi})})}]$ for the multi-step attack are larger than interactions $\sum_{a, b\in\Omega}[{I_{a b}(\delta^{(\text{il})})}]$ for the IL Attack.
        However, we can roughly consider that the term $\left[(D_{ii}\gamma_i)^2 - (D^{\text{(il)}}_{ii} \gamma_i)^2 \right]$ in Eq.~\eqref{eq:multi-il-interaction-diff}  monotonically decreases with the index $i$.
        The reason is as follows.
        Considering the base perturbation for the IL Attack is typically weak, according to Eq.~\eqref{eq:beta2-beta1}, we have $\beta_2 \ll \beta_1$.
        Then,  $(D_{ii})^2 =(\frac{\exp(\beta_1\lambda_i)-1}{\lambda_i})^{2}$ decreases more rapidly with the eigenvalue $\lambda_i$ than $(D^{\text{(il)}}_{ii})^2 =\eta^2(\frac{\exp(\beta_2\lambda_i)-1}{\lambda_i})^{2}$, where $\eta$ is a constant.
     In other words, we have
     \begin{equation}
        (D_{11} \gamma_1)^2 - (D^{\text{(il)}}_{11} \gamma_1)^2 \ge  (D_{22} \gamma_2)^2 - (D^{\text{(il)}}_{22} \gamma_2)^2 \ge \dots \ge  (D_{nn} \gamma_n)^2 - (D^{\text{(il)}}_{nn} \gamma_n)^2.
     \end{equation}

     Based on this, we can use the Chebyshev's sum inequality in Lemma~\ref{lemma:csinequality} to prove that  interactions $\sum_{a, b\in\Omega}[{I_{a b}(\delta^{(\text{multi})})}]$ for the multi-step attack are likely to be larger than interactions $\sum_{a, b\in\Omega}[{I_{a b}(\delta^{(\text{il})})}]$ for the IL Attack,~\textit{i.e.,} Eq.~\eqref{eq:multi-il-interaction-diff} is positive.
    \begin{align}
        &\sum_{a, b\in\Omega}[{I_{a b}(\delta^{(\text{multi})})}]-\sum_{a, b\in\Omega}[{I_{a b}(\delta^{(\text{il})})}] \nonumber \\
        &= \sum_{k=1}^n \lambda_i \left[(D_{ii}\gamma_i)^2 - (D^{\text{(il)}}_{ii} \gamma_i)^2 \right] \nonumber \\
        &\ge n \left(\frac{1}{n}\sum_{i=1}^n \lambda_i\right)  \left[\frac{1}{n} \sum_{i=1}^n \left((D_{ii}\gamma_i)^2 - (D^{\text{(il)}}_{ii} \gamma_i)^2  \right) \right]
         \quad // \quad \text{According to Chebyshev's sum inequality in Lemma~\ref{lemma:csinequality}} \nonumber
         \\
        &= \left(\frac{1}{n}\sum_{i=1}^n \lambda_i\right)  \left[ \|\delta^{(\text{multi})}\|^{2}_2 -\|\delta^{(\text{il})}\|^{2}_2 \right]
        \nonumber
         \\
        &=0.
        \quad // \quad \text{According to Eq.~(\ref{eq:il-same-norm})}.
    \end{align}
    
Thus, Proposition~\ref{propositionappendix:il} is proven.    
\end{proof}

\section{Proof that the single-step VR Attack generates perturbations with smaller interactions than the normal single-step attack}
\label{append:vr}
In this section, we prove Proposition~\ref{pro:vr} in Section~\ref{sec:explaining}, which shows that perturbations generated by the single-step VR Attack~\cite{wu2018understanding} exhibit smaller interactions than perturbations generated the single-step attack with the cross-entropy loss.
Before the proof of Proposition~\ref{pro:vr}, we first revisit the VR Attack.

\subsection{VR Attack}

Given an input $x$ and a DNN  using ReLU activations, we use $\textit{Loss}(x)$ to denote the classification loss.
Let there be  $c$ categories in the classification.
Based on Assumption~\ref{assumption:activation}, the output of the ReLU network before the softmax operation can be written as 
\begin{equation} \label{eq:ia-linear}
    z = \tilde{W}_x^T x + b\in\mathbb{R}^c, 
\end{equation} where $\tilde{W}_x^T =  W_d^T \Sigma_{d-1} W_{d-1}^T \dots \Sigma_{1} W_{1}^T \in \mathbb{R}^{c\times n}$. 
Given the ground-truth label $y\in \{1, 2, \dots, c \}$, the loss function  is formulated as a cross-entropy loss upon  the softmax function, as follows.
\begin{align} 
        p_i &= \frac{\exp(z_i)}{\sum_{j=1}^c \exp(z_j)}, \nonumber \\
        \textit{Loss}(x) &= -\sum_{i=1}^c p^*_i \log(p_i)
\end{align}
where if $i=y, p^*_i =1$; otherwise $p^*_i=0$.

The VR attack adds Gaussian noises on the input $x$ and computes the average loss when adding different noises as follows.
\begin{equation} \label{eq:ce-loss-vr}
    \begin{split}
        \hat{z} &= \tilde{W}_x^T (x + \xi) + b \\
        \hat{p}_i &=  \frac{\exp(\hat{z}_i)}{\sum_{j=1}^c \exp(\hat{z}_j)} \\
        L^{(\text{vr})}(x) & =  \mathbb{E}_{\xi\sim\mathcal{N}(0, \sigma^2 I)}\left[-\sum_{i=1}^c p^*_i \log(\hat{p}_i)\right] 
    \end{split}
\end{equation}
where $\xi\sim\mathcal{N}(0, \sigma^2 I)$ denotes the Gaussian noise.

\textbf{Normalizing adversarial perturbations generated by the normal single-step attack and  the single-step VR Attack for fair comparisons.}
For the normal single-step attack, given an input sample $x$, if we conduct the normal single-step attack with the cross-entropy loss, then the generated perturbation is given as follows.
\begin{equation} \label{eq:single-ce}
    \delta^{(\text{single})}_{\text{ce}} = \eta_{\text{ce}} g,
\end{equation}
where $\eta_{\text{ce}}$ denotes the step size in the normal single-step attack, and $g= \nabla_x \textit{Loss}(x)$. According to Lemma~\ref{lemma:g-H-in-p-form}, if $\textit{Loss}(x)$ is formulated as the cross-entropy loss, then the gradient $g$ is given as follows.
\begin{equation} \label{eq:ce-g}
    g = \tilde{W}_x (p- Y) = \sum_{i\ne y} p_i (\tilde{W}_x)_i + (p_y - 1) (\tilde{W}_x)_y,
\end{equation}
where $Y_y = 1$; and $\forall i\ne y, Y_i = 0$.

Based on Assumption~\ref{assum:ortho-weights}, the weight vectors of different categories $(\tilde{W}_x)_i$ are orthogonal to each other and have the same strength $\sqrt{\kappa}$.
Therefore, we can consider that $p_i$ and $1-p_y$ represent the strength of the gradient in the directions of $(\tilde{W}_x)_i$ and $(\tilde{W}_x)_y$.
Thus, the strength of the gradient $g$ in the normal single-step attack can be approximated by the sum of strengths on each weight direction, as follows.
\begin{equation}\label{eq:strength of normal single-step attack}
    \sqrt{\kappa} \left(\sum_{i\ne y} p_i + (p_y - 1) \right) = 2\sqrt{\kappa}(1-p_y)
\end{equation}

For the single-step VR Attack, given an input sample $x$,  the perturbation is given as follows.
\begin{equation} \label{eq:single-vr}
    \delta^{(\text{single})}_{\text{vr}} = \eta_{\text{vr}} g^{\text{(vr)}},
\end{equation}
where $\eta_{\text{vr}} \in \mathbb{R}$ denotes the step size in the single-step VR Attack, and $g^{\text{(vr)}} = \nabla_x L^{(\text{vr})}(x)$.
The gradient of the loss $L^{(\text{vr})}(x)$ in Eq.~(\ref{eq:ce-loss-vr}) is represented as follows.
\begin{align} \label{eq:vr-g}
    g^{(\text{vr})} &= \frac{\partial L^{(\text{vr})}(x)}{\partial x} \nonumber \\
        &= \mathbb{E}_{\xi\sim\mathcal{N}(0, \sigma^2 I)}\left[\frac{\partial \textit{Loss}(\hat{x})}{\partial \hat{x}}\right] \nonumber\\
        &= \mathbb{E}_{\xi\sim\mathcal{N}(0, \sigma^2 I)}\left[ \tilde{W}_x  (\hat{p}-Y)\right]
        \quad//\quad\text{According to Lemma~\ref{lemma:g-H-in-p-form}}\nonumber \\
        &=\mathbb{E}_{\xi\sim\mathcal{N}(0, \sigma^2 I)}\left[\sum_{i=1}^c \hat{p}_i (\tilde{W}_x)_i- (\tilde{W}_x)_y\right]  \nonumber\\
        &=\sum_{i=1}^c \left[\mathbb{E}_{\xi\sim\mathcal{N}(0, \sigma^2 I)}\left[\hat{p}_i \right]  (\tilde{W}_x)_i\right] -  (\tilde{W}_x)_y \nonumber \\
        &= \tilde{W}_x (\bar{p}- Y) = \sum_{i\ne y} \bar{p}_i (\tilde{W}_x)_i + (\bar{p}_y - 1) (\tilde{W}_x)_y,
\end{align}

In to avoid ambiguity of notation, we use $q$ to denote the average probability, as follows.
\begin{align} \label{eq:p-bar}
    q\triangleq \bar{p} &= \mathbb{E}_{\xi\sim\mathcal{N}(0, \sigma^2 I)}[\hat{p}]
\end{align}

Similar to Eq.~\eqref{eq:strength of normal single-step attack}, the overall strength of the gradient $g^{(\text{vr})}$ in the VR Attack can be approximated by the sum of strengths on each weight direction, as follows.
\begin{equation}\label{eq:strength of single-step VR attack}
    \sqrt{\kappa} \left(\sum_{i\ne y} q_i + (q_y - 1) \right) = 2\sqrt{\kappa}(1-q_y)
\end{equation}

Therefore, in order to fairly compare the interaction between $\delta^{(\text{single})}_{\text{ce}}$ and $\delta^{(\text{single})}_{\text{vr}}$, we normalize $\delta^{(\text{single})}_{\text{ce}}$ and $\delta^{(\text{single})}_{\text{vr}}$ by $2\sqrt{\kappa}(1-p_y)$ and $2\sqrt{\kappa}(1-q_y)$, respectively, according to Eq.~\eqref{eq:strength of normal single-step attack} and Eq.~\eqref{eq:strength of single-step VR attack}.
\begin{equation} \label{eq:norm-ce-vr}
    \begin{split}
        \eta_{\text{ce}} &= \frac{\eta}{2\sqrt{\kappa}(1-p_y)}, \\
        \eta_{\text{vr}} &= \frac{\eta}{2\sqrt{\kappa}(1-q_y)},
    \end{split}
\end{equation}
where $\eta$ is a constant.

Based on Eq.~\eqref{eq:ce-g} and Eq.~\eqref{eq:norm-ce-vr}, the sum of pair-wise interactions inside $\delta^{(\text{single})}_{\text{ce}}$ is expressed as follows.
\begin{align} \label{eq:ce-interaction}
    \sum_{a, b\in \Omega} \left[I_{ab}(\delta^{(\text{single})}_{\text{ce}}) \right] &=  (\eta_{\text{ce}})^2 g^T H g \nonumber \\
    &= (\eta_{\text{ce}} )^2 (p - Y)^T\tilde{W}_x^T  \tilde{W}_x \left( diag(p) - p p^T\right) \tilde{W}_x^T \tilde{W}_x (p - Y) \quad//\quad \text{According to Lemma.~\ref{lemma:g-H-in-p-form}} \nonumber \\
    & = (\eta_{\text{ce}} \kappa)^2 \left[ \sum_{i\ne y} p_i^3 + p_y (p_y - 1)^2 - \left( \sum_{i\ne y} p_i^2  + p_y(p_y-1) \right)^2  \right] \nonumber \\
    &= (\frac{\eta}{2} \sqrt{\kappa})^2 \left[ \sum_{i\ne y} p_i (\frac{p_i}{1-p_y})^2 + p_y - \left( \sum_{i\ne y} p_i\frac{p_i}{1-p_y}  - p_y \right)^2  \right]
\end{align}

Based on Eq.~\eqref{eq:vr-g} and Eq.~\eqref{eq:norm-ce-vr}, the sum of pair-wise interactions inside $\delta^{(\text{single})}_{\text{vr}}$ is expressed as follows.

\begin{align} \label{eq:vr-interaction}
    \sum_{a, b\in \Omega} \left[I_{ab}(\delta^{(\text{single})}_{\text{vr}}) \right] &= (\eta_{\text{vr}})^2 (g^{(\text{vr})})^T H g^{(\text{vr})} \nonumber \\
    &=  (\eta_{\text{vr}})^2 (q - Y)^T\tilde{W}_x^T  \tilde{W}_x \left( diag(p) - p p^T\right) \tilde{W}_x^T \tilde{W}_x (q - Y) \quad//\quad \text{According to Lemma.~\ref{lemma:g-H-in-p-form}} \nonumber \\
    & = (\eta_{\text{vr}} \kappa)^2 \left[ \sum_{i\ne y} p_i q_i^2 + p_y (q_y - 1)^2 - \left( \sum_{i\ne y} p_i q_i + p_y(q_y-1) \right)^2  \right] \nonumber \\
    & =(\frac{\eta}{2} \sqrt{\kappa})^2 \left[ \sum_{i\ne y} p_i (\frac{q_i}{1-q_y})^2 + p_y - \left( \sum_{i\ne y} p_i \frac{q_i}{1-q_y} - p_y \right)^2  \right]
\end{align}

\subsection{Proof of Proposition~\ref{proappendix:p-bar-approx}}
In this subsection, we prove the relationship between the probability  $p$ of the original sample $x$ and the average probability under different noises $q = \mathbb{E}_{\xi\sim \mathcal{N}(0, \sigma^2 I)} [\hat{p}]$.

Based on Assumption~\ref{assum:ortho-weights}, we have $\tilde{W}_x^T \tilde{W}_x = \kappa I$.
Note that given a random variable $\Phi\in\mathbb{R}^{n}\sim\mathcal{N}(\mu, \Sigma)$, let $\Gamma$ denote the affine transformation of $\Phi$, \emph{i.e.} $\Gamma = A\Phi + b$, where $A\in\mathbb{R}^{m\times n}$ and $b\in\mathbb{R}^m$.
Then, $\Gamma$ also follows a multivariate normal distribution, \emph{i.e.,} $\Gamma\sim \mathcal{N}(b+A\mu, A\Sigma A^T)$.
Based on this, we can prove that adding Gaussian noises on the input is equivalent to adding Gaussian noises on the network output as follows.
\begin{align} \label{eq:noise-on-z}
    \Delta z &= z-\hat{z} = (\tilde{W}_x^Tx+b)-(\tilde{W}_x^T(x+\xi)+b)=\tilde{W}_x^T\xi,\; \text{where }\xi\sim\mathcal{N}(0,\sigma^2I) \nonumber\\
    \Rightarrow \Delta z&\sim  \mathcal{N}(0, \tilde{W}_x^T \sigma^2I \tilde{W}_x) = \mathcal{N}(0, \kappa\sigma^2 I).
\end{align}

Here, we use $\xi^{(z)}=\tilde{W}_x^T \xi \sim \mathcal{N}(0, \kappa\sigma^2 I)$ to denote the equivalent noise added on the network output $z$.
In this way, the average probability of the $i$-th category $q_i = \mathbb{E}_{\xi\sim \mathcal{N}(0, \sigma^2 I)}[\hat{p}_i]$ in Eq.~\eqref{eq:p-bar} can be represented as follows.

\begin{align} \label{eq:p-bar-solution}
    \hat{p}_i &=  \frac{\exp(z_i + \xi^{(z)}_i)}{\sum_{j=1}^c \exp(z_j + \xi^{(z)}_j)}, \nonumber \\
    q_i &= \mathbb{E}_{\xi^{(z)} \sim \mathcal{N}(0, \kappa\sigma^2 I)} \left[\hat{p}_i \right].
\end{align}

Note that $q_i$ in Eq.~\eqref{eq:p-bar-solution} does not have an analytic solution.
Thus, to simply the analysis, we consider a simple scenario as follows.
We can roughly consider exclusively adding two symmetrical noises on the output of the $i$-th category $z_i$ without changing outputs of other categories $z_j\, (j\ne i)$ as an approximation of $q_i$. 
We just prove Proposition~\ref{proappendix:p-bar-approx} based on the simplified probability in Eq.~\eqref{eq:pi-tau}. 
We can roughly consider the VR Attack is to simultaneously conduct lots of attacks with different values of $\tau$ on different output dimensions $i$, and to combine adversarial perturbations of all such attacks. 
Therefore, the analysis based on Eq.~\eqref{eq:pi-tau} can also reflect the nature of the VR Attack, to some extent.

\begin{align} \label{eq:pi-tau}
    \bar{p}_i (\tau) \triangleq \frac{1}{2} \left[ \frac{\exp(z_i + \tau)}{\exp(z_i + \tau) + \sum_{j\ne i} \exp(z_j)} +  \frac{\exp(z_i - \tau)}{\exp(z_i - \tau) + \sum_{j\ne i} \exp(z_j)}\right],
\end{align}
where $\tau\in\mathbb{R}$ is a scalar.
In this way, we consider $\bar{p}_i (\tau)$ as an approximation of $\bar{p}_i$ for the further analysis.

\begin{propositionappendix}\label{proappendix:p-bar-approx}
    If $p_i \le 0.5$, then $\bar{p}_i (\tau) - p_i \ge 0$; if $p_i \ge 0.5$, then $\bar{p}_i (\tau) - p_i \le 0$.
\end{propositionappendix} 
\begin{proof}
The value of $\bar{p}_i (\tau)$ can be expressed as follows.
\begin{align}
    \bar{p}_i (\tau) &= \frac{1}{2} \left[ \frac{\exp(z_i + \tau)}{\exp(z_i + \tau) + \sum_{j\ne i} \exp(z_j)} +  \frac{\exp(z_i - \tau)}{\exp(z_i - \tau) + \sum_{j\ne i} \exp(z_j)}\right] \nonumber \\
    &= \frac{1}{2} \left[ \frac{\exp(z_i + \tau) \left( \exp(z_i - \tau) + \sum_{j\ne i} \exp(z_j)\right) + \exp(z_i - \tau) \left( \exp(z_i + \tau) + \sum_{j\ne i} \exp(z_j)\right) }{\left(\exp(z_i + \tau) + \sum_{j\ne i} \exp(z_j) \right) \left(\exp(z_i - \tau) + \sum_{j\ne i} \exp(z_j) \right)} \right] \nonumber \\
    &= \frac{1}{2} \left[ \frac{ \left( \exp(2 z_i) + \sum_{j\ne i} \exp(z_j + z_i + \tau)\right) + \left( \exp(2 z_i) + \sum_{j\ne i} \exp(z_j + z_i - \tau)\right) }{ \exp(2 z_i) + \left( \exp(-\tau) + \exp(\tau)\right)\left(\sum_{j\ne i} \exp(z_i + z_j) \right) + \left(\sum_{j\ne i} \exp(z_j) \right)^2} \right] \nonumber \\
    &= \frac{1}{2} \left[ \frac{  2\exp(2 z_i) +\left( \exp(-\tau) + \exp(\tau)\right)\left(\sum_{j\ne i} \exp(z_i + z_j) \right) }{ \exp(2 z_i) + \left( \exp(-\tau) + \exp(\tau)\right)\left(\sum_{j\ne i} \exp(z_i + z_j) \right) + \left(\sum_{j\ne i} \exp(z_j) \right)^2} \right] \nonumber \\
    &= \frac{1}{2} \left[ \frac{ \frac{2 \exp(z_i)}{ \sum_{j\ne i} \exp(z_j)} + \left( \exp(-\tau) + \exp(\tau)\right) }{ \frac{\exp(z_i)}{ \sum_{j\ne i} \exp(z_j)} + \left( \exp(-\tau) + \exp(\tau)\right) + \frac{\sum_{j\ne i} \exp(z_j)}{\exp(z_i)}}\right].
\end{align}

Note that $\frac{ \exp(z_i)}{ \sum_{j\ne i} \exp(z_j)} = \frac{p_i}{1-p_i}$.
Moreover, let $\omega = \frac{\exp(-\tau) + \exp(\tau)}{2}$.
In this way, 
\begin{align}
   \bar{p}_i (\tau) &=\frac{ \frac{ p_i}{ 1-p_i} + \omega }{ \frac{ p_i}{ 1-p_i} + 2\omega +  \frac{ 1-p_i}{ p_i}}  \nonumber \\
    &= \frac{p_i^2 + \omega p_i(1-p_i) }{p_i^2 + 2\omega p_i (1-p_i) + (1-p_i)^2} \nonumber \\
    &= \frac{p_i^2 + \omega p_i(1-p_i) }{(p_i^2 -2 \omega p_i^2 + p_i^2) + (2\omega p_i - 2p_i)  + 1} \nonumber \\
    &= \frac{(1-\omega)p_i^2 +\omega p_i }{2(1-\omega)p_i^2 + 2(\omega-1)p_i  + 1}.
\end{align}

Thus, the difference between $\bar{p}_i (\tau)$ and $p_i$ can expressed as follows.
\begin{align}
    \bar{p}_i (\tau) - p_i &= \frac{(1-\omega)p_i^2 +\omega p_i }{2(1-\omega)p_i^2 + 2(\omega-1)p_i  + 1} - p_i \nonumber \\
    &= \frac{(1-\omega)p_i^2 +\omega p_i }{2(1-\omega)p_i^2 + 2(\omega-1)p_i  + 1} - \frac{2(1-\omega)p_i^3 + 2(\omega-1)p_i^2  + p_i}{2(1-\omega)p_i^2 + 2(\omega-1)p_i  + 1} \nonumber \\
    &= \frac{2 (\omega-1)p_i^3 - 3(\omega-1) p_i^2 + (\omega-1)p_i}{2(1-\omega)p_i^2 + 2(\omega-1)p_i  + 1} \nonumber \\
    &=\frac{(\omega-1) p_i (1-2p_i)(1-p_i)}{ 2(\omega-1)p_i (1-p_i)  + 1}.
\end{align}

Note that $\omega = \frac{\exp(-\tau) + \exp(\tau)}{2} \ge 1 $, thus $2(\omega-1)p_i (1-p_i)  + 1 \ge 0$.
Therefore, 
\begin{equation}
\begin{split}
        \bar{p}_i (\tau) - p_i\le 0, \quad&\quad \text{if $p_i \ge 0.5$} \\
        \bar{p}_i (\tau) - p_i\ge 0, \quad&\quad \text{if $p_i \le 0.5$}
\end{split}
\end{equation}
\end{proof}

\subsection{Proof of Proposition~\ref{pro:vr} \label{propositionappendix:vr}}

Proposition~\ref{proappendix:p-bar-approx} shows that $\bar{p}(\tau)$ is more balanced than $p$, thus we make the following Assumption~\ref{assum:vr-balance} based on this.
In order to fairly compare the single-step VR Attack and the normal single-step attack, we ensure perturbations generated by the single-step VR Attack to have the same attacking strength as perturbations generated by the normal single-step attack by using the normalization in Eq.~\eqref{eq:norm-ce-vr}.
Then, under Assumption~\ref{assum:ortho-weights} and Assumption~\ref{assum:vr-balance}, we can prove that perturbations generated by the single-step VR Attack exhibit smaller interactions than those generated by the normal single-step attack.

Here, we assume that the input $x$ can be correctly classified by the DNN with high confidence, \emph{i.e.} $p_y \ge 0.5$.
Then, we can derive the properties of the average probability $q$ based on Proposition~\ref{proappendix:p-bar-approx} as follows.

\textbullet{ According to Proposition~\ref{proappendix:p-bar-approx}, $\bar{p}_y(\tau) \le p_y$.
Because $\bar{p}_y(\tau)$ defined in Eq.~\eqref{eq:pi-tau} is an approximation of $q_y$, we can consider that $q_y$ is also smaller than $p_y$ as follows.} 
\begin{equation} \label{eq:q_y-le-p_y}
    q_y \le p_y
\end{equation}

\textbullet{ Proposition~\ref{proappendix:p-bar-approx} demonstrates that when the probability value of the $i$-th category on the sample without noise is small ($p_i\le 0.5$), then the the added noises boost the probability value.
In contrast, if the probability value of the $i$-th category on the sample without noise is large ($p_i\ge 0.5$), then the the added noises reduce the probability value.
In this way, the probability values in $q$ are more balanced than probability values in $p$.} 

Therefore, just like Definition~\ref{def:balance}, we make the following assumption of more balanced probabilities as follows.
\begin{assumption} \label{assum:vr-balance}
We assume $p_y\ge 0.5$ for the ground-truth category $y$ of the input sample $x$. 
For prediction probabilities of the other $c-1$ categories, without loss of generality, let $p_{i_1}\ge p_{i_2} \ge \dots \ge p_{i_{c-1}}$ and $q_{j_1} \ge q_{j_2} \ge \dots \ge q_{j_{c-1}}$.
Because in general, $q$ is more balanced than $p$, we assume $\forall 1\le k \le c-2$ $q_{j_{k}} - q_{j_{k+1}} \le p_{i_{k} } - p_{i_{k+1}}$.
\end{assumption}

$ $\newline

\begin{propositionappendix} 
    [Perturbations generated by the single-step VR Attack exhibit smaller interactions  than those generated by the normal single-step attack]
    \label{theoremappendix:vr}
    {Let the classification loss $\textit{Loss}(x)$  be formulated as the cross-entropy loss in Eq.~\eqref{eq:softmax-ce}}.
    In order to fairly compare the single-step VR Attack and the normal single-step attack, we ensure perturbations generated by the single-step VR Attack to have the same attacking strength as perturbations generated by the normal single-step attack by using the normalization in Eq.~\eqref{eq:norm-ce-vr}.
    Then, under Assumption~\ref{assum:ortho-weights} and Assumption~\ref{assum:vr-balance},  we have $\sum_{a, b\in\Omega} \left[I_{ab}\left(\delta^{(\text{single})}_{{\text{vr}}}\right)\right] \le  \sum_{a, b\in\Omega} \left[I_{ab}\left(\delta^{(\text{single})}_{{\text{ce}}}\right) \right]$.
\end{propositionappendix}
\begin{proof}

According to Eq.~\eqref{eq:q_y-le-p_y}, we have $\frac{1-p_y}{1-q_y} \le 1$.
Thus, according to Assumption~\ref{assum:vr-balance}, we have 
\begin{align}
    &\frac{1-p_y}{1-q_y}( q_{j_{k}} - q_{j_{k+1}}) \le p_{i_{k}} - p_{i_{k+1}} \nonumber \\
    \Rightarrow &\frac{q_{j_k}}{1-q_y} - \frac{q_{j_{k+1}}}{1-q_y} \le \frac{p_{i_{k}}}{1-p_y} - \frac{p_{i_{k+1}}}{1-p_y}
\end{align}

To simplify the notation, let $a_{i_k} = \frac{p_{i_{k}}}{1-p_y}$ and $b_{j_k} = \frac{q_{i_{k}}}{1-q_y}$.
In this way, we have $b_{j_k} - b_{j_{k+1}} \le  a_{i_k} - a_{i_{k+1}}$.
In other words,  we have 
\begin{equation} \label{eq:a-b-order}
    a_{i_1} - b_{j_{1}} \ge a_{i_2} - b_{j_{2}} \ge \dots \ge a_{i_{c-1}} - b_{j_{c-1}}
\end{equation}

According to Eq.~(\ref{eq:ce-interaction}) and Eq.~(\ref{eq:vr-interaction}), the difference between interactions $\sum_{a, b\in \Omega} \left[I_{ab}(\delta^{(\text{single})}_{\text{ce}}) \right]$ and $\sum_{a, b\in \Omega} \left[I_{ab}(\delta^{(\text{single})}_{\text{vr}}) \right]$ can be represented as follows.

\begin{align}
    &\sum_{a, b\in \Omega} \left[I_{ab}(\delta^{(\text{single})}_{\text{ce}}) \right] -  \sum_{a, b\in \Omega} \left[I_{ab}(\delta^{(\text{single})}_{\text{vr}}) \right] \nonumber\\
    &=(\frac{\eta\sqrt{\kappa}}{2})^2 \Big\{ \sum_{i\ne y} p_i \left((\frac{p_i}{1-p_y})^2 - (\frac{q_i}{1-q_y})^2 \right) \nonumber \\
    &\qquad-  \left( \sum_{i\ne y} p_i (\frac{p_i}{1-p_y} - \frac{q_i}{1-q_y}) \right) \left(\sum_{i\ne y} p_i \frac{p_i}{1-q_y} - p_y + \sum_{i\ne y} p_i \frac{q_i}{1-q_y} - p_y \right)\Big\} \nonumber \\
    &= (\frac{\eta \sqrt{\kappa}}{2})^2 \Big\{ \underbrace{\sum_{k=1}^{c-1} p_{i_k} \left(a_{i_k}^2 - b_{j_k}^2 \right)}_{\text{(I)}} - \underbrace{\left(\sum_{k=1}^{c-1}p_{i_k} (a_{i_k} - b_{j_k}) \right)}_{\text{(II)}} \underbrace{\left( \sum_{k=1}^{c-1} p_{i_k} a_{i_k} - p_y + \sum_{k=1}^{c-1} p_{i_k} b_{j_k} - p_y\right)}_{\text{(III)}}
\end{align}

\textbullet{ Term (I) $\ge 0$. } 
Because $q_{j_1}\ge q_{j_2} \ge \dots \ge q_{j_{c-1}} \ge 0$, then $b_{j_1}\ge b_{j_2} \ge \dots \ge b_{j_{c-1}} \ge 0$.
Because $p_{i_1}\ge p_{i_2} \ge \dots \ge p_{i_{c-1}}\ge 0$, then we have $a_{i_1}\ge a_{i_2} \ge \dots \ge a_{i_{c-1}} \ge 0$ and $p_{i_1} \left(a_{i_1} + b_{j_1} \right) \ge p_{i_2} \left(a_{i_2} + b_{j_2} \right)  \ge \dots \ge p_{i_{c-1}} \left(a_{i_{c-1}} + b_{j_{c-1}} \right)  \ge 0$.
Note that $\sum_{k=1}^{c-1} a_{i_k}  = \sum_{k=1}^{c-1} \frac{p_{i_k}}{1-p_y} = 1$ and $\sum_{k=1}^{c-1} b_{i_k}  = \sum_{k=1}^{c-1} \frac{q_{j_k}}{1-q_y} = 1$.
Thus, according to Eq.~\eqref{eq:a-b-order} and Lemma~\ref{lemma:csinequality} (Chebyshev's sum inequality), we have
\begin{align}
    \text{(I)} &= \sum_{k=1}^{c-1} p_{i_k} \left(a_{i_k}^2 - b_{j_k}^2 \right) \nonumber \\
    &= \sum_{k=1}^{c-1} p_{i_k} \left(a_{i_k} + b_{j_k} \right) \left(a_{i_k} - b_{j_k} \right) \nonumber \\
    &\ge \frac{1}{c-1} \underbrace{\left( \sum_{k=1}^{c-1}p_{i_k} \left(a_{i_k} + b_{j_k} \right) \right)}_{\ge 0} \underbrace{\left( \sum_{k=1}^{c-1} \left( a_{i_k} - b_{j_k} \right)\right)}_{=0} = 0
\end{align}

\textbullet{ Term (II) $\ge 0$. } 
Because $a_{i_1}\ge a_{i_2} \ge \dots \ge a_{i_{c-1}} \ge 0$ and $b_{j_1}\ge b_{j_2} \ge \dots \ge b_{j_{c-1}} \ge 0$, we have $a_{i_1} + b_{j_1}\ge a_{i_2} + b_{j_2} \ge \dots \ge a_{i_{c-1}} + b_{j_{c-1}} \ge 0$.
So according to and Eq.~\eqref{eq:a-b-order} and Lemma~\ref{lemma:csinequality} (Chebyshev's sum inequality), we have
\begin{align}
    \text{(II)} &= \sum_{k=1}^{c-1} p_{i_k} \left(a_{i_k} - b_{j_k} \right) \nonumber \\
    &\ge \frac{1}{c-1}  \underbrace{\left( \sum_{k=1}^{c-1}p_{i_k} \right)}_{\ge 0} \underbrace{\left( \sum_{k=1}^{c-1} \left( a_{i_k} - b_{j_k} \right)\right)}_{=0} = 0
\end{align}

\textbullet{ Term (III) $\le 0$. } 
Let us focus on $ \sum_{k=1}^{c-1} p_{i_k} a_{i_k} - p_y$, which is bounded as follows.
\begin{align}
    p_{i_k} a_{i_k} &= p_{i_k} \frac{p_{i_k}}{1-p_y} \le p_{i_k} \frac{1-p_y}{1-p_y} = p_{i_k} \nonumber \\
\Rightarrow & \sum_{k=1}^{c-1} p_{i_k} a_{i_k} \le \sum_{k=1}^{c-1} p_{i_k} = 1-p_y \nonumber \\
\Rightarrow & \sum_{k=1}^{c-1} p_{i_k} a_{i_k} - p_y \le 1 -2p_y \le 0 
\quad//\quad \text{According to Assumption~\ref{assum:vr-balance}, $p_y\ge 0.5$}
\end{align}

Similarly, we have 
\begin{align}
    p_{i_k} b_{i_k} &= p_{i_k} \frac{q_{j_k}}{1-q_y} \le p_{i_k} \frac{1-q_y}{1-q_y} = p_{i_k} \nonumber \\
\Rightarrow & \sum_{k=1}^{c-1} p_{i_k} b_{i_k} \le \sum_{k=1}^{c-1} p_{i_k} = 1-p_y \nonumber \\
\Rightarrow & \sum_{k=1}^{c-1} p_{i_k} b_{i_k} - p_y \le 1 -2p_y \le 0
\end{align}

Thus, (III) $\le 0$.

In summary,  we have
\begin{align} 
&\sum_{a, b\in \Omega} \left[I_{ab}(\delta^{(\text{single})}_{\text{ce}}) \right] -  \sum_{a, b\in \Omega} \left[I_{ab}(\delta^{(\text{single})}_{\text{vr}}) \right] \nonumber\\
    &= \underbrace{(\frac{\eta \sqrt{\kappa}}{2})^2}_{\ge 0} \Big\{ \underbrace{\sum_{k=1}^{c-1} p_{i_k} \left(a_{i_k}^2 - b_{j_k}^2 \right)}_{\text{(I)} \ge 0} - \underbrace{\left(\sum_{k=1}^{c-1}p_{i_k} (a_{i_k} - b_{j_k}) \right)}_{\text{(II)} \ge 0} \underbrace{\left( \sum_{k=1}^{c-1} p_{i_k} a_{i_k} - p_y + \sum_{k=1}^{c-1} p_{i_k} b_{j_k} - p_y\right)}_{\text{(III)} \le 0} \nonumber \\
    &\ge 0
\end{align}

Thus, the proposition is proven.

\end{proof}

\section{More discussion of the conclusion that the SGM Attack generates perturbations with smaller interactions than the multi-step attack} \label{append:sgm}
The SGM Attack~\cite{Wu2020Skip} modifies the gradient in the back-propagation of ResNets to reduce gradient information from the residual modules.
Specifically, a ResNet~\cite{he2016deep} consists several residual blocks, where the feature of the $l$-th residual block is formulated as follows.
\begin{equation}
    x_{l+1} = x_l + F_{l} (x_l)
\end{equation}
where $x_l \in \mathbb{R}^{n_l}$ and $x_{l+1} \in \mathbb{R}^{n_{l+1}}$ denote the input and the output of the $l$-th residual block, respectively.
$F_l (\cdot)$ represents the nonlinear transformation of the $l$-th residual block.
The normal back-propagation rule of the residual block is given as follows.
\begin{equation}\label{eq:normal grad}
    \frac{\partial x_{l+1}}{\partial x_l^T} = I + \frac{\partial F_l(x_l)}{\partial x_l^T}
\end{equation}

The SGM Attack computes the forward propagation as normal, but adds a decay factor to back-propagation as follows.
\begin{equation}
    (\frac{\partial x_{l+1}}{\partial x_l^T})_{\text{sgm}} = I + \gamma\frac{\partial F_l(x_l)}{\partial x_l^T}
\end{equation}
where $\gamma \in (0, 1]$.
In this way, the gradient information of the residual module is partially reduced.
Such an operation can be understood as adding a specific dropout operation on the gradient.
Moreover, in our previous work~\cite{dropout2020zhang}, we have proven that the dropout operation can reduce the significance of interactions. 
Thus, compared with the multi-step attack based on gradients in Eq.~\eqref{eq:normal grad}, the SGM Attack generates perturbations with smaller interactions.

\section{More discussion of the conclusion that the GhostNet generates perturbations with smaller interactions than the multi-step attack} \label{append:ghost}
The GhostNet~\cite{li2020learning} uses two types of erosion to revise a pretrained DNN, in order to generate more transferable perturbations on the revised DNN.

\textbf{Dropout Erosion.} Given the input $x$, let $f_{l+1}$ denote the non-linear transformation in the $(l+1)$-th layer, and let $x_l$ denote the input feature of the $(l+1)$-th layer. Then in the GhostNet, the output feature of the $(l+1)$-th layer is given as follows.
\begin{equation}
    \begin{split}
        x_{l+1}=f_{l}\left(\frac{r_{l} \odot x_{l}}{1-p}\right),\\
        r_{l} \sim \text{Bernoulli} (1-p)
    \end{split}
\end{equation}
where $\odot$ denotes element-wise production. The element-wise production between $r_l$ and $x_l$ represents the dropout operation~\cite{srivastava2014dropout}, and $1-p$ denotes the probability that each element in $x_l$ is dropped.
In the GhostNet, the dropout operation is densely applied to each layer of the DNN.

 In our previous work~\cite{dropout2020zhang}, we have proven that the dropout operation can decrease the significance of interactions. Thus, compared with the multi-step attack on normally-trained DNNs without additional dropout layers, the multi-step attack on the GhostNet generates perturbations with smaller interactions.

\textbf{Skip Connection Erosion.}
In ResNets~\cite{he2016deep},  the residual block is formulated as follows.
\begin{equation}
    x_{l+1} = x_l + F_l (x_l)
\end{equation}
where $x_l$ and $x_{l+1}$ denote the input and the output of the $l$-th residual block, respectively.
$F_l (\cdot)$ represents the nonlinear transformation of the $l$-th residual block.

For DNNs with skip connection blocks, the GhostNet uses skip connection erosion to revise the DNN as follows.
\begin{equation}\label{eq:forward in GhostNet}
    x_{l+1} = \eta_l x_l + F_l (x_l)
\end{equation}
where $\eta_l$ is drawn from a uniform distribution.
\emph{I.e.} given $\tau > 0$, $\eta_l\sim U[1-\tau, 1+\tau]$.

To simplify our analysis, we consider the DNN where the non-linear transformation in $F_l$ is the ReLU activation.
In this way, the forward propagation in Eq.~\eqref{eq:forward in GhostNet} can be written as follows.
\begin{equation}
    \begin{split}
        x_{l+1} &= \eta_l x_l + \sigma\left( W^T_l x_l  \right) \\
    \end{split}
\end{equation}
where $W_l$ denotes the weight matrix of the $l$-th residual block, and $\sigma(\cdot)$ represents the activation function.
According to Assumption~\ref{assumption:activation}, if the activation function is set as the ReLU activation, then
\begin{equation}
    \begin{split}
        x_{l+1} &= \eta_l x_l + \sigma\left( W^T_l x_l  \right) \\
        &= \eta_l x_l + \Sigma_{l}\left( W^T_l (x_l) \right) \\
        &= \left[\eta_{l} I + \Sigma_{l} W_{l}^T \right]  x_l,
    \end{split}
\end{equation}
where $\Sigma_l$ is a diagonal matrix, and the diagonal element  $(\Sigma_l)_{ii}\in\{0, 1\}$ represents the activation state of the $i$-element in the ReLU activation function.
Suppose the DNN contains a total of $d$ residual blocks.
Then, the output of the network is given as follows.
\begin{equation} \label{eq:ghost-skip-forward}
    \begin{split}
        \hat{z} &= W_d^T \left[\eta_{d-1} I + \Sigma_{d-1} W_{d-1}^T \right] \dots \left[\eta_{1} I + \Sigma_{1} W_{1}^T\right] x 
    \end{split}
\end{equation}

Moreover, the output of the ReLU network with skip connection erosion can be disentangled as follows.
\begin{align}
    \hat{z}&= W_d^T \left[ I+ \Sigma_{d-1} W_{d-1}^T + (\eta_{d-1} - 1) I \right] \dots \left[I + \Sigma_{1} W_{1}^T + (\eta_{1} - 1) I \right] x \nonumber \\
        &=  W_d^T \left[ I + \Sigma_{d-1} W_{d-1}^T \right] \dots \left[ I + \Sigma_{1} W_{1}^T\right] x + W_d^T \left[\prod_{l=1}^{d-1} (\eta_{l}-1) \right]x + \dots \nonumber \\
        &= z + \Delta z
\end{align}
where $z =  W_d^T \left[ I + \Sigma_{d-1} W_{d-1}^T \right] \dots \left[ I + \Sigma_{1} W_{1}^T\right] x$ represents the normal network output without skip connection erosions.
Note that $\eta_l\sim U[1-\tau, 1+\tau]$.
If the number of residual blocks $d$ is large, we can approximately assume that $\Delta z$ is a Gaussian noise, \emph{i.e.} $\Delta z\sim \mathcal{N}(0, \sigma_z^2 I)$.
As discussed in Appendix~\ref{append:vr}, adding Gaussian noise to the input is roughly equivalent to adding noise to the network output, and adding noise on the network work can make the DNN generates perturbations with smaller interactions.
Therefore, we can consider the skip connection erosion also make the DNN generates perturbations with smaller interactions.

\section{More discussion of the conclusion that the FD Attack generates perturbations with smaller interactions than the multi-step attack}

The FD Attack~\cite{Inkawhich2020Transferable} puts a two-layer network for each category upon the feature of an
intermediate layer in the DNN.
More specifically, the FD Attack retrains a two-layer network for each category by using the intermediate-layer featur es of a pre-trained DNN as the input. Then, the adversarial perturbation is generated by maximizing the output probability of the two-layer network corresponding to the target category.

Note that each two-layer network is trained for binary classification in the one-versus-all way.
Let $h^{(i)} (\cdot)$ denote the two-layer network trained for the category $i\in \{1, 2, \dots, c\}$, and let $f_l(x)$ denote the output feature of the $l$-th layer in the pretrained DNN.
The FD Attack trains $h^{(i)} (f_l (x))$ to estimate the prediction probability of the sample $x$ belonging to the $i$-th category, which is given as follows.
\begin{equation}
    P\left(C(x) = i \mid f_l(x) \right) = \text{Sigmoid}\left( h^{(i)} (f_l (x)) \right),
\end{equation}
where $C(x) = i$ represents that the input sample $x$ belongs to category $i$.
Note that the FD Attack only focuses on the targeted attack.
In this way, after obtaining $c$ two-layer networks, the FD Attack generates adversarial perturbations by maximizing the output probability of the two-layer network corresponding to the target category $t$ as follows.
\begin{equation}
    \max_{ \|\delta \|_p \le \epsilon} P\left(C(x) = t \mid f_l(x + \delta) \right)
\end{equation}

Obviously, the newly trained shallow networks $h^{(i)} (f_l (x))$ usually encode more linear feature representations than the original deep network.
In this way, we can consider eigenvalues of the Hessian matrix in the shallow network are smaller than those in the original DNN.
If the newly-trained two-layer network exhibits the similar classification power to the original deep network, then, according to Lemma~\ref{lemma:interaction}, the two-layer network learned by the FD Attack is very likely to generate perturbations with smaller interactions than the original DNN.

\section{More discussion of the conclusion that the FI Attack generates perturbations with smaller interactions than the multi-step attack}

The FI Attack~\cite{wang2021feature} generates adversarial perturbations as follows. 
For the input $x$, the FI Attack first randomly masks the input by the mask matrix $M\in\{0,1\}^n$.
Then, the FI Attack feeds the input under different masks into the DNN.
For each specific mask $M$, let $Loss(M\odot x)$ denote the loss function on the masked input, and let $f_k(M\odot x)$ the feature of the $k$-th layer $f_k(M\odot x)$, where $\odot$ denotes the element-wise multiplication.
Then, $g^{(M)}=\frac{\partial Loss(M\odot x)}{\partial f_k(M\odot x)}$ denote the gradient of the loss \emph{w.r.t.} the feature of the $k$-th layer computed on the masked input.
The FI Attack computes the average gradient over different masks as follows.
\begin{equation}
  \bar g = \mathbb{E}_{M\sim \text{Bernoulli}(q)}\frac{\partial Loss(M\odot x)}{\partial f_k(M\odot x)}
\end{equation}
The objective of the FI Attack is to estimate the adversarial perturbation to further push intermediate-layer features away from the direction of the average gradient $\bar g$.
\begin{equation}
  \min_{\delta} \sum {\bar g \odot f_k(x+\delta)}
\end{equation}

In order to prove that the FI Attack generates perturbations with smaller interactions than the multi-step attack, we will prove the following three propositions.
(1) In the pairwise interactions between perturbation units, high-order interactions are the major components, while low-order interactions constitute only a small proportion. (2) In the FI Attack, when the input is masked, its destruction on high-order interactions are more significant than its destruction on low-order interactions. In other words, most high-order interactions are destroyed and low-order interactions are remained. Therefore, the generated perturbations mainly encode low-order interactions. (3) Because low-order interactions are usually small, the perturbation that mainly encode low-order interactions is more likely to exhibit smaller interactions than the perturbation that contain both high-order and low-order interactions.
Therefore, interactions between perturbation units generated by the FI Attack are smaller than those generated by the multi-step attack.

\begin{figure}[t!]
    \centering
    \includegraphics[width=0.5\linewidth]{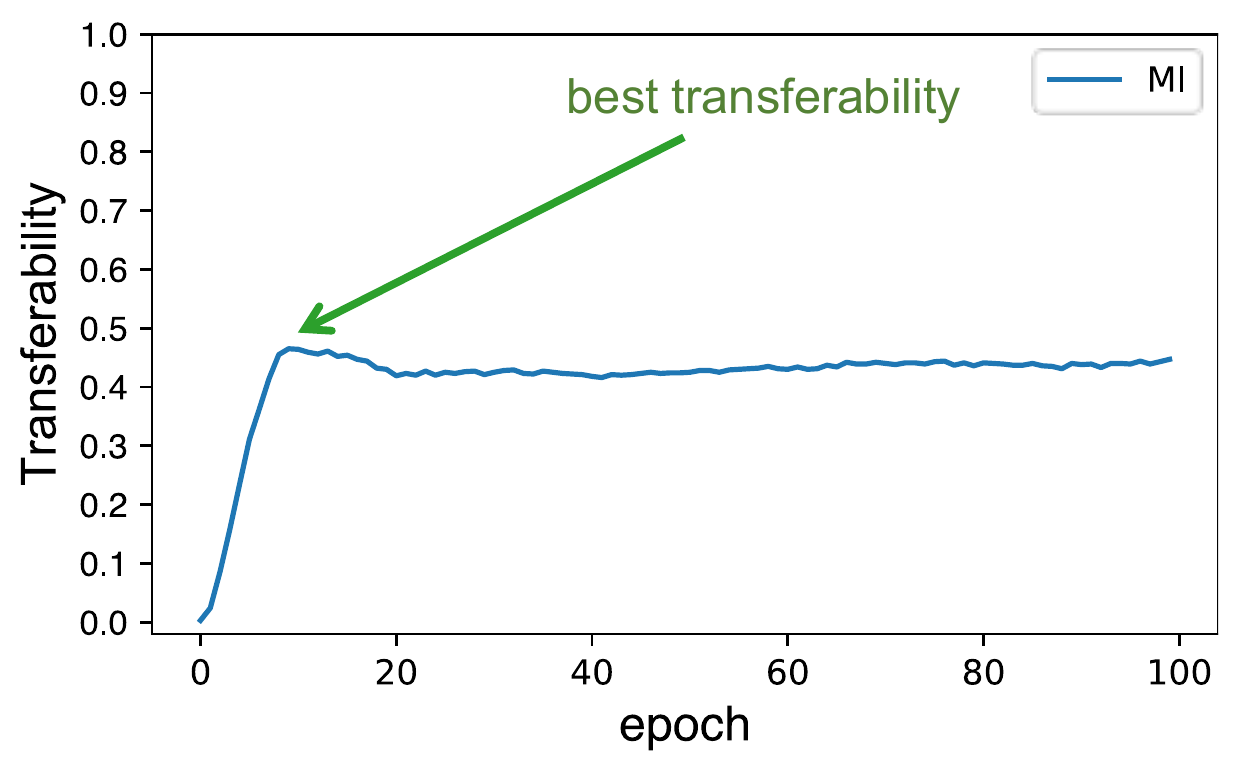}
    \caption{The curve of adversarial transferability in different steps.}
    \label{fig:loo}
\end{figure}

\section{Evaluation of adversarial transferability via leave-one-out validation} \label{append:loo}

This Section provides the motivation and details of using the leave-one-out validation in experiments.
The motivation of using the leave-one-out validation is the following phenomenon.
As shown in Figure~\ref{fig:loo}, when we conduct the MI Attack for multiple steps, the highest transferability of the MI Attack is reached in an intermediate step, rather than in the last step.
This phenomenon presents a challenge for the fair comparison of adversarial transferability between different attacking methods.

To  enable impartial comparisons of adversarial transferability, we estimate the adversarial perturbations with the highest transferability for each input image via the leave-one-out (LOO) validation as follows.
Given a set of clean examples $\{(x_i, y_i)\}_{i=1}^N$, where $y_i\in\{1, 2, \dots, C\}$, ${x}^t_i$ denotes the adversarial example at step $t$ \emph{w.r.t.} the clean example $x_i$, where $t\in\{1, 2, \dots, T\}$, and $T$ is the number of total steps.
Given a target DNN $h(\cdot)$ and an input example $x$, where $h(\cdot)$ denotes the output before the softmax layer, the prediction of the example $x$ is given as  $\mathcal{C}(x)=\arg\max_{k} h_k(x),k\in\{1,\dots,C\}$.
For the example $x_i$, we use its adversarial example generated after $t_i^*$ steps, where $t_i^*$ is determined via leave-one-out validation, which is defined as follows.
\begin{equation*}
\hat{x}_i \triangleq{x_i^{t^*_i}},  \; s.t. \quad t^*_i = \arg\max_t \mathbb{E}_{i'\in\{1,2,\dots, N\}\setminus \{i\}}\left[{\mathbbm{1}}[\mathcal{C}({x}^t_{i'})\ne y_{i'}] \right],
\end{equation*}
where $\mathbbm{1}[\cdot]$ is the indicator function.
Then the average transferability is given as follows.
\begin{equation*}
    \textit{Transferability} = \mathbb{E}_i \left[ \mathbbm{1}[\mathcal{C}(\hat{x}_i)\ne y_{i}]\right].
\end{equation*}

\begin{table}[t!]
\renewcommand{\arraystretch}{1.1}
\renewcommand\tabcolsep{3.0pt}
\small
\centering
\caption{
The success rates of $L_\infty$ black-box attacks crafted by MI and MI$+$IR on four source models (RN-34/152, DN-121/201) against seven target models. The interaction loss can boost the transferability of the MI Attack.
}
\resizebox{0.9\linewidth}{!}{
\begin{tabular}{c|c|ccccccc}
\hline
  Source& Method & VGG-16 & RN152 & DN-201 & SE-154 & IncV3 & IncV4 & IncResV2  \\ \hline
  \multirow{2}{*}{{RN-34}}
  & MI & 80.1$\pm$0.5 & 73.0$\pm$2.3 & 77.7$\pm$0.5 & 48.9$\pm$0.8 & 46.2$\pm$1.2 & 39.9$\pm$0.5 & 34.8$\pm$2.5 \\
  & MI$+$IR & \textbf{90.0$\pm$0.5} & \textbf{85.7$\pm$0.3} & \textbf{88.5$\pm$0.6} & \textbf{67.0$\pm$0.1} & \textbf{66.9$\pm$1.8} & \textbf{60.2$\pm$0.7} & \textbf{53.9$\pm$2.3} \\ \hline
  \multirow{2}{*}{{RN-152}}
  & MI & 70.3$\pm$0.6 & -- & 74.8$\pm$1.4 & 51.7$\pm$0.8 & 47.1$\pm$0.9 & 40.5$\pm$1.6 & 36.8$\pm$2.7 \\
  & MI$+$IR & \textbf{78.9$\pm$1.4} & -- & \textbf{82.2$\pm$2.0} & \textbf{68.3$\pm$0.3} & \textbf{63.6$\pm$1.2} & \textbf{59.0$\pm$0.4} & \textbf{56.3$\pm$1.0} \\\hline
  \multirow{2}{*}{{DN-121}}
  & MI & 83.0$\pm$4.9 & 72.0$\pm$0.7 & 91.5$\pm$0.2 & 58.4$\pm$2.6 & 54.6$\pm$1.6 & 49.2$\pm$2.4 & 43.9$\pm$1.5 \\
  & MI$+$IR & \textbf{89.0$\pm$0.8} & \textbf{83.2$\pm$1.5} & \textbf{93.4$\pm$0.6} & \textbf{74.2$\pm$0.7} & \textbf{69.6$\pm$0.9} & \textbf{64.7$\pm$0.5} & \textbf{58.2$\pm$2.3} \\\hline
  \multirow{2}{*}{{DN-201}}
  & MI & 77.3$\pm$0.8 & 74.8$\pm$1.4 & -- & 64.6$\pm$1.0 & 56.5$\pm$2.5 & 51.1$\pm$2.1 & 47.8$\pm$1.9 \\
  & MI$+$IR & \textbf{87.3$\pm$0.3} & \textbf{81.6$\pm$2.0} & -- & \textbf{75.4$\pm$0.6} & \textbf{66.6$\pm$3.3} & \textbf{60.0$\pm$1.0} & \textbf{62.1$\pm$0.7} \\ \hline
\end{tabular}}
\vspace{-10pt}
\label{table:mi+ir}
\end{table}

\begin{table}[t!]
\renewcommand{\arraystretch}{1.1}
\renewcommand\tabcolsep{3.0pt}
\small
\centering
\caption{
The success rates of $L_\infty$ black-box attacks crafted by SGM and SGM$+$IR on four source models (RN-34/152, DN-121/201) against seven target models. The interaction loss can boost the transferability of the SGM Attack.
}
\resizebox{0.9\linewidth}{!}{
\begin{tabular}{c|c|ccccccc}
\hline
  Source& Method & VGG-16 & RN152 & DN-201 & SE-154 & IncV3 & IncV4 & IncResV2  \\ \hline
  \multirow{2}{*}{{RN-34}}
  & SGM & 91.8$\pm$0.6 & 89.0$\pm$0.9 & 90.0$\pm$0.4 & 68.0$\pm$1.4 & 63.9$\pm$0.3 & 58.2$\pm$1.1 & 54.6$\pm$1.2 \\
  & SGM$+$IR & \textbf{94.7$\pm$0.6} & \textbf{91.7$\pm$0.6} & \textbf{93.4$\pm$0.8} & \textbf{72.7$\pm$0.4} & \textbf{68.9$\pm$0.9} & \textbf{64.1$\pm$1.3} & \textbf{61.3$\pm$1.0} \\ \hline
  \multirow{2}{*}{{RN-152}}
  & SGM & 88.2$\pm$0.5 & -- & 90.2$\pm$0.3 & 72.7$\pm$1.4 & 63.2$\pm$0.7 & 59.1$\pm$1.5 & 58.1$\pm$1.2 \\
  & SGM$+$IR & \textbf{92.0$\pm$1.0} & -- & \textbf{92.5$\pm$0.4} & \textbf{79.3$\pm$0.1} & \textbf{69.6$\pm$0.8} & \textbf{66.2$\pm$1.0} & \textbf{63.6$\pm$0.9} \\ \hline
  \multirow{2}{*}{{DN-121}}
  & SGM & 88.7$\pm$0.9 & 88.1$\pm$1.0 & \textbf{98.0$\pm$0.4} & 78.0$\pm$0.9 & 64.7$\pm$2.5 & 65.4$\pm$2.3 & 59.7$\pm$1.7 \\
  & SGM$+$IR & \textbf{91.7$\pm$0.2} & \textbf{90.4$\pm$0.4} & {94.3$\pm$0.1} & \textbf{87.0$\pm$0.4} & \textbf{78.8$\pm$1.3} & \textbf{79.5$\pm$0.2} & \textbf{75.8$\pm$2.7} \\ \hline
  \multirow{2}{*}{{DN-201}}
  & SGM & 87.3$\pm$0.3 & \textbf{92.4$\pm$1.0} & -- & 82.9$\pm$0.2 & 72.3$\pm$0.3 & 71.3$\pm$0.6 & 68.8$\pm$0.5 \\
  & SGM$+$IR & \textbf{89.5$\pm$0.9} & {91.8$\pm$0.7} & -- & \textbf{87.3$\pm$1.2} & \textbf{82.5$\pm$0.8} & \textbf{80.3$\pm$0.3} & \textbf{81.5$\pm$0.5} \\ \hline
\end{tabular}}
\label{table:sgm+ir}
\end{table}

\begin{table}[t!]
\renewcommand{\arraystretch}{1.1}
\renewcommand\tabcolsep{3.0pt}
\small
\centering
\caption{
The success rates of $L_\infty$ black-box attacks crafted by VR and VR$+$IR on four source models (RN-34/152, DN-121/201) against seven target models. The interaction loss can boost the transferability of the VR Attack.
}
\vspace{1pt}
\resizebox{0.7\linewidth}{!}{
\begin{tabular}{c|c|ccccccc}
\hline
  Source& Method & VGG-16 & RN152 & DN-201 & SE-154 & IncV3 & IncV4 & IncResV2  \\ \hline
  \multirow{2}{*}{{RN-34}}
  & VR & 85.1 & 85.3 & 87.0 & 55.7 & 54.3 & 50.7 & 43.7 \\
  & VR$+$IR & \textbf{90.8} & \textbf{92.2} & \textbf{93.3} & \textbf{75.4} & \textbf{75.4} & \textbf{67.5} & \textbf{66.1} \\ \hline
  \multirow{2}{*}{{DN-121}}
  & VR & 88.8 & 88.4 & \textbf{98.2} & 72.9 & 73.5 & 72.5 & 63.6 \\
  & VR$+$IR & \textbf{93.0} & \textbf{93.5} & 96.2 & \textbf{83.7} & \textbf{82.8} & \textbf{84.0} & \textbf{79.8} \\ \hline
\end{tabular}}
\label{table:vr+ir}
\end{table}

\section{Interaction reduction on other attacks}\label{append:ir+}

To further demonstrate the effectiveness of the interaction loss, we have applied the interaction loss on other attacks besides the PGD Attack, including the MI Attack, the SGM Attack, and the VR Attack.
More specifically, we added the interaction loss on the MI Attack (namely the MI$+$IR Attack), the SGM Attack (namely the SGM$+$IR Attack), and the VR Attack (namely the VR$+$IR Attack), respectively.

For the MI Attack and the SGM Attack, we directly applied Eq.\eqref{eq:inte loss} to generate adversarial perturbations, because these attacks were compatible with the interaction loss. Besides, for the VR Attack, its objective function is expressed as follows.
\begin{equation}
\underset{\delta}{\text{maximize}}\quad \mathbb{E}_{\xi \sim \mathcal{N}\left(0, \sigma^{2}I\right)}\left[ \textit{Loss}(x+\delta+\xi)\right] \quad \text{s.t.} \quad \|\delta\|_p\le \epsilon,\; x+\delta\in[0, 1]^n,
\end{equation}
Therefore, the VR$+$IR Attack was implemented via sampling as follows.
\begin{equation}
\begin{split}
    \underset{\delta}{\text{maximize}}\quad &\frac{1}{K}\sum_{k=1}^K  \left[ \textit{Loss}(x+\delta+\xi_k) -\mathbb{E}_{a b}\left[ I_{a b}(\delta)\right]\right],\quad \xi_k \sim \mathcal{N}\left(0, \sigma^{2}I\right) \\ &\text{s.t.}  \quad \|\delta\|_p\le \epsilon,\; x+\delta\in[0, 1]^n,
\end{split}
\end{equation}
where the interaction loss was computed by considering the input image as $x+\xi_k$, rather than $x$ in Eq.~(\ref{eq:expectation}).
When applying the interaciton loss to the VR Attack, a significant issue was the extremely high computational cost. 
Therefore, in order to implement the VR$+$IR Attack, we set $K=5$ and reduced the number of steps from 100 to 50. 
Besides, we also employed the LOO strategy for evaluation of the MI$+$IR, the VR$+$IR, and the SGM$+$IR attacks.

Table~\ref{table:mi+ir}, Table~\ref{table:sgm+ir}, and Table~\ref{table:vr+ir} compare the success rates of attacks with and without the interaction loss. 
These results demonstrate that the performance of the MI Attack,  the VR Attack , and the SGM Attack can be further improved by adding the interaction loss, which directly reduced interactions inside perturbations.

\end{document}